\documentclass[11pt]{article}

\usepackage{setspace}

\usepackage{natbib}
\usepackage{hyperref}
\usepackage{amsmath,amsfonts,mathtools,amssymb,amsthm}

\usepackage{graphicx}
\usepackage{subcaption,caption}

\usepackage{enumitem}
\usepackage{xr}

\usepackage{fullpage}

\usepackage{algorithm}
\usepackage{algpseudocode}

\makeatletter
\newcommand{\vast}{\bBigg@{4}} 
\newcommand{\Vast}{\bBigg@{5}} 
\makeatother
\algblock{Input}{EndInput}
\algnotext{EndInput}
\algblock{Output}{EndOutput}
\algnotext{EndOutput}

\algnewcommand{\Initialize}[1]{%
  \State \textbf{Initialize:}
  \Statex \hspace*{\algorithmicindent}\parbox[t]{.8\linewidth}{\raggedright #1}
}

\DeclareMathOperator*{\argmin}{arg\,min}

\newtheorem{assumption}{Assumption}
\newtheorem{definition}{Definition}
\newtheorem{proposition}{Proposition}
\newtheorem{lemma}{Lemma}
\newtheorem{theorem}{Theorem}

\title{Optimal differentially private kernel learning with random projection}
\author{
Bonwoo Lee\thanks{Department of Mathematical Sciences, KAIST, Daejeon 34141, Korea. Email: \texttt{righthim@kaist.ac.kr}}
\and
Cheolwoo Park\thanks{Department of Mathematical Sciences, KAIST, Daejeon 34141, Korea. Email: \texttt{parkcw2021@kaist.ac.kr}}
\and
Jeongyoun Ahn\thanks{Department of Industrial and Systems Engineering, KAIST, Daejeon 34141, Korea. Email: \texttt{jyahn@kaist.ac.kr}}
}
\date{}
\begin{document}
\doublespacing
\maketitle
\vspace{-2cm}
\vspace{1cm}

\begin{abstract}
Differential privacy has become a cornerstone in the development of privacy-preserving learning algorithms. This work addresses optimizing differentially private kernel learning within the empirical risk minimization (ERM) framework. We propose a novel differentially private kernel ERM algorithm based on random projection in the reproducing kernel Hilbert space using Gaussian processes. Our method achieves minimax-optimal excess risk rates for both the squared loss and Lipschitz-smooth convex loss functions under a local strong convexity condition. We further show that existing approaches based on alternative dimension reduction techniques, such as random Fourier feature mappings or $\ell_2$ regularization, yield suboptimal excess risk bounds. Our key theoretical contribution also includes the derivation of dimension-free excess risk bounds for objective perturbation-based private linear ERM—marking the first such result that does not rely on noisy gradient-based mechanisms. Additionally, we obtain sharper excess risk bounds for existing differentially private kernel ERM algorithms. Empirical evaluations support our theoretical claims, demonstrating that random projection enables statistically efficient and optimally private kernel learning. These findings provide new insights into the design of differentially private algorithms and highlight the central role of dimension reduction in balancing privacy and utility.
\end{abstract}
\section{Introduction}\label{sec:intro}
Differential privacy (DP), first formalized by \cite{dwork2006calibrating}, provides a rigorous framework for algorithmic indistinguishability and has become the standard for quantifying privacy in machine learning. Growing concerns about information leakage in machine learning models have further underscored the need for formal privacy guarantees \citep{phong, wang,privacy_leakage2,privacy_leakge3,privacy_leakage4,privacy_leakage5,privacy_leakage6}. 

In response, a substantial body of literature has emerged to develop differentially private algorithms across a range of statistical learning procedures. Within the realm of kernel methods, a key challenge is the high model complexity of the hypothesis space, which amplifies the privacy cost. To address this, dimension reduction techniques are often employed. For example, \cite{chaudhuri2011differentially} used random Fourier feature (RFF) maps to reduce kernel learning to a linear one in a randomized low-dimensional space. Similarly, \cite{hall2013differential} applied the reproducing kernel Hilbert space (RKHS) norm regularization, which can be viewed as a soft form of dimension reduction that limits the complexity of the hypothesis space. These two strategies—based on either RFF mappings or RKHS regularization—have since become essential tools in designing differentially private kernel learning algorithms \citep{balog2018differentially,harder2021dp,raj2019differentially,rubinstein2009learning,bassily2022differentially,smith2018differentially,smith2021differentially, wang2019privacy}.

Despite their widespread adoption, the optimality of differentially private kernel learning algorithms that rely on existing dimension reduction techniques remains largely unexplored. To the best of our knowledge, while prior work on differentially private kernel empirical risk minimization (ERM) has established excess risk bounds \citep{chaudhuri2011differentially,hall2013differential,jain2013differentially,bassily2022differentially,NEURIPS2024, wang2024differentially}, these results stop short of assessing whether the employed dimension reduction strategies achieve \emph{optimal} excess risk rates or characterizing minimax rates, which are commonly used to evaluate the optimality of differentially private algorithms \citep{duchi2013local,karwa_et_al:LIPIcs.ITCS.2018.44,cai2021cost,NEURIPS2023,cai2024optimal}. Notably, \citet{chaudhuri2011differentially} raised concerns about the statistical inefficiency of methods based on RFF. This limitation motivates a fundamental question: can more efficient dimension reduction methods lead to tighter excess risk bounds and improved statistical guarantees in differentially private kernel learning?

In this paper, we propose a new dimension reduction approach based on random projections and develop differentially private kernel ERM algorithms for both squared loss and general Lipschitz-smooth loss functions. Specifically, we define random projections on RKHSs via Gaussian processes, which naturally extend the projection methods used in compressed learning\footnote{We note a terminological inconsistency with \citet{chaudhuri2011differentially}, where the random Fourier features (RFF) map is referred to as a \emph{random projection}. In our work, we distinguish between the two: the term “random projection” is reserved for a separate construction detailed in Section~\ref{sec:alg}, which differs both conceptually and technically from the RFF-based mapping.}. This approach offers several advantages over existing ones. First, it applies to arbitrary kernels, whereas RFF-based algorithms are restricted to shift-invariant kernels, a limitation arising from their fundamentally different construction. Second, our method achieves the fastest excess risk rate than any other existing methods, and matches the optimal excess risk rate. Specifically, under the standard eigenvalue assumption—which capture the complexity of the hypothesis space—we derive excess risk bounds for the proposed algorithms and evaluate them against existing methods. We further establish minimax lower bounds for differentially private kernel ERM, and compare them with our upper bounds. The proposed random projection method yields better rates than the existing methods, and attains the minimax-optimal rates. This result highlights that the choice of dimension reduction technique is crucial to the performance of differentially private kernel ERM. Furthermore, we establish refined excess risk upper bounds and minimax lower bounds under specific structural assumptions, such as the source condition and local strong convexity. Under these favorable conditions, the proposed method achieves faster excess risk rates than existing approaches. Notably, we demonstrate that the embedding property significantly influences the minimax lower bound for differentially private kernel learning, a phenomenon absent in non-private settings. While random projections have previously been used in privacy-preserving linear classification and optimization \citep{rp_svm_linear,arora2022differentially}, their role in achieving optimality for differentially private kernel methods has not been explored. This work provides the first result demonstrating that random projection can yield minimax-optimal performance in the differentially private kernel learning setting. Our findings show that the trade-off between privacy and utility in kernel methods can be optimally balanced through random projection, offering a principled approach to designing private algorithms that are both theoretically grounded and practically scalable.

Our analysis also leads to several byproduct results of independent interest. First, we obtain improved excess risk bounds for existing differentially private kernel ERM algorithms. Second, we propose a novel objective perturbation algorithm for differentially private linear ERM that achieves \emph{dimension-free} excess risk bounds. Previously, such guarantees were attainable only through noisy gradient descent methods \citep{rp_svm_linear,arora2022differentially}. Unlike gradient-based approaches, our method requires fewer hyperparameters, thereby reducing the privacy overhead associated with private hyperparameter tuning and offering a practical alternative for differentially private learning. Related materials are deferred in Appendix~\ref{sec:linear}.

The remainder of the paper is organized as follows.
Section~\ref{sec:problem} formulates the differentially private kernel ERM problem and introduces key definitions regarding differential privacy and kernel models. Section~\ref{sec:alg} proposes differentially private kernel ERM algorithms based on random projection and discusses their privacy guarantees. Section~\ref{subsec:rp} begins by extending random projection from Euclidean spaces to RKHS. Section~\ref{subsec:rp_alg} presents two algorithms: one for the squared loss (Algorithm~\ref{alg:dprp_reg}) and another for the Lipschitz-smooth loss (Algorithm~\ref{alg:dprp_cls}). The respective privacy guarantees are stated in Theorems~\ref{thm:dprp_reg_priv} and \ref{thm:dprp_cls_priv}.

Section~\ref{sec:general} provides the excess risk analysis for both the proposed algorithms and existing differentially private kernel ERM methods for squared loss. Sections~\ref{subsec:alg_reg} and~\ref{subsec:error_rff} derive the excess risk bounds for the proposed and existing methods, respectively, and establish the minimax lower bound for differentially private kernel regression. This shows that the proposed algorithm achieves the minimax-optimal rate, and is faster than the RFF-based algorithm. Section~\ref{sec:exp} presents empirical results validating the optimality of our proposed algorithm for kernel learning with squared loss. We compare its performance against an RFF-based baseline on both simulated and real-world datasets.
Section~\ref{subsec:lipschitz} then turns to the case of general Lipshitz-smooth loss functions. As in Section~\ref{sec:general}, we present the excess risk bounds of our method under the standard assumptions.
Finally, in Section~\ref{sec:faster}, we derive excess risk bounds for both the proposed and existing methods, and establish the minimax lower bound for differentially private kernel ERM under favorable assumptions.
All omitted proofs and supporting technical lemmas are provided in the Appendix. 

\subsection{Notation}
Let $X$ be a random variable with marginal distribution denoted by $P_X$. We define $\mathcal{X}\subset\mathbb{R}^d$ and $\mathcal{Y} \subset \mathbb{R}$ as measurable spaces from which the input and output data are drawn, respectively. Let $k:\mathcal{X} \times \mathcal{X} \rightarrow \mathbb{R}_{\geq 0}$ be a measurable kernel—that is, a bivariate, symmetric, and semi-positive definite function. The RKHS associated with $k$ is denoted by  $\mathcal{H}_k$, which is a subspace of $L^2(P_X)$. Let $l:\mathbb{R}^2 \to \mathbb{R}$ represent the loss function. For any normed space $V$, we write $\lVert\cdot\rVert_V$ for its norm, omitting the subscript when $V = L^2(P_X)$. Also, we denote the diameter of $V$ with respect to its norm as $\lVert V\rVert$. We use $\lVert\cdot\rVert_2$ to denote the standard $\ell_2$ norm on Euclidean space. We denote a Gaussian process with mean 0 and covariance $k$ as $\mathcal{GP}(0, k)$. For an operator or square matrix $A$, we define $A_\lambda \coloneqq  A + \lambda I$, where $I$ is the identity operator or matrix. Given a prediction function $f$, its risk is defined as $\mathcal{E}(f) \coloneqq  \mathbb{E}\left[l(Y, f(X))\right]$ under the joint data distribution and specified loss function. We use the notation $\widetilde{O}$ to denote asymptotic bounds that suppress logarithmic factors. Lastly, we write $a_n\lesssim b_n$ if there exists a constant $c>0$, independent of $n$, such that $a_n\leq c b_n$ and $a_n \asymp b_n$ if $a_n\lesssim b_n$ and $b_n\lesssim a_n$. Finally, given a sequence $a_1,\ldots,a_m$, we use the shorthand notation $a_{1:m}$.

\subsection{Related works} Early algorithmic frameworks to differentially private kernel ERM were introduced by \cite{chaudhuri2011differentially} and \cite{hall2013differential}. The former proposed a method based on RFF mapping, which reduces kernel learning to a linear one in a randomized feature space, followed by a differentially private linear learner. In contrast, the latter work proposed a functional perturbation approach, in which privacy is enforced by injecting noise into the learned function in an RKHS. Although the framework of \citet{chaudhuri2011differentially} is designed for general bounded random feature mappings, both their excess risk analysis and typical implementations have focused specifically on RFF. These two methods—random Fourier feature mapping and functional perturbation—have become the main ingredients in the design of private algorithms for a variety of kernel-based learning tasks beyond kernel ERM. The RFF map approach has been employed in private algorithms for mean embedding discrepancy estimation \citep{balog2018differentially,harder2021dp}, two-sample testing \citep{raj2019differentially}, support vector machines (SVMs) \citep{rubinstein2009learning,bassily2022differentially}, multi-class SVMs \citep{park2023efficient}, ridge regression \citep{wang2024differentially}, and fast kernel density estimation \citep{wagner2023fast}. Functional perturbation has similarly been applied to private variants of Gaussian processes \citep{smith2018differentially}, sparse Gaussian processes \citep{smith2021differentially}, two-sample testing \citep{hall2013differential}, Q-learning \citep{wang2019privacy}, meta-learned regression \citep{raisa2024noise}, and adaptive distributed classification \citep{auddy2024minimax}.

While these methods have shown empirical promise, their theoretical guarantees remain relatively limited. Several studies have established excess risk bounds, primarily relying on RFF mappings and functional perturbation \citep{chaudhuri2011differentially,hall2013differential,jain2013differentially,bassily2022differentially, wang2024differentially}. However, these results do not provide an optimality analysis, and some of the bounds are demonstrated to be suboptimal in this work. For instance, Theorem 25 in \cite{chaudhuri2011differentially} establishes a privacy cost of order $\widetilde{O}((n\epsilon)^{-1/4})$ under convex Lipschitz-smooth loss functions using the RFF method. In contrast, our analysis shows that the achievable privacy cost using the RFF method is in fact $\widetilde{O}((n\epsilon)^{-2/3})$, and can be further improved under local strong convexity and appropriate source conditions (Theorem~\ref{thm:dprff_cls_rate}). Similarly, \cite{hall2013differential} analyzed a statistical error bound under a fixed hyperparameter $\lambda$, thereby overlooking the effects of proper tuning. We provide a refined bound under optimally selected $\lambda$ (Theorem~\ref{thm:dphall_cls_rate}). Finally, \cite{wang2024differentially} studied differentially private kernel ridge regression under the same eigenvalue and source conditions considered in our work, but their excess risk bound is suboptimal compared to ours (Theorem~\ref{thm:dprff_reg_rate}).

Next, we review the kernel approximation literature and situate our random projection method within the context of existing techniques. Kernel approximation methods have been widely studied to address the computational challenges of kernel learning. Prominent approaches include RFF \citep{rahimi2007random}, Nyström methods \citep{original_Nystrom}, and random projection methods \citep{LOPEZSANCHEZ2018130,BalcanBlumVempala2006KernelsAsFeatures}. These techniques transform kernel learning problems into low-dimensional linear learning problems. That is, they construct a feature map that embeds the data into an $M(<n)$-dimensional vector space and then apply linear learning methods to the transformed data. This greatly reduces the computational complexity of kernel methods. For example, solving kernel ridge regression involves the calculation of the inverse of an $n\times n$ kernel matrix, consuming $O(n^3)$ time complexity, where $n$ is the number of samples. In contrast, kernel approximation methods reduce kernel ridge regression to linear ridge regression for $M$-dimensional data, resulting $O(nM^2)$ time complexity.

Among these approaches, RFF and Nyström methods are the most popular ones. RFF constructs feature maps based on the Fourier spectral density of the kernel, while Nyström methods build feature representations by subsampling columns of the kernel matrix. These approaches have become standard in the literature and are supported by extensive theoretical analyses \citep{Nystrom_generalization,randomfeature,sketching,Li2021unified,JunhongLin0,falkon,SGD_RFF}. However, they differ fundamentally from the random projection-based approach. This line of research is motivated by classical random projection in Euclidean space, which preserves data geometry while reducing dimensionality. Existing works explore different ways to extend the concept from Euclidean space to the RKHS setting. For instance, \citet{BalcanBlumVempala2006KernelsAsFeatures} developed a random projection that preserves the margin of the data distribution within the RKHS, and \citet{LOPEZSANCHEZ2018130} proposed a method specifically for the RKHS of homogeneous polynomial kernels. Our approach offers two distinct advantages over these existing methods. First, our framework is more versatile; while the method in \citet{LOPEZSANCHEZ2018130} is restricted to homogeneous kernels, ours is applicable to general kernels. Furthermore, unlike \citet{BalcanBlumVempala2006KernelsAsFeatures}, which requires access to the underlying data distribution, our projection operates independently of such distributional knowledge. Second, we provide rigorous theoretical guarantees for the excess risk of kernel learning via random projection. Such analyses have largely been absent in prior random projection–based approaches, especially when compared with the extensive theoretical results available for RFF and Nyström methods.

\section{Problem formulation}\label{sec:problem}
In this section, we present the theoretical assumptions of the kernel model for excess risk analysis and introduce basic concepts regarding differential private learning.
\subsection{Kernel method}\label{subsec:model}

Let $k$ be a positive definite kernel $k$ such that $\sup_x k(x,x)\leq\kappa^2$ for some $\kappa>0$. For a given loss function $l$, the goal of kernel-based learning is to identify the target function $f_{\mathcal{H}_k} \in \mathcal{H}_k$ that minimizes the expected loss\footnote{We assume the existence of $f_{\mathcal{H}_k}$, which is a common assumption in kernel learning theory \citep{caponnetto2007optimal,randomfeature}}:
\begin{equation*} f_{\mathcal{H}_k} \coloneqq  \argmin_{f \in \mathcal{H}_k} \mathbb{E}_{P_{X,Y}} \left[ l(Y, f(X)) \right]. \end{equation*}
In this work, we employ regularized kernel ERM \citep{scholkopf2018learning}. Given an i.i.d. sample $\{(x_1,y_1), \ldots, (x_n,y_n)\}$ from $P_{X,Y}$, the regularized estimator is defined as 
\begin{equation*} \hat{f}_\lambda\coloneqq \argmin_{f \in \mathcal{H}_k} \frac{1}{n} \sum_{i=1}^n l(y_i, f(x_i)) + \frac{\lambda}{2} \lVert f \rVert_{\mathcal{H}_k}^2, \end{equation*}
 where $\lambda$ is a penalty parameter. Our objective is to find $f_{\mathcal H_k}$ in a privacy-preserving manner. 

We consider two types of loss functions in our analysis: (i) the squared loss, defined as $l(y, \hat{y}) \coloneqq  (y - \hat{y})^2$, and (ii) a general class of convex, bounded loss functions that satisfy Lipschitz continuity and smoothness conditions, as formalized below. Lipschitz continuity and smoothness are standard assumptions in differentially private learning, as they are essential for delivering privacy guarantees \citep{chaudhuri2011differentially,kifer2012private}. 
\begin{assumption}[Lipschitz and smooth conditions]\label{assump:lip}
    The loss function $l: \mathbb{R}^2 \rightarrow \mathbb{R}$ is convex and twice differentiable in its second argument. Also, there exist constants $c_0 ,c_1,c_2 > 0$ such that $|l(y,0)|$ is bounded by $c_0$,
    \begin{equation*}
        |l(y,z)-l(y,z^\prime)|\leq c_1|z-z^\prime|,
    \end{equation*}
    and
    \begin{equation*}
        l(y,z)\leq l(y,z^\prime)+l_z(y,z^\prime)(z-z^\prime)+\frac{c_2}{2}( z-z^\prime)^2,
    \end{equation*}
    for all $y,z,z^\prime\in\mathbb{R}$, where $l_z(y, z’)$ denotes the partial derivative of $l$ with respect to its second argument evaluated at $(y, z’)$.
\end{assumption}

Our excess risk analysis requires different distributional assumptions on the response variable $Y$, depending on the loss function. For Lipschitz-smooth loss, no additional assumptions on $Y$ are required, while for the squared loss, we assume the conditional distribution of the response variable $Y|X$ is light-tailed.
\begin{assumption}[Bernstein condition]\label{assump:bern}
For some constants $B, \sigma > 0$ and all integers $p \geq 2$,
\[
\mathbb{E}\left[(Y-f_{\mathcal{H}_k}(X))^p|X=x\right]\leq\frac{1}{2}p!B^{p-2}\sigma^2
\]
for all $x\in\mathcal{X}$.
\end{assumption}

We now state the eigenvalue assumption, a standard condition in kernel learning theory \citep{caponnetto2007optimal,Nystrom_generalization,falkon,randomfeature}. The assumption characterizes the complexity of the hypothesis space $\mathcal{H}_k$. To this end, we define the integral operator $L : L^2(P_X) \rightarrow L^2(P_X)$ as
\begin{equation*}
    L(f)(x^\prime)=\int_{\mathcal{X}}k(x,x^\prime)f(x)dP_X(x).
\end{equation*}
The eigenvalue assumption quantifies the complexity of the hypothesis space through the spectral decay of the operator $L$. For $l\geq 1$, let $\mu_l$ be the $l$th eigenvalue of $L$.
\begin{assumption}[Eigenvalue assumption]\label{assump:cap}
   There exists $0<a<b$ and $\gamma\in [0,1)$ such that $\mu_l\in [al^{-1/\gamma},b^{-1/\gamma}]$ for all $l\geq 1$.
\end{assumption}
The decay rate $\gamma$ characterizes the polynomial decay of eigenvalues of $L$. The parameter $\gamma$ can be interpreted as a measure of complexity of the hypothesis space $\mathcal{H}_k$ \citep{caponnetto2007optimal}. Smaller values of $\gamma$ correspond to more restricted spaces, with $\gamma=0$ representing the finite-dimensional case. We note that the excess risk upper bounds presented in subsequent sections require only the upper bound of the eigenvalues, $\mu_l\leq bl^{-1/\gamma}$. This condition is satisfied in various practical settings. For instance, if the given kernel is Gaussian and $P_X$ is sub-Gaussian, Assumption~\ref{assump:cap} holds with $\gamma=0$. Additionally, if the density of $P_X$ is $O((1+\lVert x\rVert_2)^{-\alpha})$ for some $\alpha>d$, the assumption holds for any $\gamma>\alpha^{-1}$. Similar results have been established for other widely used kernels, including Matérn, Student, and sinc kernels \citep[see Corollaries 27--29]{HarchaouiBachMoulines2008KFDA}. We present the assumption in its current stringent form as it is essential for the minimax lower bound analysis.
\subsection{Privacy}\label{subsec:DP}
To ensure privacy, we require the learning algorithm to satisfy differential privacy, which guarantees that the output distribution of the algorithm remains nearly indistinguishable when a single individual’s data is modified. Specifically, we employ the notion of $(\epsilon, \delta)$-differential privacy. Let $\mathcal{M}: \mathcal{X}^n \rightarrow \mathcal{P}(\mathcal{Z})$ be a randomized algorithm that maps datasets of size $n$ to probability distributions over a measurable output space $\mathcal{Z}$. Two datasets $D$ and $D'$ are said to be \emph{neighboring} (denoted $D \sim D'$) if they differ in exactly one data point.

\begin{definition}[$(\epsilon,\delta)$-DP, \cite{dwork2006our}]\label{def:ed-DP}
    A randomized algorithm $\mathcal{M}$ is said to satisfy $(\epsilon,\delta)$-DP if the following holds:
\begin{equation*}
    \sup_{D,D'\subset\mathcal{X}^n, D\sim D'}\sup_{A\subset\mathcal{Z}}\mathbb{P}\left(\mathcal{M}(D)\in A\right)-e^\epsilon\mathbb{P}\left(\mathcal{M}(D')\in A\right)\leq \delta.
\end{equation*}
\end{definition}
Although this definition may initially seem unintuitive, a useful interpretation is provided by the hypothesis testing perspective introduced by \citet{dong2019gaussian}. Letting $Z$ denote the output of the algorithm, differential privacy can be understood as a bound on the ability to distinguish between neighboring datasets via hypothesis testing:
\begin{equation*}
    H_0:Z\sim\mathcal{M}(D)\quad\text{ vs. }\quad H_1:Z\sim \mathcal{M}(D^\prime).
\end{equation*}
Under this view, the $(\epsilon,\delta)$-DP constrains the trade-off between type I and type II errors, limiting the power to detect the presence and absence of an individual's data. 

To satisfy the $(\epsilon, \delta)$-DP criterion in Definition~\ref{def:ed-DP}, a widely used method is to add Gaussian noise to the output of a deterministic algorithm. The magnitude of this noise is scaled to the algorithm's $\ell_2$ sensitivity, denoted by $\Delta$, which measures the maximum change in the output when a single data point in the input dataset is altered. This approach is formalized by the following result:
\begin{proposition}[Gaussian mechanism, \cite{nikolov2013geometry}]\label{prop:gaussian}
    Let $\mathcal{A}:\mathcal{X}^n\rightarrow\mathbb{R}^d$ be a deterministic algorithm satisfying $\sup_{D,D'\subset\mathcal{X}^n,D\sim D'}\lVert \mathcal{A}(D)-\mathcal{A}(D')\rVert_2\leq\Delta$. Then the randomized algorithm defined by $\mathcal{M}(D)\coloneqq \mathcal{A}(D)+\frac{\Delta\left(1+\sqrt{2\log\frac{1}{\delta}}\right)}{\epsilon}\varepsilon$, where $\varepsilon$ follows the standard normal distribution on $\mathbb{R}^d$, satisfies $(\epsilon,\delta)$-DP.
\end{proposition}

Many statistical procedures, including the DP algorithms we propose, involve multiple stages of accessing the data or performing data-dependent computations. The following result formalizes the privacy guarantees for such composed procedures.

\begin{proposition}[Composition theorem, \cite{dwork2006our}]\label{thm:comp}
    Let $\mathcal{M}_1:\mathcal{X}^n\rightarrow\mathcal{P}(\mathcal{Z})$ and $\mathcal{M}_2:\mathcal{X}^n\times\mathcal{Z}\rightarrow\mathcal{P}(\mathcal{W})$ be randomized algorithms. Define $\mathcal{M}:\mathcal{X}^n\rightarrow\mathcal{W}$ by $\mathcal{M}(D)\coloneqq \mathcal{M}_2(D,\mathcal{M}_1(D))$ for $D\subset\mathcal{X}^n$. Then $\mathcal{M}$ is $(\epsilon_1+\epsilon_2,\delta_1+\delta_2)$-DP if $M_1$ is $(\epsilon_1,\delta_1)$-DP and $M_2(\cdot,z)$ is $(\epsilon_2,\delta_2)$-DP for every $z\in\mathcal{Z}$.
\end{proposition}
Finally, to assess the optimality of the proposed method within the minimax framework, we define the differentially private minimax risk as follows.
\begin{definition}[$(\epsilon,\delta)$-DP minimax risk]
    For a given family of data distributions $\mathcal{P}$, we define the minimax risk of differentially private kernel learning as
\begin{equation*}
    \mathfrak{R}(\mathcal{P},\mathcal{E},\epsilon,\delta)\coloneqq \inf_{M\in\mathcal{M}_{\epsilon,\delta}}\sup_{P\in\mathcal{P}}\left(\mathbb{E}_{P}[\mathcal{E}(\widetilde{f})]-\min_{f\in\mathcal{H}_k}\mathcal{E}(f)\right),
\end{equation*}
where $\mathcal{M}_{\epsilon,\delta}$ denotes the class of $(\epsilon,\delta)$-DP mechanisms that return a function estimator $\widetilde{f}$.
\end{definition}
The $(\epsilon,\delta)$-DP minimax risk characterizes the fundamental lower limit of the excess risk for any differentially private kernel learning algorithm under the worst-case scenario over $\mathcal{P}$. This quantity captures the inherent trade-off between privacy and utility in differentially private kernel learning.
\section{Random projection-based DP kernel ERM algorithms}\label{sec:alg}
In this section, we define random projection in an RKHS and propose differentially private kernel ERM algorithms for both squared loss and general Lipschitz-smooth loss functions, along with their respective privacy guarantees.


\subsection{Random projection on RKHS}\label{subsec:rp}

We construct the random projection map via centered Gaussian processes. To this end, we begin by recalling the definition of a Gaussian process.
\begin{definition}[Gaussian process]
Given a positive semi-definite covariance function $C:\mathcal{X} \times \mathcal{X} \rightarrow \mathbb{R}$, a mean function $m:\mathcal{X} \rightarrow \mathbb{R}$, and a probability space $\Omega$, a measurable function $\psi:\mathcal{X} \times \Omega \rightarrow \mathbb{R}$ is said to be a Gaussian process with mean function $m$ and covariance function $C$ if, for an arbitrary positive integer $N$, and $t_1, \ldots, t_N \in \mathcal{X}$, the random vector $(\psi(t_1,\omega), \ldots, \psi(t_N,\omega))$ follows a multivariate normal distribution:
\begin{align*}
    N\left((m(t_1),\ldots,m(t_N))^\top,\{ C(t_i,t_j)\}_{N\times N}\right).
\end{align*}
If $m\equiv0$, then $\psi$ is called a centered Gaussian process.
\end{definition}
Conceptually, a Gaussian process generalizes the Gaussian distribution to function spaces. While a Gaussian distribution generates random scalars or vectors, a Gaussian process produces entire functions as its realizations. In this sense, a sample drawn from a Gaussian process is not a single point or a vector, but a real-valued function defined on $\mathcal{X}$. For a given kernel $k$, we assume the Gaussian process is almost surely bounded.
\begin{assumption}\label{assump:kernel}
    If $h\sim\mathcal{GP}(0,k)$, then $|\sup_x h(x,\omega)|$ is almost surely bounded. That is, $\mathbb{P}(|\sup_x h(x,\omega)|<\infty)=1$. 
\end{assumption}
Assumption~\ref{assump:kernel} holds in many practical settings, including widely used kernels such as inner product and Gaussian kernels on bounded domains. Notably, the assumption also holds for certain non-smooth kernels like the Laplace kernel, even when $\mathcal{X}$ is unbounded.

Given a projection dimension $M > 0$, we define the random projection from the RKHS $\mathcal{H}_k$ to $\mathbb{R}^M$ as follows:
\begin{definition}[Random projection on RKHS]\label{def:rp}
Let $h(\cdot,\omega_1),\ldots,h(\cdot,\omega_M)$ be i.i.d. random functions drawn from a centered Gaussian process with covariance function $k$. The random projection on the RKHS $\mathcal{H}_k$ is defined as the map $h_M:\mathcal{X}\rightarrow\mathbb{R}^M$ given by
    \begin{eqnarray*}
        h_M(x)\coloneqq \frac{1}{\sqrt{M}}\left(h(x,\omega_1),\ldots,h(x,\omega_M)\right)^\top.
    \end{eqnarray*}
\end{definition}
For a given input $x\in\mathcal{X}$, our random projection first collects the $M$ values of i.i.d. random functions, $h(\cdot,\omega_1),\ldots,h(\cdot,\omega_M)$, at $x$. Aggregating these numbers into $M$-dimensional vector, we obtain the random projection $h_M(x)$.

\begin{figure}[htp]
    \centering
    \includegraphics[width=0.9\textwidth]{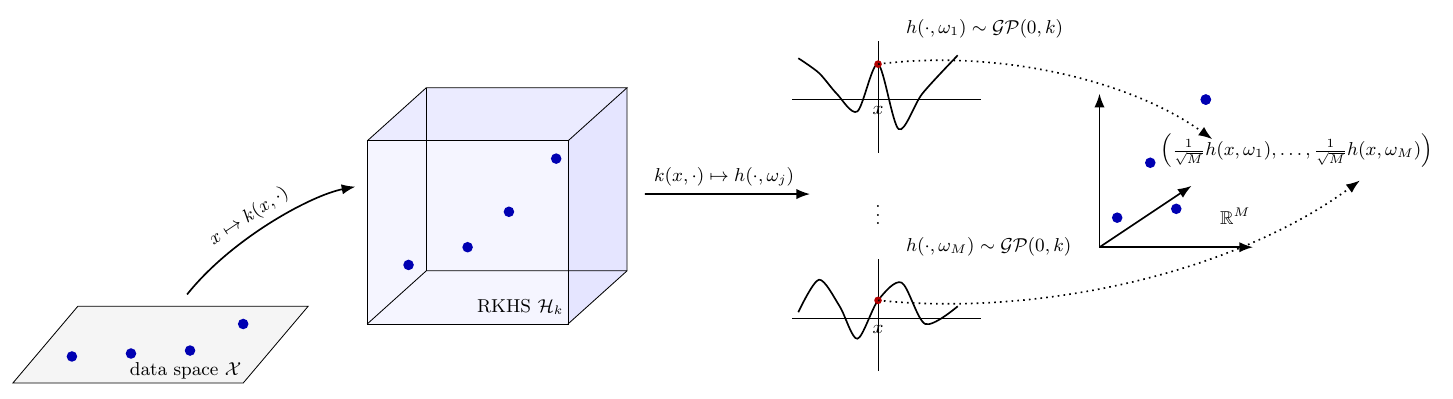}
    \caption{Random projection on $\mathcal{H}_k$. The random projection first maps the input $x \in \mathcal{X}$ to its feature representation $k(x, \cdot)$ in the RKHS $\mathcal{H}_k$. This is followed by a projection onto an $M$-dimensional space $\mathbb{R}^M$, where each coordinate of the projected vector corresponds to a realization of $M$ i.i.d. Gaussian processes $h(\cdot, \omega_j) \sim \mathcal{GP}(0, k)$ evaluated at $x$.}
    \label{fig:rp_example}
\end{figure}
Definition~\ref{def:rp} is directly inspired by the classical random projection in Euclidean space, formally defined below.
\begin{definition}[Random projection in Euclidean space, \cite{achlioptas2001database}]\label{def:rp_euc}
   For $x\in\mathbb{R}^d$, the random projection of $x$ is defined as $\mathbf{R}x$, where $\mathbf{R} = \{ r_{ij}/\sqrt{M}\}_{M \times d}$ is a random matrix, where $r_{ij}$ are i.i.d. standard normal random variables.
\end{definition}In Euclidean space, random projection is performed by multiplying the data by a Gaussian random matrix. Analogously, our random projection on the RKHS can be viewed as applying an infinite-width random Gaussian matrix to the feature mapping $k(x,\cdot)$. Then each component of the projected vector corresponds to a realization of a Gaussian process evaluated at $x$, as illustrated in Fig.~\ref{fig:rp_example}. A detailed derivation of this correspondence is provided below.

Let $\{\phi_l\}_{l=1}^\infty$ and $\{\mu_l\}_{l=1}^\infty$ denote the eigenfunctions and eigenvalues, respectively, of the integral operator $L$. Then, $\{\sqrt{\mu_l}\phi_l\}_{l=1}^\infty$ forms an orthonormal basis for $\mathcal{H}_k$. To construct the random projection on $\mathcal{H}_k$, we first consider the feature mapping $k(x,\cdot)\in\mathcal{H}_k$ for a given $x\in \mathcal{X}$. Under the orthonormal basis $\{\sqrt{\mu_l}\phi_l\}_{l=1}^\infty$, the coordinate representation of $k(x,\cdot)$ is given by $(\sqrt{\mu_1}\phi_1(x),\sqrt{\mu_2}\phi_2(x),\ldots)^\top$. This follows from Mercer's theorem \citep{mercer1909,mercer}, which provides the spectral decomposition of the kernel function:
\begin{equation}\label{eq:mercer} k(x,y) = \sum_{l=1}^\infty \mu_l \phi_l(x) \phi_l(y). \end{equation}
As in Definition~\ref{def:rp_euc}, we can define the random projection by multiplying this coordinate representation by an infinite-width random Gaussian matrix. Consider a random matrix $\mathbf{R}=\{r_{jl}/\sqrt{M}\}_{M\times\infty}$, where $r_{jl}$ are i.i.d. standard normal random variables. Applying $\mathbf{R}$ to the coordinates of $k(x,\cdot)$ yields: \begin{equation}\label{eq:Karhunen}
    \mathbf{R}(k(x,\cdot))=\frac{1}{\sqrt{M}}\left(\sum_{l=1}^\infty r_{1l}\sqrt{\mu_l}\phi_l(x),\ldots,\sum_{l=1}^\infty r_{Ml}\sqrt{\mu_l}\phi_l(x)\right)^\top.
\end{equation}
Importantly, the expansion in \eqref{eq:Karhunen} corresponds to the Karhunen–Loève expansion of a centered Gaussian process with covariance function $k$. Specifically, for each $j=1,\ldots,M$, the $j$th component $\sum_{l=1}^\infty r_{jl}\sqrt{\mu_l}\phi_l(x)$ converges to the Gaussian process \citep{karhunen1947lineare}. 
Consequently, each component of $\mathbf{R}(k(x,\cdot))$ represents an i.i.d. random function $h(\cdot,\omega_j)$ as defined in Definition~\ref{def:rp}. This formulation establishes a natural bridge between matrix-based projections in Euclidean space and Gaussian process-based projections in $\mathcal{H}_k$. Notably, when $\mathcal{X} \subset \mathbb{R}^d$ and $k(x, y) = \langle x, y \rangle$, Definition~\ref{def:rp} recovers the classical Euclidean random projection.

\subsection{Algorithms and their privacy analysis}\label{subsec:rp_alg}
We propose two differentially private kernel ERM algorithms based on random projection: one for the squared loss function presented in Algorithm \ref{alg:dprp_reg} and the other for general Lipschitz-smooth loss functions presented in Algorithm~\ref{alg:dprp_cls}. Both algorithms solve the following finite-dimensional ERM problem under a differential privacy constraint:
\begin{equation}\label{eq:rp_reg}
\hat{\beta}_{M,\lambda}\coloneqq \argmin_{\beta\in\mathbb{R}^M}\frac{1}{n}\sum_{i=1}^n l(y_i,\beta^\top h_M(x_i))+\frac{\lambda}{2}\lVert\beta\rVert_2^2.
\end{equation}
Algorithm~\ref{alg:dprp_reg} is based on the sufficient statistics perturbation method for private linear regression proposed by \cite{wang2018revisiting}, while
Algorithm~\ref{alg:dprp_cls} is based on the objective perturbation method for $(\epsilon, \delta)$-DP by \cite{kifer2012private}. Following theorems establish the privacy guarantees for Algorithms~\ref{alg:dprp_reg} and~\ref{alg:dprp_cls}.
\begin{algorithm}
    \caption{Random projection-based DP kernel ridge regression}\label{alg:dprp_reg}
\begin{algorithmic}
  \Input:\textrm{  }Given data $\{(x_i,y_i)\}_{i=1}^n$, kernel $k$, kernel size $\kappa$, projection dimension $M$, penalty parameter $\lambda$, truncation level $T$, privacy parameters $\epsilon,\delta$.
  \EndInput

  \Output:\textrm{  }A differentially private estimator $\widetilde{f}$.
  \EndOutput
\begin{enumerate}
    \item Compute the sensitivities:

    $\Delta_1=2\kappa^2\left(1+2\sqrt{\frac{\log\frac{8}{\delta}}{M}}+\frac{2\log\frac{8}{\delta}}{M}\right)$ and $\Delta_2=2\kappa T\sqrt{1+2\sqrt{\frac{\log\frac{8}{\delta}}{M}}+\frac{2\log\frac{8}{\delta}}{M}}$.
    \item Process data:

    (i) Project $x_i$ to $z_i\coloneqq h_M(x_i)$ for $i=1,\ldots,n$, where $h_M$ is from Definition~\ref{def:rp}.

    (ii) Truncate responses $y_i$ to $[y_i]_T\coloneqq \min\{\max\{y_i,-T\},T\}$ for $i=1,\ldots,n$.
    \item Compute the sufficient statistics for linear regression on the projected data:

    $\hat{C}_{M}\coloneqq \frac{1}{n}\sum_{i=1}^nz_iz_i^\top$ and $\hat{u}_T\coloneqq \frac{1}{n}\sum_{i=1}^n[y_i]_Tz_i$.
    \item Privatize the sufficient statistics:

    $\widetilde{C}_{M}\coloneqq \hat{C}_{M}+\frac{2\Delta_1\left(1+\sqrt{2\log\frac{4}{\delta}}\right)}{n\epsilon}\left(\frac{\varepsilon_{M\times M}+\varepsilon_{M\times M}^\top}{2}\right)$ and $\widetilde{u}\coloneqq \hat{u}_T+\frac{2\Delta_2\left(1+\sqrt{2\log\frac{4}{\delta}}\right)}{n\epsilon}\varepsilon_{M}$ where $\varepsilon_{M}$ and $\varepsilon_{M\times M}$ are a $M$-dimensional random vector and $M\times M$ matrix, respectively, with i.i.d. standard normal entries.
    \item Compute $\widetilde{\beta}\coloneqq \widetilde{C}_{M,\lambda}^{-1}\widetilde{u}$, where $\widetilde{C}_{M,\lambda}\coloneqq \widetilde{C}_M+\lambda I$.
\end{enumerate}
 \Return $\widetilde{f}\coloneqq \widetilde{\beta}^\top h_M$.

\end{algorithmic}
\end{algorithm}
\begin{algorithm}
    \caption{Random projection-based DP kernel objective perturbation}\label{alg:dprp_cls}
\begin{algorithmic}
  \Input:\textrm{  }Given data $\{(x_i,y_i)\}_{i=1}^n$, kernel $k$, kernel size $\kappa$, projection dimension $M$, penalty parameter $\lambda$, Lipschitz parameter $c_1$, smoothness parameter $c_2$, privacy parameters $\epsilon,\delta$.
  \EndInput

  \Output:\textrm{  }A differentially private estimator $\widetilde{f}$.
  \EndOutput

1. Compute the sensitivity $\Delta_3=\kappa\sqrt{1+2\sqrt{\frac{\log\frac{4}{\delta}}{M}}+\frac{2\log\frac{4}{\delta}}{M}}$.

2. Obtain the random projection map $h_M$ in Definition \ref{def:rp}.

3. Generate perturbation vector $b\sim \mathcal{N}\left(0,\frac{4c_1^2\Delta_3^2\left(2\log\frac{4}{\delta}+\epsilon\right)}{\epsilon^2}I_M\right)$.

4. Solve $\widetilde{\beta}\coloneqq \argmin_{\beta\in\mathbb{R}^M}\frac{1}{n}\sum_{i=1}^n l(y_i,\beta^\top h_M(x_i))+\frac{\lambda_0}{2}\lVert\beta\rVert_2^2+\frac{b^\top \beta}{n}$ where $\lambda_0\coloneqq \max\Big\{\lambda,\frac{c_2\Delta_3^2}{n}\left(e^{\epsilon/4}-1\right)^{-1}\Big\}$.

 \Return $\widetilde{f}\coloneqq \widetilde{\beta}^\top h_M$.

\end{algorithmic}
\end{algorithm}
\begin{theorem}\label{thm:dprp_reg_priv}
    Under Assumption \ref{assump:kernel}, Algorithm \ref{alg:dprp_reg} is $(\epsilon,\delta)$-DP.
\end{theorem}
Theorem~\ref{thm:dprp_reg_priv} can be proved by leveraging the composition property of differential privacy (Proposition~\ref{thm:comp}), provided that the release of statistics $\widetilde{u}$ and $\widetilde{C}_M$ is differentially private. A primary technical challenge arises from the fact that the randomly projected data $h_M(x)$ is not bounded, which implies that the $\ell_2$ sensitivities of these statistics are infinite. We address this issue by demonstrating that the projected data remains bounded with high probability. Similarly, we utilize this probabilistic boundedness in conjunction with the privacy guarantees for objective perturbation \citep[e.g.,][]{kifer2012private} to prove Theorem~\ref{thm:dprp_cls_priv}. The complete proofs of these theorems are provided in Appendix~\ref{sec:proof_dprp_reg_priv} and Appendix~\ref{sec:proof_dprp_cls_priv}, respectively.
\begin{theorem}\label{thm:dprp_cls_priv}
    Under Assumption \ref{assump:lip}, Algorithm~\ref{alg:dprp_cls} is $(\epsilon,\delta)$-DP.
\end{theorem}

\section{Excess risk analysis of DP kernel learning with squared loss}\label{sec:general}
In this section, we first establish the minimax rate to assess the optimality of differentially private kernel learning with squared loss. Specifically, we derive high-probability excess risk bounds for the proposed random projection-based method and the existing RFF-based method. Then we establish the minimax lower bound to show the optimality of the proposed method. The excess risk of a given prediction function $f$ is defined as follows:
\begin{align*}
    \mathcal{E}(f)-\min_{f\in\mathcal{H}_k}\mathcal{E}(f).
\end{align*}
\subsection{DP kernel ridge regression via random projection}\label{subsec:alg_reg}
We analyze the excess risk of the differentially private kernel ridge regression algorithm presented in Algorithm \ref{alg:dprp_reg}.  Under the eigenvalue assumption (Assumption~\ref{assump:cap}), Theorem~\ref{thm:dprp_reg_rate} provides a high-probability bound on the excess risk.

\begin{theorem}\label{thm:dprp_reg_rate_simple} Under Assumptions \ref{assump:bern}--\ref{assump:kernel}, for $t\in (0,e^{-1}]$, there exist $n_0>0$ and choices of $T$, $\lambda$, and $M$ such that, Algorithm \ref{alg:dprp_reg} achieves an excess risk of order $O\Big (\Big(n^{-\frac{1}{1+\gamma}}+(n\epsilon)^{-\frac{2}{2+\gamma}}\Big)\log^2\frac{1}{t}\Big)$ with probability at least $1-t$ if  $\epsilon\geq n^{-1}$ and $n\geq n_0$.
\end{theorem}
As $\epsilon\rightarrow\infty$, the excess risk bound becomes $O(n^{-\frac{1}{1+\gamma}})$ recovering the excess risk rate of kernel ridge regression in non-private setting \citep{caponnetto2007optimal}. Hence, the privacy constraint introduces an additional error term of order $O\left((n\epsilon)^{-\frac{2}{2 + \gamma}}\right)$ in the high privacy regime.

The excess risk bound presented in Theorem~\ref{thm:dprp_reg_rate_simple} omits the explicit dependency of the data dimension $d$ by treating it as a constant. However, the dependency can be critical in high-dimensional settings. In the proof of Theorem~\ref{thm:dprp_reg_rate_simple}, we rely on the uniform boundedness of the random projection function $h_M$ with high probability. The following lemma establishes that the supremum of the Gaussian process realizations, $\sup_{1\le l \le M}\sup_{x\in\mathcal{X}}|h(x,\omega_l)|$, remains bounded with high probability.
\begin{lemma}\label{lem:general_bound}
For a given $t\in (0,1]$, the Gaussian processes $\{h(\cdot,\omega_j)\}_{j=1}^M$ satisfy
    \begin{eqnarray*}
    \mathbb{P}_{\omega_{1:M}}\left(\sup_{1\le j \le M}\sup_{x\in\mathcal{X}}|h(x,\omega_j)|\leq H\right)\geq 1-t,
\end{eqnarray*}
where $H\coloneqq H_0+\kappa\sqrt{2\log\frac{2M}{t}}$ and $H_0\coloneqq\mathbb{E}_\omega[\sup_x|h(x,\omega)|]$ is a constant that depends on the kernel. \end{lemma}
The term $H_0$ enters the excess risk bound as a multiplicative factor. While $H_0$ depends on the data dimension $d$, it is small enough to be treated as a constant factor in many practical settings. For example, for the inner product or Gaussian kernels, it is established that $H_0=O(\log d)$ suffices \citep[e.g.,][Example 4.3.2]{adler2007gaussian}. This suggests that the dimensionality of the original space $\mathcal{X}$ may not substantially degrade the performance of Algorithm~\ref{alg:dprp_reg} in many practical scenarios.

Next, we derive the minimax lower bound of the excess risk for differentially private kernel learning with squared loss to analyze the optimality of Algorithm~\ref{alg:dprp_reg}. Let $\mathcal{P}_{1}$ denote the class of data distributions over $\mathcal{X}\times\mathcal{Y}$ satisfying Assumptions~\ref{assump:bern} and \ref{assump:cap}. In the following theorem, we establish the minimax lower bound for differentially private kernel ridge regression over $\mathcal{P}_1$.
\begin{theorem}\label{thm:lower_bd_simple} Under Assumption~\ref{assump:kernel} and $\delta=o(n^{-1})$, the minimax lower bound is given by
\begin{equation*}
    \mathfrak{R}(\mathcal{P}_1,\mathcal{E},\epsilon,\delta)\gtrsim
    n^{-\frac{1}{1+\gamma}}+(n\epsilon)^{-\frac{2}{2+\gamma}-\eta}\wedge1,
\end{equation*}
for $\eta>0$.
\end{theorem}
Taking $\eta \to 0^+$, the lower bound in Theorem~\ref{thm:lower_bd_simple}
can be made arbitrarily close to $O(n^{-\frac{1}{1+\gamma}}+(n\epsilon)^{-\frac{2}{2+\gamma}}\wedge1)$. Hence, the minimax rate is $O(n^{-\frac{1}{1+\gamma}}+(n\epsilon)^{-\frac{2}{2+\gamma}}\wedge1)$ up to an arbitrarily small slack in the exponent. Notably, this lower bound matches the excess risk upper bound presented in Theorem~\ref{thm:dprp_reg_rate_simple}. Thus, Algorithm~\ref{alg:dprp_reg} achieves the minimax-optimal rate. We note that the eigenvalue condition is necessary for the minimax risk analysis. In particular, one needs both upper and lower bounds on the eigenvalues to ensure the existence of a `worst-case' scenario, while the upper bound condition of the eigenvalues is sufficient to derive the excess risk upper bound. This distinction is standard in the kernel learning literature; for example, \citet{caponnetto2007optimal} uses only the upper bound on risk convergence while employing a more specific decay condition to prove minimax optimality.

\subsection{DP kernel ridge regression via random feature mapping}\label{subsec:error_rff}
To compare our method to the existing method for differentially private kernel regression, we study a widely used technique: random feature mapping. In particular, we present the RFF-based DP kernel ridge regression algorithm \citep{rahimi2007random, chaudhuri2011differentially} and provide its excess risk bound.

We begin by introducing the concept of a random feature map:
\begin{definition}[Random feature map, \cite{rahimi2007random}]\label{def:rff}
Let $k : \mathcal{X} \times \mathcal{X} \rightarrow \mathbb{R}$ be a positive definite kernel. Let $\varphi:\mathcal{X}\times\Omega\rightarrow\mathbb{R}$ be a bounded measurable function such that for all $x,x^\prime\in\mathcal{X}$,
    \begin{equation*}
k(x,x^\prime)=\mathbb{E}_{\omega\sim\pi}\left[\varphi(x,\omega)\varphi(x^\prime,\omega)\right]
    \end{equation*}
    for some distribution $\pi$ over $\Omega$. Given a projection dimension $M$, the associated random feature map $\varphi_M : \mathcal{X} \rightarrow \mathbb{R}^M$ is defined as
\[
    \varphi_M(x) \coloneqq  \frac{1}{\sqrt{M}} \left( \varphi(x, \omega_1), \ldots, \varphi(x, \omega_M) \right),
\]
where $\omega_1, \ldots, \omega_M \overset{\mathrm{i.i.d.}}{\sim} \pi$.
\end{definition}
When the kernel $k$ is shift-invariant, that is, $k(x,x^\prime)$ is a $k_0(x-x^\prime)$ for some function $k_0:\mathcal{X}\rightarrow\mathbb{R}$, there exists a bounded random feature map $\varphi:\mathcal{X}\times\Omega\rightarrow[-\sqrt{2},\sqrt{2}]$ such that the condition in Definition~\ref{def:rff} is satisfied. This construction, known as the random Fourier feature map~\citep{rahimi2007random}, is the most commonly used random feature mapping in DP kernel learning. We emphasize that random projection is neither an RFF map nor a random feature map in the sense of Definition~\ref{def:rff}, which has been used in prior work on DP kernel learning. While the RFF map is defined only for shift-invariant kernels, random projection can be applied to general kernels. Moreover, the random projection map is typically unbounded and thus does not satisfy the boundedness condition required in Definition~\ref{def:rff}.

Algorithm~\ref{alg:dprff_reg}, presents the DP kernel ridge regression method based on RFF map. Its privacy guarantee is given in Theorem~\ref{thm:dprff_reg_priv}.

\begin{algorithm}
    \caption{RFF-based DP kernel ridge regression}\label{alg:dprff_reg}
\begin{algorithmic}
  \Input:\textrm{  }Given data $\{(x_i,y_i)\}_{i=1}^n$, shift-invariant kernel $k$, kernel size $\kappa$, projection dimension $M$, penalty parameter $\lambda$, truncation level $T$, privacy parameters $\epsilon,\delta$.
  \EndInput

  \Output:\textrm{  }A differentially private solution $\widetilde{f}$.
  \EndOutput

1. Set the sensitivities as $\Delta_1 = 2$ and $\Delta_2 = \sqrt{2}$.

2. Obtain a random feature map $\varphi_M$ from Definition \ref{def:rff}.

3. Compute the sufficient statistics for linear regression of projected data: $\hat{C}_{M}\coloneqq \frac{1}{n}\sum_{i=1}^n\varphi_M(x_i)\varphi_M(x_i)^\top$ and $\hat{u}_T\coloneqq \frac{1}{n}\sum_{i=1}^n [y_i]_T\varphi_M(x_i)$.

4. Privatize the sufficient statistics: $\widetilde{C}_{M}\coloneqq \hat{C}_{M}+\frac{2\Delta_1\left(1+\sqrt{2\log\frac{2}{\delta}}\right)}{n\epsilon}\left(\frac{\varepsilon_{M\times M}+\varepsilon_{M\times M}^\top}{2}\right)$ and $\widetilde{u}_T\coloneqq \hat{u}_T+\frac{2\Delta_2T\left(1+\sqrt{2\log\frac{2}{\delta}}\right)}{n\epsilon}\varepsilon_{M}$ where $\varepsilon_{M}$ and $\varepsilon_{M\times M}$ are a random vector and matrix, respectively, with i.i.d. standard normal entries.

5. Compute $\widetilde{\beta}\coloneqq \widetilde{C}_{M,\lambda}^{-1}\widetilde{u}_T$.

 \Return $\widetilde{f}\coloneqq \widetilde{\beta}^\top\varphi_M$.

\end{algorithmic}
\end{algorithm}
\begin{theorem}\label{thm:dprff_reg_priv}
    Algorithm \ref{alg:dprff_reg} is $(\epsilon,\delta)$-DP.
\end{theorem}

To analyze the excess risk of Algorithm \ref{alg:dprff_reg}, we introduce the \emph{compatibility condition}, a standard theoretical tool in kernel learning with random feature maps ~\citep{randomfeature,li2022optimal,li2019towards,liu2021random, wang2024towards, wang2024optimal}.

\begin{assumption}[Compatibility condition, \cite{randomfeature}]\label{assump:compat} Let $\varphi_M$ be a random feature map constructed from a bounded base function $\varphi:\mathcal{X}\times\Omega\rightarrow\mathbb{R}$. For $\lambda>0$, define
    \begin{equation*}
        \mathcal{F}_\infty(\lambda)\coloneqq \sup_{w\in\Omega}\left\| L_\lambda^{-\frac{1}{2}}\varphi(\cdot,\omega)\right\|^2.
    \end{equation*} 
    We say the compatibility condition holds if there exist constants $F>0$ and $\alpha\in[\gamma,1]$ such that $\mathcal{F}_\infty(\lambda)\leq F\lambda^{-\alpha}$.
\end{assumption}

The quantity $\mathcal{F}_\infty(\lambda)$ captures the efficiency of the dimension reduction induced by the random feature map $\varphi_M$ in kernel learning. Intuitively, when functions $\varphi(\cdot,\omega)$ align well with the leading eigenfunctions of the integral operator $L$, the value of $\mathcal{F}_\infty(\lambda)$ is small. In such cases, a relatively small number of random features $\varphi(\cdot,\omega_j)$ suffices for the $\mathrm{span}\{\varphi(\cdot,\omega_1),\ldots,\varphi(\cdot,\omega_M)\}$ to approximate the RKHS $\mathcal{H}_k$ effectively. Conversely, if $\varphi(\cdot, \omega)$ poorly captures the spectral structure of $L$, a larger number of features is required to achieve a comparable approximation. It is worth noting that $\mathcal{F}_\infty(\lambda)\in [N(\lambda),\sup_{x\in\mathcal{X},\omega\in\Omega}|\varphi(x,\omega)|^2\lambda^{-1}]$. This implies that there always exists some $\alpha\in[\gamma,1]$ such that Assumption \ref{assump:compat} holds for any given feature map $\varphi_M$. In particular, existing theoretical results on RFF-based algorithms that make no assumptions on RFF can be compared with ours under the compatibility condition where $\alpha=1$.

We next present the excess risk bound for DP kernel ridge regression using random feature maps under the compatibility condition.
\begin{theorem}\label{thm:dprff_reg_rate_simple}
     Under Assumptions \ref{assump:bern}-\ref{assump:compat}, for $t\in (0,e^{-1}]$, there exist $n_0>0$ and choices of $\lambda$ and $M$ such that Algorithm \ref{alg:dprff_reg} achieves an excess risk $\widetilde{O}\Big(\Big(n^{-\frac{1}{1+\gamma}}+(n\epsilon)^{-\frac{2}{2+\alpha}}\Big)\log^2\frac{1}{t}\Big)$ with probability at least $1-t$ if $\epsilon\geq n^{-1}$ and $n\geq n_0$.
\end{theorem}
By setting $\alpha=1$, we obtain an excess risk bound of $O(n^{-\frac{1}{1+\gamma}}+(n\epsilon)^{-\frac{2}{3}})$ for differentially private kernel regression based on general RFFs. This result significantly improves upon the previously established rate of $O(n^{-\frac{1}{1+\gamma}}+n^{\frac{2-2\gamma}{1+\gamma}}\epsilon^{-2})$ by~\cite{wang2024differentially}, thereby offering tighter excess risk guarantees for private kernel ridge regression using RFFs.

The excess risk bound in Theorem~\ref{thm:dprff_reg_rate_simple} underscores the critical role of dimension reduction efficiency in the performance of DP kernel ridge regression. The bound improves with decreasing $\alpha$, indicating that more effective dimension reduction—characterized by a better alignment between the random feature map and the kernel eigenspectrum—leads to stronger excess risk guarantees in the private setting. Specifically, the projection dimension $M$ required for our random projection approach is significantly smaller than the number of features typically needed for RFF. For example, the number of features ($M$) required to reach the optimal rate in RFF-based method ranges from $\widetilde{O}\left(n^{\frac{\gamma}{1+\gamma}}\wedge(n\epsilon)^{\frac{2\gamma}{2+\gamma}}\right)$ to $\widetilde{O}\left(n^{\frac{1}{1+\gamma}}\wedge(n\epsilon)^{\frac{2}{3}}\right)$. In contrast, the random projection-based method achieves the optimal rate with $M=O\left(n^{\frac{\gamma}{1+\gamma}}\wedge (n\epsilon)^{\frac{2\gamma}{2+\gamma}}\right)$. This demonstrates that random projection offers the most efficient dimension reduction than any bounded random feature map. Consequently, the excess risk bounds for existing RFF-based methods are generally slower than our algorithm when $\gamma < 1$, rendering them suboptimal in this regime.

\section{Experiments}\label{sec:exp}
We conducted two experiments to evaluate the performance of differentially private kernel ridge regression algorithms based on random projection and RFF. The first experiment was performed on simulated data, while the second was based on a real-world dataset. In both settings, we compared the test errors of the two algorithms-one using random projections (Algorithm~\ref{alg:dprp_reg}) and the other using RFF (Algorithm~\ref{alg:dprff_reg})-under $(\epsilon,\delta)$-DP with $\epsilon\in\{10^{-1},10^{-0.5},1,10^{0.5},10^1\}$ and $\delta=n^{-1.1}$ for the given sample size $n$ of the training data.

For the synthetic experiment, we employ the spline kernel to construct a data-generating distribution that exactly matches the capacity and compatibility conditions in our analysis. This allows us to explicitly verify the minimax-optimality gap between random projection and RFF-based kernel regression. For $q>1$, define the spline kernel on $[0,1]$ as
\begin{align*}
    \Lambda_q(x,x^\prime):=\sum_{k=1}^\infty \frac{2}{|k|^q}\cos(2\pi k(x-x^\prime)).
\end{align*}
For a given $\gamma\in(0,1)$, we consider kernel learning with the kernel $\Lambda_{1/\gamma}$. The eigenvalues of the integral operator of the kernel decay as $l^{-1/\gamma}$, which matches Assumption~\ref{assump:cap}.

The data are generated i.i.d. from the model:
\begin{align*}
    Y=\Lambda_{1/2\gamma}(X,0.5)+\varepsilon,
\end{align*}
where $X\sim\mathrm{Unif}[0,1]$ and $\varepsilon$ follows a zero-mean normal distribution with variance $0.01$, truncated at $\pm 0.1$. The corresponding RFF map can be obtained by applying the Fourier transform to the kernel, as given below:
\begin{align*}
    \varphi_M(\omega,x):=\sqrt{\frac{2}{M}}(\cos(2\pi \omega_1 x+\theta_1),\ldots,\cos(2\pi \omega_M x+\theta_M)),
\end{align*}
where $\omega_1,\ldots,\omega_M$ and $\theta_1,\ldots\theta_M$ are i.i.d. with
\begin{align*}
    \mathbb{P}(\omega_1=k)=\frac{k^{-1/\gamma}}{\sum_{k=1}^\infty k^{-1/\gamma}}\textrm{ for }k\geq1, \quad
    \theta_1\sim\mathrm{Unif}[0,2\pi].
\end{align*}
Under this construction, $\mathcal{F}_{\infty}(\lambda)\asymp \lambda^{-1}$, which implies that the compatibility condition holds with exponent $\alpha=1$ and fails for any $\alpha<1$.

We empirically compare the performance of the two methods across different eigenvalue decay regimes. For the synthetic data experiment, we considered $\gamma\in\{1/10,1/8,1/4\}$, and generated $n=100,000$ samples each for training and testing. For each method, we evaluated a range of hyperparameter pairs $(M, \lambda)$, where the projection dimension $M$ varied from 10 to 1000, and $\lambda$ was selected from the set $\{2^{-i+2}|0\leq i\leq 15\}$. For each privacy level $\epsilon$, we reported the minimum test error achieved across all hyperparameter combinations. The experiments were repeated 100 times. Note that tuning $\lambda$ does not compromise privacy, whereas selecting $M$ does and should ideally be performed using differentially private tuning algorithms, such as those proposed in \cite{chaudhuri2011differentially} and \cite{priv_hyper}. However, incorporating such procedures would only introduce additional logarithmic factors in the excess risk bounds. Since our focus is on the primary scaling behavior, we do not include private hyperparameter tuning in the experiment.

\begin{figure}[htp]
    \centering
    \begin{subfigure}[b]{0.32\textwidth}\includegraphics[width=\textwidth]{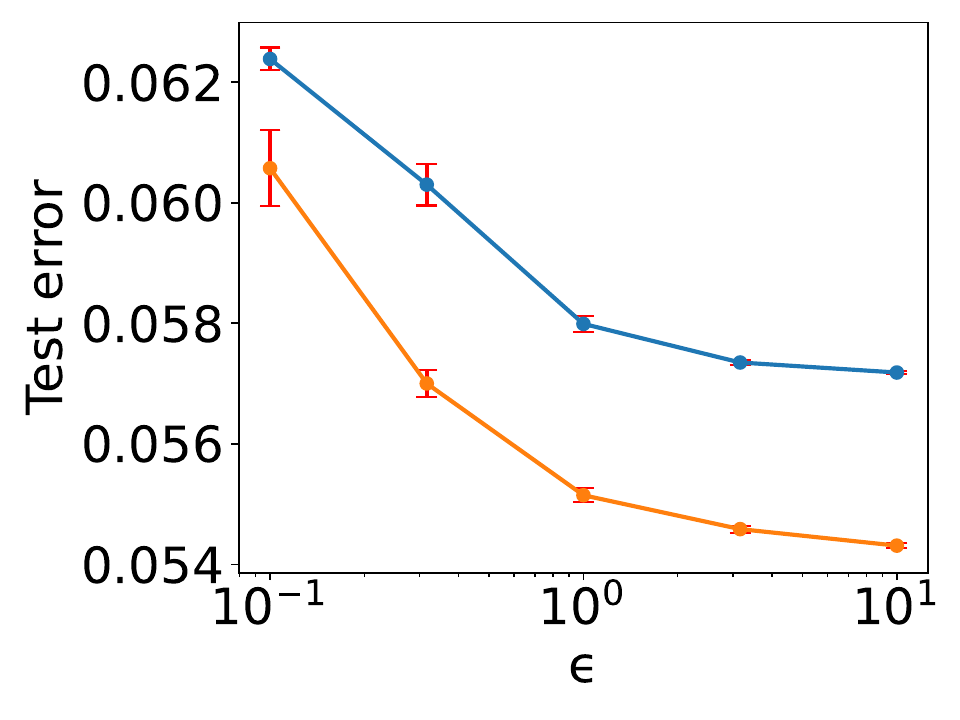}\caption{$\gamma = \frac{1}{10}$}\label{subfig:10dim}
    \end{subfigure}
    \hfill
    \begin{subfigure}[b]{0.32\textwidth}\includegraphics[width=\textwidth]{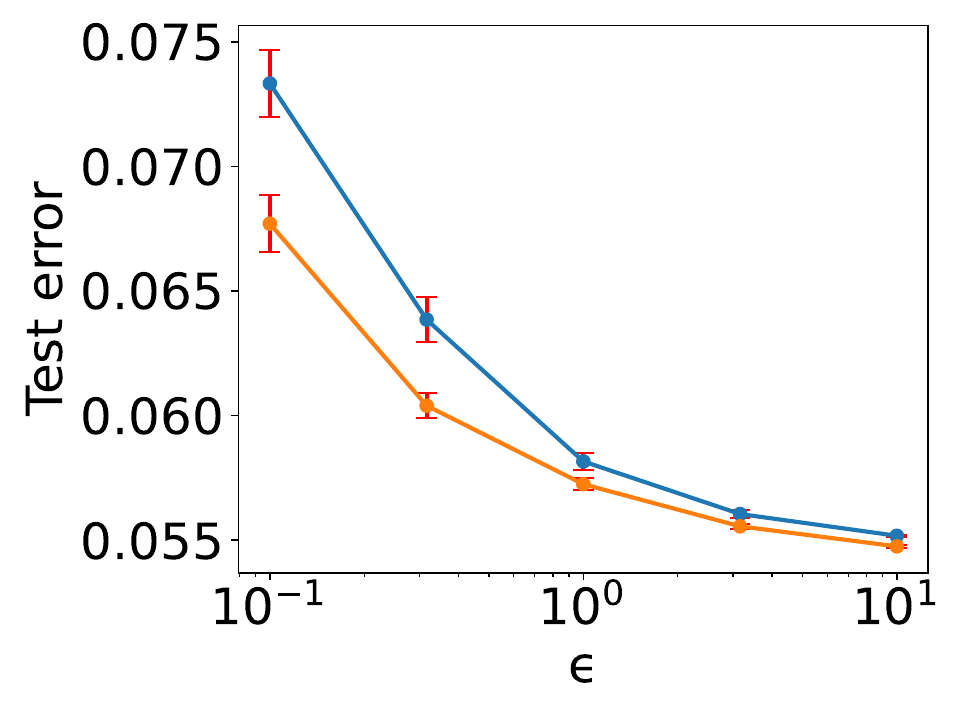}\caption{$\gamma = \frac{1}{8}$}\label{subfig:20dim}
    \end{subfigure}
    \hfill
    \begin{subfigure}[b]{0.32\textwidth}\includegraphics[width=\textwidth]{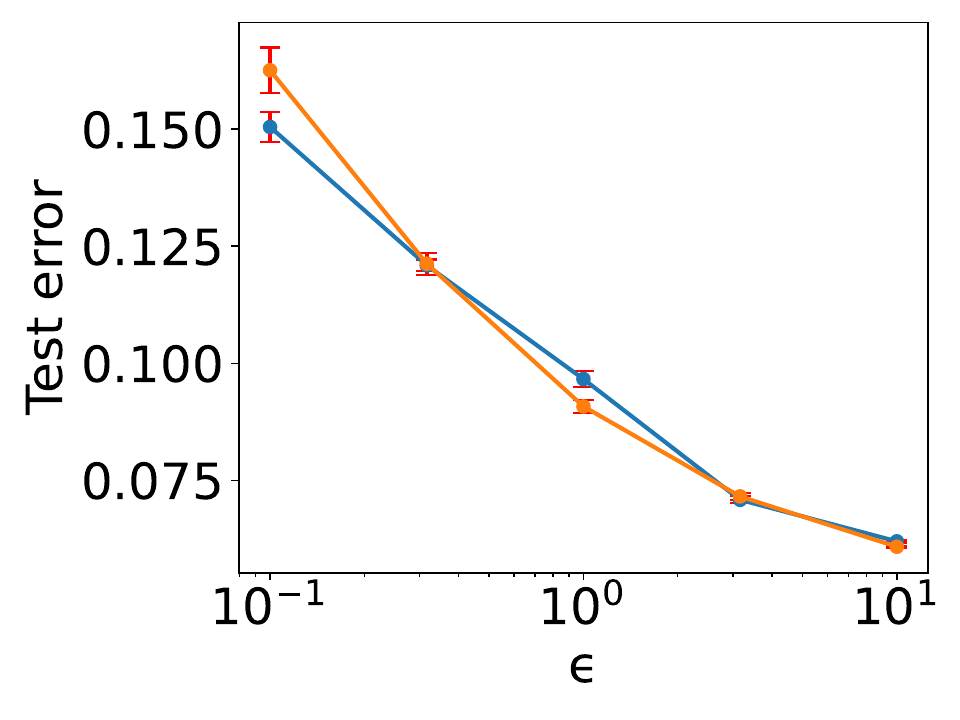}\caption{$\gamma = \frac{1}{4}$}\label{subfig:30dim}
    \end{subfigure}
    \caption{Average test errors for DP kernel ridge regression based on random projection (orange) and RFF (blue) over 100 repetitions. The random projection consistently outperforms RFF across a wide range of privacy budgets, with the performance gap widening as the dimensionality of the original data increases. }
    \label{fig:rp_rff}
\end{figure}

The test errors in Fig.~\ref{fig:rp_rff} demonstrate that the random projection-based algorithm outperforms the RFF-based algorithm across most of the considered privacy budget range except in the case $\gamma=1/4$. Moreover, the performance gap becomes more pronounced as $\gamma$ decreases. This trend is consistent with our theoretical analysis. In particular, the excess risk of the random projection-based method scales as $O((n\epsilon)^{-\frac{2}{2+\gamma}})$, while that of the RFF-based method scales as $O((n\epsilon)^{-\frac{2}{3}})$. As a result, the rate of the RFF-based method does not improve as $\gamma$ decreases, whereas the random projection-based method achieves faster rates, leading to an increasing performance gap. The relatively small gap observed for $\gamma=1/4$ can be attributed to finite-sample effects. In particular, unlike RFF, the random projection-based method induces unbounded features, which leads to a larger noise scale in the Gaussian mechanism in order to ensure differential privacy. This results in a larger constant factor in the excess risk. As this effect does not change the asymptotic rates predicted by our theory, we expect the random projection-based method will outperform the RFF-based method.

\begin{figure}[htp]
    \centering
   \begin{subfigure}[b]{0.32\textwidth}\includegraphics[width=\textwidth]{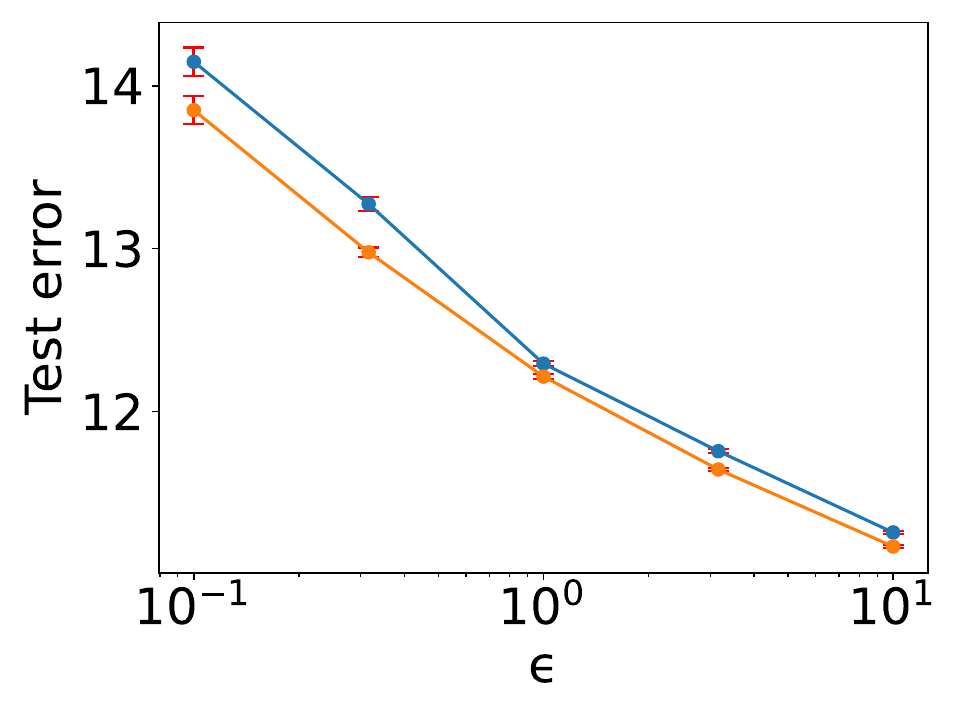}\label{subfig:cali1}
   \end{subfigure}
   \begin{subfigure}[b]{0.32\textwidth}\includegraphics[width=\textwidth]{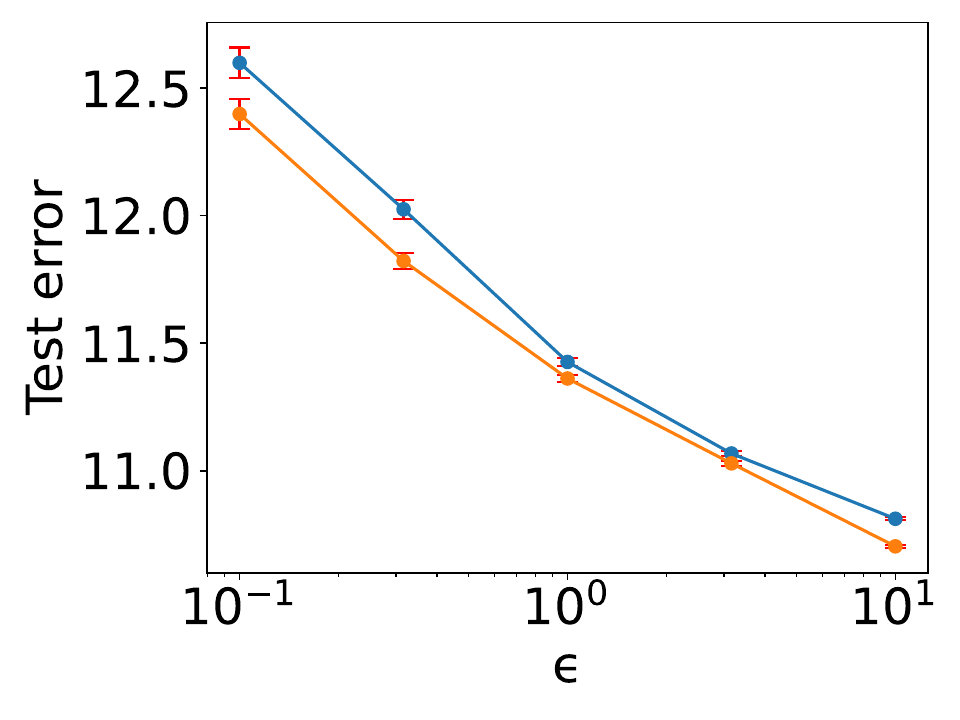}\label{subfig:cali2}
   \end{subfigure}
   \begin{subfigure}[b]{0.32\textwidth}\includegraphics[width=\textwidth]{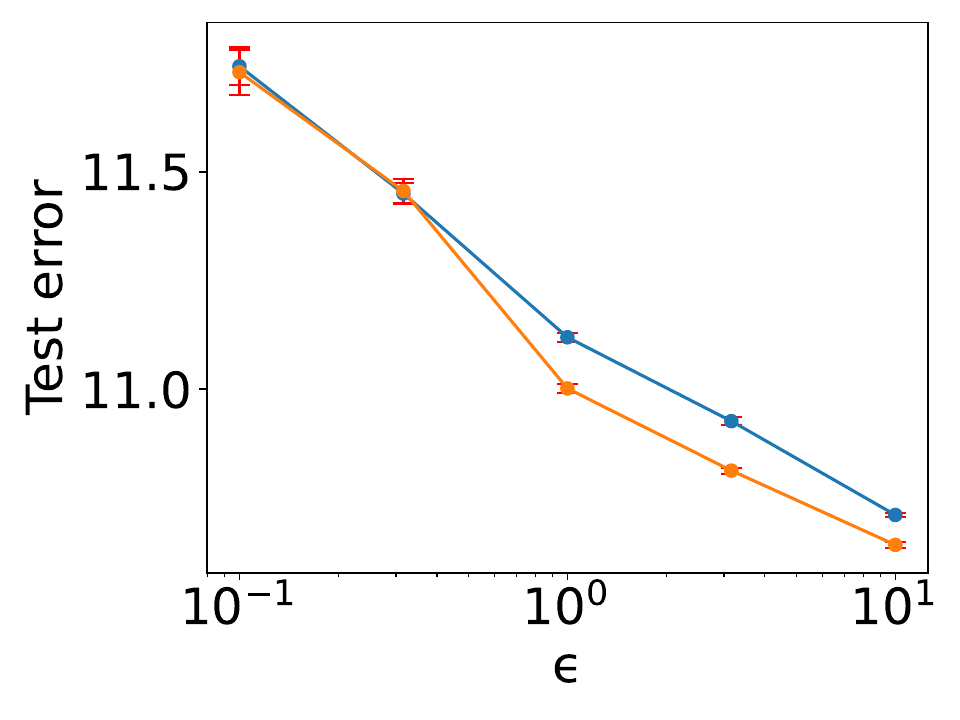}\label{subfig:cali3}
   \end{subfigure}
    \caption{Test errors of DP kernel ridge regression for Million Song dataset with bandwidth parameters $\sigma=$ $2^4$ (left), $2^{4.5}$ (middle), $2^5$ (right).}
    \label{fig:cali}
\end{figure}

For the real-data experiment, we used the Million Song dataset from the UCI machine learning repository \citep{year_prediction_msd_203}, the dataset for the regression task of predicting the year in which a song is released based on audio features
associated with the song. The dataset consists of 515,345 samples with 90 features and a response within $[1922,2011]$. For the experiments, we randomly subsampled 2,0000 data points, using half for training and half for testing. The features were standardized to have a zero mean and unit variance based on the training set. Note that this step may introduce privacy leakage if performed non-privately; in practice, differentially private standardization should be used to maintain privacy guarantees. The responses were centered so that they lie on $[-44.5,44.5]$. Gaussian kernel with multiple bandwidth parameters $\sigma \in \{2^{4},2^{4.5},2^{5}\}$ were used. The experiments were repeated 100 times. The result is presented in Fig.~\ref{fig:cali}. Consistent with the simulation results, the random projection-based method yields lower test errors than the RFF-based method on real data across most of the privacy budget range.

\section{DP kernel learning with Lipschitz-smooth loss}\label{subsec:lipschitz}
In this section, we consider a more general class of loss functions: Lipschitz-smooth loss. We analyze the excess risk of the DP kernel-based ERM presented in Algorithm~\ref{alg:dprp_cls}, comparing it with the existing algorithms of \citet{chaudhuri2011differentially} and \citet{hall2013differential} (Algorithms~\ref{alg:dprff_cls} and~\ref{alg:hall_cls}, respectively). Theorem~\ref{thm:dprp_cls_smooth} provides a high-probability bound on the excess risk of Algorithm~\ref{alg:dprp_cls}.
\begin{theorem}\label{thm:dprp_cls_smooth}
    Under Assumptions~\ref{assump:lip}, \ref{assump:cap}-\ref{assump:kernel}, for $t\in(0,e^{-1}]$, there exist $n_0>0$ and choices of $\lambda$ and $M$ such that Algorithm~\ref{alg:dprp_cls} achieves the excess risk $O\left(\left(n^{-\frac{1}{2}}+(n\epsilon)^{-\frac{2}{2+\gamma}}\right)\log^2\frac{1}{t}\right)$ with probability at least $1-t$ if $\epsilon\geq n^{-1}$ and $n\geq n_0$.
\end{theorem}
Theorem~\ref{thm:dprp_cls_smooth} shows the privacy cost of Algorithm~\ref{alg:dprp_cls} is $(n\epsilon)^{-\frac{2}{2+\gamma}}$. As in the regression, the excess risk rate of random projection-based algorithm is improved under the faster eigenvalue decay. Next, we analyze the existing differentially private kernel ERM algorithms for Lipschitz-smooth loss. We first turn to the RFF-based algorithm, shown in Algorithm~\ref{alg:dprff_cls}. Algorithm~\ref{alg:dprff_cls} is an $(\epsilon,\delta)$-DP version of Algorithm 3 in \cite{chaudhuri2011differentially}, the differentially private kernel ERM algorithm for convex Lipschitz-smooth loss functions based on RFF maps.
\begin{algorithm}
    \caption{RFF-based DP kernel objective perturbation algorithm}\label{alg:dprff_cls}
\begin{algorithmic}
  \Input:\textrm{  }Given data $\{(x_i,y_i)\}_{i=1}^n$, kernel $k$, kernel size $\kappa$, projection dimension $M$, penalty parameter $\lambda$, Lipschitz-smooth parameters $c_1$ and $c_2$, privacy parameters $\epsilon,\delta$.
  \EndInput

  \Output:\textrm{  }A differentially private solution $\widetilde{f}$.
  \EndOutput

1. Set the sensitivity as $\Delta=\sqrt{2}$.

2. Obtain the random projection map $\varphi_M$ from Definition \ref{def:rff}.

3. Generate perturbation vector $b\sim \mathcal{N}\left(0,\frac{4c_1^2\Delta^2\left(2\log\frac{2}{\delta}+\epsilon\right)}{\epsilon^2}I_M\right)$

4. Solve $\widetilde{\beta}\coloneqq \argmin_{\beta\in\mathbb{R}^M}\frac{1}{n}\sum_{i=1}^n l(y,\beta^\top \varphi_M(x_i))+\max\Big\{\frac{\lambda}{2},\frac{c_2}{n}\left(e^{\epsilon/4}-1\right)^{-1}\Big\}\lVert\beta\rVert_2^2+\frac{b^\top \beta}{n}$.

 \Return $\widetilde{f}\coloneqq \widetilde{\beta}^\top \varphi_M$.

\end{algorithmic}
\end{algorithm}
Using the similar arguments by \citet{chaudhuri2011differentially} and \citet{kifer2012private} we can show the privacy guarantee of the algorithm.
\begin{theorem}\label{thm:dprff_cls_priv}
    Under Assumption \ref{assump:lip}, Algorithm \ref{alg:dprff_cls} is $(\epsilon,\delta)$-DP.
\end{theorem}
Theorem~\ref{thm:dprff_cls_smooth} establishes the excess risk rate of Algorithm~\ref{alg:dprff_cls}.
\begin{theorem}\label{thm:dprff_cls_smooth}
    Under Assumptions~\ref{assump:lip}, \ref{assump:cap}--~\ref{assump:compat} for $t\in(0,e^{-1}]$, there exist $n_0>0$ and choices of $\lambda$ and $M$ such that Algorithm~\ref{alg:dprp_cls} achieves the excess risk $\widetilde{O}\Big(\Big(n^{-\frac{1}{2}}+(n\epsilon)^{-\frac{2}{2+\alpha}}\Big)\log^2\frac{1}{t}\Big)$ with probability at least $1-t$ if $\epsilon\geq n^{-1}$ and $n\geq n_0$.
\end{theorem}
The rate given in Theorem~\ref{thm:dprff_cls_smooth} improves upon previously known results for RFF-based differentially private kernel ERM with Lipschitz-smooth loss. In particular, Theorem 25 of \cite{chaudhuri2011differentially} establishes the excess risk rate $O(n^{-1/2}+(n\epsilon)^{-1/4})$ under the same theoretical setting of ours with $\alpha=1$, which is strictly slower than the guarantee of Theorem~\ref{thm:dprff_cls_smooth}, $\widetilde{O}(n^{-1/2} + (n\epsilon)^{-2/3})$. However, even the improved rate is slower compared to the random projection-based method (Algorithm~\ref{alg:dprp_cls}).

We briefly comment on the role of the DP linear ERM subroutine in Algorithm~\ref{alg:dprp_cls} and~\ref{alg:dprff_cls} in the excess risk rates. If the same analysis is applied to a variant of these algorithms using output perturbation instead of objective perturbation, the resulting excess risk rates are strictly worse than the ones stated in Theorem~\ref{thm:dprp_cls_smooth} and~\ref{thm:dprff_cls_smooth}. The limitations of output perturbation were first noted in \cite{chaudhuri2011differentially}, and further examined by \cite{arora2022differentially}, who showed that output perturbation can achieve the minimax-optimal rate. Their theoretical setting corresponds to the special case of ours. Putting $\gamma=1$ with the standard inner product kernel, Theorem~\ref{thm:dprp_cls_smooth} recovers their rate $O(n^{-\frac{1}{2}}+(n\epsilon)^{-\frac{1}{3}})$. However, in a situation where the eigenvalue decays faster, the combination of objective perturbation and dimension reduction leads to significantly sharper rates across a broad range of $\gamma$. As in the least squares regression setting, this underscores the importance of selecting the appropriate private linear ERM algorithm to attain strong privacy-utility trade-offs.

Finally, we examine the functional perturbation approach proposed by \citet{hall2013differential}. The detailed description of the algorithm can be found in Algorithm~\ref{alg:hall_cls}. Although this method is frequently cited in the literature, its statistical performance under general structural assumptions, such as the eigenvalue assumption, has not been thoroughly examined. Our analysis reveals that this approach yields the weakest excess risk guarantee among the algorithms considered.
\begin{theorem}\label{thm:dphall_cls_smooth_rate}
     Under the same assumptions of Theorem~\ref{thm:dprp_cls_smooth}, for $c>0$, there exists $n_0>0$ and a choice of $\lambda$ such that the algorithm of \cite{hall2013differential} achieves an excess risk of order: \[\widetilde{O}\left(\left(n^{-\frac{1}{2}}+ (n\epsilon)^{-\frac{2}{3}}\right)\log^2\frac{1}{t}\right),\]
     with probability at least $1-t$ if $\epsilon\geq n^{-1}$ and $n\geq n_0$.
\end{theorem}
The relatively weak excess risk guarantee of the algorithm proposed by \cite{hall2013differential} can be attributed to its lack of \emph{compression}. Although their functional perturbation approach incorporates regularization via an $\mathcal{H}_k$-norm penalty, it perturbs the infinite-dimensional estimator directly, without reducing the complexity of the hypothesis space. In contrast, Algorithms~\ref{alg:dprp_cls} and~\ref{alg:dprff_cls} perform explicit dimension reduction before privatization, followed by $\ell_2$-regularized optimization in a finite-dimensional subspace. This combination not only reduces sensitivity but also lowers the privacy cost. The ordering of the excess risk bounds across the considered methods reinforces the conclusion that efficient dimension reduction is essential for achieving strong utility guarantees under differential privacy constraints.

We discuss the difference between our approach and \citet{hall2013differential} in detail. Both our algorithm and that of \cite{hall2013differential} rely on Gaussian processes, but for fundamentally different reasons with contrasting effects on their performances. \cite{hall2013differential} use Gaussian processes to inject noise directly into the empirical risk minimizer in the RKHS to ensure privacy. In contrast, Algorithm~\ref{alg:dprp_cls} uses a Gaussian process to construct a randomized finite-dimensional function space. This enables a form of statistical compression through random projection, which precedes the privatization step and plays a crucial role in improving excess risk guarantees.

\section{Improved excess risk rates under additional conditions}\label{sec:faster}
Faster excess risk rates and tighter minimax lower bounds can be obtained under additional conditions. These correspond to the standard source condition and local strong convexity, which are commonly used assumptions in kernel learning theory. In this section, we establish the excess risk rates and minimax lower bounds for both the proposed and existing differentially private kernel ERM algorithms under these favorable conditions. The source condition imposes a regularity assumption on $f_{\mathcal{H}_k}$ by specifying its alignment with the leading eigenfunctions of the integral operator $L$ \citep{caponnetto2007optimal,randomfeature,embedding_property}.
\begin{assumption}[Source condition]\label{assump:source}
    There exists $r\in[1/2,1]$ such that $f_{\mathcal{H}_k}=L^rg$ for some $g\in L^2(P_X)$ with $\lVert g\rVert\leq R$.
\end{assumption}
Note that for any $f \in \mathcal{H}_k$, there exists a function $g \in L^2(P_X)$ such that $f = L^{1/2} g$. Therefore, Assumption~\ref{assump:source} is always satisfied for some $g$ when $r \geq 1/2$. As $r$ increases, the representation of $f_{\mathcal{H}_k}=L^rg$ places greater weights on the leading eigenfunctions of $L$, corresponding to a smoother function in the input space. This parameter characterizes the impact of the smoothness of the target function in the excess risk bound. Larger the $r$ is, a faster rate can be achieved  \citep{falkon,randomfeature}.
We establish excess risk bounds of differentially private kernel ridge regression via random projection and random features in Theorems~\ref{thm:dprp_reg_rate} and~\ref{thm:dprff_reg_rate}, respectively.
\begin{theorem}\label{thm:dprp_reg_rate} Under Assumption~\ref{assump:source} and the assumptions of Theorem~\ref{thm:dprp_reg_rate_simple}, for $t\in (0,e^{-1}]$, there exist $n_0>0$ and choices of $T$, $\lambda$, and $M$ such that, Algorithm \ref{alg:dprp_reg} achieves an excess risk of order $O\Big (\Big(n^{-\frac{2r}{2r+\gamma}}+(n\epsilon)^{-\frac{4r}{4r+\gamma}}\Big)\log^2\frac{1}{t}\Big)$ with probability at least $1-t$ if  $\epsilon\geq n^{-1}$ and $n\geq n_0$.
\end{theorem}
The excess risk bound for non-private kernel ridge regression under the same assumption is $O\left(n^{-\frac{2r}{2r+\gamma}}\right)$ \citep{caponnetto2007optimal}. Hence, the privacy constraint introduces an additional error term of order $O\left((n\epsilon)^{-\frac{4r}{4r + \gamma}}\right)$ for Algorithm~\ref{alg:dprp_reg} in the high privacy regime. The rate becomes faster as $r$ increases and recovers the bound of Theorem~\ref{thm:dprp_reg_rate_simple} when $r=1/2$.
\begin{theorem}\label{thm:dprff_reg_rate}
Under Assumption~\ref{assump:source} and the assumptions of Theorem~\ref{thm:dprff_reg_rate_simple}, for $t\in (0,e^{-1}]$, there exist $n_0>0$ and choices of $\lambda$ and $M$ such that Algorithm \ref{alg:dprff_reg} achieves an excess risk $\widetilde{O}\Big(\Big(n^{-\frac{2r}{2r+\gamma}}+(n\epsilon)^{-\frac{4r}{4r+\gamma_\alpha}}\Big)\log^2\frac{1}{t}\Big)$ with probability at least $1-t$ if $\epsilon\geq n^{-1}$ and $n\geq n_0$, where $\gamma_{\alpha}\coloneqq (2r-1)\gamma+(2-2r)\alpha$.
\end{theorem}
The rate established in Theorem~\ref{thm:dprff_reg_rate} recovers the result of Theorem~\ref{thm:dprff_reg_rate_simple} when $r=1/2$, ensuring consistency with our previous analysis. A comparison with Theorem~\ref{thm:dprp_reg_rate} further indicates that our proposed method outperforms RFF-based approaches even under these favorable conditions. Interestingly, while the excess risk rate improves as $\alpha$ decreases, the magnitude of this improvement depends on the value of $r$. When $r=1/2$, employing random features with a smaller $\alpha$ leads to a faster rate. However, when $r=1$, no such improvement is observed. This suggests that when the target function is sufficiently smooth, the algorithm can effectively compensate for the inefficiency of the random features, regardless of the value of $\alpha$.

Finally, we derive a minimax lower bound under the additional assumptions. Let $\mathcal{P}_{2}$ denote the class of data distributions over $\mathcal{X}\times\mathcal{Y}$ satisfying Assumptions~\ref{assump:bern}-\ref{assump:cap} and~\ref{assump:source}-\ref{assump:lsc}. We establish the minimax lower bound for differentially private kernel ridge regression over $\mathcal{P}_2$.
\begin{theorem}\label{thm:lower_bd} Under Assumption \ref{assump:kernel} and $\delta=o(n^{-1})$, the minimax risk is lower bounded by
\begin{equation*}
    \mathfrak{R}(\mathcal{P}_2,\mathcal{E},\epsilon,\delta)\gtrsim
    n^{-\frac{2r}{2r+\gamma}}+(n\epsilon)^{-\frac{4r}{2r+\gamma+1}-\eta}\wedge1,
\end{equation*}
for any $\eta>0$.
\end{theorem}
By letting $\eta \to 0^+$, the lower bound in Theorem~\ref{thm:lower_bd}
can be made arbitrarily close to $O(n^{-\frac{2r}{2r+\gamma}}+(n\epsilon)^{-\frac{4r}{2r+\gamma+1}}\wedge1)$. Hence, the minimax rate is $O(n^{-\frac{2r}{2r+\gamma}}+(n\epsilon)^{-\frac{4r}{2r+\gamma+1}}\wedge1)$ up to an arbitrarily small slack in the exponent. First, as $\epsilon\rightarrow\infty$, our lower bound recovers the non-private minimax rate $O(n^{-\frac{2r}{2r+\gamma}})$ for kernel regression under the same assumption \citep{caponnetto2007optimal}. 

Unlike in the squared loss setting, the source condition fails to improve the excess risk bound for Lipschitz-smooth losses. This is because such loss functions may lack sufficient curvature. Specifically, while the source condition characterizes the regularity of the target function via the spectral structure of the integral operator $L$, this structure cannot be effectively exploited when the loss is merely convex and Lipschitz. To address this, recent studies \citep{li2019towards, quantile, liu2025improved} have shown that many commonly used loss functions satisfy a refined curvature property known as local strong convexity.
\begin{assumption}[Local strong convexity, \cite{li2019towards}]\label{assump:lsc}
    For a given data distribution $P_{X,Y}$, there exist $\mu>0$ and $R_{\mathrm{lsc}}>0$ such that
    \begin{eqnarray}\label{eq:lsc}
       \mathcal{E}(f)-\mathcal{E}(f_{\mathcal{H}_k})\geq \frac{\mu}{2}\lVert f-f_{\mathcal{H}_k}\rVert^2,
    \end{eqnarray}
    for all $f\in L^2(P_X)$ satisfying $\lVert f-f_{\mathcal{H}_k}\rVert_{}\leq R_{\mathrm{lsc}}$.
\end{assumption}
The condition ensures that the risk function exhibits positive curvature in a neighborhood around the target function $f_{\mathcal{H}_k}$. Under this condition, significantly faster convergence rates can be achieved in the non-private setting.
In the remainder of this section, we show that similar gains are possible under differential privacy as well. This suggests that, in practice, the privacy cost may be much smaller than the bound in Theorem~\ref{thm:dprp_cls_smooth} implies, provided that the loss function exhibits localized curvature near the optimum.

\begin{theorem}\label{thm:dprp_cls_rate} Under Assumptions~\ref{assump:source}--\ref{assump:lsc}, and the same assumptions of Theorem~\ref{thm:dprp_cls_smooth}, for $c>0$, there exist $n_0>0$ and choices of $\lambda$ and $M$ such that Algorithm~\ref{alg:dprp_cls} achieves an excess risk bound $O\Big(\Big(n^{-\frac{2r}{2r+\gamma}}+(n\epsilon)^{-\frac{4r}{4r+\gamma}}\Big)\log^2n\Big)$ with probability at least $1-n^{-c}$, if $\epsilon\geq n^{-1}$ and $n\geq n_0$.
\end{theorem}
The rate in Theorem \ref{thm:dprp_cls_rate} shows that we can achieve a smaller privacy cost under the local strong convexity condition compared to Theorem \ref{thm:dprp_cls_smooth}.

\begin{theorem}\label{thm:dprff_cls_rate}
    Under Assumptions~\ref{assump:source}--\ref{assump:lsc}, and the same assumptions of Theorem~\ref{thm:dprff_cls_smooth}, for $c>0$, there exist $n_0>0$ and choices of $\lambda$ and $M$ such that Algorithm \ref{alg:dprff_cls} achieves $O\Big(\Big(n^{-\frac{2r}{2r+\gamma}}+(n\epsilon)^{-\frac{4r}{4r+\gamma_\alpha}}\Big)\log^2n\Big)$ with probability at least $1-n^{-c}$ if $\epsilon\geq n^{-1}$ and $n\geq n_0$.
\end{theorem}

Finally, we compare our excess risk analysis with that of a differentially private kernel ERM algorithm based on \emph{functional perturbation}, originally proposed by \cite{hall2013differential}. Although this method is frequently cited in the literature, its statistical performance under general structural assumptions, such as the eigenvalue and source conditions, has not been thoroughly examined. Our analysis reveals that this approach yields the weakest excess risk guarantee among the algorithms considered, under the same assumptions as Theorem~\ref{thm:dprp_cls_rate} when $r=1/2$. We provide the excess risk bound of the functional perturbation algorithm in the theorem below.
\begin{theorem}\label{thm:dphall_cls_rate}
Under the same assumptions of Theorem~\ref{thm:dprp_cls_rate}, for $c>0$, there exists $n_0>0$ and a choice of $\lambda$ such that the algorithm of \cite{hall2013differential} achieves an excess risk of order: \[O\left(\left(n^{-\frac{2r}{2r+\gamma}}+ (n\epsilon)^{-\frac{4r}{2r+2}}\right)\log^2n\right),\]
     with probability $1-n^{-c}$ if $\epsilon\geq n^{-1}$ and $n\geq n_0$.
\end{theorem}
As a side note, we state that the local strongly convexity assumption can be replaced by $\ell_\infty$ condition. That is, \eqref{eq:lsc} holds for $f\in\mathcal{H}_k$ satisfying $\lVert f\rVert_{\infty}\leq R_{\mathrm{lsc}}\lVert f_{\mathcal{H}_k}\rVert_{\mathcal{H}_k}$ instead of $f$ satisfying $\lVert f-f_{\mathcal{H}_k}\rVert_{}\leq R_{\mathrm{lsc}}$.

Finally, we establish the minimax risk for differentially private kernel learning when the loss function satisfies the Lipschitz smoothness and local strong convexity conditions (Assumptions~\ref{assump:lip} and~\ref{assump:lsc}). Let $\mathcal{P}_3$ be a family of distribution satisfying Assumptions~\ref{assump:bern}--\ref{assump:cap} and~\ref{assump:source}.
\begin{theorem}\label{thm:lower_bd_general} Under Assumption~\ref{assump:kernel} and with $\delta = O(n^{-1})$, there exists a loss function $l$ satisfying Assumptions~\ref{assump:lip} and~\ref{assump:lsc} such that the minimax lower bound is given by
\begin{equation*}
    \mathfrak{R}(\mathcal{P}_3,\mathcal{E},\epsilon,\delta)\gtrsim
    n^{-\frac{2r}{2r+\gamma}}+(n\epsilon)^{-\frac{4r}{\gamma+2r+1}-\eta}\wedge1,
\end{equation*}
for any $\eta>0$.
\end{theorem}
The result in Theorem~\ref{thm:lower_bd_general} mirrors the findings in Theorem~\ref{thm:lower_bd} for kernel ridge regression, establishing that the same minimax lower bound applies under the locally strongly convex loss function. This equivalence underscores that, given sufficient local curvature, private classification exhibits learning behavior analogous to private least squares regression.
\section{Discussion}
In this work, we developed a family of differentially private kernel ERM algorithms based on random projections, with formal guarantees under $(\epsilon, \delta)$-differential privacy. We analyzed both the squared loss and general Lipschitz-smooth losses, and derived upper bounds on the excess risk for the proposed methods. These results strictly improve upon the guarantees of existing differentially private kernel ERM algorithms. We further established a minimax lower bound and showed that the proposed algorithm attains the minimax-optimal rate. As a consequence, we obtained a class of objective perturbation algorithms that are computationally efficient and enjoy dimension-independent risk bounds. Finally, we validated the theoretical findings with empirical results, which confirm the practical advantages of the proposed approach. To our knowledge, this is the first differentially private kernel ERM method to achieve minimax-optimality.

There are two important implications arising from our results. First, our theoretical analysis highlights the critical role of dimension reduction in achieving strong excess risk guarantees for private kernel ERM algorithms. Specifically, we showed that random projection achieves risk guarantees that match the most efficient methods within a compatible algorithmic class, thereby establishing minimax-optimality. This result partially addresses the open question posed by \cite{chaudhuri2011differentially}: \textit{Is there a more statistically efficient solution to privacy-preserving learning with kernels?}, where they noted the statistical inefficiency of algorithms based on RFF. Our findings suggest that random projection offers a viable and theoretically sound alternative. Extending random projection-based dimension reduction to other private kernel learning problems presents a promising direction for future work and may lead to additional minimax-optimal procedures. 

Second, our excess risk analysis suggests that prior studies may have overestimated the privacy cost associated with differentially private kernel ERM algorithms based on RFF. In particular, we established tighter bounds on the privacy loss for both least squares regression \citep{wang2024differentially} and classification \citep{chaudhuri2011differentially}. Notably, we showed that substantially smaller privacy costs can be achieved under realistic conditions—especially when the loss function satisfies a local strong convexity condition. These results indicate that, contrary to earlier assessments, the practical privacy cost of differentially private kernel ERM may be significantly lower than previously believed.

Finally, we acknowledge some limitations of our work. The primary drawback of the proposed random projection-based differentially private kernel ERM algorithm is its computational complexity, which scales as $O(n^3)$ due to the reliance on Gaussian processes. Designing more computationally efficient variants, possibly through approximate projections or sparse representations, remains an important direction for future research. In addition, our optimality analysis reveals a limitation related to the embedding property. Under favorable assumptions, while fast excess risk rates are possible, they do not match the corresponding minimax lower bounds. A key open problem is whether it is possible to construct a private algorithm that achieves the minimax-optimal rate under the source condition and the embedding property.


\section*{Acknowledgments and Disclosure of Funding}
All correspondence should be addressed to Cheolwoo Park and Jeongyoun Ahn (corresponding authors) at \href{mailto:parkcw2021@kaist.ac.kr}{parkcw2021@kaist.ac.kr} and \href{mailto:jyahn@kaist.ac.kr}{jyahn@kaist.ac.kr}. The work of Jeongyoun Ahn and Cheolwoo Park was partially supported by the National Research Foundation of Korea (RS-2022-NR068758).


\newpage

\appendix
\section{Proofs of omitted lemmas and theorems}
\subsection{Proof of Theorem~\ref{thm:dprp_reg_priv}}\label{sec:proof_dprp_reg_priv}
To show the privacy guarantee we use the following lemma and proposition.
\begin{lemma}\label{lem:priv_bound}
Under Assumption \ref{assump:kernel}, for $t\in (0,1]$ and $x\in\mathcal{X}$,
\begin{equation*}
    \mathbb{P}_{\omega_{1:M}}\left(\lVert h_M(x)\rVert_2^2\leq \kappa^2\left(1+2\sqrt{\frac{\log\frac{1}{t}}{M}}+\frac{2\log\frac{1}{t}}{M}\right)\right)\geq 1-t.
\end{equation*}
\end{lemma}
\begin{proof}
We use the lemma by \cite{laurentandmassart}:
    \begin{proposition}[\citet{laurentandmassart}]\label{prop:gaussian_vec}
    Let $(Y_1,\ldots,Y_n)$ be i.i.d. standard Gaussian variables, and let $a=(a_1,\ldots,a_n)$ be a nonnegative vector. Define $S=\sum_{i=1}^n a_i(Y_i^2-1)$. Then for any $x>0$, the following inequalities hold:
\begin{align*}
\mathbb{P}\left(S\geq2\left\lVert a\right\rVert_2\sqrt{x}+2\left\lVert a\right\rVert_\infty x\right)\leq& e^{-x}\\\mathbb{P}\left(S\leq-2\left\lVert a\right\rVert_2\sqrt{x}\right)\leq& e^{-x}.
\end{align*}
\end{proposition}
For a fixed $x$, the squared $\ell_2$ norm of $h_M(x)$ can be expressed as $\frac{Z_1^2+\cdots+Z_M^2}{M}$, where $Z_i\sim \mathcal{N}(0,k(x,x))$. Applying Proposition~\ref{prop:gaussian_vec} with $a_i=\frac{1}{M}$ yields the desired bound.
\end{proof}
Let $D=\{(x_1,y_1),\ldots,(x_{n-1},y_{n-1}),(x,y)\}$ and $D^{\prime}=\{(x_1,y_1),\ldots,(x_{n-1},y_{n-1}),(x^\prime,y^\prime)\}$ be neighboring datasets. By Lemma \ref{lem:priv_bound},
\begin{equation*}
    \mathbb{P}_{\omega_{1:M}}\left(\lVert h_M(x)\rVert_2^2\geq \kappa^2\left(1+2\sqrt{\frac{\log\frac{1}{t}}{M}}+\frac{2\log\frac{1}{t}}{M}\right)\right)\leq t. 
\end{equation*}
Letting $t=\delta/8$, define
\begin{equation*}
    E_1\coloneqq \Bigg\{\omega_{1:M}:\left\lVert h_M(x^\prime)h_M(x^\prime)^\top-h_M(x)h_M(x)^\top\right\rVert_F\leq 2\kappa^2\left(1+2\sqrt{\frac{\log\frac{8}{\delta}}{M}}+\frac{2\log\frac{8}{\delta}}{M}\right)\Bigg\}.
\end{equation*}
Then,
\begin{eqnarray*}
\mathbb{P}\left(E_1^c\right)&\leq&\mathbb{P}\left(\lVert h_M(x)h_M(x)^\top\rVert_F\geq\kappa^2\left(1+2\sqrt{\frac{\log\frac{8}{\delta}}{M}}+\frac{2\log\frac{8}{\delta}}{M}\right)\right)\\
&&+\mathbb{P}\left(\lVert h_M(x^\prime)h_M(x^\prime)^\top\rVert_F\geq\kappa^2\left(1+2\sqrt{\frac{\log\frac{8}{\delta}}{M}}+\frac{2\log\frac{8}{\delta}}{M}\right)\right)\\
    &\leq&\frac{\delta}{4}.
\end{eqnarray*}
Denote $\widetilde{C}_{M}^0\coloneqq \frac{1}{n}\sum_{i=1}^nz_iz_i^\top+\varepsilon_{M\times M}$. Then, for any measurable subset $A$ of $\mathcal{Z}$, the following inequality holds: {\small
\begin{eqnarray*}\mathbb{P}\left(\widetilde{C}_{M}^0\in A|D\right)
    &\leq&\frac{1}{4}\delta+\mathbb{P}\left(\hat{C}_{M}+\frac{2\Delta_1\left(1+\sqrt{2\log\frac{4}{\delta}}\right)}{n\epsilon}\varepsilon_{M\times M}\in A,E_1\vast|D\right)\\
    &\leq&\frac{1}{2}\delta+e^{\frac{1}{2}\epsilon}\mathbb{P}\left(\hat{C}_{M}+\frac{2\Delta_1\left(1+\sqrt{2\log\frac{4}{\delta}}\right)}{n\epsilon}\varepsilon_{M\times M}\in A\vast|D^\prime\right)
\end{eqnarray*}}
by Proposition~\ref{prop:gaussian}, where we treat $\hat{C}_{M,\lambda}$ as the $M^2$-dimensional vector. 
Hence, releasing $\widetilde{C}_{M}=\frac{1}{2}\left(\widetilde{C}_{M}^{0}+\widetilde{C}_{M}^{0\top}\right)$ is $\left(\frac{1}{2}\epsilon,\frac{1}{2}\delta\right)$-DP.

The privacy of $\widetilde{u}$ can be established in a similar manner. Define
\begin{equation*}
    E_2\coloneqq \left \{\omega_{1:M}:\lVert [y^\prime]_T h_M(x^\prime)-[y]_Th_M(x)\rVert_2\leq \Delta_2\right \}.
\end{equation*}
Then $\mathbb{P}\left(E_2\right)\geq1-\frac{\delta}{4}$. Thus, for any measurable subset $A$ of $\mathcal{Z}$, the following inequality holds:
{\small
\begin{eqnarray*}
    \mathbb{P}\left(\widetilde{u}\in A|D\right)
    &\leq&\frac{1}{4}\delta+\mathbb{P}\left(\hat{u}_T+\frac{2\Delta_2\left(1+\sqrt{2\log\frac{4}{\delta}}\right)}{n\epsilon}\varepsilon\in A,E_2\vast|D\right)\\
    &\leq&\frac{1}{2}\delta+e^{\frac{1}{2}\epsilon}\mathbb{P}\left(\hat{u}_T+\frac{2\Delta_2\left(1+\sqrt{2\log\frac{4}{\delta}}\right)}{n\epsilon}\varepsilon\in A\vast|D^\prime\right).
\end{eqnarray*}}
Therefore, releasing $\widetilde{u}$ is also $\left(\frac{1}{2}\epsilon,\frac{1}{2}\delta\right)$-DP. By Proposition~\ref{thm:comp}, the algorithm is $(\epsilon,\delta)$-DP.
\subsection{Proof of Theorem~\ref{thm:dprff_reg_priv}}\label{sec:proof_dprff_reg_priv}
The theorem can be shown in the same way as in the proof of Theorem~\ref{thm:dprp_reg_priv}. Furthermore, we can omit the step of bounding $\lVert h_M(x)\rVert_2$ with high probability since $\lVert\varphi_M(x)\rVert_2$ is bounded by $\sqrt{2}$.

\subsection{Proof of Theorem~\ref{thm:dprp_cls_priv}}\label{sec:proof_dprp_cls_priv}
Denote
    \begin{equation*}
        E_3:=\Bigg\{\lVert h_M(x)h_M(x)^\top\rVert_F,\lVert h_M(x^\prime)h_M(x^\prime)^\top\rVert_F\leq\Delta_3^2\Bigg\}.
    \end{equation*}
    Then, $\mathbb{P}(E_3^c)\leq\frac{\delta}{2}$ by Lemma~\ref{lem:priv_bound}.
    
From the definition of $\widetilde{\beta}_D$ we have
\begin{equation*}
    0=\frac{1}{n}\sum_{i=1}^nl_{\hat{y}}(y_i,\widetilde{\beta}_D^\top h_M(x_i))h_M(x_i)+\lambda_0\widetilde{\beta}_D+\frac{1}{n}b.
\end{equation*}
Denote
\begin{equation*}
    b(\beta;D):=-\sum_{i=1}^nl_{\hat{y}}(y_i,\beta^\top h_M(x_i))h_M(x_i)-n\lambda_0\beta,
\end{equation*}
then $b(\widetilde{\beta}_D;D)\sim N\left(0,\sigma_b^2I\right)$ holds where $\sigma_b^2=\frac{4c_1^2\Delta^2\left(2\log\frac{4}{\delta}+\epsilon\right)}{\epsilon^2}$. Thus, for $\beta\in\mathbb{R}^M$, we have
\begin{equation*}
    p_D(\beta)=p_b(b(\beta;D))|\mathrm{det}\nabla_{\beta}b(\beta;D)|,
\end{equation*}
where $p_D$ and $p_b$ denote the densities of $\widetilde{\beta}_D$ and $N\left(0,\sigma_b^2I\right)$ and $\nabla$ denotes the Jacobian matrix of $b$ with respect to $\beta$.

Let $D=\{(x_1,y_1),\ldots,(x_{n-1},y_{n-1}),(x,y)\}$ and $D^\prime=\{(x_1,y_1),\ldots,(x_{n-1},y_{n-1}),(x^\prime,y^\prime)\}$ be neighboring datasets.

We bound the ratio $\frac{p_D(\beta)}{p_{D^\prime}(\beta)}=\frac{p_b(b(\beta;D))}{p_b(b(\beta;D^\prime))}\frac{|\mathrm{det}\nabla_{\beta}b(\beta;D)|}{|\mathrm{det}\nabla_{\beta}b(\beta;D^\prime)|}$.

The first term of the ratio can be bounded as follows:
\begin{align*}
\frac{p_b(b(\beta;D))}{p_b(b(\beta;D^\prime))}=&\mathrm{exp}\left(\frac{\lVert b(\beta;D^\prime)\rVert_2^2-\lVert b(\beta;D)\rVert_2^2}{2\sigma_b^2}\right)\\
=&\mathrm{exp}\left(\frac{\lVert b(\beta;D^\prime)-b(\beta;D)\rVert_2^2+2\left\langle b(\beta;D^\prime)-b(\beta;D),b(\beta;D)\right\rangle}{2\sigma_b^2}\right).
\end{align*}
Note that $b(\beta;D^\prime)-b(\beta;D)=l_{\hat{y}}(y,\beta^\top h_M(x))h_M(x)-l_{\hat{y}}(y^\prime,\beta^\top h_M(x^\prime))h_M(x^\prime)$ is independent of $b$. Thus, we have 
\begin{equation*}
    \left\langle b(\beta;D^\prime)-b(\beta;D),b(\beta;D)\right\rangle\sim N(0,\sigma_b^2\lVert b(\beta;D^\prime)-b(\beta;D)\rVert_2^2).
\end{equation*}
Denote
\begin{equation*}
    E_4:=\left\{\left|\left\langle b(\beta;D^\prime)-b(\beta;D),b(\beta;D)\right\rangle\right|\leq\sigma_b\lVert b(\beta;D^\prime)-b(\beta;D)\rVert_2\sqrt{2\log\frac{4}{\delta}}\right\}.
\end{equation*}
Then $\mathbb{P}(E_4^c)\leq\delta/2$ by Lemma 17 in \cite{kifer2012private}, which states $\mathbb{P}(|Z|>t)\leq \mathrm{exp}(-t^2/2)$ if $t\geq 1$.

Also,
\begin{equation*}
    \lVert b(\beta;D^\prime)-b(\beta;D)\rVert_2\leq 2c_1\Delta_3,
\end{equation*}
on $E_3$. Thus,
\begin{align*}
\frac{p_b(b(\beta;D))}{p_b(b(\beta;D^\prime))}\leq&\mathrm{exp}\left(\frac{4c_1^2\Delta_3^2+4c_1\sigma_b\Delta_3\sqrt{2\log\frac{4}{\delta}}}{2\sigma_b^2}\right)\\
\leq&\mathrm{exp}\left(\frac{\epsilon}{2}\right)
\end{align*}
on $E_3\cap E_4$.

Next, we bound the second term. Note that
\begin{equation*}
    \nabla_{\beta}b(\beta;D^\prime)=-\sum_{i=1}^nl_{\hat{y}\hat{y}}(y_i,\beta^\top h_M(x_i))h_M(x_i)h_M(x_i)^\top-n\lambda_0 I.
\end{equation*}
Thus, $\lVert \nabla_{\beta}b(\beta;D^\prime)\rVert_2\geq n\lambda_0$. We bound the second term using the following lemma:
\begin{lemma}[Lemma 10 in \cite{chaudhuri2011differentially}]
If $A$ is full rank, and if $B$ has rank at most 2, then,
\begin{eqnarray*}
    \frac{\mathrm{det}(A+B)-\mathrm{det}(A)}{\mathrm{det}(A)}=\lambda_1(A^{-1}B)+\lambda_2(A^{-1}B)+\lambda_1(A^{-1}B)\lambda_2(A^{-1}B)
\end{eqnarray*}
where $\lambda_j(\cdot)$ is the $j$-th eigenvalue of the given matrix.
\end{lemma}
Put $A=\nabla_{\beta}b(\beta;D^\prime)$ and $B=\nabla_{\beta}b(\beta;D)-\nabla_{\beta}b(\beta;D^\prime)$. Then
\begin{eqnarray*}
    |\lambda_j(A^{-1}B)|\leq \lVert A^{-1}\rVert|\lambda_j(B)|\leq\frac{|\lambda_j(B)|}{n\lambda_0},
\end{eqnarray*}
for $j=1,2$. Also,
\begin{align*}
|\lambda_1(B)|+|\lambda_2(B)|=&\lVert B\rVert_*\\
    \leq&\lVert l_{\hat{y}}(y,\beta^\top h_M(x))h_M(x)h_M(x)^\top-l_{\hat{y}}(y^\prime,\beta^\top h_M(x^\prime))h_M(x^\prime)h_M(x^\prime)^\top \rVert_*\\
    \leq&2c_2\Delta^2,
\end{align*}
where $\lVert\cdot\rVert_{*}$ denotes the trace norm of the given matrix.
\begin{align*}
\frac{|\mathrm{det}\nabla_{\beta}b(\beta;D)|}{|\mathrm{det}\nabla_{\beta}b(\beta;D^\prime)|}=&\left|\frac{\mathrm{det}\nabla_{\beta}b(\beta;D)}{\mathrm{det}\nabla_{\beta}b(\beta;D^\prime)}\right|\\
    =&1+\frac{2c_2\Delta_3^2}{n\lambda_0}+\frac{c_2^2\Delta_3^4}{n^2\lambda_0^2}\\
    \leq&\mathrm{exp}\left(\frac{\epsilon}{2}\right).
\end{align*}
Therefore, $\frac{p_D(\beta)}{p_{D^\prime}(\beta)}\leq e^\epsilon$ on $E_3\cap E_4$. Then for any measurable set $A\subset\mathbb{R}^M$, we have
\begin{align*}
\mathbb{P}\left(\widetilde{\beta}_D\in A\right)\leq&\mathbb{P}\left(\widetilde{\beta}_D\in A,E_3\cap E_4\right)+\mathbb{P}\left(E_3^c\right)+\mathbb{P}\left(E_4^c\right)\\
    \leq&e^\epsilon\mathbb{P}\left(\widetilde{\beta}_{D^\prime}\in A,E_3\cap E_4\right)+\mathbb{P}\left(E_3^c\right)+\mathbb{P}\left(E_4^c\right)\\
    \leq&e^\epsilon\mathbb{P}\left(\widetilde{\beta}_{D^\prime}\in A,E_3\cap E_4\right)+\delta.
\end{align*}

\subsection{Proof of Lemma~\ref{lem:general_bound}}
We apply the Borel–TIS inequality to obtain a high-probability bound on $\sup_{x\in\mathcal{X}}h(x,\omega)$.
\begin{proposition}[Borel-TIS inequality in \citet{adler2007gaussian}]\label{prop:borel_tis}
    \begin{eqnarray*}
    \mathbb{P}_\omega \left(\sup_{x\in\mathcal{X}}h(x,\omega)-\mathbb{E}\left[\sup_{x\in\mathcal{X}}h(x,\omega)\right]>\sqrt{2\sigma_T^2\log\frac{1}{t}}\right)\leq t
\end{eqnarray*}
where $\sigma_T^2:=\sup_{x\in\mathcal{X}}\mathbb{E}_\omega \left[|h(x,\omega)|^2\right]$.
\end{proposition}
Therefore, the following holds:
\begin{eqnarray*}
    \mathbb{P}_{\omega_{1:M}} \left(\sup_{1\leq j\leq M}\sup_{x\in\mathcal{X}}|h(x,\omega_j)|>H\right)\leq t.
\end{eqnarray*}

\subsection{Proof of Theorem~\ref{thm:dprp_reg_rate_simple}}
By setting $r=1/2$ in Theorem~\ref{thm:dprp_reg_rate}, we obtain Theorem~\ref{thm:dprp_reg_rate_simple}.
\subsection{Proof of Theorem~\ref{thm:dprp_reg_rate}}
We prove the theorem under a weaker assumption than Assumption~\ref{assump:cap}.
\begin{assumption}\label{assump:cap_real}
   For $\lambda>0$, define the \emph{effective dimension} as $N(\lambda)\coloneqq \mathrm{tr}(L_\lambda^{-1}L)$, where $L_\lambda\coloneqq L+\lambda I$. We assume that there exist constants $Q>0$ and $\gamma\in [0,1]$ such that $N(\lambda)\leq Q\lambda^{-\gamma}$ for all $\lambda>0$.
\end{assumption}
Assumption~\ref{assump:cap_real} is weaker than Assumption~\ref{assump:cap} since there exists a constant $b>0$ such that $\mu_l\leq bl^{-1/\gamma}$ under Assumption~\ref{assump:cap_real}. See Section~\ref{sec:cap_cap_real} for the proof.

Denote $\mathcal{H}_M\coloneqq \mathrm{span}\{h(\cdot,\omega_1),\ldots,h(\cdot,\omega_M)\}$. The output of Algorithm~\ref{alg:dprp_reg} resides in $\mathcal{H}_M$. For any prediction function $f\in \mathcal{H}_M$, the risk $\mathcal{E}(f)$ can be decomposed as follows:
\begin{equation*}
    \mathcal{E}(f)-\min_{f\in\mathcal{H}_k}\mathcal{E}(f)=\lVert f-f_{\mathcal{H}_k}\rVert^2.
\end{equation*}
Let $\hat{f}_{M,\lambda}\coloneqq \hat{\beta}_{M,\lambda}^\top h_M$ denote the (non-private) estimator obtained from kernel ridge regression using random projection.  Let $\widetilde{f}$ denote the output of the private Algorithm~\ref{alg:dprp_reg}. Applying the triangle inequality yields the following decomposition of the excess risk:
\begin{equation*}
    \mathcal{E}(\widetilde{f})-\min_{f\in\mathcal{H}_k}\mathcal{E}(f)\leq (\lVert f_{\mathcal{H}_k}-\hat{f}_{M,\lambda}\rVert+\lVert \hat{f}_{M,\lambda}-\widetilde{f}\rVert)^2.
\end{equation*}
The first term captures the generalization error and the approximation error due to finite samples and random projection, respectively. The second term accounts for the perturbation error introduced to preserve differential privacy.

The following lemmas provide high-probability bounds on the two components of the excess risk.
\begin{lemma}\label{lem:nonpriv_rp}
For $t\in(0,e^{-1}]$, if  $M\gtrsim \left(\frac{\kappa}{\lVert L\rVert}+N(\lambda)\right)\log\frac{2}{t},\lambda\in[0,\lVert L\rVert]$, and $n\gtrsim \frac{H^2}{\lambda}\log\frac{N(\lambda)}{t}$, then
    \begin{equation*}
\lVert f_{\mathcal{H}_k}-\hat{f}_{M,\lambda}\rVert\lesssim \frac{H^2\log\frac{1}{t}}{n\sqrt{\lambda}}+\sqrt{\frac{H^2N(\lambda)\log\frac{1}{t}}{n}}+\sqrt{\frac{\lambda N(\lambda)}{M}}\log\frac{1}{t}+ \lambda^r
\end{equation*}
holds with probability at least $1-t$.
\end{lemma}
\begin{lemma}\label{lem:priv_ridge} For $t\in (0,e^{-1}]$, let $T=\kappa R+2B\log\frac{n}{t}+\sqrt{2\log\frac{n}{t}}\sigma$. Then the following holds
    \begin{eqnarray*}
        \lVert \hat{f}_{M,\lambda}- \widetilde{f}\rVert
        \lesssim\frac{\sqrt{M}}{n\sqrt{\lambda}\epsilon}\left(\kappa^{2r-1}R+T+\frac{BH\log\frac{1}{t}}{n\lambda}+\sqrt{\sigma^2\frac{N(\lambda)\log\frac{1}{t}}{n\lambda}}+\frac{\sqrt{M}}{n\lambda\epsilon}\right)
    \end{eqnarray*}
    with probability at least $1-t$ if $\frac{n^2\lambda^2\epsilon^2}{\Delta_1^2\log^2\frac{1}{\delta}}\gtrsim M\gtrsim \max\left\{\left(\frac{\kappa}{\lVert L\rVert}+N(\lambda)\right)\log\frac{2}{t},\log\frac{2}{t}\right\}$, $n\gtrsim \frac{H^2}{\lambda}\log\frac{N(\lambda)}{t}$, and $\lambda\in[0,\lVert L\rVert]$.
\end{lemma}
Combining Lemmas~\ref{lem:nonpriv_rp} and~\ref{lem:priv_ridge}, we obtain a high-probability bound on the excess risk of Algorithm~\ref{alg:dprp_reg}. By appropriately choosing the hyperparameters $M$ and $\lambda$, we recover the rates stated in Theorem~\ref{thm:dprp_reg_rate}. Specifically, setting $\lambda=O\left(n^{-\frac{1}{2r+\gamma}}\vee (n\epsilon)^{-\frac{2}{4r+\gamma}}\right)$ and $M=O\left(n^{\frac{2r+\gamma-1}{2r+\gamma}}\wedge (n\epsilon)^{\frac{4r+2\gamma-2}{4r+\gamma}}\right)$ yields the rate given in Theorem~\ref{thm:dprp_reg_rate}. 

\subsection{Proof of Lemma~\ref{lem:nonpriv_rp}}
Define the following operators:
\begin{align*}
&L_M:L^2(P_X)\rightarrow L^2(P_X)\textrm{ defined as }L_Mf:=\frac{1}{M}\sum_{i=1}^Mh(\cdot,\omega_i)\langle h(\cdot,\omega_i),f\rangle_{L^2(P_X)},\\
&S_M:\mathbb{R}^M\rightarrow L^2(P_X)\textrm{ defined as }(S_Mu)(\cdot)=h_M(\cdot)^\top u,\\
&S_M^\top :L^2(P_X)\rightarrow \mathbb{R}^M\textrm{ defined as }S_M^\top f=\langle h_M,f\rangle_{L^2(P_X)}=\mathbb{E}\left[h_M(X)f(X)\right].
\end{align*}
Then $L_M=S_MS_M^\top $ holds.

The following are the covariance and sample covariance matrices arising in linear ridge regression of processed data $\{(h_M(x_i),y_i)\}_{i=1}^n$:
\begin{align*}
&C_M:\mathbb{R}^M\rightarrow \mathbb{R}^M\textrm{ defined as }C_M=\mathbb{E}_X\left[h_M(X)h_M(X)^\top\right],\\
&\hat{C}_M:\mathbb{R}^M\rightarrow \mathbb{R}^M\textrm{ defined as }\hat{C}_M:=\frac{1}{n}\sum_{i=1}^n h_M(x_i)h_M(x_i)^\top,\\
&\hat{S}_M :\mathbb{R}^M\rightarrow\mathbb{R}^n \textrm{ defined as }\hat{S}_Mu:=\left(h_M(x_1)^\top u,\ldots,h_M(x_n)^\top u\right),\\
&\hat{S}_M^{\top} :\mathbb{R}^n\rightarrow \mathbb{R}^M\textrm{ defined as }\hat{S}_M^{\top} \hat{y}=\frac{1}{n}\sum_{i=1}^nh_M(x_i)y_i,
\end{align*}
where $\hat{y}=(y_1,\ldots,y_n)$. Then $\hat{C}_M=\hat{S}_M^{\top} \hat{S}_M$. Similarly, $C_M=S_M^\top S_M$. Note that the defined operators are bounded with probability one since Gaussian processes lie in $L^2(P_X)$ with probability one.

Now we define our kernel regression estimator using $h_M$, which corresponds to linear ridge regression for $\{(h_M(x_i),y_i)\}_{i=1}^n$:
\begin{equation*}
    \hat{f}_{\lambda}:=S_M\hat{C}_{M,\lambda}^{-1}\hat{S}_M^{\top} \hat{y}
\end{equation*}
Let $P:L^2(P_X)\rightarrow L^2(P_X)$ denote the projection onto the range of $L$. For $x\in\mathcal{X}$, define $f_0(x)\coloneqq \mathbb{E}\left[Y|X=x\right]$. Then $f_{\mathcal{H}_k}=Pf_0$. We first note that, for $f\in\mathcal{H}_M$, the excess risk can be expressed as the square of the $L^2(P_X)$ distance, that is, $
    \mathcal{E}(f)-\min_{f\in\mathcal{H}_k}\mathcal{E}(f)=\lVert f-f_{\mathcal{H}_k}\rVert^2$:
\begin{align*}
    &\mathbb{E}[(Y-f(X))^2]\\
    =&\lVert f-f_0\rVert^2+\mathbb{E}\left[(Y-f_0(X))^2\right]\\
    =&\lVert f-f_{\mathcal{H}_k}\rVert^2+2\langle f-f_{\mathcal{H}_k},f_{\mathcal{H}_k}-f_0\rangle+\lVert f_{\mathcal{H}_k}-f_0\rVert^2+\mathbb{E}\left[(Y-f_0(X))^2\right]\\
    =&\lVert f-f_{\mathcal{H}_k}\rVert^2+2\langle f-f_{\mathcal{H}_k},f_{\mathcal{H}_k}-f_0\rangle+\mathbb{E}\left[(Y-f_{\mathcal{H}_k}(X))^2\right]\\
    =&\lVert f-f_{\mathcal{H}_k}\rVert^2+2\langle f-f_{\mathcal{H}_k},(I-P)f_0\rangle+\mathbb{E}\left[(Y-f_{\mathcal{H}_k}(X))^2\right]\\
    =&\lVert f-f_{\mathcal{H}_k}\rVert^2+2\langle (I-P)f,f_0\rangle-2\langle (I-P)f_{\mathcal{H}_k},f_0\rangle+\mathbb{E}\left[(Y-f_{\mathcal{H}_k}(X))^2\right]\\
    =&\lVert f-f_{\mathcal{H}_k}\rVert^2+2\langle (I-P)f,f_0\rangle+\mathbb{E}\left[(Y-f_{\mathcal{H}_k}(X))^2\right]\\
    =&\lVert f-f_{\mathcal{H}_k}\rVert^2+\mathbb{E}\left[(Y-f_{\mathcal{H}_k}(X))^2\right],
\end{align*}
where the sixth equality follows from the fact that $(I-P)f_{\mathcal{H}_k}\equiv0$, since $Pf_{\mathcal{H}_k}=f_{\mathcal{H}_k}$; and the seventh equality follows from the fact that $(I-P)f\equiv0$ almost surely, as stated in Lemma~\ref{lem:secondterm}.
\begin{align*}
\sqrt{\mathcal{E}(\hat{f}_{M,\lambda})-\min_{f\in\mathcal{H}_k}\mathcal{E}(f)} \leq &\left\lVert\hat{f}_{M,\lambda}-S_M\hat{C}_{M,\lambda}^{-1}S_M^\top f_0\right\rVert\\
    &+\left\lVert S_M\hat{C}_{M,\lambda}^{-1}S_M^\top (I-P)f_0\right\rVert
    \\
    &+\left\lVert S_M \hat{C}_{M,\lambda}^{-1}S_M^\top Pf_0-L_ML_{M,\lambda}^{-1}Pf_0\right\rVert\\
    &+\left\lVert L_ML_{M,\lambda}^{-1}Pf_0-LL_{\lambda}^{-1}Pf_0\right\rVert\\
    &+\left\lVert LL_\lambda^{-1}Pf_0-Pf_0\right\rVert.
\end{align*}
The first term represents the error due to response noise; the second term arises from restricting the hypothesis space to $\mathcal{H}_k$; the third term corresponds to the covariance approximation error from the randomly projected data; the fourth term reflects information loss due to the dimension reduction via random projection; and the fifth term captures the bias introduced by the ridge penalty. Using the bounds in \citet{randomfeature}, we can bound these terms as follows:
\begin{lemma}\label{lem:firstterm}
    \begin{equation*}
        \left\lVert\hat{f}_{M,\lambda}-S_M\hat{C}_{M,\lambda}^{-1}S_M^\top f_0\right\rVert\leq\left\lVert S_M\hat{C}_{M,\lambda}^{-1}C_{M,\lambda}^{\frac{1}{2}}\right\rVert\left\lVert C_{M,\lambda}^{-\frac{1}{2}}(\hat{S}_M^{\top} \hat{y}-S_M^\top f_0)\right\rVert.
    \end{equation*}
\end{lemma}
\begin{lemma}\label{lem:secondterm}
    $\mathbb{P}_\omega \left(\left\lVert S_M\hat{C}_{M,\lambda}^{-1}S_M^\top (I-P)f_0\right\rVert=0\right)=1$. Actually, $\left\lVert S_M^\top (I-P)\right\rVert=0$ almost surely.
\end{lemma}
\begin{lemma}\label{lem:thirdterm}
\begin{align*}
&\left\lVert S_M\hat{C}_{M,\lambda}^{-1}S_M^\top Pf_0-L_ML_{M,\lambda}Pf_0\right\rVert\\
\leq& R\kappa^{2r-1}\left\lVert L_{M,\lambda}^{-\frac{1}{2}}L_\lambda^{\frac{1}{2}}\right\rVert\left\lVert S_M\hat{C}_{M,\lambda}^{-1}C_{M,\lambda}^{\frac{1}{2}}\right\rVert\left\lVert C_{M,\lambda}^{-\frac{1}{2}}(C_M-\hat{C}_M)\right\rVert
\end{align*}
\end{lemma}
\begin{lemma}\label{lem:fourthterm}
\begin{align*}
    &\left\lVert L_ML_{M,\lambda}^{-1}Pf_0-LL_{\lambda}^{-1}Pf_0\right\rVert\\
    \leq& R\sqrt{\lambda}\left\lVert L_{M,\lambda}^{-\frac{1}{2}}L_\lambda^{\frac{1}{2}}\right\rVert\left\lVert L_\lambda^{-\frac{1}{2}}(L-L_M)\right\rVert^{2r-1}\left\lVert L_\lambda^{-\frac{1}{2}}(L-L_{M})L_\lambda^{-\frac{1}{2}}\right\rVert^{2-2r}
\end{align*}
\end{lemma}
\begin{lemma}\label{lem:fifthterm}
    \begin{equation*}
        \left\lVert LL_{\lambda}^{-1}Pf_0-Pf_0\right\rVert\leq R\lambda^r.
    \end{equation*}
\end{lemma}
One can observe that the upper bounds consist of repeated terms, which we divide into two groups. The first group includes $\left\lVert S_M\hat{C}_{M,\lambda}^{-1}C_{M,\lambda}^{\frac{1}{2}}\right\rVert$ and $\left\lVert L_{M,\lambda}^{-\frac{1}{2}}L_\lambda^{\frac{1}{2}}\right\rVert$. The second group includes $\left\lVert C_{M,\lambda}^{-\frac{1}{2}}(\hat{S}_M^{\top} \hat{y}-S_M^\top f_0)\right\rVert$, $\left\lVert C_{M,\lambda}^{-\frac{1}{2}}(C_M-\hat{C}_M)\right\rVert$, $\Big\lVert L_\lambda^{-\frac{1}{2}}(L-L_M)\Big\rVert$, and $\left\lVert L_\lambda^{-\frac{1}{2}}(L-L_M)L_\lambda^{-\frac{1}{2}}\right\rVert$. We bound each of these terms to derive the excess risk bound for the non-private estimator with random projection.

For given $t\in(0,1]$, let $E_{g}:=\Big\{\omega_{1:M}:\sup_{1\leq j\leq M}\sup_{x\in\mathcal{X}}|h(x,\omega_j)| \leq H\Big\}$ where $H$ is defined in Lemma~\ref{lem:general_bound}.

We bound these quantities as follows:
\begin{lemma}\label{lem:half_operator}
Let $E_{half}$ be defined as
    \begin{equation*}
        E_{half}:=\Bigg\{\omega_{1:M}:\left\lVert L_\lambda^{-\frac{1}{2}}(L-L_M)\right\rVert_F\leq\frac{2\kappa\sqrt{32N(\lambda)}\log\frac{2}{t}}{M}+\sqrt{\frac{64\kappa^2N(\lambda)\log\frac{2}{t}}{M}}\Bigg\}.
    \end{equation*}
    Then for $t\in (0,1]$, the following inequality holds:
    \begin{eqnarray*}
        \mathbb{P}_\omega \left(E_{half}\right)\geq 1-t.
    \end{eqnarray*}
    \end{lemma}
\begin{lemma}\label{lem:full_operator}
There exists a universal constant $C>0$ such that, for $t\in (0,e^{-1}]$, the following inequality holds:
\begin{equation*}\mathbb{P}_\omega \left(
    E_{full}\right)\geq1-t
\end{equation*}
where
\begin{equation*}
    E_{full}:=\Bigg\{\omega_{1:M}:
    \left\lVert L_\lambda^{-\frac{1}{2}}(L-L_M)L_\lambda^{-\frac{1}{2}}\right\rVert\leq C\max\left(\sqrt{\frac{N(\lambda)\log\frac{1}{t}}{M}},\frac{N(\lambda)\log\frac{1}{t}}{M}\right)\Bigg\}.
\end{equation*}
\end{lemma}
\begin{lemma}\label{lem:constants}
For $\omega_1,\ldots,\omega_M\in E_{full}$, we have
\begin{equation*}
    \left\lVert L_{M,\lambda}^{-\frac{1}{2}}L_{\lambda}^{\frac{1}{2}}\right\rVert\leq\sqrt{2}.
\end{equation*}
Furthermore, let $E_{cov}$ be defined as
\begin{equation*}
    E_{cov}:=\left\{\omega_{1:M}:\lVert L-L_M\rVert\leq\frac{\lVert L\rVert}{2}\right\}.
\end{equation*}
Then, for $t\in (0,e^{-1}]$ the following holds:
\begin{equation*}
    \mathbb{P}_{\omega_{1:M}}(E_{cov})\geq 1-t.
\end{equation*}
Now, let $E_{cov-mat}$ be defined as
\begin{equation*}
    E_{cov-mat}:=\left\{X_{1:n}:\left\lVert C_{M,\lambda}^{-\frac{1}{2}}\left(\hat{C}_{M}-C_{M}\right) C_{M,\lambda}^{-\frac{1}{2}}\right\rVert\leq\frac{1}{4}\right\}.
\end{equation*}
For $\omega_1,\ldots,\omega_M\in E_g\cap E_{cov}$, the following holds:
\begin{eqnarray*}
\mathbb{P}_{X_{1:n}|\omega_{1:M}}\left(E_{cov-mat}\right)\geq 1-t,
\end{eqnarray*}
if $M\gtrsim \left(\frac{\kappa}{\lVert L\rVert}+N(\lambda)\right)\log\frac{1}{t}$, $\lambda\in[0,\lVert L\rVert]$, and $n\gtrsim H^2\lambda^{-1}\log\frac{2\mathrm{tr}(C_{M,\lambda}^{-1}C_M)}{t}$. Moreover, the bound $\left\lVert C_{M,\lambda}^{-\frac{1}{2}}\left(\hat{C}_{M}-C_{M}\right) C_{M,\lambda}^{-\frac{1}{2}}\right\rVert\leq \frac{1}{4}$ implies that $\left\lVert S_M\hat{C}_{M,\lambda}^{-1}C_{M,\lambda}^{\frac{1}{2}}\right\rVert\leq\frac{4}{3}$.
\end{lemma}
The following lemmas bound the quantities in the second group:
\begin{lemma}\label{lem:noise_err} Let $E_{noise}$ be defined as
\begin{align*}
&E_{noise}\\
:=&\Bigg\{(X_{1:n},Y_{1:n})
:\left\lVert C_{M,\lambda}^{-\frac{1}{2}}\left(\hat{S}_M^\top \hat{y}-S_M^\top f_0\right)\right\rVert_2\leq\frac{2BH\log\frac{2}{t}}{n\sqrt{\lambda}}+\sqrt{\frac{8\sigma^2 \mathrm{tr}\left(C_{M,\lambda}^{-1}C_M\right)\log\frac{2}{t}}{n}}\Bigg\}.
\end{align*}
Then for $\omega_1,\ldots,\omega_M\in E_g$, the following holds:
\begin{eqnarray*}
    \mathbb{P}_{X_{1:n},Y_{1:n}|\omega_{1:M}}\left(E_{noise}\right)\geq 1-t.
\end{eqnarray*}
\end{lemma}
\begin{lemma}\label{lem:rf_err} Let $E_{cov-mat-half}$ be defined as
\begin{align*}
    &E_{cov-mat-half}\\
    :=&\left\{X_{1:n}:\left\lVert C_{M,\lambda}^{-\frac{1}{2}}\left(C_M-\hat{C}_M\right)\right\rVert_F\leq\frac{4H\log\frac{2}{t}}{n\sqrt{\lambda}}+\sqrt{\frac{8H^2\mathrm{tr}\left(C_{M,\lambda}^{-1}C_M\right)\log\frac{2}{t}}{n}}\right\}
\end{align*}
Then for $\omega_1,\ldots,\omega_M\in E_g$, the following holds:
\begin{equation*}
    \mathbb{P}_{X_{1:n}|\omega_{1:M}}\left(E_{cov-mat-half}\right)\geq 1-t.
\end{equation*}
\end{lemma}
\begin{lemma}\label{lem:rf_dim}
Let $E_{rf-dim}$ be defined as
\begin{align*}
    E_{rf-dim}:=\{w:\mathrm{tr}(C_{M,\lambda}^{-1}C_M)\leq 2N(\lambda)\}
\end{align*}
Then,
\begin{equation*}
    \mathbb{P}_\omega \left(E_{rf-dim}\right)\geq 1-3t
\end{equation*}
    if $M\gtrsim  N(\lambda)\log^2\frac{2}{t}$. 
\end{lemma}
By aggregating the presented lemmas, it follows that for $\omega_1,\ldots,\omega_M\in E_g\cap E_{full}\cap E_{half}\cap E_{cov}\cap E_{rf-dim}$:
\begin{align*}
\left\lVert\hat{f}_{M,\lambda}-S_M\hat{C}_{M,\lambda}^{-1}S_M^\top f_0\right\rVert \leq &\frac{8BH\log\frac{2}{t}}{3n\sqrt{\lambda}}+\sqrt{\frac{256\sigma^2N(\lambda)\log\frac{2}{t}}{9n}}\\
    \left\lVert S_M\hat{C}_{M,\lambda}^{-1}S_M^\top (I-P)f_0\right\rVert=&0\\
    \left\lVert S_M \hat{C}_{M,\lambda}S_M^\top Pf_0-L_ML_{M,\lambda}^{-1}Pf_0\right\rVert \leq &\frac{16\sqrt{2}}{3}\kappa^{2r-1}R\left(\frac{H\log\frac{2}{t}}{n\sqrt{\lambda}}+\sqrt{\frac{H^2N(\lambda)\log\frac{2}{t}}{n}}\right)\\
    \left\lVert L_ML_{M,\lambda}^{-1}Pf_0-LL_{\lambda}^{-1}Pf_0\right\rVert \leq &\sqrt{2}C^{2-2r}\left(\frac{2\sqrt{32\log\frac{2}{t}}\kappa}{\sqrt{M}}+\sqrt{64\kappa^2}\right)^{2r-1}\\
    &\cdot\left(1+\sqrt{\frac{N(\lambda)\log\frac{2}{t}}{M}}\right)^{2-2r}\\
    &\cdot R\sqrt{\frac{\lambda N(\lambda)\log\frac{2}{t}}{M}}
    \\\left\lVert LL_\lambda^{-1}Pf_0-Pf_0\right\rVert \leq & \lambda^rR
\end{align*}
with probability at least $1-3t$ if $M\gtrsim \left(\frac{\kappa}{\lVert L\rVert}+N(\lambda)\right)\log\frac{2}{t},\lambda\in[0,\lVert L\rVert]$, and $n\gtrsim H^2\lambda^{-1}\log\frac{4N(\lambda)}{t}$. Thus, we have the following theorem.
\begin{theorem}\label{thm:nonpriv_rp}
For given $t\in (0,1]$, the excess risk of $\hat{f}_{M,\lambda}$ is bounded by,
\begin{align*}
\sqrt{\mathcal{E}(\hat{f}_{M,\lambda})-\min_{f\in\mathcal{H}_k}\mathcal{E}(f)}\lesssim& \frac{\left(B+\kappa^{2r-1}R\right)H\log\frac{2}{t}}{n\sqrt{\lambda}}+(\sigma+H)\sqrt{\frac{N(\lambda)\log\frac{2}{t}}{n}}\\
&+(\kappa^{2r-1}+\kappa^{r-1/2}\lVert L\rVert^{r-1/2})R\sqrt{\frac{N(\lambda)\log\frac{2}{t}}{M}}+R\lambda^r
\end{align*}
with probability at least $1-10t$ if $M\gtrsim \left(\frac{\kappa}{\lVert L\rVert}+N(\lambda)\right)\log\frac{2}{t},\lambda\in[0,\lVert L\rVert]$, and $n\gtrsim H^2\lambda^{-1}\log\frac{4N(\lambda)}{t}$.
\end{theorem}
\subsubsection{Proof of Lemma~\ref{lem:half_operator}}
We show the following stronger statement:
\begin{align*}
\left\lVert L_{\lambda}^{-\frac{1}{2}}(L-L_M)\right\rVert_F\leq \frac{2\kappa\sqrt{32N(\lambda)}\log\frac{2}{t}}{M}+\sqrt{\frac{64\kappa^2N(\lambda)\log\frac{2}{t}}{M}},
\end{align*}
holds with probability at least $1-t$.

We use the Bernstein's inequality for random vectors $\{L_{\lambda}^{-\frac{1}{2}}h(\cdot,\omega_j)\otimes h(\cdot,\omega_j)\}_{j=1}^M$ in $L^2(P_X)\times L^2(P_X)$.
\begin{proposition}[Bernstein]\label{prop:bernstein_vec} Let $z_1,\ldots,z_n$ be a sequence of i.i.d. random vectors on a separable Hilbert space $\mathcal{H}$. Assume $\mu=\mathbb{E}\left[z_i\right]$ exists and let $\sigma,M\geq 0$ such that
\begin{equation*}
    \mathbb{E}\left[\left\lVert z_i-\mu\right\rVert_{\mathcal{H}}^p\right]\leq\frac{1}{2}p!\sigma^2M^{p-2}
\end{equation*}
if $p\geq 2$. Then for any $t\in(0,1)$ the following holds:
\begin{equation*}
    \mathbb{P}\left(\left\lVert\frac{1}{n}\sum_{i=1}^nz_i-\mu\right\rVert_{\mathcal{H}}\leq\frac{2M\log\frac{2}{t}}{n}+\sqrt{\frac{2\sigma^2\log\frac{2}{t}}{n}}\right)\geq 1-t.
\end{equation*}
\end{proposition}
Applying Proposition~\ref{prop:gaussian_vec}, one can show that the following inequalities hold for $x\geq0$:
\begin{align*}
\mathbb{P}_{\omega}\left(\left\lVert h(\cdot,\omega)\right\rVert^2\geq \kappa^2(1+2\sqrt{x}+2x)\right)\leq& e^{-x},\\
        \mathbb{P}_{\omega}\left(\left\lVert L_\lambda^{-\frac{1}{2}}h(\cdot,\omega)\right\rVert^2\geq N(\lambda)(1+2\sqrt{x}+2x)\right)\leq& e^{-x}.
\end{align*}
Detailed proofs of the inequalities are followed by the proof of Lemma~\ref{lem:gp_moments}. Combining the inequalities, we obtain
\begin{align*}
    \mathbb{P}\left(\left\lVert L_\lambda^{-\frac{1}{2}}h\right\rVert\left\lVert h\right\rVert\geq \kappa\sqrt{N(\lambda)}(1+2\sqrt{x}+2x)\right)\leq 2e^{-x}.
\end{align*}
    Then the $p$th moment of $\left\lVert L_\lambda^{-\frac{1}{2}}h\right\rVert\left\lVert h\right\rVert$ is bounded as
    \begin{align*}
&\mathbb{E}\left[\left\lVert L_\lambda^{-\frac{1}{2}}h\right\rVert^p\left\lVert h\right\rVert^p\right]\\
=&\int_0^\infty
        \mathbb{P}\left(\left\lVert L_\lambda^{-\frac{1}{2}}h\right\rVert^p\left\lVert h\right\rVert^p\geq t\right)dt\\
        \leq&\left(\kappa\sqrt{N(\lambda)}\right)^p+\int_{\left(\kappa\sqrt{N(\lambda)}\right)^p}^\infty
        \mathbb{P}\left(\left\lVert L_\lambda^{-\frac{1}{2}}h\right\rVert^p\left\lVert h\right\rVert^p\geq t\right)dt\\
        =&\left(\kappa\sqrt{N(\lambda)}\right)^p\Big(1+\int_0^\infty \sqrt{2}p(1+\sqrt{2x})^{2p-1}x^{-\frac{1}{2}}\\
        &\cdot\mathbb{P}\left(\left\lVert L_\lambda^{-\frac{1}{2}}h\right\rVert^p\left\lVert h\right\rVert^p\geq \left(\kappa\sqrt{N(\lambda)}\right)^p(1+\sqrt{2x})^{2p}\right)dx\Big)\\
        \leq&\left(\kappa\sqrt{N(\lambda)}\right)^p\left(1+\int_0^\infty\sqrt{2}p(1+\sqrt{2x})^{2p-1}x^{-\frac{1}{2}}e^{-x}dx\right)\\
        \leq&\left(\kappa\sqrt{N(\lambda)}\right)^p\left(1+2^{2p}p+\int_{\frac{1}{2}}^\infty\sqrt{2}p(1+\sqrt{2x})^{2p-1}x^{-\frac{1}{2}}e^{-x}dx\right)\\
        \leq&\left(\kappa\sqrt{N(\lambda)}\right)^p\left(1+2^{2p}p+\int_{0}^\infty\sqrt{2}p(2\sqrt{2x})^{2p-1}x^{-\frac{1}{2}}e^{-x}dx\right)\\
        =&\left(\kappa\sqrt{N(\lambda)}\right)^p\left(1+2^{2p}p+2^{3p-1}\Gamma(p+1)\right)\\
        \leq&\left(\kappa\sqrt{N(\lambda)}\right)^p2^{3p}\Gamma(p+1).
\end{align*}

Then,

\begin{align*}
&\mathbb{E}\left[\left\lVert L_\lambda^{-\frac{1}{2}}h(\cdot,\omega)\otimes h(\cdot,\omega)-\mathbb{E}\left[L_\lambda^{-\frac{1}{2}}h(\cdot,\omega)\otimes h(\cdot,\omega)\right]\right\rVert^p\right]\\\leq&2^{p-1}\mathbb{E}\left[\left\lVert L_\lambda^{-\frac{1}{2}}h(\cdot,\omega)\otimes h(\cdot,\omega)\right\rVert^p+\left\lVert\mathbb{E}\left[L_\lambda^{-\frac{1}{2}}h(\cdot,\omega)\otimes h(\cdot,\omega)\right]\right\rVert^p\right]\\
    \leq&2^{p}\mathbb{E}\left[\left\lVert L_\lambda^{-\frac{1}{2}}h(\cdot,\omega)\otimes h(\cdot,\omega)\right\rVert^p\right]\\
    \leq&\frac{1}{2}p!(\kappa\sqrt{32N(\lambda)})^{p}.
\end{align*}
Applying Proposition~\ref{prop:bernstein_vec}, we obtain the desired result.

\subsubsection{Proof of Lemma~\ref{lem:full_operator}}
\begin{proposition}[Corollary 2 in \cite{gaussian_covariance}]\label{prop:gaussian_cov}
    For i.i.d. centered Gaussian random variables $x_1,\ldots,x_M$ in the Hilbert space with covariance operator $\Sigma$ and $t\in (0,e^{-1}]$, there exists a universal constant $C>0$ such that
\begin{equation*}
\left\lVert\frac{1}{M}\sum_{i=1}^M x_ix_i^\top-\Sigma\right\rVert_2\leq C\max\left(\sqrt{\frac{\left\lVert\Sigma\right\rVert_2\mathrm{tr}(\Sigma)}{M}},\frac{\mathrm{tr}(\Sigma)}{M},\sqrt{\frac{\left\lVert\Sigma\right\rVert_2^2\log\frac{1}{t}}{M}},\frac{\left\lVert\Sigma\right\rVert_2\log\frac{1}{t}}{M}\right),
\end{equation*}
holds with probability at least $1-t$.
\end{proposition}
Setting $x_i=L_\lambda^{-\frac{1}{2}}h(\cdot,w_i)$, the covariance operator $\Sigma$ of $x_i$ is $\Sigma=L_{\lambda}^{-\frac{1}{2}}LL_\lambda^{-\frac{1}{2}}=L_\lambda^{-1}L$. Thus,
\begin{equation*}
    \mathrm{tr}(\Sigma)=N(\lambda),\left\lVert\Sigma\right\rVert_2=\frac{\left\lVert L\right\rVert}{\left\lVert L\right\rVert+\lambda}.
\end{equation*}
Therefore,
\begin{eqnarray*}
    \left\lVert L_\lambda^{-\frac{1}{2}}(L-L_M)L_\lambda^{-\frac{1}{2}}\right\rVert\leq C\max\left(\sqrt{\frac{\left\lVert L\right\rVert}{\left\lVert L\right\rVert+\lambda}\frac{ N(\lambda)}{M}},\frac{N(\lambda)}{M},\frac{\left\lVert L\right\rVert}{\left\lVert L\right\rVert+\lambda}\sqrt{\frac{ \log\frac{1}{t}}{M}},\frac{\left\lVert L\right\rVert}{\left\lVert L\right\rVert+\lambda}\frac{\log\frac{1}{t}}{M}\right)
\end{eqnarray*}
holds with probability at least $1-t$. Note that $1\geq N(\lambda)\geq\frac{\lVert L\rVert}{\lVert L\rVert+\lambda}$ and $\log\frac{1}{t}\geq 1$. Thus,
\begin{align*}
&\max\left(\sqrt{\frac{\left\lVert L\right\rVert}{\left\lVert L\right\rVert+\lambda}\frac{ N(\lambda)}{M}},\frac{N(\lambda)}{M},\frac{\left\lVert L\right\rVert}{\left\lVert L\right\rVert+\lambda}\sqrt{\frac{ \log\frac{1}{t}}{M}},\frac{\left\lVert L\right\rVert}{\left\lVert L\right\rVert+\lambda}\frac{\log\frac{1}{t}}{M}\right)\\
    \leq&\max\left(\sqrt{\frac{ N(\lambda)\log\frac{1}{t}}{M}},\frac{N(\lambda)\log\frac{1}{t}}{M},\frac{\left\lVert L\right\rVert}{\left\lVert L\right\rVert+\lambda}\sqrt{\frac{ \log\frac{1}{t}}{M}},\frac{\left\lVert L\right\rVert}{\left\lVert L\right\rVert+\lambda}\frac{\log\frac{1}{t}}{M}\right)\\
    \leq&\max\left(\sqrt{\frac{ N(\lambda)\log\frac{1}{t}}{M}},\frac{N(\lambda)\log\frac{1}{t}}{M},\sqrt{\frac{\left\lVert L\right\rVert}{\left\lVert L\right\rVert+\lambda}}\sqrt{\frac{N(\lambda)\log\frac{1}{t}}{M}},\frac{N(\lambda)\log\frac{1}{t}}{M}\right)\\
    \leq&\max\left(\sqrt{\frac{ N(\lambda)\log\frac{1}{t}}{M}},\frac{N(\lambda)\log\frac{1}{t}}{M}\right).
\end{align*}
Therefore,
\begin{eqnarray*}
    \left\lVert L_\lambda^{-\frac{1}{2}}(L-L_M)L_\lambda^{-\frac{1}{2}}\right\rVert\leq C\max\left(\sqrt{\frac{N(\lambda)\log\frac{1}{t}}{M}},\frac{N(\lambda)\log\frac{1}{t}}{M}\right),
\end{eqnarray*}
holds with probability at least $1-t$.

\subsubsection{Proof of Lemma~\ref{lem:constants}}
Note that
\begin{align*}
\left\lVert L_{M,\lambda}^{-\frac{1}{2}}L_\lambda^{\frac{1}{2}}\right\rVert^2=&\left\lVert L_\lambda^{\frac{1}{2}}L_{M,\lambda}^{-1}L_\lambda^{\frac{1}{2}}\right\rVert\\
    =&\left\lVert \left(L_\lambda^{-\frac{1}{2}}L_{M,\lambda}L_\lambda^{-\frac{1}{2}}\right)^{-1}\right\rVert\\
    =&\left\lVert \left(I+L_\lambda^{-\frac{1}{2}}\left(L_{M,\lambda}-L_\lambda\right)L_\lambda^{-\frac{1}{2}}\right)^{-1}\right\rVert\\
    \leq&\left(1-\left\lVert L_\lambda^{-\frac{1}{2}}\left(L_{M}-L\right)L_\lambda^{-\frac{1}{2}}\right\rVert\right)^{-1}.
\end{align*}
Thus, for $M\geq (4C^2+2C)N(\lambda)\log\frac{1}{t}$, the following holds:
\begin{equation*}
    \mathbb{P}\left(\left\lVert L_\lambda^{-\frac{1}{2}}\left(L_{M}-L\right)L_\lambda^{-\frac{1}{2}}\right\rVert\leq \frac{1}{2}\right)\geq 1-t.
\end{equation*}
Therefore,
\begin{equation*}
    \mathbb{P}\left(\left\lVert L_{M,\lambda}^{-\frac{1}{2}}L_\lambda^{\frac{1}{2}}\right\rVert\leq \sqrt{2}\right)\geq 1-t.
\end{equation*}
Next, we show $\mathbb{P}_{\omega_{1:M}}(E_{cov})\geq 1-t$. Applying Proposition~\ref{prop:gaussian_cov}, we get
\begin{eqnarray*}
    \mathbb{P}_{\omega_{1:M}}\left(\left\lVert L-L_M\right\rVert\leq C\max\left(\sqrt{\frac{\lVert L\rVert\mathrm{tr}(L)}{M}},\frac{\mathrm{tr}(L)}{M},\sqrt{\frac{\lVert L\rVert^2\log\frac{1}{t}}{M}},\frac{\lVert L\rVert\log\frac{1}{t}}{M}\right)\right)\geq 1-t.
\end{eqnarray*}
Setting $M\geq \frac{\kappa(4C^2+2C)}{\lVert L\rVert}\log\frac{1}{t}$, we get $\mathbb{P}_{\omega_{1:M}}(\lVert L-L_M\rVert\leq \lVert L\rVert/2)\geq 1-t$

Next, we show
\begin{equation*}
    \mathbb{P}_{X|\omega_{1:M}}\left(\left\lVert C_{M,\lambda}^{-\frac{1}{2}}\left(\hat{C}_{M}-C_M\right)C_{M,\lambda}^{-\frac{1}{2}}\right\rVert\leq \frac{1}{4}\right)\geq 1-t.
\end{equation*}
For $\omega_1,\ldots,\omega_M\in E_g$ we have $\sup_{x\in\mathcal{X}}|h(x,\omega_i)|\leq H$. Then for $i=1,2,\ldots,n$ the operator norm of the random matrix $C_{M,\lambda}^{-\frac{1}{2}}(h_M(x_i)h_M(x_i)^\top-C_M)C_{M,\lambda}^{-\frac{1}{2}}$ is bounded by $2H^2\lambda^{-1}$. Thus, we can apply Bernstein's inequality for bounded self-adjoint random matrices:
\begin{proposition}[Proposition 3 in \cite{randomfeature}]\label{prop:Bernstein_matrx}
     Let $\mathcal{H}$ be a separable Hilbert space and let $X_1,...,X_n$ be a sequence of independent and identically distributed self-adjoint positive random operators on $\mathcal{H}$. Assume that $\mathbb{E}\left[X_i\right]=0$ and there exists $T>0$ such that the largest eigenvalue of $X_i$ is bounded by $T$ almost surely for every $i\in\{1,\ldots,n\}$. Let $S$ be a positive operator such that $\mathbb{E}\left[X_i^2\right]\preceq S$. Then for any $t\in(0,1]$ the following holds:
     \begin{eqnarray*}
         \mathbb{P}\left(\lambda_{\max}\left(\frac{1}{n}\sum_{i=1}^nX_i\right)\leq\frac{2T\log\frac{2\mathrm{tr}(S)}{\lVert S\rVert t}}{3n}+\sqrt{\frac{2\lVert S\rVert\log\frac{2\mathrm{tr}(S)}{\lVert S\rVert t}}{n}}\right)\geq 1-t.
     \end{eqnarray*}
\end{proposition}
Note that
\begin{eqnarray*}
    \mathbb{E}\left[\left( C_{M,\lambda}^{-\frac{1}{2}}h_M(x_i)h_M(x_i)^\top-C_M)C_{M,\lambda}^{-\frac{1}{2}}\right)^2\right]&\preceq&
    \mathbb{E}\left[\left( C_{M,\lambda}^{-\frac{1}{2}}h_M(x_i)h_M(x_i)^\top C_{M,\lambda}^{-\frac{1}{2}}\right)^2\right]\\
    &\preceq& H^2\lambda^{-1}
    \mathbb{E}\left[C_{M,\lambda}^{-\frac{1}{2}}h_M(x_i)h_M(x_i)^\top C_{M,\lambda}^{-\frac{1}{2}}\right]\\
    &\preceq& H^2\lambda^{-1}C_{M,\lambda}^{-1}C_M.
\end{eqnarray*}
Therefore, we can apply Proposition~\ref{prop:Bernstein_matrx} for $T=2H^2\lambda^{-1}$ and $S=H^2\lambda^{-1}C_{M,\lambda}^{-1}C_M$ to conclude that
\begin{eqnarray*}
\left\lVert C_{M,\lambda}^{-\frac{1}{2}}(C_M-\hat{C}_M)C_{M,\lambda}^{-\frac{1}{2}}\right\rVert\leq\frac{2H^2\log\frac{4\mathrm{tr}(C_{M,\lambda}^{-1}C_M)}{\left\lVert C_{M,\lambda}^{-1}C_M\right\rVert t}}{3n\lambda}+\sqrt{\frac{2H^2\log\frac{4\mathrm{tr}(C_{M,\lambda}^{-1}C_M)}{\left\lVert C_{M,\lambda}^{-1}C_M\right\rVert t}}{n\lambda}},
\end{eqnarray*}
holds with probability at least $1-t$.

Note that
\begin{align*}
\left\lVert C_M\right\rVert=&\lVert L_M\rVert\\
    \geq&\lVert L\rVert-\lVert L-L_M\rVert,
\end{align*}
Thus, for $\omega_1,\ldots,\omega_M\in E_{cov}$, we have $\lVert C_M\rVert\geq\frac{\lVert L\rVert}{2}$. Since $\lambda\in [0,\lVert L\rVert]$, we have
\begin{eqnarray*}
    \mathbb{P}_{\omega_{1:M}}\left(\left\lVert C_{M,\lambda}^{-1}C_M\right\rVert\geq \frac{1}{3}\right)\geq 1-t.
\end{eqnarray*}
Therefore, for $\omega_1,\ldots,\omega_M\in E_{g}\cap E_{cov}$,
\begin{eqnarray*}
\left\lVert C_{M,\lambda}^{-\frac{1}{2}}(C_M-\hat{C}_M)C_{M,\lambda}^{-\frac{1}{2}}\right\rVert\leq\frac{2H^2\log\frac{12\mathrm{tr}(C_{M,\lambda}^{-1}C_M)}{t}}{3n\lambda}+\sqrt{\frac{2H^2\log\frac{12\mathrm{tr}(C_{M,\lambda}^{-1}C_M)}{t}}{n\lambda}},
\end{eqnarray*}
holds with probability at least $1-t$.

Thus, if $n\geq 40H^2\lambda^{-1}\log\frac{12\mathrm{tr}(C_{M,\lambda}^{-1}C_M)}{t}$, the following holds:
\begin{eqnarray*}
    \mathbb{P}_{X|\omega_{1:M}}\left(\left\lVert C_{M,\lambda}^{-\frac{1}{2}}(C_M-\hat{C}_M)C_{M,\lambda}^{-\frac{1}{2}}\right\rVert\leq\frac{1}{4}\right)\geq 1-t.
\end{eqnarray*}
Finally, we show $\left\lVert S_M\hat{C}_{M,\lambda}^{-1}C_{M,\lambda}^{\frac{1}{2}}\right\rVert\leq\frac{4}{3}$ if $\left\lVert C_{M,\lambda}^{-\frac{1}{2}}(C_M-\hat{C}_M)C_{M,\lambda}^{-\frac{1}{2}}\right\rVert\leq\frac{1}{4}$:
\begin{align*}
\left\lVert S_M\hat{C}_{M,\lambda}^{-1}C_{M,\lambda}^{\frac{1}{2}}\right\rVert \leq&\left\lVert S_M\hat{C}_{M,\lambda}^{-\frac{1}{2}}\right\rVert\left\lVert \hat{C}_{M,\lambda}^{-\frac{1}{2}} C_{M,\lambda}^{\frac{1}{2}}\right\rVert\\
    =&\left\lVert \hat{C}_{M,\lambda}^{-\frac{1}{2}\top}C_M\hat{C}_{M,\lambda}^{-\frac{1}{2}}\right\rVert^{\frac{1}{2}}\left\lVert \hat{C}_{M,\lambda}^{-\frac{1}{2}} C_{M,\lambda}^{\frac{1}{2}}\right\rVert\\
    \leq&\left\lVert \hat{C}_{M,\lambda}^{-\frac{1}{2}\top}C_{M,\lambda}\hat{C}_{M,\lambda}^{-\frac{1}{2}}\right\rVert^{\frac{1}{2}}\left\lVert \hat{C}_{M,\lambda}^{-\frac{1}{2}} C_{M,\lambda}^{\frac{1}{2}}\right\rVert\\
    =&\left\lVert C_{M,\lambda}^{\frac{1}{2}}\hat{C}_{M,\lambda}^{-\frac{1}{2}}\right\rVert\left\lVert \hat{C}_{M,\lambda}^{-\frac{1}{2}} C_{M,\lambda}^{\frac{1}{2}}\right\rVert\\
    =&\left\lVert \hat{C}_{M,\lambda}^{-\frac{1}{2}} C_{M,\lambda}^{\frac{1}{2}}\right\rVert^2\\
    =&\left\lVert C_{M,\lambda}^{\frac{1}{2}}\hat{C}_{M,\lambda}^{-1} C_{M,\lambda}^{\frac{1}{2}}\right\rVert\\
    =&\left\lVert C_{M,\lambda}^{-\frac{1}{2}}\hat{C}_{M,\lambda} C_{M,\lambda}^{-\frac{1}{2}}\right\rVert^{-1}\\
    =&\left\lVert\left( I+C_{M,\lambda}^{-\frac{1}{2}}\left(\hat{C}_{M}-C_{M}\right) C_{M,\lambda}^{-\frac{1}{2}}\right)^{-1}\right\rVert\\
    \leq&\left(1-\left\lVert C_{M,\lambda}^{-\frac{1}{2}}\left(\hat{C}_{M}-C_{M}\right) C_{M,\lambda}^{-\frac{1}{2}}\right\rVert\right)^{-1}.
\end{align*}

\subsubsection{Proof of Lemma~\ref{lem:noise_err}}
We apply Proposition~\ref{prop:bernstein_vec} to the random vectors $\{C_{M,\lambda}^{-\frac{1}{2}}y_ih_M(x_i)\}_{i=1}^n$. Note that for $\omega_1,\ldots,\omega_M\in E_g$, we have:
\begin{align*}
&\mathbb{E}\left[\left\lVert C_{M,\lambda}^{-\frac{1}{2}}y_1h_M(x_1)-\mathbb{E}\left[C_{M,\lambda}^{-\frac{1}{2}}y_1h_M(x_1)\right]\right\rVert^p\right]\\
\leq&2^{p-1}\mathbb{E}\left[\left\lVert C_{M,\lambda}^{-\frac{1}{2}}y_1h_M(x_1)\right\rVert^p+\left\lVert\mathbb{E}\left[C_{M,\lambda}^{-\frac{1}{2}}y_1h_M(x_1)\right]\right\rVert^p\right]\\
    \leq&2^{p}\mathbb{E}\left[\left\lVert C_{M,\lambda}^{-\frac{1}{2}}y_1h_M(x_1)\right\rVert^p\right]\\
    \leq&2^{p}\mathbb{E}\left[\left\lVert C_{M,\lambda}^{-\frac{1}{2}}h_M(x_1)\right\rVert^p\mathbb{E}\left[|y_1|^p|x_1\right]\right]\\
    \leq&\frac{1}{2}p!(2\sigma)^2(2B)^{p-2}\mathbb{E}_{x_1}\left[\left\lVert C_{M,\lambda}^{-\frac{1}{2}}h_M(x_1)\right\rVert^p\right]\\
    \leq&\frac{1}{2}p!(2\sigma)^2(2BH\lambda^{-\frac{1}{2}})^{p-2}\mathbb{E}_{x_1}\left[\left\lVert C_{M,\lambda}^{-\frac{1}{2}}h_M(x_1)\right\rVert^2\right]\\
    =&\frac{1}{2}p!(2\sigma\sqrt{\mathrm{tr}\left(C_{M,\lambda}^{-1}C_M\right)})^2(2BH\lambda^{-\frac{1}{2}})^{p-2}.
\end{align*}
Applying Proposition~\ref{prop:bernstein_vec} yields the following:
\begin{align*}
    \mathbb{P}\left(\left\lVert C_{M,\lambda}^{-\frac{1}{2}}\left(\hat{S}_M^\top \hat{y}-S_M^\top f_0\right)\right\rVert_2\leq\frac{2BH\log\frac{2}{t}}{n\sqrt{\lambda}}+\sqrt{\frac{8\sigma^2 \mathrm{tr}\left(C_{M,\lambda}^{-1}C_M\right)\log\frac{2}{t}}{n}}\right)\geq 1-t.
\end{align*}

\subsubsection{Proof of Lemma~\ref{lem:rf_err}}
We apply Proposition~\ref{prop:bernstein_vec} to the random vectors $\{C_{M,\lambda}^{-\frac{1}{2}}h_M(x_i)h_M(x_i)^\top\}_{i=1}^n$ in $\mathbb{R}^{M\times M}$. Note that for $\omega_1,\ldots,\omega_M\in E_g$, we have:
\begin{align*}
&\mathbb{E}\left[\left\lVert C_{M,\lambda}^{-\frac{1}{2}}h_M(x_i)h_M(x_i)^\top-\mathbb{E}\left[C_{M,\lambda}^{-\frac{1}{2}}h_M(x_i)h_M(x_i)^\top\right]\right\rVert_F^p\right]\\
    \leq&2^{p-1}\mathbb{E}\left[\left\lVert C_{M,\lambda}^{-\frac{1}{2}}h_M(x_i)h_M(x_i)^\top\right\rVert_F^p+\left\lVert \mathbb{E}\left[C_{M,\lambda}^{-\frac{1}{2}}h_M(x_i)h_M(x_i)^\top\right]\right\rVert_F^p\right]\\
    \leq&2^{p}\mathbb{E}\left[\left\lVert C_{M,\lambda}^{-\frac{1}{2}}h_M(x_i)h_M(x_i)^\top\right\rVert_F^p\right]\\
    \leq&(2H)^{2}(2H\lambda^{-\frac{1}{2}})^{p-2}\mathbb{E}\left[\left\lVert C_{M,\lambda}^{-\frac{1}{2}}h_M(x_i)\right\rVert_2^2\right]\\
    =&\left(2H\sqrt{\mathrm{tr}\left(C_{M,\lambda}^{-1}C_M\right)}\right)^{2}(2H\lambda^{-\frac{1}{2}})^{p-2}\\
    \leq&\frac{1}{2}p!\left(2H\sqrt{\mathrm{tr}\left(C_{M,\lambda}^{-1}C_M\right)}\right)^{2}(2H\lambda^{-\frac{1}{2}})^{p-2}
\end{align*}
for $p\geq 2$. Therefore,
\begin{eqnarray*}
    \mathbb{P}\left(\left\lVert C_{M,\lambda}^{-\frac{1}{2}}\left(C_M-\hat{C}_M\right)\right\rVert_F\leq\frac{4H\log\frac{2}{t}}{n\sqrt{\lambda}}+\sqrt{\frac{8H^2\mathrm{tr}\left(C_{M,\lambda}^{-1}C_M\right)\log\frac{2}{t}}{n}}\right)\geq 1-t.
\end{eqnarray*}

\subsubsection{Proof of Lemma~\ref{lem:rf_dim}}
Proposition 7 in \cite{randomfeature} states:
    \begin{eqnarray*}
        \left|\mathrm{tr}(C_{M,\lambda}^{-1}C_{M})- N(\lambda)\right|\leq\left(\frac{d(\lambda)^2}{(1-c(\lambda))N(\lambda)}+\frac{\lambda e(\lambda)}{N(\lambda)}\right)N(\lambda)
    \end{eqnarray*}
    where $c(\lambda)=\lambda_{\max}\left( \widetilde{B}\right), d(\lambda)=\lVert \widetilde{B}\rVert_F$, and $e(\lambda)=|\mathrm{tr}(L_{\lambda}^{-\frac{1}{2}}\widetilde{B}L_{\lambda}^{-\frac{1}{2}})|$ where $\widetilde{B}=L_\lambda^{-\frac{1}{2}}(L-L_M)L_{\lambda}^{-\frac{1}{2}}$. We bound $c(\lambda),\lambda e(\lambda)$, and $d(\lambda)$.

    By Lemma~\ref{lem:full_operator}, we can bound $c(\lambda)=O(1)$ as $M=O(N(\lambda)\log\frac{1}{t})$. Note that,
    \begin{eqnarray*}
        \lambda e(\lambda)=\frac{1}{M}\sum_{i=1}^M\lVert\sqrt{\lambda} L_\lambda^{-1}h(\cdot,w_i)\rVert^2-\mathbb{E}\left[\lVert\sqrt{\lambda} L_\lambda^{-1}h(\cdot,w_i)\rVert^2\right],
    \end{eqnarray*}
    and
\begin{align*}
\mathbb{E}_\omega \left[\left(\lVert\sqrt{\lambda} L_\lambda^{-1}h(\cdot,w_i)\rVert^2-\mathbb{E}\left[\lVert\sqrt{\lambda} L_\lambda^{-1}h(\cdot,w_i)\rVert^2\right]\right)^p\right]
         \leq& 2p!\left(4\mathbb{E}\left[\lVert\sqrt{\lambda} L_\lambda^{-1}h(\cdot,w_i)\rVert^2\right]\right)^p\\
         \leq& 2p!\left(4\mathbb{E}\left[\lVert L_\lambda^{-\frac{1}{2}}h(\cdot,w_i)\rVert^2\right]^2\right)^p\\
         =&2p!\left(4N(\lambda)\right)^p\\
         =&\frac{1}{2}p!\left(8N(\lambda)\right)^p,
\end{align*}
    by Lemma~\ref{lem:gp_moments}. Thus,
    \begin{eqnarray*}
        \mathbb{P}\left(\lambda e(\lambda)\leq \frac{8N(\lambda)\log\frac{2}{t}}{M}+\frac{8N(\lambda)\log\frac{2}{t}}{\sqrt{M}}\right)\geq1-t,
    \end{eqnarray*}
    by Proposition~\ref{prop:bernstein_vec}.
    Therefore, $\frac{\lambda e(\lambda)}{N(\lambda)}=o(1)$.
    
    Next, we show the high-probability bound of $d(\lambda)$. Note that,
\begin{align*}
\mathbb{E}\left[\lVert L_\lambda^{-\frac{1}{2}}(L-h(\cdot,w_i)h(\cdot,w_i)^\top)L_\lambda^{-\frac{1}{2}}\rVert_F^p\right]=&2^{p}\mathbb{E}\left[\lVert L_\lambda^{-\frac{1}{2}}h(\cdot,w_i)h(\cdot,w_i)^\top L_\lambda^{-\frac{1}{2}}\rVert_F^p\right]\\
        \leq&2^{p}\mathbb{E}\left[\lVert L_\lambda^{-\frac{1}{2}}h(\cdot,w_i)\rVert_F^{2p}\right]\\
        \leq&2^{p+1}p!\left(8\mathbb{E}\left[\lVert L_\lambda^{-\frac{1}{2}}h(\cdot,w_i)\rVert_F^{2}\right]\right)^p\\
        \leq&2^{p+1}p!\left(8N(\lambda)\right)^p\\
        \leq&\frac{1}{2}p!\left(32N(\lambda)\right)^p.
\end{align*}
    Applying Proposition~\ref{prop:bernstein_vec} yields
    \begin{eqnarray*}
        \mathbb{P}\left(d(\lambda)\leq \frac{32N(\lambda)\log\frac{2}{t}}{M}+\frac{32N(\lambda)\log\frac{2}{t}}{\sqrt{M}}\right)\geq1-t.
    \end{eqnarray*}
    Then $\frac{d(\lambda)^2}{N(\lambda)}=O(1)$. Therefore, there exists a constant $c$ such that
     \begin{equation*}
        \mathbb{P}_\omega \left(\mathrm{tr}(C_{M,\lambda}^{-1}C_M)\leq 2N(\lambda)\right)\geq 1-3t
    \end{equation*}
    if $M\geq c N(\lambda)\log^2\frac{2}{t}$.

\subsection{Proof of Lemma~\ref{lem:priv_ridge}}
\begin{proposition}[Example 2.11 in \citet{wainwright2019high}]\label{prop:gauss_vec_concen}
    Let $v$ be an $m$-dimensional random vector with i.i.d. components following the standard normal distribution. Then
    \begin{eqnarray*}
        \mathbb{P}\left(\left\lVert v\right\rVert_2\geq\sqrt{m}+\sqrt{8\log\frac{2}{t}}\right)\leq t,
    \end{eqnarray*}
    holds for any $t\in (0,1]$.
\end{proposition}
\begin{proposition}[Theorem 4.4.5 in \citet{vershynin2018high}]\label{prop:gauss_mat_concen} 
    Let $A$ be an $m\times n$ random matrix whose entries $A_{ij}$ are independent, mean zero, standard Gaussian random variables. Then, for any $t\in(0,1]$, we have
    \begin{eqnarray*}
        \left\lVert A\right\rVert\leq C^\prime\left(\sqrt{m}+\sqrt{n}+\sqrt{2\log\frac{2}{t}}\right)
    \end{eqnarray*}
with probability at least $1-t$ where $C^\prime$ is a universal constant.
\end{proposition}
Define
\begin{align*}
E_1&:=\Bigg\{\left\lVert\widetilde{C}_{M,\lambda}-\hat{C}_{M,\lambda}\right\rVert\leq \frac{4C^\prime \Delta_1\sqrt{M\log\frac{1}{\delta}}}{n\epsilon}\Bigg\},\\
    E_2&:=\Bigg\{\left\lVert\widetilde{u}-\hat{u}_T\right\rVert\leq \frac{2\Delta_2\sqrt{M\log\frac{1}{\delta}}}{n\epsilon},[y_i]_T=y_i\textrm{ }\forall i\Bigg\}.
\end{align*}
From Proposition~\ref{prop:gauss_mat_concen}, we can establish that $\mathbb{P}\left(E_1\right)\geq 1-t$ if $M\geq \frac{1}{2}\log\frac{2}{t}$. This is shown by the following:
\begin{align*}
&\mathbb{P}\left(\left\lVert\widetilde{C}_{M,\lambda}-\hat{C}_{M,\lambda}\right\rVert\leq\frac{4C^\prime\Delta_1\sqrt{M\log\frac{1}{\delta}}}{n\epsilon}\right)\\\geq& \mathbb{P}\left(\left\lVert\widetilde{C}_{M,\lambda}-\hat{C}_{M,\lambda}\right\rVert\leq\frac{C^\prime\Delta_1\sqrt{\log\frac{1}{\delta}}}{n\epsilon}\left(2\sqrt{M}+\sqrt{2\log\frac{2}{t}}\right)\right)\\
    \geq&1-t.
\end{align*}
Also $\mathbb{P}\left(E_2\right)\geq 1-t$ if $M\geq 8\log\frac{2}{t}$ by Proposition~\ref{prop:gauss_vec_concen}. First, note that $\mathbb{P}\left([y_i]_T=y_i\textrm{ }\forall i\right)\geq 1-t$, because
\begin{equation*}
    \mathbb{P}\left(Y_i-\mathbb{E}\left[Y_i\right]\geq 2B\log\frac{n}{t}+\sqrt{2\log\frac{n}{t}}\sigma\right)\leq t/n.
\end{equation*}
Then,
\begin{align*}
&\mathbb{P}\left(\left\lVert\widetilde{u}-\hat{u}_T\right\rVert\leq \frac{2\Delta_2\sqrt{M\log\frac{1}{\delta}}}{n\epsilon}\Bigg|[y_i]_T=y_i\textrm{ }\forall i\right)\\\geq&\mathbb{P}\left(\left\lVert\widetilde{u}-\hat{u}_T\right\rVert\leq \frac{\Delta_2\sqrt{\log\frac{1}{\delta}}}{n\epsilon}\left(\sqrt{M}+\sqrt{8\log\frac{2}{t}}\right)\Bigg|[y_i]_T=y_i\textrm{ }\forall i\right)\\
    \geq&1-t.
\end{align*}
Therefore, $\mathbb{P}\left(E_1\cup E_2\right)\geq 1-3t$ if $M\geq 8\log\frac{2}{t}$.

Assume $M\gtrsim \max\left\{\left(\frac{\kappa}{\lVert L\rVert}+N(\lambda)\right)\log\frac{2}{t},8\log\frac{2}{t}\right\}$ and $n\gtrsim H^2\lambda^{-1}\log\frac{4N(\lambda)}{t}$.

Now we bound the excess risk of $\widetilde{f}$ in Algorithm~\ref{alg:dprp_reg}. Let $\hat{u}:=\frac{1}{n}\sum_{i=1}^ny_iz_i$. Then $\hat{\beta}_{M,\lambda}=C_{M,\lambda}^{-1}\hat{u}=C_{M,\lambda}^{-1}\hat{w}_T$.
\begin{align*}
    \lVert \widetilde{f}-\hat{f}_{M,\lambda}\rVert=&\left\lVert S_M(\widetilde{\beta}-\hat{\beta}_{M,\lambda})\right\rVert\\
    =&\left\lVert S_M\left(\widetilde{C}_{M,\lambda}^{-1}\widetilde{u}-\hat{C}_{M,\lambda}^{-1}\hat{u}_T\right)\right\rVert\\
    =&\left\lVert S_M\widetilde{C}_{M,\lambda}^{-1}(\hat{C}_{M,\lambda}-\widetilde{C}_{M,\lambda})\hat{C}_{M,\lambda}^{-1}\widetilde{u}\right\rVert+\left\lVert S_M\hat{C}_{M,\lambda}^{-1}(\widetilde{u}-\hat{u}_T)\right\rVert.
\end{align*}
We bound each term. Suppose $\omega_1,\ldots,\omega_M\in E_{g}\cap E_{half}\cap E_{full}\cap E_{cov}\cap E_{rf-dim}$, and $X_{1:n},Y_{1:n}\in E_{cov-mat}\cap E_{noise}\cap E_{cov-mat-half}$.

The first term is bounded as below:
\begin{align*}
&\left\lVert S_M\widetilde{C}_{M,\lambda}^{-1}(\hat{C}_{M,\lambda}-\widetilde{C}_{M,\lambda})\hat{C}_{M,\lambda}^{-1}\widetilde{u}\right\rVert\\
    \leq&\left\lVert S_M\widetilde{C}_{M,\lambda}^{-1}C_{M,\lambda}^{\frac{1}{2}}C_{M,\lambda}^{-\frac{1}{2}}(\hat{C}_{M,\lambda}-\widetilde{C}_{M,\lambda})\hat{C}_{M,\lambda}^{-1}C_{M,\lambda}C_{M,\lambda}^{-1}\widetilde{u}\right\rVert\\
    \leq&\lambda^{-\frac{1}{2}}\left\lVert S_M\widetilde{C}_{M,\lambda}^{-1}C_{M,\lambda}^{\frac{1}{2}}\right\rVert\left\lVert \hat{C}_{M,\lambda}-\widetilde{C}_{M,\lambda}\right\rVert\left\lVert\hat{C}_{M,\lambda}^{-1}C_{M,\lambda}\right\rVert\left\lVert C_{M,\lambda}^{-1}\widetilde{u}\right\rVert.
\end{align*}
First, we bound $\left\lVert S_M\widetilde{C}_{M,\lambda}^{-1}C_{M,\lambda}^{\frac{1}{2}}\right\rVert$. Note that
\begin{eqnarray*}
    \left\lVert\frac{\Delta_1\sqrt{\log\frac{1}{\delta}}}{n\epsilon}\varepsilon_{M\times M}\right\rVert\leq\frac{4C^\prime\Delta_1\sqrt{M\log\frac{1}{\delta}}}{n\epsilon}\leq \frac{\lambda}{2},
\end{eqnarray*}
as $M\leq \frac{n^2\lambda^2\epsilon^2}{64C^{\prime2}\Delta_1^2\log\frac{1}{\delta}}$. Thus, $\widetilde{C}_{M,\lambda}^{-1}$ is positive definite. Then,
\begin{align*}
\left\lVert S_M\widetilde{C}_{M,\lambda}^{-1}C_{M,\lambda}^{\frac{1}{2}}\right\rVert \leq&\left\lVert S_M\widetilde{C}_{M,\lambda}^{-\frac{1}{2}}\right\rVert\left\lVert \widetilde{C}_{M,\lambda}^{-\frac{1}{2}} C_{M,\lambda}^{\frac{1}{2}}\right\rVert\\
    =&\left\lVert \widetilde{C}_{M,\lambda}^{-\frac{1}{2}}C_M\widetilde{C}_{M,\lambda}^{-\frac{1}{2}}\right\rVert^{\frac{1}{2}}\left\lVert \widetilde{C}_{M,\lambda}^{-\frac{1}{2}} C_{M,\lambda}^{\frac{1}{2}}\right\rVert\\
    \leq&\left\lVert \widetilde{C}_{M,\lambda}^{-\frac{1}{2}}C_{M,\lambda}\widetilde{C}_{M,\lambda}^{-\frac{1}{2}}\right\rVert^{\frac{1}{2}}\left\lVert \widetilde{C}_{M,\lambda}^{-\frac{1}{2}} C_{M,\lambda}^{\frac{1}{2}}\right\rVert\\
    =&\left\lVert C_{M,\lambda}^{\frac{1}{2}}\widetilde{C}_{M,\lambda}^{-\frac{1}{2}}\right\rVert\left\lVert \widetilde{C}_{M,\lambda}^{-\frac{1}{2}} C_{M,\lambda}^{\frac{1}{2}}\right\rVert\\
    =&\left\lVert \widetilde{C}_{M,\lambda}^{-\frac{1}{2}} C_{M,\lambda}^{\frac{1}{2}}\right\rVert^2\\
    =&\left\lVert C_{M,\lambda}^{\frac{1}{2}}\widetilde{C}_{M,\lambda}^{-1} C_{M,\lambda}^{\frac{1}{2}}\right\rVert\\
    =&\left\lVert \left(C_{M,\lambda}^{-\frac{1}{2}}\widetilde{C}_{M,\lambda} C_{M,\lambda}^{-\frac{1}{2}}\right)^{-1}\right\rVert\\
    =&\left\lVert \left(I+C_{M,\lambda}^{-\frac{1}{2}}(\widetilde{C}_{M}-C_M) C_{M,\lambda}^{-\frac{1}{2}}\right)^{-1}\right\rVert\\
\leq&\left(1-\left\lVert C_{M,\lambda}^{-\frac{1}{2}}\left(\widetilde{C}_{M}-C_{M}\right) C_{M,\lambda}^{-\frac{1}{2}}\right\rVert\right)^{-1}\\
\leq&\left(1-\left\lVert C_{M,\lambda}^{-\frac{1}{2}}\left(\widetilde{C}_M-\hat{C}_{M}\right) C_{M,\lambda}^{-\frac{1}{2}}\right\rVert-\left\lVert C_{M,\lambda}^{-\frac{1}{2}}\left(\hat{C}_{M}-C_{M}\right) C_{M,\lambda}^{-\frac{1}{2}}\right\rVert\right)^{-1}\\
\leq&\left(\frac{1}{2}-\left\lVert C_{M,\lambda}^{-\frac{1}{2}}\left(\hat{C}_{M}-C_{M}\right) C_{M,\lambda}^{-\frac{1}{2}}\right\rVert\right)^{-1}.
\end{align*}
Lemma~\ref{lem:constants} yields $\left\lVert S_M\widetilde{C}_{M,\lambda}^{-1}C_{M,\lambda}^{\frac{1}{2}}\right\rVert\leq4$ if $\lambda\in[0,\lVert L\rVert]$.

Next, we bound $\left\lVert \hat{C}_{M,\lambda}^{-1}C_{M,\lambda}\right\rVert$.
\begin{align*}
\left\lVert \hat{C}_{M,\lambda}^{-1}C_{M,\lambda}\right\rVert\leq&1+
    \left\lVert \left(\hat{C}_{M,\lambda}^{-1}-C_{M,\lambda}^{-1}\right)C_{M,\lambda}\right\rVert\\
    =&1+
    \left\lVert C_{M,\lambda}^{-1}\left(\hat{C}_{M}-C_{M}\right)\hat{C}_{M,\lambda}^{-1}C_{M,\lambda}\right\rVert\\
    \leq&1+
    \left\lVert C_{M,\lambda}^{-1}\left(\hat{C}_{M}-C_{M}\right)\right\rVert\left\lVert\hat{C}_{M,\lambda}^{-1}C_{M,\lambda}\right\rVert.
\end{align*}
By Lemma~\ref{lem:rf_err}, we obtain $\left\lVert C_{M,\lambda}^{-1}\left(\hat{C}_{M}-C_{M}\right)\right\rVert\leq \frac{1}{2}$ and $\left\lVert \hat{C}_{M,\lambda}^{-1}C_{M,\lambda}\right\rVert\leq 2$ since $n\geq \max\{256H^2N(\lambda)\log\frac{2}{t},8BH\lambda^{-\frac{1}{2}}\log\frac{2}{t}\}$.

Next, we have $\left\lVert \hat{C}_{M,\lambda}-\widetilde{C}_{M,\lambda}\right\rVert\leq \frac{4C^\prime\Delta_1\sqrt{M\log\frac{1}{\delta}}}{n\epsilon}$, and
\begin{align*}
\left\lVert C_{M,\lambda}^{-1}\widetilde{u}\right\rVert_2\leq&\left\lVert C_{M,\lambda}^{-1}\hat{u}_T\right\rVert_2+\frac{\left\lVert \widetilde{u}-\hat{u}_T\right\rVert_2}{\lambda}\\\leq&\left\lVert C_{M,\lambda}^{-1}\hat{u}_T\right\rVert_2+\frac{2\Delta_2\sqrt{M\log\frac{1}{\delta}}}{n\lambda\epsilon}\\
    =&\left\lVert C_{M,\lambda}^{-1}\hat{u}\right\rVert_2+\frac{2\Delta_2\sqrt{M\log\frac{1}{\delta}}}{n\lambda\epsilon}\\
    \leq&\left\lVert C_{M,\lambda}^{-1}S_M^\top f_0\right\rVert_2+\left\lVert C_{M,\lambda}^{-1}(\hat{u}-S_M^\top f_0)\right\rVert_2+\frac{2\Delta_2\sqrt{M\log\frac{1}{\delta}}}{n\lambda\epsilon}.
\end{align*}
The first term is bounded as
\begin{align*}
\left\lVert C_{M,\lambda}^{-1}S_M^\top f_0\right\rVert \leq&\left\lVert C_{M,\lambda}^{-1}S_M^\top Pf_0\right\rVert+\left\lVert C_{M,\lambda}^{-1}S_M^\top(I-P) f_0\right\rVert\\\leq&\left\lVert C_{M,\lambda}^{-1}S_M^\top L_{M,\lambda}^{\frac{1}{2}}\right\rVert\left\lVert L_{M,\lambda}^{-\frac{1}{2}}L^{\frac{1}{2}}\right\rVert\left\lVert L^{-\frac{1}{2}}Pf_0\right\rVert\\
    =&\left\lVert C_{M,\lambda}^{-1}S_M^\top L_{M,\lambda}^{\frac{1}{2}}\right\rVert\left\lVert L_{M,\lambda}^{-\frac{1}{2}}L^{\frac{1}{2}}\right\rVert\left\lVert L^{r-\frac{1}{2}}g\right\rVert\\
    \leq&\left\lVert C_{M,\lambda}^{-1}S_M^\top L_{M,\lambda}^{\frac{1}{2}}\right\rVert\left\lVert L_{M,\lambda}^{-\frac{1}{2}}L^{\frac{1}{2}}\right\rVert\kappa^{2r-1}R\\
    \leq&\sqrt{2}\kappa^{2r-1}R,
\end{align*}
which follows from the fact that $\left\lVert C_{M,\lambda}^{-1}S_M^\top L_{M,\lambda}^{\frac{1}{2}}\right\rVert=\left\lVert S_M L_{M,\lambda}^{-1} L_{M,\lambda}^{\frac{1}{2}}\right\rVert=\left\lVert L_{M,\lambda}^{-\frac{1}{2}}L_ML_{M,\lambda}^{-\frac{1}{2}}\right\rVert^{\frac{1}{2}}\leq 1$, and Lemma~\ref{lem:constants}.

Next, we bound the second term by applying Lemma~\ref{lem:noise_err}:
\begin{align*}
\left\lVert C_{M,\lambda}^{-1}(\hat{u}_T-S_M^\top f_0)\right\rVert_2 \leq&\lambda^{-\frac{1}{2}}\left\lVert C_{M,\lambda}^{-\frac{1}{2}}(\hat{u}_T-S_M^\top f_0)\right\rVert_2
\\\leq&\frac{2BH\log\frac{2}{t}}{n\lambda}+\sqrt{\frac{16\sigma^2N(\lambda)\log\frac{2}{t}}{n\lambda}}.
\end{align*}
Thus,
\begin{eqnarray*}
    \left\lVert C_{M,\lambda}^{-1}\widetilde{u}\right\rVert\leq \sqrt{2}\kappa^{2r-1}R+\frac{2BH\log\frac{2}{t}}{n\lambda}+\sqrt{\frac{16\sigma^2N(\lambda)\log\frac{2}{t}}{n\lambda}}+\frac{2\Delta_2\sqrt{M\log\frac{1}{\delta}}}{n\lambda\epsilon}.
\end{eqnarray*}
Therefore,
\begin{align*}
&\left\lVert S_M\widetilde{C}_{M,\lambda}^{-1}(\hat{C}_{M,\lambda}-\widetilde{C}_{M,\lambda})\hat{C}_{M,\lambda}^{-1}\widetilde{u}\right\rVert\\
    \leq&\frac{16C^\prime \Delta_1\sqrt{M\log\frac{1}{\delta}}}{n\sqrt{\lambda}\epsilon}\left(\sqrt{2}\kappa^{2r-1}R+\frac{2BH\log\frac{2}{t}}{n\lambda}+\sqrt{\frac{16\sigma^2N(\lambda)\log\frac{2}{t}}{n\lambda}}+\frac{2\Delta_2\sqrt{M\log\frac{1}{\delta}}}{n\lambda\epsilon}\right).
\end{align*}
Also,
\begin{align*}
    \left\lVert S_M\hat{C}_{M,\lambda}^{-1}(\widetilde{u}-\hat{u})\right\rVert \leq &\frac{2\Delta_2\sqrt{M\log\frac{1}{\delta}}}{n\epsilon}\left\lVert S_M\hat{C}_{M,\lambda}^{-1}\right\rVert\\
     \leq &\frac{2\Delta_2\sqrt{M\log\frac{1}{\delta}}}{n\sqrt{\lambda}\epsilon}\left\lVert S_M\hat{C}_{M,\lambda}^{-1}C_{M,\lambda}^{\frac{1}{2}}\right\rVert\\
     \leq &\frac{8\Delta_2\sqrt{M\log\frac{1}{\delta}}}{n\sqrt{\lambda}\epsilon}.
\end{align*}
Therefore, the following holds:
\begin{align*}
    &\lVert\widetilde{f}-\hat{f}_{M,\lambda}\rVert\\
     \leq & \frac{4C^\prime \Delta_1\sqrt{M\log\frac{1}{\delta}}}{n\sqrt{\lambda}\epsilon}\left(\sqrt{2}\kappa^{2r-1}R+\frac{2BH\log\frac{2}{t}}{n\lambda}+\sqrt{\frac{16\sigma^2N(\lambda)\log\frac{2}{t}}{n\lambda}}+\frac{2\Delta_2\sqrt{M\log\frac{1}{\delta}}}{n\lambda\epsilon}\right)\\
    &+\frac{8\Delta_2\sqrt{M\log\frac{1}{\delta}}}{n\sqrt{\lambda}\epsilon}
\end{align*}
if $\frac{n^2\lambda^2\epsilon^2}{64C^{\prime2}\Delta_1^2\log^2\frac{1}{\delta}}\geq M\gtrsim \max\left\{\left(\frac{\kappa}{\lVert L\rVert}+N(\lambda)\right)\log\frac{2}{t},8\log\frac{2}{t}\right\}$ and $n\gtrsim H^2\lambda^{-1}\log\frac{4N(\lambda)}{t}$.

\subsection{Proof of Theorem~\ref{thm:dprff_reg_rate_simple}}
By setting $r=1/2$ in Theorem~\ref{thm:dprff_reg_rate}, we obtain Theorem~\ref{thm:dprff_reg_rate_simple}.
\subsection{Proof of Theorem~\ref{thm:dprff_reg_rate}}
As in the proof of Theorem~\ref{thm:dprp_reg_rate}, we prove the theorem under Assumption~\ref{assump:cap_real}.

Using a similar argument to that in the proof of Theorem~\ref{thm:dprp_reg_rate},
if $(n\lambda\epsilon)^2\gtrsim M\gtrsim \max\Big\{\mathcal{F}_\infty(\lambda)\log\frac{1}{\lambda t},\sqrt{\log\frac{1}{t}}\Big\}$ and $n\gtrsim \lambda^{-1}\log\frac{1}{\lambda t}$, 
{\footnotesize
\begin{eqnarray*}
    &&\sqrt{\mathcal{E}(\widetilde{f})-\min_{f\in\mathcal{H}_k}\mathcal{E}(f)}\\
    &\lesssim& \frac{\sqrt{M}}{n\sqrt{\lambda}\epsilon}+\left(\frac{1}{\sqrt{\lambda }n}+\sqrt{\frac{N(\lambda)}{n}}+\left(\frac{\sqrt{\lambda \mathcal{F}_{\infty}(\lambda)}}{M^r}+\sqrt{\frac{\lambda N(\lambda)^{2r-1}\mathcal{F}_{\infty}(\lambda)^{2-2r}}{M}}\right)+\lambda^r\right)\log\frac{18}{t}
\end{eqnarray*}}
by Lemma~\ref{lem:priv_ridge} and Theorem 5 in \cite{randomfeature}.

We set $M=\widetilde{O}\left(n^{\frac{2r-1+\gamma_{\alpha}}{2r+\gamma}}\wedge(n\epsilon)^{\frac{4r+2\gamma_{\alpha}-2}{4r+\gamma_{\alpha}}}\right)$, with $\lambda=n^{-\frac{1}{2r+\gamma}}\vee (n\epsilon)^{-\frac{2}{4r+\gamma_{\alpha}}}$, to obtain the rate of $\widetilde{O}\left(n^{-\frac{2r}{2r+\gamma}}+(n\epsilon)^{-\frac{4r}{4r+\gamma_{\alpha}}}\right)$.
\subsection{Proof of Theorem~\ref{thm:lower_bd_simple}}
To show the lower bound, we use Theorem~\ref{thm:lower_bd} for $r=1/2$.
\subsection{Proof of Theorem~\ref{thm:lower_bd}}\label{sec:lowerbd}
We prove a slightly stronger version of the theorem. \begin{assumption}[Fluctuation condition]\label{assump:fluct}
For some constants $c_3>0$ and $\beta\in (\gamma,1)$, the eigenfunction $\phi_l$ satisfies $\lVert \phi_l\rVert_{L^1(P_X)}^2\leq c_3\mu_l^{\beta}$ for all positive integers $l$.
\end{assumption}
Denote the collection of distributions satisfying $\mathcal{P}_4\subset\mathcal{P}_2$ as Assumptions~\ref{assump:bern},~\ref{assump:cap},~\ref{assump:kernel},~\ref{assump:source},~\ref{assump:fluct}. We derive the minimax lower bound $O(n^{-\frac{2r}{2r+\gamma}}+(n\epsilon)^{-\frac{4r}{2r+\gamma+\beta}})$ over $\mathcal{P}_4$.

We briefly discuss the fluctuation condition. Intuitively, as the parameter $\beta$ in Assumption~\ref{assump:fluct} increases, $\mathcal{P}_4$ encompasses more challenging learning scenarios. Similar to the eigenvalue assumption, the upper bound imposed by the fluctuation condition ensures the existence of a `worst-case' scenario for the minimax analysis. Without an additional assumption, there exists a kernel and a data distribution such that Assumption~\ref{assump:fluct} is satisfied for any $\beta<1$, which will be given at the end of the proof. This implies the minimax lower bound $O(n^{-\frac{2r}{2r+\gamma}}+(n\epsilon)^{-\frac{4r}{2r+\gamma+1}+\eta})$ for any $\eta>0$. However, under a more favorable setting, for example, if $\lVert \phi_l\rVert_{L^1(P_X)}^2\geq c_3\mu_l^{\beta}$ holds for every data distribution and kernel, then the minimax lower bound becomes faster. Additionally, we show the derived minimax lower bound can be achieved under an additional assumption and the knowledge of the underlying covariate distribution $P_X$.

Let $\nu$ be the distribution satisfying Assumption~\ref{assump:cap} and~\ref{assump:fluct}. We construct $p_v\in\mathcal{P}_4$ for $v\in\{-1,+1\}^{d}$ as follows. Define $p_v$ as below.
\begin{align*}
    f_{v}(x)&:=\frac{1}{\sqrt{d}}\sum_{l=1}^{d}R\mu_{l+d}^{r}v_{l}\phi_{l+d}(x)\\
    dP_v(x,y)&:=\frac{2\kappa^{2r}R+f_v(x)}{4\kappa^{2r}R}d\nu(x)dy_{+1}(y)+\frac{2\kappa^{2r}R-f_v(x)}{4\kappa^{2r}R}d\nu(x)dy_{-1}(y).
\end{align*}
Note that $f_v=\argmin_{f\in\mathcal{H}_k}\mathcal{E}(f)$ since $\mathbb{E}_{p_v}\left[Y|X\right]=f_v(x)$ and $f_v\in\mathcal{H}_k$. Also, $f_v$ satisfies Assumption~\ref{assump:source} since $\left\lVert L^{-r}f_v\right\rVert=R$. Finally, $p_v$ is a density of an actual probability measure since
\begin{eqnarray*}
    |f_v(x)|\leq \kappa\lVert f_v\rVert_{\mathcal{H}_k}=\kappa\lVert L^{-1/2}f_v\rVert\leq \kappa\lVert L^{r-1/2}\rVert\lVert L^{-r}f_v\rVert\leq \kappa^{2r}R.
\end{eqnarray*}
Here we used the fact that $f(x)=\langle k(x,\cdot),f\rangle_{\mathcal{H}_k}$ if $f\in \mathcal{H}_k$. Thus, $p_v$ is a probability measure on $\mathcal{X}\times\mathcal{Y}$ and $P_v\in\mathcal{P}_1$.
\begin{proposition}[\cite{acharya2021differentially}]\label{prop:priv_assouad} Let $\mathcal{P}$ be a distribution over $\mathcal{X}^n$ and $\mathcal{V} \subset \mathcal{P}$ be a collection of distributions indexed by $\{-1, +1\}^d$. Denote $\theta(P)$ the estimation target corresponding to the distribution $P$. For $u,v\in\{-1,+1\}^d$, suppose the loss function $l$ satisfies
\begin{equation*}
    l(\theta(P_{u}),\theta(P_v))\geq2\tau\sum_{l=1}^d\mathbf{1}(u_l\not=v_l)
\end{equation*}
for some $\tau$ depending on $d$. For each $1\leq l\leq d$, define $p_{\pm l}^{\otimes n}:=\frac{1}{2^{d-1}}\sum_{v_l=\pm 1} p_v$. Suppose there exists $D>0$ such that for every $l$, there exists a coupling $(X^+,X^-)$ between $p_{+l}^{\otimes n}$ and $p_{-l}^{\otimes n}$ satisfying $\mathbb{E}\left[d_{H}(X^+,X^-)\right]\leq D$ where $d_H$ is the Hamming distance. Then the minimax risk is lower bounded as:
\begin{equation*}
    \inf_{M\in\mathcal{M}_{\epsilon,\delta}}\sup_{P\in\mathcal{P}}\mathbb{E}_P[l(M(X_1,\ldots,X_n),\theta(P))]\geq \frac{d\tau}{2}\left(0.9e^{-10\epsilon D}-10D\delta\right).
\end{equation*}
\end{proposition}
We derive the minimax lower bound by applying Proposition~\ref{prop:priv_assouad} for $\mathcal{P}:=\{P_{X,Y}\in\mathcal{P}_1|P_X=\nu\}$, $\mathcal{V}=\{p_v^{\otimes n}\}_{v\in\{-1,+1\}^d}\subset\mathcal{P}$, and $\theta(P):=\argmin_{f\in\mathcal{H}_k}\mathcal{E}_P(f)$. To this end, we first represent the minimax risk $\mathcal{R}(\mathcal{P}_1,\mathcal{E},\epsilon,\delta)$ as follows. Recall that $\mathcal{M}_{\epsilon,\delta}$ is a collection of $(\epsilon,\delta)$-DP algorithm that releases a prediction function $\widetilde{f}\in\mathcal{H}_k$.
\begin{align*}
    \mathcal{R}(\mathcal{P}_1,\mathcal{E},\epsilon,\delta)=&\inf_{M\in\mathcal{M}_{\epsilon,\delta}}\sup_{P\in\mathcal{P}_1}\left(\mathbb{E}_P[\mathcal{E}(M(X_1,\ldots,X_n))]-\min_{f\in\mathcal{H}_k}\mathcal{E}(f)\right)\\
    \geq&\inf_{M\in\mathcal{M}_{\epsilon,\delta}}\sup_{P\in\mathcal{P}}\left(\mathbb{E}_P[\mathcal{E}(M(X_1,\ldots,X_n))]-\min_{f\in\mathcal{H}_k}\mathcal{E}(f)\right)\\
    =&\inf_{M\in\mathcal{M}_{\epsilon,\delta}}\sup_{P\in\mathcal{P}}\left(\mathbb{E}_P[(Y-M(X_1,\ldots,X_n))^2]-\mathbb{E}_P[(Y-f_{\mathcal{H}_k}(X))^2]\right)\\
    =&\inf_{M\in\mathcal{M}_{\epsilon,\delta}}\sup_{P\in\mathcal{P}}\Big(\mathbb{E}_P[(f_{\mathcal{H}_k}(X)-M(X_1,\ldots,X_n))^2]\\
    &+2\mathbb{E}_P[(f_{\mathcal{H}_k}(X)-M(X_1,\ldots,X_n))(Y-f_{\mathcal{H}_k}(X))]\Big)\\
    =&\inf_{M\in\mathcal{M}_{\epsilon,\delta}}\sup_{P\in\mathcal{P}}\left(\mathbb{E}_\nu[(f_{\mathcal{H}_k}(X)-M(X_1,\ldots,X_n))^2]\right),
\end{align*}
where the last equality holds since $\mathbb{E}[(Y-f_{\mathcal{H}_k}(X))g(X)]=0$ for all $g\in\mathcal{H}_k$. Thus, we can apply Proposition~\ref{prop:priv_assouad} for the loss function defined as $l(\theta(P),\theta(P^\prime)):=\lVert \theta(P)-\theta(P^\prime)\rVert_{L^2(\nu)}^2$ for $P,P^\prime\in\mathcal{P}$. Next, we verify the $\tau$ and $D$ that satisfy the condition of the proposition.

\textbf{1. The condition regarding the loss function.} Since
\begin{align*}
l(\theta(P_u),\theta(P_v))=&l(f_u,f_v)\\
=&\left\lVert f_u-f_v\right\rVert^2_{L^2(\nu)}\\
=&\sum_{l=1}^{d}\frac{4R^2\mu_{l+d}^{2r}}{d}\mathbf{1}(u_l\not= v_l)\\
\geq&\frac{4R^2\mu_{2d}^{2r}}{d}\sum_{l=1}^{d}\mathbf{1}(u_l\not= v_l),
\end{align*}
the condition holds for 
$\tau=\frac{2R^2\mu_{2d}^{2r}}{d}$.

\textbf{2. Existence of the coupling for each $1\leq l\leq d$.} We construct the coupling that actually satisfies the aforementioned condition. For a given $l$, let $X^+:=(X_1^+,Y_1^+,\ldots,X_n^+,Y_n^+)$ be a random vector such that $(X_i^+,Y_i^+)|V^+=v\sim p_v$ and $V^+\sim\mathrm{Unif}(\{v\in\{-1,+1\}^d:v_l=1\})$. Then $(X^+,Y^+)\sim p^{\otimes n}_{\pm l}$ where
\begin{equation*}
    p_{\pm l}^{\otimes n}:=\frac{1}{2^{d-1}}\sum_{v_l=\pm1}p_v^{\otimes n}.
\end{equation*}
Next, we construct $(X^-,V^-):=(X_1^-,Y_1^-,\ldots,X_n^-,Y_n^-,V^-)$ for a given $(X^+,V^+)$.  For each $1\leq i\leq n$, set $(X_i^-,Y_i^-)$ according to the following cases: For convenience, for $v\in\{-1,+1\}^{d}$, denote $v^{-l}$ to be the element of $\{-1,+1\}^{d}$ such that $[v^{-l}]_j=v_j$ if $j\not=l$ and $[v^{-l}]_l=-v_l$. For example, if $v=(1,1,-1,1)$ then $v^{3-}$ is $(1,1,1,1)$.
\begin{itemize}
    \item[Case 1.] Given $X_i^+=x$ and $V^+=v$, if $\phi_{l+d}(x)\geq 0$, set $V^-=v^{-l}$ and
    \begin{align*}
        \mathbb{P}\left(X_i^-=x|X_i^+=x,V^+=v\right)=&1\\\mathbb{P}\left(Y_i^{-}=-1|Y_i^+=-1,X_i^+=x,V^+=v\right)=&1\\
        \mathbb{P}\left(Y_i^{-}=-1|Y_i^+=1,X_i^+=x,V^+=v\right)=&\frac{2}{2\kappa^{2r}R+f_v(x)}\frac{R\mu_{l+d}^r}{\sqrt{d}}v_l\phi_{l+d}(x).
    \end{align*}\item[Case 2.] Given $X_i^+=x$ and $V=v$, if $\phi_{l+d}(x)< 0$, set $V^-=v^{-l}$ and
    \begin{align*}
        \mathbb{P}\left(X_i^-=x|X_i^+=x,V^+=v\right)=&1\\\mathbb{P}\left(Y_i^{-}=1|Y_i^+=1,X_i^+=x,V^+=v\right)=&1\\
        \mathbb{P}\left(Y_i^{-}=1|Y_i^+=-1,X_i^+=x,V^+=v\right)=&-\frac{2}{2\kappa^{2r}R-f_v(x)}\frac{R\mu_{l+d}^r}{\sqrt{d}}v_l\phi_{l+d}(x).
    \end{align*}
\end{itemize}
We note that the above construction is valid. That is, the probability of changing the sign is indeed less than or equal to 1. Since $|\phi_{l+d}(x)|\leq\kappa\lVert\phi_{l+d}\rVert_{\mathcal{H}_k}\leq\kappa \mu_{l+d}^{-1/2}$ and $\mu_l\leq\kappa^2$, we have
\begin{align*}
    \frac{2}{2\kappa^{2r}R-f_v(x)}\frac{R\mu_{l+d}^r}{\sqrt{d}}|\phi_{l+d}(x)|\leq&\frac{2}{\kappa^{2r}R}\frac{R\mu_{l+d}^r}{\sqrt{d}}|\phi_{l+d}(x)|\\
    \leq&\frac{2}{\kappa^{2r-1}R}\frac{R\mu_d^{r-1/2}}{\sqrt{d}}\\
    \leq&\frac{2}{\sqrt{d}}\\
    \leq &1,
\end{align*}
where the last inequality holds by choosing $d\geq4$. Next, we show that $X^-|V^-=v^{-l}\sim p_{v^{-l}}^{\otimes n}$ for $v\in\{-1,+1\}^d$ with $v_l=1$. 
\begin{itemize}
    \item[Case 1.] If $\phi_{l+d}(x)\geq 0$ then
    \begin{align*}
        \mathbb{P}\left(Y_i^{-}=-1|X_i^{-}=x,V^-=v^{-l}\right)=&\mathbb{P}\left(Y_i^{+}=-1|X_i^{-}=x,V^+=v\right)\\&+\mathbb{P}\left(Y_i^-=-1|Y_i^+=1\right)\mathbb{P}\left(Y_i^{+}=1|X_i^{-}=x,V^+=v\right)\\
        =&\frac{2\kappa^{2r}R-f_v(x)}{4\kappa^{2r}R}+\frac{1}{4\kappa^{2r}R}\left(\frac{2R\mu_{l+d}^r}{\sqrt{d}}\phi_{l+d}(x)\right)\\
        =&\frac{2\kappa^{2r}R-f_{v^{-l}}(x)}{4\kappa^{2r}R}.
    \end{align*}
    \item[Case 2.] If $\phi_{l+d}(x)< 0$ then
    \begin{align*}
        \mathbb{P}\left(Y_i^{-}=1|X_i^{-}=x,V^-=v^{-l}\right)=&\mathbb{P}\left(Y_i^{+}=1|X_i^{-}=x,V^+=v\right)\\&+\mathbb{P}\left(Y_i^-=1|Y_i^+=-1\right)\mathbb{P}\left(Y_i^{+}=-1|X_i^{-}=x,V^+=v\right)\\
        =&\frac{2\kappa^{2r}R+f_v(x)}{4\kappa^{2r}R}+\frac{1}{4\kappa^{2r}R}\left(-\frac{2R\mu_{l+d}^r}{\sqrt{d}}\phi_{l+d}(x)\right)\\
        =&\frac{2\kappa^{2r}R+f_{v^{-l}}(x)}{4\kappa^{2r}R}.
    \end{align*}
\end{itemize}
In any case, we have $(Y_i^-,X_i^-)|V^-=v^{-l}\sim p_{v^{-l}}$. Also, $\{(Y_j^-,X_j^-)|V^-\}_{j=1}^n$ are i.i.d. by construction. Since $(Y_j^-,X_j^-)|V^-=v^{-l}\sim p_{v^{-l}}$ and $V^-\sim\mathrm{Unif}(\{v\in\{-1,+1\}^d:v_l=-1\})$, the marginal distribution of $X^-=((X_1^-,Y^-_1),\ldots,(X_n^-,Y^-_n))$ is
\begin{equation*}
    \frac{1}{2^{d-1}}\sum_{v_l=1}p_{v^{-l}}^{\otimes n}=\frac{1}{2^{d-1}}\sum_{v_l=-1}p_{v}^{\otimes n}=p_{-l}^{\otimes n}.
\end{equation*}
Thus, $X^+$ and $X^-$ are couplings of $p_{+l}^{\otimes n}$ and $p_{-l}^{\otimes n}$.

The expected Hamming distance between $X^+$ and $X^-$ is
\begin{align*}
&\mathbb{E}\left[d_H(X^+,X^-)\right]\\
=&\sum_{i=1}^n\mathbb{E}[\mathbf{1}(X^+_i\not=X^-_i)]+\mathbb{E}[\mathbf{1}(Y^+_i\not=Y^-_i)]\\
=&\sum_{i=1}^n\mathbb{E}[\mathbf{1}(Y^+_i\not=Y^-_i)]\\
=&\sum_{i=1}^n\mathbb{E}[\mathbb{E}[\mathbf{1}(Y^+_i\not=Y^-_i)|X_i^+=x,V^+=v]]\\
=&\sum_{i=1}^n\mathbb{E}_{X^+,V^+}\Big[\mathbb{P}(Y^-_i=-1|Y^+_i=1,X_i^+,V^+)\mathbb{P}(Y^+_i=1|X_i^+,V^+)\mathbf{1}(\phi_{l+d}(X_i^+)\geq0)\\
&+\mathbb{P}(Y^-_i=1|Y^+_i=-1,X_i^+,V^+)\mathbb{P}(Y^+_i=-1|X_i^+,V^+)\mathbf{1}(\phi_{l+d}(X_i^+)<0)\Big]\\
=&\sum_{i=1}^n\mathbb{E}_{X^+,V^+}\Bigg[\frac{R\mu_{l+d}^r}{2\kappa^{2r}R\sqrt{d}}|\phi_{l+d}(X_i^+)|\Bigg]\\
\leq&\sum_{i=1}^n\frac{\sqrt{c_3}R\mu_{l+d}^{r+\beta/2}}{2\kappa^{2r}R\sqrt{d}}\\
\leq&\frac{\sqrt{c_3}nR\mu_d^{r+\beta/2}}{2\kappa^{2r}R\sqrt{d}}
\end{align*}
Thus, the coupling condition holds for $D=\frac{\sqrt{c_3}nR\mu_d^{r+\beta/2}}{2\kappa^{2r}R\sqrt{d}}$.

Finally, we apply Proposition~\ref{prop:priv_assouad} to derive the minimax lower bound.
\begin{align*}
&\mathfrak{R}(\mathcal{P},\mathcal{E},\epsilon,\delta)\\
     \geq &R^2\mu_{2d}^{2r}\left(0.9\mathrm{exp}\left(-\frac{10\sqrt{c_3}(n\epsilon) R\mu_d^{r+\beta/2}}{2\kappa^{2r}R\sqrt{d}}\right)-\frac{10\sqrt{c_3}(n\delta) R\mu_d^{r+\beta/2}}{2\kappa^{2r}R\sqrt{d}}\right)\\
     \geq&R^2a(2d)^{-2r/\gamma}\left(0.9\mathrm{exp}\left(-\frac{10\sqrt{c_3}(n\epsilon) R\mu_d^{r+\beta/2}}{2\kappa^{2r}R\sqrt{d}}\right)-\frac{10\sqrt{c_3}(n\delta) R\mu_d^{r+\beta/2}}{2\kappa^{2r}R\sqrt{d}}\right)
\end{align*}
Note that $\frac{(n\epsilon)\mu_d^{r+\beta/2}}{\sqrt{d}}=\Omega((n\epsilon)d^{-\frac{r+(\gamma+\beta)/2}{\gamma}})$ and $\frac{(n\delta)\mu_d^{r+\beta/2}}{\sqrt{d}}=O((n\delta)d^{-\frac{r+(\gamma+\beta)/2}{\gamma}})=o(d^{-\frac{r+(\gamma+\beta)/2}{\gamma}})$. Setting $d=O((n\epsilon)^{\frac{2\gamma}{2r+\gamma+\beta}}\vee1)$, we obtain the minimax risk lower bound
\begin{align*}
    \mathfrak{R}(\mathcal{P},\mathcal{E},\epsilon,\delta)=\Omega((n\epsilon)^{-\frac{4r}{2r+\gamma+\beta}}\wedge1).
\end{align*}
Combining this with the non-private minimax risk, we arrive at the desired lower bound:
\begin{equation*}
    \Omega\left(n^{-\frac{2r}{2r+\gamma}}+(n\epsilon)^{-\frac{4r}{\gamma+2r+\beta}}\wedge1\right).
\end{equation*}
Now, we construct an example of kernel and a data distribution satisfying Assumption~\ref{assump:cap},~\ref{assump:kernel}, and~\ref{assump:fluct} for arbitrary $\beta\in(\gamma,1)$ based on the bump function. For simplicity we assume $\mathcal{X}$ as a unit interval. However, the argument can be extended to any bounded subset of $\mathbb{R}^d$. The bump function on a unit interval is defined as follows:
\begin{align*}
    \psi(x):=\frac{\mathrm{exp}\left(-\frac{1}{1-x^2}\right)\mathbf{1}(|x|<1)}{\sqrt{\int_{-1}^1\mathrm{exp}\left(-\frac{2}{1-x^2}\right)dx}}.
\end{align*}
For $l\geq1$ define
\begin{align*}
    \mu_l:=&al^{-1/\gamma},\\
    s_l:=&\frac{\sum_{j<l}j^{-\beta/\gamma}}{s_\infty},&\\
    A_l:=&\left[s_l,s_{l+1}\right],\\
    \phi_l(x)=&\sqrt{2s_\infty}l^{\beta/2\gamma} \psi\left(2s_\infty l^{\beta/\gamma}\left(x-s_l\right)-1\right),
\end{align*}
where $a:=2s_\infty\kappa^2$ and $s_\infty:=\sum_{j=1}^\infty j^{-\beta/\gamma}$. By the definition, $\phi_l$ is a smooth function supported on $A_l$. Now define the marginal distribution $\nu$ and the kernel $k$ as
\begin{align*}
   \nu\sim\mathrm{Unif}[0,1],\\
    k(x,y)=&\sum_l \mu_l\phi_l(x)\phi_l(y).
\end{align*}
$k(x,y)$ is well-defined since the supports of $\phi_l$s are all distinct and $\sqrt{\mu_l}\phi_l$s are bounded. Also, by the definition of $a$, $\nu$ is indeed a probability measure. Simple calculation shows that $\langle\phi_l,\phi_{l^\prime}\rangle_{L^2(\nu)}=1$ if $l=l^\prime$ and 0 otherwise. Thus, $\phi_l$ and $\mu_l$ are the eigenfunction and the eigenvalue of $k$. Then Assumptions~\ref{assump:cap} and~\ref{assump:fluct} are verified since $\mu_l\asymp l^{-1/\gamma}$ and $\lVert\phi_l\rVert_{L^1(\nu)}\asymp l^{-\beta/2\gamma}$. Finally, we verify Assumption~\ref{assump:kernel}. The boundedness of the kernel is evident from its construction. Next, we show the sub-Gaussianity of $\sup_x|h(x,\omega)|$ where $h\sim\mathcal{GP}(0,k)$. We show $h$ is almost surely bounded. That is, $\mathbb{P}(\sup_x h(x,\omega)<\infty)=1$. We use the following fact.
\begin{lemma}\label{lem:rank1_gp}
    For a function $c:\mathcal{X}\rightarrow\mathbb{R}$, let $h(\cdot,\omega)\sim\mathcal{GP}(0,c(x)c(y))$. If $\sup_{x,y} |c(x)-c(y)|,\sup_xc^2(x)<\infty$ then
    \begin{align*}
        \mathbb{P}\left(\sup_x h(x,\omega)>\sqrt{2}K\sup_{x,y}|c(x)-c(y)|+\sqrt{2\sup_xc^2(x)\log\frac{1}{t}}\right)\leq t.
    \end{align*}
\end{lemma}
\begin{proof}
We use the following proposition.
\begin{proposition}\label{prop:gaussian_process_mean}[Theorem 1.3.3 in \citet{adler2007gaussian}] There exists a universal constant $K$ such that
    \begin{equation*}
        \mathbb{E}_\omega \left[\sup_{x\in\mathcal{X}}|h(x,\omega)|\right]\leq K\int_0^{\mathrm{diam}(\mathcal{X})/2}\sqrt{\log \mathcal{N}(\mathcal{X},d,s)}ds
    \end{equation*}
    where $\mathcal{N}(\mathcal{X},d,s)$ denotes the smallest number of $d$-balls of radius $s$ required to cover $\mathcal{X}$; the metric $d$ is defined as $d(x_1,x_2):=\sqrt{\mathbb{E}_\omega \left[(h(w,x_1)-h(w,x_2))^2\right]}=\left\lVert k(x_1,\cdot)-k(x_2,\cdot)\right\rVert_{\mathcal{H}_k}$, and $\mathrm{diam}(\mathcal{X}):=\sup_{x_1,x_2\in\mathcal{X}}d(x_1,x_2)$.
\end{proposition}
From Proposition~\ref{prop:gaussian_process_mean}, we have
    \begin{align*}
        \mathbb{E}_\omega\left[\sup_x|h(x,\omega)|\right]\leq& K\int_0^{\mathrm{diam}(\mathcal{X})/2}\sqrt{\log\mathcal{N}(\mathcal{X},d,s)}ds\\
        =&K\int_0^{\mathrm{diam}(\mathcal{X})/2}\sqrt{\log\mathcal{N}(c(\mathcal{X}),|\cdot|,s)}ds\\
        \leq&K\int_0^{\sup_{x,y}|c(x)-c(y)|/2}\sqrt{\log\left(1+\frac{\sup_{x,y}|c(x)-c(y)|}{s}\right)}ds\\
        =&K\int_2^\infty\frac{\sup_{x,y}|c(x)-c(y)|}{u^2}\sqrt{\log(1+u)}du\quad\textrm{ ($u=\sup_{x,y}|c(x)-c(y)|/s$)}\\
        \leq&K\int_2^\infty\frac{\sup_{x,y}|c(x)-c(y)|}{u^2}\sqrt{u}du\\
        =&\sqrt{2}K\sup_{x,y}|c(x)-c(y)|,
    \end{align*}
    where the second inequality holds since $d(x,y)=\sqrt{k(x,x)-2k(x,y)+k(y,y)}=|c(x)-c(y)|$. Applying Proposition~\ref{prop:borel_tis}, we obtain the desired.
\end{proof}
Note that $h(x,\omega)\sim\mathcal{GP}(0,\mu_l\phi_l(x)\phi_l(y))$ on $A_l$ and
\begin{align*}
    |\mu_l^{1/2}\phi_l(x)-\mu_l^{1/2}\phi_l(y)|\leq&\sqrt{2s_\infty a}l^{-(1-\beta)/2\gamma},\\
    \mu_l\phi_l(x)\phi_l(y)\leq&2s_\infty al^{-(1-\beta)/\gamma}.
\end{align*}
By Lemma~\ref{lem:rank1_gp}, we have
\begin{align*}
    &\mathbb{P}\left(\sup_x h(x,\omega)\geq \sqrt{4s_\infty a}K+\sqrt{4s_\infty a\left(\log \frac{1}{p}+\frac{\gamma}{e(1-\beta)}\right)}\right)\\
    =&\mathbb{P}\left(\sup_l\sup_{x\in A_l} h(x,\omega)\geq \sqrt{4s_\infty a}K+\sqrt{4s_\infty a\left(\log \frac{1}{p}+\frac{\gamma}{e(1-\beta)}\right)}\right)\\
    \leq&\sum_l\mathbb{P}\left(\sup_{x\in A_l} h(x,\omega)\geq \sqrt{4s_\infty a}K+\sqrt{4s_\infty a\left(\log \frac{1}{p}+\frac{\gamma}{e(1-\beta)}\right)}\right)\\
    \leq&\sum_l\mathbb{P}\left(\sup_{x\in A_l} h(x,\omega)\geq \sqrt{4s_\infty a}Kl^{-(1-\beta)/2\gamma}+ \sqrt{4s_\infty al^{-(1-\beta)/\gamma}\left(\log \frac{1}{p}+\frac{\gamma}{e(1-\beta)}\right)}\right)\\
    \leq&\sum_l\frac{p}{l^2},
\end{align*}
where the second inequality holds since $l^{-\frac{1-\beta}{\gamma}}\log l\leq \frac{\gamma}{e(1-\beta)}$. Thus,
\begin{align*}
    \mathbb{P}(\sup_x h(x,\omega)<\infty)\geq&1-\mathbb{P}\left(\sup_x h(x,\omega)\geq \sqrt{4s_\infty a}K+\sqrt{4s_\infty a\left(\log \frac{1}{p}+\frac{\gamma}{e(1-\beta)}\right)}\right)\\
    \geq&1-\sum_l\frac{p}{l^2},
\end{align*}
for all $p>0$. Taking $p\rightarrow0^+$, we obtain $\mathbb{P}(\sup_x h(x,\omega)<\infty)=1$.

\subsection{Proof of Theorem~\ref{thm:dprp_cls_smooth}}\label{sec:proof_dprp_cls_smooth}
As in the proof of Theorem~\ref{thm:dprp_reg_rate}, we prove the theorem under Assumption~\ref{assump:cap_real}.

We assume $\lambda\geq\frac{4c_2\Delta_3^2}{n\epsilon}\geq \frac{c_2\Delta_3^2}{n}\left(e^{\epsilon/4}-1\right)^{-1}$. Throughout the proof, we show that the excess risk bound in Theorem~\ref{thm:dprp_cls_smooth} can still be attained under this choice of $\lambda$.

We combine the approaches of \citet{li2019towards} and \citet{liu2025improved}. However, the underlying theoretical settings differ. For example, we assume smoothness of the loss function, which enables more efficient dimension reduction. For instance, the $\sqrt{\lambda}$ dependence in Theorem 15 of \citet{Li2021unified} is replaced by $\lambda$ in our setting.

For $\lambda>0$, let $\mathcal{G}_\lambda$ be a map that takes a self-adjoint operator with eigenvalues $\{\xi_l\}_{l=1}^\infty$ to a similar operator with eigenvalues $\{\mathbf{1}(\xi_l\geq\lambda)/\xi_l\}_{l=1}^\infty$.

Denote the non-private random projection estimator as
\begin{eqnarray*}
    \hat{\beta}_{M,\lambda}:=\argmin_{f\in\mathcal{H}_M}\frac{1}{n}\sum_{i=1}^nl(y_i,f(x_i))+\frac{\lambda}{2}\left\lVert f\right\rVert_{\mathcal{H}_M}^2.
\end{eqnarray*}
Denote $w:=\mathcal{G}_\lambda(C_M)S_M^\top f_{\mathcal{H}_k}$.

We decompose the excess risk into three terms: \begin{align}
        \mathcal{E}(S_M\widetilde{\beta})-\mathcal{E}(f_{\mathcal{H}_k})=&\left[\mathcal{E}(S_M\widetilde{\beta})-\mathcal{E}(S_M\hat{\beta}_{M,\lambda})\right]+\left[\mathcal{E}(S_M\hat{\beta}_{M,\lambda})-\mathcal{E}(S_M w)\right]\notag\\
        &+\left[\mathcal{E}(S_M w)-\mathcal{E}(f_{\mathcal{H}_k})\right].
\label{eq:err_smooth}
\end{align}
Each term can be bounded by the following lemma:
\begin{lemma}\label{lem:last_term} If $\omega_{1:M}\in E_{half}\cap E_{full}$, $M\geq \max\{(1024\kappa^2 +4C^2+2C)N(\lambda)\lambda^{-(2r-1)}\log \frac{2}{t},1\}$, and  $\lambda\leq 1$ then the following bound holds:
\begin{equation}\label{eq:err_smooth_first}
    \mathcal{E}(S_M w)-\mathcal{E}(f_{\mathcal{H}_k})\leq 2^{2r}c_2\lambda^{2r}R^2.
\end{equation}
\end{lemma}
\begin{lemma}\label{lem:err_smooth_second_third2}
Denote
\begin{align}
    r_{non-priv}:=&\frac{4c_1H\log\frac{2}{t}}{n\lambda}+\frac{2c_1H}{\lambda}\sqrt{\frac{2\log\frac{2}{t}}{n}}+\sqrt{2^{2r+1}\left(\frac{c_2}{2}\lambda^{2r-1}+5^{2r}\kappa^{4r-2}\right)\lambda }R\label{eq:smooth_R_non_priv},\\
    r_{priv}:=&r_{non-priv}+\frac{\Delta_3\left(\sqrt{M}+\sqrt{8\log\frac{2}{t}}\right)}{(n\epsilon)\lambda}\label{eq:smooth_r_priv},\\
    E_{regul}:=&\Bigg\{(X_{1:n},Y_{1:n}):\left\lVert \beta_{M,\lambda}-\hat{\beta}_{M,\lambda}\right\rVert_{2}\leq\frac{4c_1H\log\frac{2}{t}}{n\lambda}+\frac{2c_1H}{\lambda}\sqrt{\frac{2\log\frac{2}{t}}{n}}\Bigg\}\notag,\\
    E_b:=&\left\{\left\lVert b\right\rVert_2\leq \Delta_3\left(\frac{\sqrt{M}}{\epsilon}+\frac{\sqrt{8\log\frac{2}{t}}}{\epsilon}\right)\right\}\notag,\\
    E_{priv}:=&\Bigg\{(X_{1:n},Y_{1:n}):\sup_{\left\lVert f\right\rVert_{\mathcal{H}_M}\leq r_{priv}}(\mathcal{E}-\mathcal{E}_n)(f)\leq \frac{c_1Hr_{priv}}{\sqrt{n}}+\sqrt{\frac{2\log\frac{4}{t}}{n}}\Bigg\}\notag.
\end{align}
If $\omega_1,\ldots,\omega_M\in E_{half}\cap E_{full}\cap E_{cov}$, $b\in E_b$, $(X_{1:n},Y_{1:n})\in E_{regul}\cap E_{priv}$, $\lambda\leq \min\{1,\lVert L\rVert\}$ and $M\geq \max\{(1024\kappa^2 +4C^2+2C)N(\lambda)\lambda^{-(2r-1)}\log \frac{2}{t},1\}$ then the following bound holds:
\begin{equation}\label{eq:err_smooth_second_third2_final}
    \mathcal{E}(S_M\hat{\beta}_{M,\lambda})-\mathcal{E}(S_M w)\leq 2^25^{2r}\lambda\kappa^{4r-2}R^2+\frac{2c_1 Hr_{priv}}{\sqrt{n}}+2\sqrt{\frac{2\log\frac{4}{t}}{n}}.
\end{equation}
\end{lemma}
\begin{lemma}\label{lem:err_smooth_second_third1}
    If $\omega_1,\ldots,\omega_M\in E_{half}\cap E_{full}\cap E_{cov}$, $b\in E_b$, and $(X_{1:n},Y_{1:n})\in E_{regul}\cap E_{priv}$ then the following bound holds: \begin{equation}\label{eq:err_smooth_second_third1_final}
    \mathcal{E}(S_M\widetilde{\beta})-\mathcal{E}(S_M\hat{\beta}_{M,\lambda})\leq \frac{\lambda}{2}r_{non-priv}^2+\frac{2\Delta_3^2\left(M+8\log\frac{2}{t}\right)}{(n\epsilon)^2\lambda}+\frac{2c_1Hr_{priv}}{\sqrt{n}}+2\sqrt{\frac{2\log\frac{4}{t}}{n}}.
\end{equation}
\end{lemma}
Combining \eqref{eq:err_smooth}, \eqref{eq:err_smooth_first}, \eqref{eq:err_smooth_second_third2_final}, and \eqref{eq:err_smooth_second_third1_final}, yields the excess risk bound:
\begin{align*}
    &\mathcal{E}(S_M\widetilde{\beta})-\mathcal{E}(f_{\mathcal{H}_k})\\
     \leq &2^{2r}c_2\lambda^{2r}R^2+2^25^{2r}\kappa^{4r-2}\lambda R^2+\frac{\lambda}{2}r_{non-priv}^2+\frac{2\Delta_3(M+8\log\frac{2}{t})}{(n\epsilon)^2\lambda}+\frac{4c_1Hr_{priv}}{\sqrt{n}}+4\sqrt{\frac{2\log\frac{4}{t}}{n}},
\end{align*}
if $\omega_1,\ldots,\omega_M\in E_{half}\cap E_{full}\cap E_{cov}$, $b\in E_b$, $(X_{1:n},Y_{1:n})\in E_{regul}\cap E_{priv}$, $\lambda\leq \min\{1,\lVert L\rVert\}$ and $M\geq \max\{(1024\kappa^2 +4C^2+2C)N(\lambda)\lambda^{-(2r-1)}\log \frac{2}{t},1\}$.

Note that if
\begin{align}
    \max\left\{\frac{4c_1H}{\sqrt{2^{2r-1}\left(\frac{c_2}{2}+5^{2r}\kappa^{4r-2}\right)}R}\left(\frac{1}{n}+\sqrt{\frac{1}{2n}}\right),\frac{\Delta_3\left(\sqrt{M/\log\frac{2}{t}}+\sqrt{8}\right)}{\sqrt{2^{5-2r}\left(\frac{c_2}{2}+5^{2r}\kappa^{4r-2}\right)}R}\frac{1}{n\epsilon}\right\}\leq\lambda,\label{eq:thecondition}
\end{align}
then $r_{priv}\leq 3\sqrt{2^{2r+1}\left(\frac{c_2}{2}+5^{2r}\kappa^{4r-2}\right)}R\log\frac{2}{t}$. Denote $\widetilde{R}:=\sqrt{2^{2r+1}\left(\frac{c_2}{2}+5^{2r}\kappa^{4r-2}\right)}R$. Thus, choosing $\lambda=O(n^{-\frac{1}{2}}\vee (n\epsilon)^{-\frac{2}{2r+\gamma+1}})$ and $M=O(N(\lambda)\lambda^{-(2r-1)}\log\frac{1}{t})=O(\lambda^{-\gamma-(2r-1)}\log\frac{1}{t})$ ensures that condition~\eqref{eq:thecondition} is satisfied. Consequently, we have $r_{priv}=O\left(\widetilde{R}\log\frac{1}{t}\right)$, and the resulting error is bounded by
\begin{eqnarray}\label{eq:smooth_general_error}
    O\left(\lambda\log^2\frac{1}{t}+\frac{M}{(n\epsilon)^2\lambda}+\frac{\sqrt{\log\frac{1}{t}}}{\sqrt{n}}\right),
\end{eqnarray}
which simplifies to $O((n^{-\frac{1}{2}}+ (n\epsilon)^{-\frac{2}{2r+\gamma+1}})\log^2\frac{1}{t})$ if $\omega_1,\ldots,\omega_M\in E_{half}\cap E_{full}\cap E_{cov}$, $b\in E_b$, and $(X_{1:n},Y_{1:n})\in E_{regul}\cap E_{priv}$. 

Finally, we establish that the event that the bound holds occurs with high probability by applying the lemmas stated below:
\begin{lemma}\label{lem:empi_pop_difference}
For $(\omega_1,\ldots,\omega_M)\in E_g$ and $t\in (0,1]$ the following inequality holds:
    \begin{eqnarray}\label{eq:regul}
    \mathbb{P}_{(X_{1:n},Y_{1:n})|\omega_{1:M}}\left(E_{regul}\right)\geq 1-t.
\end{eqnarray}
\end{lemma}
\begin{lemma}\label{lem:rademacher_rkhs}
For $\omega_1,\ldots,\omega_M\in E_g$ and $t\in (0,1]$ the following inequality holds:
\begin{eqnarray*}
    \mathbb{P}_{(X_{1:n},Y_{1:n})|\omega_{1:M}}\left(\sup_{\left\lVert f\right\rVert_{\mathcal{H}_M}\leq r}|(\mathcal{E}-\mathcal{E}_n)(f)|\leq \frac{rc_1H}{\sqrt{n}}+\sqrt{\frac{2\log\frac{4}{t}}{n}}\right)\geq 1-t.
\end{eqnarray*}
\end{lemma}
\begin{lemma}\label{lem:b} The event $E_b$ occurs with probability at least $1-t$.
\end{lemma}
Combining Lemmas~\ref{lem:half_operator}, \ref{lem:full_operator}, \ref{lem:constants}, \ref{lem:rf_dim}, \ref{lem:priv_bound}, \ref{lem:b}, and \ref{lem:empi_pop_difference}, we establish that the event occurs with high probability:
\begin{align*}
    \mathbb{P}\left(\omega_{1:M}\in E_g\cap  E_{half}\cap E_{full}\cap E_{cov}\cap E_{rf-dim}, (X_{1:n},Y_{1:n})\in E_{regul},b\in E_b\right)\geq 1-7t.
\end{align*}
Thus, the excess risk bound holds with high probability. Setting $r=1/2$, we obtain the desired rate.
\subsubsection{Proof of Lemma~\ref{lem:last_term}}
The smooth assumption in Assumption \ref{assump:lip} and Lemma~\ref{lem:zero_grad} yields:
\begin{align*}
    \mathcal{E}(S_M w)-\mathcal{E}(f_{\mathcal{H}_k}) \leq &\mathbb{E}\left[l_{\hat{y}}(Y,f_{\mathcal{H}_k}(X))(S_M w(X)-f_{\mathcal{H}_k}(X))\right]+\frac{c_2}{2}\lVert S_M w-f_{\mathcal{H}_k}\rVert^2\\
    =&\frac{c_2}{2}\lVert S_M w-f_{\mathcal{H}_k}\rVert^2.
\end{align*}
Applying Lemmas~\ref{lem:liu_prop1} and \ref{lem:liu_lem6}, we obtain the following:
\begin{align*}
\frac{c_2}{2}\lVert S_M w-f_{\mathcal{H}_k}\rVert^2 \leq &\frac{c_2}{2}\lambda^{2r}\left\lVert L_{M,\lambda}^{-r}L^r\right\rVert^2R^2\\
 \leq &\frac{c_2}{2}\lambda^{2r}\left(1-\lambda^{-r+\frac{1}{2}}\left\lVert L_\lambda^{-\frac{1}{2}}(L_M-L)\right\rVert^{2r-1}\left\lVert L_\lambda^{-\frac{1}{2}}(L_M-L)L_\lambda^{-\frac{1}{2}}\right\rVert^{2-2r}\right)^{-2}\\
&\cdot\left(1-\left\lVert L_\lambda^{-\frac{1}{2}}(L_M-L)L_\lambda^{-\frac{1}{2}}\right\rVert\right)^{2-2r}R^2.
\end{align*}
Since $\omega_1,\ldots,\omega_M\in E_{half}$ and $M\geq \max\left\{(1024\kappa^2+4C^2+2C) N(\lambda)\lambda^{-(2r-1)}\log \frac{2}{t},1\right\}$, we can bound the $\left\lVert L_{\lambda}^{-\frac{1}{2}}(L_M-L)\right\rVert$ as below:
\begin{align*}
    \left\lVert L_{\lambda}^{-\frac{1}{2}}(L_M-L)\right\rVert \leq &\frac{2\kappa\sqrt{32N(\lambda)\log\frac{2}{t}}}{M}+\sqrt{\frac{64\kappa^2N(\lambda)\log\frac{2}{t}}{M}}\\
     \leq &\frac{2\kappa\sqrt{32N(\lambda)\log\frac{2}{t}}}{M}+\sqrt{\frac{1}{16}}\lambda^{r-\frac{1}{2}}\\
     \leq &\frac{1}{4}\lambda^{2r-1}+\frac{1}{4}\lambda^{r-\frac{1}{2}}\\
    <&\frac{1}{2}\lambda^{r-\frac{1}{2}}.
\end{align*}
Also, since $\omega_1,\ldots,\omega_M\in E_{full}$ and $M\geq (4C^2+4C)N(\lambda)\lambda^{-(2r-1)}\log\frac{2}{t}$, we can bound the term $\left\lVert L_\lambda^{-\frac{1}{2}}\left(L_M-L\right)L_\lambda^{-\frac{1}{2}}\right\rVert$ as below:
\begin{align*}
    \left\lVert L_\lambda^{-\frac{1}{2}}\left(L_M-L\right)L_\lambda^{-\frac{1}{2}}\right\rVert \leq & C\max\left(\sqrt{\frac{N(\lambda)\log\frac{1}{t}}{M}},\frac{N(\lambda)\log\frac{1}{t}}{M}\right)\\
     \leq & \frac{1}{2}\lambda^{r-\frac{1}{2}}
\end{align*}
Therefore, $\lVert L_{M,\lambda}^{-r}L_\lambda^r\rVert^2\leq 2^{2r+1}$ and
\begin{eqnarray*}
    \mathcal{E}(S_M w)-\mathcal{E}(f_{\mathcal{H}_k})\leq c_22^{2r}\lambda^{2r}R^2.
\end{eqnarray*}
\subsubsection{Proof of Lemma~\ref{lem:err_smooth_second_third2}}
By the definition of $\hat{\beta}_{M,\lambda}$, we obtain
\begin{equation}\label{eq:err_smooth_second_third2}
    \mathcal{E}(S_M \hat{\beta}_{M,\lambda})-\mathcal{E}(S_M w)\leq\frac{\lambda}{2}\left\lVert w\right\rVert_2^2+(\mathcal{E}-\mathcal{E}_n)(S_M \hat{\beta}_{M,\lambda})-(\mathcal{E}-\mathcal{E}_n)(S_M w).
\end{equation}
As in the proof of Lemma~\ref{lem:last_term}, applying Lemmas~\ref{lem:liu_prop1}, \ref{lem:liu_lem6}, and \ref{lem:constants} we obtain the bound $\lVert L_{M,\lambda}^{-r}L_\lambda^r\rVert\leq 2^{r+\frac{1}{2}}$. Thus, we bound the first term as below:
\begin{align*}
    \frac{\lambda}{2}\left\lVert w\right\rVert_2^2 \leq & \frac{\lambda}{2}\left(\lVert L_M\rVert+\lambda\right)^{2r}\lVert L_M\rVert^{-1} \lVert L_{M,\lambda}^{-r}L^r\rVert^2 R^2\\
     \leq & 2^{2r+1}\lambda\left(\lVert L_M\rVert+\lambda\right)^{2r}\lVert L_M\rVert^{-1}R^2\\
     \leq &2^{2r+2}\lambda\left(\frac{3}{2}\lVert L\rVert+\lambda\right)^{2r}\lVert L\rVert^{-1}R^2\\
     \leq & 2^{2r+2}\lambda\left(\frac{5}{2}\lVert L\rVert\right)^{2r}\lVert L\rVert^{-1}R^2\\
    =&2^25^{2r}\lambda\lVert L\rVert^{2r-1}R^2\\
     \leq &2^25^{2r}\lambda\kappa^{4r-2}R^2.
\end{align*}
Thus,
\begin{equation}\label{eq:smooth_w}
    \frac{\lambda}{2}\left\lVert w\right\rVert_2^2\leq 2^25^{2r}\lambda\kappa^{4r-2}R^2.
\end{equation}
Next, we bound $\left\lVert w\right\rVert_{2}$ and $\left\lVert \hat{\beta}_{M,\lambda}\right\rVert_{2}$.

Denote
\begin{eqnarray*}
    \beta_{M,\lambda}:=\argmin_{\beta\in\mathbb{R}^M}\mathcal{E}(S_M \beta)+\frac{\lambda}{2}\left\lVert \beta\right\rVert_2^2.
\end{eqnarray*}
By the definition of $\beta_{M,\lambda}$, we obtain
\begin{eqnarray*}
    \mathcal{E}(S_M \beta_{M,\lambda})+\frac{\lambda}{2}\left\lVert \beta_{M,\lambda}\right\rVert_{2}^2\leq \mathcal{E}(S_M w)+\frac{\lambda}{2}\left\lVert w\right\rVert_{2}^2.
\end{eqnarray*}
Then
\begin{align*}
    \frac{\lambda}{2}\left\lVert\beta_{M,\lambda}\right\rVert_2^2 \leq & \mathcal{E}(S_M w)-\mathcal{E}(S_M \beta_{M,\lambda})+\frac{\lambda}{2}\left\lVert w\right\rVert_2^2\\
     \leq & \mathcal{E}(S_M w)-\mathcal{E}(f_{\mathcal{H}_k})+\mathcal{E}(f_{\mathcal{H}_k})-\mathcal{E}(S_M \beta_{M,\lambda})+\frac{\lambda}{2}\left\lVert w\right\rVert_2^2\\
    =& \mathcal{E}(S_M w)-\mathcal{E}(f_{\mathcal{H}_k})+\mathcal{E}(f_{\mathcal{H}_k})-\mathcal{E}(PS_M \beta_{M,\lambda})+\frac{\lambda}{2}\left\lVert w\right\rVert_2^2\\
     \leq & \frac{c_2}{2}\lVert S_M w-f_{\mathcal{H}_k}\rVert^2+5^{2r}\lambda\kappa^{4r-2}R^2\\
     \leq & 2^{2r}c_2\lambda^{2r}R^2+2^25^{2r}\lambda\kappa^{4r-2}R^2,
\end{align*}
where equality in the third line follows from the fact that $(I-P)S_M \beta\equiv0$ almost surely, as stated in Lemma~\ref{lem:secondterm}.

Thus,
\begin{align*}
    \frac{\lambda}{2}\left\lVert\beta_{M,\lambda}\right\rVert_2^2\leq2^{2r+1}\left(\frac{c_2}{2}\lambda^{2r-1}+5^{2r}\kappa^{4r-2}\right)\lambda R^2\label{eq:smooth_beta},
\end{align*}
As $(X_{1:n},Y_{1:n})\in E_{regul}$, we bound $\lVert\hat{\beta}_{M,\lambda}\rVert_2$ as below:
\begin{eqnarray}\label{eq:smooth_beta_hat}
    \lVert \hat{\beta}_{M,\lambda}\rVert_{2}\leq r_{non-priv}<r_{priv}.
\end{eqnarray}
Therefore, we arrive at the following bounds of $|(\mathcal{E}-\mathcal{E}_n)(S_M w)|,|(\mathcal{E}-\mathcal{E}_n)(S_M \hat{\beta}_{M,\lambda})|$.
\begin{eqnarray}\label{eq:err_smooth_rademacheR_nonpriv}
    |(\mathcal{E}-\mathcal{E}_n)(S_M w)|,|(\mathcal{E}-\mathcal{E}_n)(S_M \hat{\beta}_{M,\lambda})|\leq \frac{c_1Hr_{priv}}{\sqrt{n}}+\sqrt{\frac{2\log\frac{4}{t}}{n}},
\end{eqnarray}
by \eqref{eq:smooth_w} and \eqref{eq:smooth_beta_hat}.

Combining \eqref{eq:err_smooth_second_third2}, \eqref{eq:smooth_beta_hat}, \eqref{eq:smooth_priv_noise}, \eqref{eq:err_smooth_rademacheR_nonpriv}, we obtain the desired.
\subsubsection{Proof of Lemma~\ref{lem:err_smooth_second_third1}}
By the definition of $\widetilde{\beta}$, we obtain
{\small\begin{equation}\label{eq:err_smooth_second_third1}
    \mathcal{E}(S_M\widetilde{\beta})-\mathcal{E}(S_M \hat{\beta}_{M,\lambda})\leq\frac{\lambda}{2}\lVert \hat{\beta}_{M,\lambda}\rVert_2^2+\frac{b^\top(\hat{\beta}_{M,\lambda}-\widetilde{\beta})}{n}-(\mathcal{E}-\mathcal{E}_n)(S_M \hat{\beta}_{M,\lambda})+(\mathcal{E}-\mathcal{E}_n)(S_M \widetilde{\beta})
\end{equation}}
The $\lVert\hat{\beta}_{M,\lambda}\rVert$ is bounded in \eqref{eq:smooth_beta_hat}. We bound $\frac{b^\top(\hat{\beta}_{M,\lambda}-\widetilde{\beta})}{n}$ using the lemma below:
\begin{lemma}[Lemma 7 in \citet{chaudhuri2011differentially}]\label{lem:chaudhuri}
    Let $G_1,G_2:\mathbb{R}^M\rightarrow\mathbb{R}$ be continuous, differentiable, $\lambda$-strongly convex functions. If $\beta_j:=\argmin_{\beta}G_j(\beta)$ for $j=1,2$, then
    \begin{equation*}
        \lVert \beta_1-\beta_2\rVert_2\leq\frac{1}{\lambda}\max_{\beta}\lVert\nabla (G_1-G_2)(\beta)\rVert.
    \end{equation*}
\end{lemma}
Applying Lemma~\ref{lem:chaudhuri} to $G_1=\frac{1}{n}\sum_{i}l(y_i,\beta^\top h_M(x_i))+\frac{\lambda_0}{2}\lVert\beta\rVert_2^2$ and $G_2=G_1+\frac{b^\top\beta}{n}$, we derive the following bound:
\begin{eqnarray}\label{eq:obj_pert}
    \left\lVert \widetilde{\beta}-\hat{\beta}_{M,\lambda}\right\rVert_2\leq\frac{\left\lVert b\right\rVert}{n\lambda}.
\end{eqnarray}
Thus,
\begin{equation}\label{eq:smooth_priv_noise}
 \frac{b^\top(\widetilde{\beta}-\hat{\beta}_{M,\lambda})}{n}\leq\frac{\left\lVert b\right\rVert_2^2}{n^2\lambda}\leq\frac{2\Delta_3^2(M+8\log\frac{2}{t})}{(n\epsilon)^2\lambda}
\end{equation}
since $b\in E_b$.

To use Lemma~\ref{lem:rademacher_rkhs}, we bound $\left\lVert S_M\widetilde{\beta}\right\rVert_{\mathcal{H}_M}$.
Combining \eqref{eq:smooth_beta_hat} and \eqref{eq:obj_pert} yields
\begin{equation}\label{eq:smooth_beta_tilt}
    \left\lVert S_M\widetilde{\beta}\right\rVert_{\mathcal{H}_M}\leq r_{priv}.
\end{equation}
Thus,
\begin{eqnarray}\label{eq:err_smooth_rademacher_priv}
|(\mathcal{E}-\mathcal{E}_n)(S_M \widetilde{\beta})|\leq \frac{c_1Hr_{priv}}{\sqrt{n}}+\sqrt{\frac{2\log\frac{4}{t}}{n}}
\end{eqnarray}
by \eqref{eq:smooth_beta_tilt}.

Combining \eqref{eq:err_smooth_second_third1}, \eqref{eq:smooth_beta_hat}, \eqref{eq:smooth_priv_noise}, \eqref{eq:err_smooth_rademacher_priv}, we establish the desired lemma.
\subsubsection{Proof of Lemma~\ref{lem:rademacher_rkhs}}
Let $\mathcal{F}$ be a family of real functions defined on $\mathcal{X}\times\mathcal{Y}$ and $P_{X,Y}$ a distribution over $\mathcal{X}\times\mathcal{Y}$. The Rademacher complexity $R_n(\mathcal{F})\in\mathbb{R}$ for $\mathcal{F}$ and $P_{X,Y}$ is defined as
\begin{equation*}
R_n(\mathcal{F}):=\mathbb{E}_{(X_{1:n},Y_{1:n})\sim P_{X,Y}^{\otimes n}}\left[\mathbb{E}_{\sigma_i}\left[\sup_{f\in\mathcal{F}}\frac{1}{n}\sum_{i=1}^n \sigma_if(X_i,Y_i)\Bigg|X_{1:n},Y_{1:n}\right]\right]
\end{equation*}
where $\sigma_i$s are i.i.d. Bernoulli random variables with $\mathbb{P}(\sigma_i=1)=\mathbb{P}(\sigma_i=-1)=\frac{1}{2}$.
\begin{proposition}[Theorem 8 in \citet{bartlett2002rademacher}]\label{prop:Rademacher}
Consider a loss function $l:\mathcal{Y}\times \mathcal{Y}\rightarrow[0,1]$. Let $\mathcal{F}$ be a family of functions mapping from $\mathcal{X}$ to $\mathcal{Y}$ and let $\{(X_i,Y_i)\}_{i=1}^n$ be i.i.d. samples drawn from $P_{X,Y}$. Then, for $n\geq 1$ and $t\in(0,1)$,
\begin{eqnarray}
    \mathbb{P}_{(X_{1:n},Y_{1:n})}\left(\sup_{f\in\mathcal{F}}(\mathcal{E}-\mathcal{E}_n)(f)\leq R_n(l\circ \mathcal{F})+\sqrt{\frac{8\log\frac{2}{t} }{n}}\right)\geq 1-t
\end{eqnarray}
where $l\circ\mathcal{F}:=\{h:(x,y)\mapsto l(y,f(x))-l(y,0)|f\in\mathcal{F}\}$.
\end{proposition}
In our case, $\mathcal{Y}=[-Hr,Hr]$ since $|f(x)|=|\langle k_M(x,\cdot),f\rangle_{\mathcal{H}_M}|\leq \lVert k_M(x,\cdot)\rVert_{\mathcal{H}_M}\lVert f\rVert_{\mathcal{H}_M}\leq Hr$. Also, $l$ is bounded by $c_1|Hr|+c_0$ since $l$ is Lipschitz.
\begin{proposition}[Lemma 22 in \citet{bartlett2002rademacher}]\label{prop:empirical_rademacher}Let $k$ be a positive definite kernel. If $\mathcal{F}$ is the ball of radius r in $\mathcal{H}_k$ then
    \begin{equation*}
\mathbb{E}_{\sigma_i}\left[\sup_{f\in\mathcal{F}}\frac{1}{n}\sum_{i=1}^n \sigma_if(X_i)\Bigg|X_1,\ldots,X_n\right]\leq\frac{2r}{n}\sqrt{\sum_{i=1}^nk(X_i,X_i)}.
\end{equation*}
\end{proposition}
\textit{Proof of Lemma~\ref{lem:rademacher_rkhs}}. It is known that $R_n(l\circ \mathcal{F})\leq c_1 R_n(\mathcal{F})$ if $l$ is $c_1$-Lipschitz, for example, Theorem 7 in \citet{meir2003generalization}. By Proposition~\ref{prop:empirical_rademacher}, we have
\begin{align*}
R_n(l\circ \mathcal{H}_M)\leq c_1 R_n(\mathcal{H}_M) \leq &\frac{2rc_1}{n}\mathbb{E}_{(X_{1:n},Y_{1:n})}\left[\sqrt{\sum_{i=1}^nk_M(X_i,X_i)}\right]\\
 \leq &\frac{2rc_1\sup_{x\in\mathcal{X}}\sqrt{k_M(x,x)}}{\sqrt{n}}.
\end{align*}
The final term is bounded by $\frac{2rc_1H}{\sqrt{n}}$ if $\omega_1,\ldots,\omega_M\in E_g$.
    Applying Proposition~\ref{prop:Rademacher}, we obtain the desired inequality.
\subsubsection{Proof of Lemma~\ref{lem:empi_pop_difference}}
We follow the approach of \citet{christmann2008support}.
\begin{proposition}[Corollary 5.10 in \citet{christmann2008support}]\label{prop:emp_pop_reg}
    For Lipschitz loss $l$ and kernel $k$ with RKHS $\mathcal{H}_k$ such that $\sup_{x\in\mathcal{X}}k(x,x)\leq\kappa^2$ satisfies
\begin{eqnarray*}
    \left\lVert f_\lambda-\hat{f}_\lambda\right\rVert_{\mathcal{H}_k}\leq \frac{1}{\lambda}\left\lVert \mathbb{E}_{P_{X,Y}}\left[l_{\hat{y}}(Y,f_\lambda(X))k(X,\cdot)\right]-\frac{1}{n}\sum_{i=1}^nl_{\hat{y}}(y_i,f_\lambda(x_i))k(x_i,\cdot)\right\rVert_{\mathcal{H}_k}
\end{eqnarray*}
holds where
\begin{eqnarray*}
    \hat{f}_\lambda:=\argmin_{f\in\mathcal{H}_k}\frac{1}{n}\sum_{i=1}^n l(y_i,f(x_i))+\frac{\lambda}{2}\left\lVert f\right\rVert_{\mathcal{H}_k}^2
\end{eqnarray*}
and
\begin{eqnarray*}
    f_\lambda:=\argmin_{f\in\mathcal{H}_k}\mathbb{E}_{P_{X,Y}}\left[l(Y,f(X))\right]+\frac{\lambda}{2}\left\lVert f\right\rVert_{\mathcal{H}_k}^2.
\end{eqnarray*}
\end{proposition}
Since the original statement involves technical details for full generality, we provide a simplified proof tailored to our setting.
\begin{proof}
    Denote $\delta_f:=\hat{f}_\lambda-f_\lambda$. By the definition of $f_\lambda$ and $\hat{f}_\lambda$, for $t\in [0,1]$, the following inequalities hold:
\begin{align*}
\mathbb{E}\left[l(Y,f_\lambda(X))\right]+\frac{\lambda}{2}\lVert f_\lambda\rVert_{\mathcal{H}_k}^2 \leq &\mathbb{E}\left[l(Y,(f_\lambda+t\delta_f)(X))\right]+\frac{\lambda}{2}\lVert f_\lambda+t\delta_f\rVert_{\mathcal{H}_k}^2\\
    \frac{1}{n}\sum_{i=1}^nl(y_i,\hat{f}_\lambda(x_i))+\frac{\lambda}{2}\lVert \hat{f}_\lambda\rVert_{\mathcal{H}_k}^2 \leq &\frac{1}{n}\sum_{i=1}^nl(y_i,(\hat{f}_\lambda-t\delta_f)(x_i))+\frac{\lambda}{2}\lVert \hat{f}_\lambda-t\delta_f\rVert_{\mathcal{H}_k}^2.
\end{align*}
They are simplified to
\begin{align*}
    \lambda(t-t^2)\lVert \delta_f\rVert_{\mathcal{H}_k}^2=&\frac{\lambda}{2}\left(\lVert f_\lambda\rVert_{\mathcal{H}_k}^2-\lVert f_\lambda+t\delta_f\rVert_{\mathcal{H}_k}^2+\lVert \hat{f}_\lambda\rVert_{\mathcal{H}_k}^2-\lVert \hat{f}_\lambda-t\delta_f\rVert_{\mathcal{H}_k}^2\right)\\
     \leq &\mathbb{E}\left[l(Y,(f_\lambda+t\delta_f)(X))-l(Y,f_\lambda(X))\right]\\
    &+\frac{1}{n}\sum_{i=1}^nl(y_i,(\hat{f}_\lambda-t\delta_f)(x_i))-l(y_i,\hat{f}_\lambda(x_i))\\
     \leq & t\mathbb{E}\left[l_{\hat{y}}(Y,(f_\lambda+t\delta_f)(X))\delta_f(X)\right]\\
    &-\frac{t}{n}\sum_{i=1}^nl_{\hat{y}}(y_i,(\hat{f}_\lambda-t\delta_f)(x_i))\delta_f(x_i).
\end{align*}
Dividing by $t$ yields
\begin{eqnarray*}
    \lambda(1-t)\lVert \delta_f\rVert_{\mathcal{H}_k}^2\leq \mathbb{E}\left[l_{\hat{y}}(Y,(f_\lambda+t\delta_f)(X))\delta_f(X)\right]-\frac{1}{n}\sum_{i=1}^nl_{\hat{y}}(y_i,(\hat{f}_\lambda-t\delta_f)(x_i))\delta_f(x_i).
\end{eqnarray*}
Taking $t\rightarrow0^{+}$ we obtain
\begin{align*}
    \lVert \delta_f\rVert_{\mathcal{H}_k}^2 \leq & \frac{1}{\lambda}\mathbb{E}\left[l_{\hat{y}}(Y,f_\lambda(X))\delta_f(X)\right]-\frac{1}{n}\sum_{i=1}^nl_{\hat{y}}(y_i,\hat{f}_\lambda(x_i))\delta_f(x_i)\\
     \leq & \frac{1}{\lambda}\left\langle\mathbb{E}\left[l_{\hat{y}}(Y,f_\lambda(X))k(X,\cdot)\right]-\frac{1}{n}\sum_{i=1}^nl_{\hat{y}}(y_i,\hat{f}_\lambda(x_i))k(x_i,\cdot),\delta_f\right\rangle_{\mathcal{H}_k}\\
     \leq & \frac{1}{\lambda}\left\langle\mathbb{E}\left[l_{\hat{y}}(Y,f_\lambda(X))k(X,\cdot)\right]-\frac{1}{n}\sum_{i=1}^nl_{\hat{y}}(y_i,f_\lambda(x_i))k(x_i,\cdot),\delta_f\right\rangle_{\mathcal{H}_k}\\
     \leq &\frac{1}{\lambda}\left\lVert\mathbb{E}\left[l_{\hat{y}}(Y,f_\lambda(X))k(X,\cdot)\right]-\frac{1}{n}\sum_{i=1}^nl_{\hat{y}}(y_i,f_\lambda(x_i))k(x_i,\cdot)\right\rVert_{\mathcal{H}_k}\lVert \delta_f\rVert_{\mathcal{H}_k}.
\end{align*}
where the third inequality holds since
\begin{align*}
    &\left\langle\frac{1}{n}\sum_{i=1}^n(l_{\hat{y}}(y_i,\hat{f}_\lambda(x_i))-l_{\hat{y}}(y_i,f_\lambda(x_i)))k(x_i,\cdot),\delta_f\right\rangle_{\mathcal{H}_k}\\
    =&\frac{1}{n}\sum_{i=1}^n(l_{\hat{y}}(y_i,\hat{f}_\lambda(x_i))-l_{\hat{y}}(y_i,f_\lambda(x_i)))(\hat{f}_\lambda(x_i)-f_\lambda(x_i))\\
    \geq&0,
\end{align*}
by the convexity of $l$. Therefore,
\begin{eqnarray*}
    \left\lVert f_{\lambda}-\hat{f}_\lambda\right\rVert_{\mathcal{H}_k}\leq \frac{1}{\lambda}\left\lVert\mathbb{E}\left[l_{\hat{y}}(Y,f_\lambda(X))k(X,\cdot)\right]-\frac{1}{n}\sum_{i=1}^nl_{\hat{y}}(y_i,f_\lambda(x_i))k(x_i,\cdot)\right\rVert_{\mathcal{H}_k}.
\end{eqnarray*}
\end{proof}

Proposition~\ref{prop:emp_pop_reg} gives
\begin{align*}
&\left\lVert S_M(\beta_{M,\lambda}-\hat{\beta}_{M,\lambda})\right\rVert_{\mathcal{H}_{M}}\\
 \leq & \frac{1}{\lambda}\left\lVert \mathbb{E}_{P_{X,Y}}\left[l_{\hat{y}}(Y,S_M \beta_\lambda(X))k_M(X,\cdot)\right]-\frac{1}{n}\sum_{i=1}^nl_{\hat{y}}(y_i,S_M \hat{\beta}_\lambda(x_i))k_M(x_i,\cdot)\right\rVert_{\mathcal{H}_{M}}
\end{align*}
Since $\left\lVert l_{\hat{y}}(y,S_M \beta_\lambda(x))k_M(x,\cdot)\right\rVert_{\mathcal{H}_{M}}\leq c_1H$, we can apply Proposition~\ref{prop:bernstein_vec}, which yields the following inequality
\begin{align*}
&\left\lVert \mathbb{E}_{P_{X,Y}}\left[l_{\hat{y}}(Y,S_M \beta_\lambda(X))k_M(X,\cdot)\right]-\frac{1}{n}\sum_{i=1}^nl_{\hat{y}}(y_i,S_M \beta_\lambda(x_i))k_M(x_i,\cdot)\right\rVert_{\mathcal{H}_{M}}\\
 \leq &\frac{4c_1H\log\frac{2}{t}}{n}+2c_1H\sqrt{\frac{2\log\frac{2}{t}}{n}}
\end{align*}
holds with probability at least $1-t$. Therefore,
\begin{eqnarray*}
    \mathbb{P}\left(\left\lVert S_M \beta_\lambda-S_M\hat{\beta}_{M,\lambda}\right\rVert_{\mathcal{H}_M}\leq\frac{4c_1H\log\frac{2}{t}}{n\lambda}+\frac{2c_1H}{\lambda}\sqrt{\frac{2\log\frac{2}{t}}{n}}\right)\leq 1-t.
\end{eqnarray*}

\subsection{Proof of Theorem~\ref{thm:dprp_cls_smooth} for random Fourier features}\label{subsec:dprff_cls_smooth}
As in the proof of Theorem~\ref{thm:dprp_reg_rate}, we prove the theorem under Assumption~\ref{assump:cap_real}.

The argument is almost the same as in the proof of Theorem~\ref{thm:dprp_cls_smooth}, with only two adjustments required. First, Lemmas~\ref{lem:half_operator} and \ref{lem:full_operator} should be replaced by Lemmas~\ref{lem:half_operator_rff} and \ref{lem:full_operator_rff}, since the original versions do not hold for general random features. Second, the constant $H$ in the proof of Theorem~\ref{thm:dprp_cls_smooth} should be replaced by $\sqrt{2}$, as we have $\max_{x\in\mathcal{X}}|\varphi(x,\omega)|\leq \sqrt{2}$. Moreover, the event $E_g$ is no longer needed in excess risk analysis.

Setting $\lambda=O(n^{-\frac{1}{2}}\vee (n\epsilon)^{-\frac{2}{2r+1+\gamma_\alpha}})$ and $M=O((\lambda^{-1} N(\lambda))^{2r-1}\mathcal{F}_\infty(\lambda)^{2-2r}\log\frac{1}{\lambda t})$, the $r_{non-priv}$ and $r_{priv}$ for general random feature map become $\widetilde{O}(1)$. Therefore, the excess risk is bounded as follows:
\begin{align*}
\widetilde{O}\left(\lambda\log^2\frac{1}{t}+\frac{M}{(n\epsilon)^2\lambda}+\frac{\sqrt{\log\frac{1}{t}}}{\sqrt{n}}\right)=&\widetilde{O}\left(\left(n^{-\frac{1}{2}}\vee (n\epsilon)^{-\frac{2}{2r+1+\gamma_\alpha}}\right)\log^2\frac{1}{t}\right)\\
=&\widetilde{O}\left(\left(n^{-\frac{1}{2}}+ (n\epsilon)^{-\frac{2}{2r+1+\gamma_\alpha}}\right)\log^2\frac{1}{t}\right).
\end{align*}
Setting $r=1/2$, we obtain the desired rate.
\subsection{Proof of Theorem~\ref{thm:dprp_cls_rate}}\label{pf:lsc}
As in the proof of Theorem~\ref{thm:dprp_reg_rate}, we prove the theorem under Assumption~\ref{assump:cap_real}.

Theorem~\ref{thm:dprp_cls_rate} can be shown using a similar argument to that in Section~\ref{sec:proof_dprp_cls_smooth}. We decompose the excess risk as below:
\begin{align}
    \mathcal{E}(S_M\widetilde{\beta})-\mathcal{E}(f_{\mathcal{H}_k}) \leq &\frac{c_2}{2}\lVert S_M\widetilde{\beta}-f_{\mathcal{H}_k}\rVert^2\notag\\
     \leq & c_2\left(\lVert S_M \widetilde{\beta}-S_M w\rVert^2+\lVert S_M w-f_{\mathcal{H}_k}\rVert^2\right)
    \label{eq:err_lsc}.
\end{align}
Denote
\begin{align*}
    E_{rad}:=\Bigg\{&(X_{1:n},Y_{1:n}):\\
    &|(\mathcal{E}-\mathcal{E}_n)(f)-(\mathcal{E}-\mathcal{E}_n)(g)|\\
     \leq &  \left(4ec_1\sqrt{\frac{2N(\lambda)}{n}}+ec_1\sqrt{\frac{2\log\frac{2 e \log(R_{\mathrm{lsc}} n)}{t}}{n}}+\frac{8ec_1H \log\frac{2 e \log(R_{\mathrm{lsc}} n)}{t}}{3n\sqrt{\lambda}}\right)\lVert L_{M,\lambda}^{\frac{1}{2}}(f-g)\rVert_{\mathcal{H}_M}\\&\textrm{ if }\lVert L_\lambda^{\frac{1}{2}}(f-g)\rVert_{\mathcal{H}_M}\in[n^{-1},\alpha]\Bigg\}.
\end{align*}
We assume $\omega_{1:M}\in E_g\cap  E_{half}\cap E_{full}\cap E_{cov}\cap E_{rf-dim}$, $(X_{1:n},Y_{1:n})\in E_{rad}$, $b\in E_b$.

We bound the first term as follows:
\begin{align*}
    \frac{c_2}{2}\lVert (S_M\widetilde{\beta}-S_Mw)\rVert^2=&\frac{c_2}{2}\lVert L_{M}^{\frac{1}{2}}(S_M\widetilde{\beta}-S_Mw)\rVert_{\mathcal{H}_M}^2\\
     \leq &\frac{c_2}{2}\lVert L_{M,\lambda}^{\frac{1}{2}}(S_M\widetilde{\beta}-S_Mw)\rVert_{\mathcal{H}_M}^2.
\end{align*}
The latter can be bounded by using the following lemma:
\begin{lemma}\label{lem:priv_err_lsc_first}
For $\omega_{1:M}\in E_{half}\cap E_{full}\cap E_{cov}$, $(X_{1:n},Y_{1:n})\in E_{rad}$, $b\in E_b$, and
\begin{align*}
\lambda \geq & \frac{450e^2c_1^2}{(\mu\wedge2)^2R_{\mathrm{lsc}}^2}\left(\frac{32\kappa^2+2(1+4H\sqrt{2}n^{-1/2}/3)^2\log\frac{2e\log (R_{\mathrm{lsc}} n)}{t}}{n}\right)\vee\sqrt{\frac{27}{2}\lVert L\rVert}\frac{\Delta_3\left(\sqrt{M}+\sqrt{8\log\frac{2}{t}}\right)}{R_{\mathrm{lsc}}n\epsilon}\\
    \lambda \leq &\frac{1600e^2c_1^2 H^2 }{9(\mu\wedge 2)^2}\log^2\frac{2 e \log(R_{\mathrm{lsc}} n)}{t}\wedge\left(\frac{R_{\mathrm{lsc}}}{3R}\right)^{1/r}\left(2^{r+\frac{1}{2}}+\frac{2^{\frac{r+7}{2}}5^r\lVert L\rVert^{-\frac{1-r}{2}}}{\mu\wedge2}+\frac{2^{r+1}\sqrt{c_2+\mu}}{\sqrt{\mu\wedge2}}\right)^{-1/r}\wedge\lVert L\rVert\wedge1\\
    M \geq &(1024\kappa^2+4C^2+2C) N(\lambda)\lambda^{-(2r-1)}\log \frac{2}{t}\vee1,
\end{align*}
the following inequality holds:
\begin{align*}
    &\left\lVert L_{M,\lambda}^{\frac{1}{2}}S_M (\widetilde{\beta}-w)\right\rVert_{\mathcal{H}_M}\\
    \leq & \frac{4}{\mu\wedge 2}\left(4ec_1\sqrt{\frac{2N(\lambda)}{n}}+ec_1\sqrt{\frac{2\log\frac{2 e \log(R_{\mathrm{lsc}} n)}{t}}{n}}+\frac{8ec_1H \log\frac{2 e \log(R_{\mathrm{lsc}} n)}{t}}{3n\sqrt{\lambda}}\right)\\
    &+\frac{\Delta_3\left(\sqrt{M}+\sqrt{8\log\frac{2}{t}}\right)}{n\epsilon\sqrt{\lambda}}\\
    &+\left(\frac{2^{\frac{5-r}{2}}5^r\lVert L\rVert^{-\frac{1-r}{2}}}{\mu\wedge2}+\frac{2^{r+1}\sqrt{c_2+\mu}}{\sqrt{\mu\wedge2}}\right)\lambda^rR.
\end{align*}
\end{lemma}
To establish the excess risk bound, we consider two cases.

First, if $r=1/2$ and $\gamma=0$, then setting $t=n^{-c}$, $\lambda= O\left(\frac{\log n}{n}\vee\frac{\log n}{n\epsilon}\right)$, and $M=O\left(\log n\right)$ satisfies the assumption of Lemma~\ref{lem:priv_err_lsc_first} for sufficiently large $n$. Thus, we have
\begin{align*}
    \frac{c_2}{2}\lVert S_M(\widetilde{\beta}-w)\rVert^2=O\left(\left(n^{-1}+(n\epsilon)^{-1}\right)\log^2 n\right).
\end{align*}
Second, if $r\not=1/2$ or $\gamma\not=0$, setting $t=n^{-c}$, $\lambda= O\left(n^{-\frac{1}{2r+\gamma}}\vee(n\epsilon)^{-\frac{2}{4r+\gamma}}\right)$, and $M=O\left(\left(n^{\frac{2r+\gamma-1}{2r+\gamma}}\wedge (n\epsilon)^{\frac{4r+2\gamma-2}{4r+\gamma}}\right)\log n\right)$ satisfies the assumption of Lemma~\ref{lem:priv_err_lsc_first} for sufficiently large $n$. Thus, we have
\begin{align*}
    \frac{c_2}{2}\lVert S_M(\widetilde{\beta}-w)\rVert^2=O\left(\left(n^{-\frac{2r}{2r+\gamma}}+(n\epsilon)^{-\frac{4r}{4r+\gamma}}\right)\log^2 n\right).
\end{align*}
Recall that we bounded the second term in the proof of Lemma~\ref{lem:last_term}:
\begin{align*}
    \frac{c_2}{2}\lVert S_M w-f_{\mathcal{H}_k}\rVert^2\leq 2^{3-2r}c_2\lambda^{2r}R^2.
\end{align*}
The following lemma shows that $E_{\mathrm{rad}}$ is an event that occurs with high probability:
\begin{lemma}\label{lem:rad}
    For $\omega_1,\ldots,\omega_{M}\in E_g\cap E_{rf-dim}$, the following holds:
    \begin{align*}
        \mathbb{P}\left(E_{rad}^c\right)\leq t,
    \end{align*}
    for $t\in (0,1]$.
\end{lemma}
Combining Lemmas~\ref{lem:half_operator}, \ref{lem:full_operator}, \ref{lem:constants}, \ref{lem:rf_dim}, \ref{lem:priv_bound}, \ref{lem:b}, and \ref{lem:rad} we obtain
\begin{align*}
    \mathbb{P}\left(\omega_{1:M}\in E_g\cap  E_{half}\cap E_{full}\cap E_{cov}\cap E_{rf-dim}, (X_{1:n},Y_{1:n})\in E_{rad},b\in E_b\right)\geq 1-7t.
\end{align*}
Therefore, the excess risk is bounded by $O\left(\left(n^{-\frac{2r}{2r+\gamma}}+(n\epsilon)^{-\frac{4r}{4r+\gamma}}\right)\log^2 n\right)$ with probability $1-n^{-c}$.
\subsubsection{Proof of Lemma~\ref{lem:priv_err_lsc_first}}
We first show $\lVert L_{M,\lambda}^{\frac{1}{2}}S_M(\widetilde{\beta}-w)\rVert_{\mathcal{H}_M}\leq R_{\mathrm{lsc}}$ and $\lVert S_M \widetilde{\beta}-f_{\mathcal{H}_k}\rVert\leq  R_{\mathrm{lsc}}$. Note that these quantities can be bounded as follows:
\begin{align*}
    \lVert L_{M,\lambda}^{\frac{1}{2}}S_M(\widetilde{\beta}-w)\rVert_{\mathcal{H}_M} \leq &\lVert L_{M,\lambda}^{\frac{1}{2}}S_M (\widetilde{\beta}-\hat{\beta}_{M,\lambda})\rVert_{\mathcal{H}_M}+\lVert L_{M,\lambda}^{\frac{1}{2}}S_M (\hat{\beta}_{M,\lambda}-w)\rVert_{\mathcal{H}_M}\\
    \lVert S_M \widetilde{\beta}-f_{\mathcal{H}_k}\rVert \leq & 
    \lVert S_M (\widetilde{\beta}-\hat{\beta}_{M,\lambda})\rVert+
    \lVert S_M (\hat{\beta}_{M,\lambda}-w)\rVert+\lVert S_M w-f_{\mathcal{H}_k}\rVert\\
    =&\lVert L_M^{\frac{1}{2}} S_M (\widetilde{\beta}-\hat{\beta}_{M,\lambda})\rVert_{\mathcal{H}_M}+
    \lVert L_M^{\frac{1}{2}}S_M (\hat{\beta}_{M,\lambda}-w)\rVert_{\mathcal{H}_M}+\lVert S_M w-f_{\mathcal{H}_k}\rVert\\
     \leq &\lVert L_{M,\lambda}^{\frac{1}{2}} S_M (\widetilde{\beta}-\hat{\beta}_{M,\lambda})\rVert_{\mathcal{H}_M}+
    \lVert L_{M,\lambda}^{\frac{1}{2}}S_M (\hat{\beta}_{M,\lambda}-w)\rVert_{\mathcal{H}_M}+\lVert S_M w-f_{\mathcal{H}_k}\rVert.
\end{align*}
Thus, it suffices to show
\begin{align*}\lVert L_{M,\lambda}^{\frac{1}{2}} S_M (\widetilde{\beta}-\hat{\beta}_{M,\lambda})\rVert_{\mathcal{H}_M}+
    \lVert L_{M,\lambda}^{\frac{1}{2}}S_M (\hat{\beta}_{M,\lambda}-w)\rVert_{\mathcal{H}_M}+\lVert S_M w-f_{\mathcal{H}_k}\rVert\leq R_{\mathrm{lsc}}.
\end{align*}
We bound each term. The first term is bounded as follows: \begin{align*}
    \lVert L_{M,\lambda}^{\frac{1}{2}} S_M (\widetilde{\beta}-\hat{\beta}_{M,\lambda})\rVert_{\mathcal{H}_M} \leq & \lVert L_{M,\lambda}\rVert_{\mathcal{H}_M}^{\frac{1}{2}}\lVert S_M (\widetilde{\beta}-\hat{\beta}_{M,\lambda})\rVert_{\mathcal{H}_M}\\
    =&\lVert L_{M,\lambda}\rVert^{\frac{1}{2}}\lVert \widetilde{\beta}-\hat{\beta}_{M,\lambda}\rVert_2\\
     \leq &\sqrt{\frac{3}{2}\lVert L\rVert+\lambda}\lVert \widetilde{\beta}-\hat{\beta}_{M,\lambda}\rVert_2\\
     \leq & \sqrt{\frac{5}{2}\lVert L\rVert}\frac{\lVert b\rVert}{n\lambda}\\
     \leq & \sqrt{\frac{5}{2}\lVert L\rVert}\frac{\Delta_3\left(\sqrt{M}+\sqrt{8\log\frac{2}{t}}\right)}{n\epsilon\lambda},
\end{align*}
where the second inequality holds since $\omega_1,\ldots,\omega_M\in E_{cov}$, the third inequality holds by Lemma~\ref{lem:chaudhuri}, and the final inequality holds since $b\in E_b$.

The second term is bounded by the following lemma:
\begin{lemma}\label{lem:nonpriv_lsc}
For $\omega_{1:M}\in E_{half}\cap E_{full}\cap E_{cov}$, $(X_{1:n},Y_{1:n})\in E_{rad}$, if
\begin{align*}
    \lambda&\in [0,\lVert L\rVert\wedge1]\\
    M \geq & (1024\kappa^2 +4C^2+2C)N(\lambda)\lambda^{-(2r-1)}\log \frac{2}{t}\vee1\\
    \mathrm{err}(n,\lambda,t) \leq & R_{\mathrm{lsc}}-2^{r+\frac{1}{2}}\lambda^rR,
\end{align*}
then
\begin{eqnarray*}
    \left\lVert L_{M,\lambda}^{\frac{1}{2}}S_M (\hat{\beta}_{M,\lambda}-w)\right\rVert_{\mathcal{H}_M}<\mathrm{err}(n,\lambda,t),
\end{eqnarray*}
where
\begin{align*}
    \mathrm{err}(n,\lambda,t)=&\Bigg[\frac{5}{\mu\wedge 2} \left(4ec_1\sqrt{\frac{2N(\lambda)}{n}}+ec_1\sqrt{\frac{2\log\frac{2 e \log(R_{\mathrm{lsc}} n)}{t}}{n}}+\frac{8ec_1H \log\frac{2 e \log(R_{\mathrm{lsc}} n)}{t}}{3n\sqrt{\lambda}}\right)\\
    &+\left(\frac{2^{\frac{r+7}{2}}5^r\lVert L\rVert^{-\frac{1-r}{2}}}{\mu\wedge2}+\frac{2^{r+1}\sqrt{c_2+\mu}}{\sqrt{\mu\wedge2}}\right)\lambda^rR\Bigg]\vee\frac{1}{n}.
\end{align*}
\end{lemma}
Note that, $\frac{5}{\mu\wedge2}\frac{8ec_1H \log\frac{2 e \log(R_{\mathrm{lsc}} n)}{t}}{3n\sqrt{\lambda}}\geq\frac{1}{n}$ since $\lambda\leq \frac{1600e^2c_1^2 H^2 }{9(\mu\wedge 2)^2}\log^2\frac{2 e \log(R_{\mathrm{lsc}} n)}{t}$. Thus,
\begin{align*}
    \mathrm{err}(n,\lambda,t)=&\frac{5}{\mu\wedge 2}\left(4ec_1\sqrt{\frac{2N(\lambda)}{n}}+ec_1\sqrt{\frac{2\log\frac{2 e \log(R_{\mathrm{lsc}} n)}{t}}{n}}+\frac{8ec_1H \log\frac{2 e \log(R_{\mathrm{lsc}} n)}{t}}{3n\sqrt{\lambda}}\right)\\
    &+\left(\frac{2^{\frac{r+7}{2}}5^r\lVert L\rVert^{-\frac{1-r}{2}}}{\mu\wedge2}+\frac{2^{r+1}\sqrt{c_2+\mu}}{\sqrt{\mu\wedge2}}\right)\lambda^rR.
\end{align*}
Also,
\begin{align*}
    \mathrm{err}(n,\lambda,t)+2^{r+\frac{1}{2}}\lambda^rR \leq &\frac{5}{\mu\wedge2}\left(4ec_1\sqrt{\frac{2\kappa^2\lambda^{-1}}{n}}+ec_1\sqrt{\frac{2\log\frac{2 e \log(R_{\mathrm{lsc}} n)}{t}}{n}}+\frac{8ec_1H \log\frac{2 e \log(R_{\mathrm{lsc}} n)}{t}}{3n\sqrt{\lambda}}\right)\\
    &+\left(2^{r+\frac{1}{2}}+\frac{2^{\frac{r+7}{2}}5^r\lVert L\rVert^{-\frac{1-r}{2}}}{\mu\wedge2}+\frac{2^{r+1}\sqrt{c_2+\mu}}{\sqrt{\mu\wedge2}}\right)\lambda^rR\\
    =&\frac{5}{\mu\wedge2}\left(\frac{4\sqrt{2}\kappa+(1+4H\sqrt{2}n^{-1/2}/3)\sqrt{2\log\frac{2e\log (R_{\mathrm{lsc}} n)}{t}}}{\sqrt{n}}\right)\frac{ec_1}{\sqrt{\lambda}}\\
    &+\left(2^{r+\frac{1}{2}}+\frac{2^{\frac{r+7}{2}}5^r\lVert L\rVert^{-\frac{1-r}{2}}}{\mu\wedge2}+\frac{2^{r+1}\sqrt{c_2+\mu}}{\sqrt{\mu\wedge2}}\right)\lambda^rR\\
     \leq &\frac{2}{3}R_{\mathrm{lsc}}\\
    <&R_{\mathrm{lsc}}
\end{align*}
since
\begin{align*}
    \lambda\geq\frac{450e^2c_1^2}{(\mu\wedge2)^2R_{\mathrm{lsc}}^2}\left(\frac{32\kappa^2+2(1+4H\sqrt{2}n^{-1/2}/3)^2\log\frac{2e\log (R_{\mathrm{lsc}} n)}{t}}{n}\right)
\end{align*}
implies
\begin{align*}
  \frac{5}{\mu\wedge2}\left(\frac{4\sqrt{2}\kappa+(1+4H\sqrt{2}n^{-1/2}/3)\sqrt{2\log\frac{2e\log (R_{\mathrm{lsc}} n)}{t}}}{\sqrt{n}}\right)\frac{ec_1}{\sqrt{\lambda}}\leq\frac{R_{\mathrm{lsc}}}{3},
\end{align*}
and
\begin{align*}
\lambda\leq \left(\frac{R_{\mathrm{lsc}}}{3R}\right)^{1/r}\left(2^{r+\frac{1}{2}}+\frac{2^{\frac{r+7}{2}}5^r\lVert L\rVert^{-\frac{1-r}{2}}}{\mu\wedge2}+\frac{2^{r+1}\sqrt{c_2+\mu}}{\sqrt{\mu\wedge2}}\right)^{-1/r}
\end{align*}
implies
\begin{align*}
    \left(2^{r+\frac{1}{2}}+\frac{2^{\frac{r+7}{2}}5^r\lVert L\rVert^{-\frac{1-r}{2}}}{\mu\wedge2}+\frac{2^{r+1}\sqrt{c_2+\mu}}{\sqrt{\mu\wedge2}}\right)\lambda^rR\leq\frac{R_{\mathrm{lsc}}}{3}.
\end{align*}
Thus, the second term is bounded as
\begin{align*}
    \left\lVert L_{M,\lambda}^{\frac{1}{2}}S_M (\hat{\beta}_{M,\lambda}-w)\right\rVert_{\mathcal{H}_M} \leq & \frac{5}{\mu\wedge 2}\Bigg(4ec_1\sqrt{\frac{2N(\lambda)}{n}}+ec_1\sqrt{\frac{2\log\frac{2 e \log(R_{\mathrm{lsc}} n)}{t}}{n}}\\
    &+\frac{8ec_1H \log\frac{2 e \log(R_{\mathrm{lsc}} n)}{t}}{3n\sqrt{\lambda}}\Bigg)\\
    &+\left(\frac{2^{\frac{r+7}{2}}5^r\lVert L\rVert^{-\frac{1-r}{2}}}{\mu\wedge2}+\frac{2^{r+1}\sqrt{c_2+\mu}}{\sqrt{\mu\wedge2}}\right)\lambda^rR.
\end{align*}
The final term is bounded in the proof of Lemma~\ref{lem:last_term} as
\begin{align*}
    \lVert S_M w-f_{\mathcal{H}_k}\rVert \leq & 2^{r+\frac{1}{2}}\lambda^r R.
\end{align*}
Thus,
\begin{align*}
    &\lVert L_{M,\lambda}^{\frac{1}{2}} S_M (\widetilde{\beta}-\hat{\beta}_{M,\lambda})\rVert_{\mathcal{H}_M}+
    \lVert L_{M,\lambda}^{\frac{1}{2}}S_M (\hat{\beta}_{M,\lambda}-w)\rVert_{\mathcal{H}_M}+\lVert S_M w-f_{\mathcal{H}_k}\rVert\\
     \leq &\sqrt{\frac{3}{2}\lVert L\rVert}\frac{\Delta_3\left(\sqrt{M}+\sqrt{8\log\frac{2}{t}}\right)}{n\lambda\epsilon}+\mathrm{err}(n,\lambda,t)+2^{r+\frac{1}{2}}\lambda^r R \\
     \leq &\sqrt{\frac{3}{2}\lVert L\rVert}\frac{\Delta_3\left(\sqrt{M}+\sqrt{8\log\frac{2}{t}}\right)}{n\lambda\epsilon}\\
    &+\frac{5}{\mu\wedge2}\left(\frac{4\sqrt{2}\kappa+(1+4H\sqrt{2}n^{-1/2}/3)\sqrt{2\log\frac{2e\log (R_{\mathrm{lsc}} n)}{t}}}{\sqrt{n}}\right)\frac{ec_1}{\sqrt{\lambda}}\\
    &+\left(2^{r+\frac{1}{2}}+\frac{2^{\frac{r+7}{2}}5^r\lVert L\rVert^{-\frac{1-r}{2}}}{\mu\wedge2}+\frac{2^{r+1}\sqrt{c_2+\mu}}{\sqrt{\mu\wedge2}}\right)\lambda^rR\\
     \leq &R_{\mathrm{lsc}}
\end{align*}
since
\begin{align*}
    \lambda\geq\sqrt{\frac{27}{2}\lVert L\rVert}\frac{\Delta_3\left(\sqrt{M}+\sqrt{8\log\frac{2}{t}}\right)}{R_{\mathrm{lsc}}n\epsilon},
\end{align*}
implies
\begin{align*}
    \sqrt{\frac{3}{2}\lVert L\rVert}\frac{\Delta_3\left(\sqrt{M}+\sqrt{8\log\frac{2}{t}}\right)}{n\lambda\epsilon}\leq\frac{R_{\mathrm{lsc}}}{3}.
\end{align*}
By the definitions of $\widetilde{\beta}$ and $\hat{\beta}_{M,\lambda}$, we have
\begin{align*}
\mathcal{E}_n(S_M\widetilde{\beta})+\frac{\lambda}{2}\lVert\widetilde{\beta}\rVert_2^2+\frac{b^\top\widetilde{\beta}}{n} \leq &\mathcal{E}_n(S_M w)+\frac{\lambda}{2}\lVert w\rVert_2^2+\frac{b^\top w}{n},
\end{align*}
which can be represented as below:
\begin{align}
\mathcal{E}(S_M\widetilde{\beta})-\mathcal{E}(S_M w) \leq &\left[\frac{\lambda}{2}(\left\lVert w\right\rVert_2^2-\lVert\widetilde{\beta}\rVert_2^2)+\frac{b^\top(\widetilde{\beta}-\hat{\beta}_{M,\lambda})}{n}\right]\notag\\
&+(\mathcal{E}-\mathcal{E}_n)(S_M\widetilde{\beta})-(\mathcal{E}-\mathcal{E}_n)(S_M w).\label{eq:priv_err_lsc_first}
\end{align}
We bound each term of \eqref{eq:priv_err_lsc_first}.

The first term can be bounded as follows.
\begin{align*}
&\frac{\lambda}{2}\left\lVert w\right\rVert_2^2-\frac{\lambda}{2}\lVert \widetilde{\beta}\rVert_2^2\\
\leq&-\frac{\lambda}{2}\left\lVert \widetilde{\beta}- w\right\rVert_2^2+ \lambda\langle S_Mw,S_M(w-\widetilde{\beta})\rangle_{\mathcal{H}_M}\\
=&-\frac{\lambda}{2}\left\lVert \widetilde{\beta}- w\right\rVert_2^2+ \lambda\langle L_{M,\lambda}^{-r/2}S_Mw,L_{M,\lambda}^{r/2}S_M(w-\widetilde{\beta})\rangle_{\mathcal{H}_M}\\
\leq&-\frac{\lambda}{2}\left\lVert \widetilde{\beta}- w\right\rVert_2^2+ \sqrt{2}\lambda^{\frac{1+r}{2}}\lVert L_M^{-r/2}S_Mw\rVert_{\mathcal{H}_M}\lVert L_{M,\lambda}^{1/2}S_M(w-\widetilde{\beta})\rVert_{\mathcal{H}_M},
\end{align*}
where the last inequality holds since
\begin{align*}
    \lambda^{\frac{1-r}{2}}L_M^{r/2}\preceq(1-r)\lambda^{1/2}+rL_M^{1/2}\leq \sqrt{2}L_{M,\lambda}^{1/2}.
\end{align*}
Also,
\begin{align*}
\left\lVert L_M^{-\frac{r}{2}}S_M w\right\rVert_{\mathcal{H}_M}\leq&\lambda^{-\frac{1-r}{2}}(\lVert L_M\rVert+\lambda)^r\lVert L_M\rVert^{-\frac{1+r}{2}}\lVert L_{M,\lambda}^{-r}L^r\rVert R\\
    \leq&\lambda^{-\frac{1-r}{2}}\left(\frac{5}{2}\lVert L\rVert\right)^r\left(\frac{\lVert L\rVert}{2}\right)^{-\frac{1+r}{2}}\lVert L_{M,\lambda}^{-r}L^r\rVert R\\
    \leq&2^{-\frac{r}{2}}5^r\lVert L\rVert^{-\frac{1-r}{2}}\lambda^{-\frac{1-r}{2}}R.
\end{align*}
The first inequality holds by Lemma~\ref{lem:liu_prop1}. The second inequality holds since $\omega_1,\ldots,\omega_M\in E_{cov}$ and $\lambda\in [0,\lVert L\rVert]$. The final inequality holds by Lemma~\ref{lem:liu_lem6}, and by the fact that $\omega_1,\ldots,\omega_M\in E_{half}\cap E_{full}$, and $M\geq (1024\kappa^2 N(\lambda)+4C^2+2C)\lambda^{-(2r-1)}\log \frac{2}{t}\vee1$.

Thus, we have
\begin{align}
&\frac{\lambda}{2}\lVert w\rVert_2^2-\frac{\lambda}{2}\lVert\widetilde{\beta}\rVert_2^2\notag\\
     \leq &-\frac{\lambda}{2}\lVert\widetilde{\beta}-w\rVert_2^2\notag\\
    &+2^{-\frac{r}{2}}5^r\lVert L\rVert^{-\frac{1-r}{2}}\lambda^rR\left(\sqrt{\lambda}\left\lVert w-\widetilde{\beta}\right\rVert_2+\left\lVert S_M (w-\widetilde{\beta})\right\rVert\right)\label{eq:priv_err_lsc_first_first}.
\end{align}
Next, we bound the second term of \eqref{eq:priv_err_lsc_first}:
\begin{align}
    \frac{b^\top(\widetilde{\beta}-\hat{\beta}_{M,\lambda})}{n} \leq & \frac{\lVert b\rVert^2}{n\lambda}\notag\\
     \leq &\left(\frac{\Delta_3\left(\sqrt{M}+\sqrt{8\log\frac{2}{t}}\right)}{n\epsilon\sqrt{\lambda}}\right)^2\label{eq:priv_err_lsc_first_second}.
\end{align}
Finally, we bound the last term using the following inequality:
\begin{align}
    &|(\mathcal{E}-\mathcal{E}_n)(S_M\widetilde{\beta})-(\mathcal{E}-\mathcal{E}_n)(S_M w)|\notag\\
     \leq &\left(4ec_1\sqrt{\frac{2N(\lambda)}{n}}+ec_1\sqrt{\frac{2\log\frac{2 e \log(R_{\mathrm{lsc}} n)}{t}}{n}}+\frac{8ec_1H \log\frac{2 e \log(R_{\mathrm{lsc}} n)}{t}}{3n\sqrt{\lambda}}\right)\left\lVert L_{M,\lambda}^{\frac{1}{2}}S_M(\widetilde{\beta}-w)\right\rVert_{\mathcal{H}_M},\label{eq:priv_err_lsc_first_third}
\end{align}
which holds since $\left\lVert L_{M,\lambda}^{\frac{1}{2}}S_M(\widetilde{\beta}-w)\right\rVert_{\mathcal{H}_M}\in[n^{-1},R_{\mathrm{lsc}}]$ and $(X_{1:n},Y_{1:n})\in E_{rad}$.

Next, we establish a lower bound of the left-hand-side of \eqref{eq:priv_err_lsc_first}:
\begin{align}
\mathcal{E}(S_M \widetilde{\beta})-\mathcal{E}(S_M w)=&\mathcal{E}(S_M \widetilde{\beta})-\mathcal{E}(f_{\mathcal{H}_k})+\mathcal{E}(f_{\mathcal{H}_k})-\mathcal{E}(S_M w)\notag\\
 \geq &\frac{\mu}{2}\left\lVert S_M \widetilde{\beta}-f_{\mathcal{H}_k}\right\rVert^2-\frac{c_2}{2}\lVert S_M w-f_{\mathcal{H}_k}\rVert^2\notag\\
 \geq &\frac{\mu}{4}\left\lVert S_M (\widetilde{\beta}-w)\right\rVert^2-\frac{c_2+\mu}{2}\lVert S_M w-f_{\mathcal{H}_k}\rVert^2\notag\\
 \geq &\frac{\mu}{4}\left\lVert S_M (\widetilde{\beta}-w)\right\rVert^2-2^{2r}(c_2+\mu)\lambda^{2r}R^2\label{eq:priv_lsc_lhs},
\end{align}
where the first inequality holds since $\lVert S_M \widetilde{\beta}-f_{\mathcal{H}_k}\rVert=\lVert PS_M \widetilde{\beta}-f_{\mathcal{H}_k}\rVert\leq R_{\mathrm{lsc}}$ and $S_M \widetilde{\beta}=PS_M\widetilde{\beta}$ almost surely, the second inequality holds since $\left\lVert a_1+a_2\right\rVert^2\geq\frac{1}{2}\left\lVert a_1\right\rVert^2-\left\lVert a_2\right\rVert^2$ for $a_1,a_2\in L^2(P_X)$, and the last inequality holds since $\frac{c_2}{2}\lVert S_M w-f_{\mathcal{H}_k}\rVert^2\leq c_22^{2r}\lambda^{2r}R^2$ was shown in the proof of Lemma~\ref{lem:last_term}.

Combining \eqref{eq:priv_err_lsc_first}, \eqref{eq:priv_lsc_lhs}, \eqref{eq:priv_err_lsc_first_first}, \eqref{eq:priv_err_lsc_first_second}, and \eqref{eq:priv_err_lsc_first_third},
we have
\begin{align*}
&\frac{\mu}{4}\left\lVert S_M (\widetilde{\beta}-w)\right\rVert^2+\frac{\lambda}{2}\left\lVert\widetilde{\beta}-w\right\rVert_2^2\\
 \leq &2^{2r}(c_2+\mu)\lambda^{2r}R^2\\
&+\left(\frac{\Delta_3\left(\sqrt{M}+\sqrt{8\log\frac{2}{t}}\right)}{n\epsilon\sqrt{\lambda}}\right)^2\\
&+2^{-\frac{r}{2}}5^r\lVert L\rVert^{-\frac{1-r}{2}}\lambda^rR\left(\sqrt{\lambda}\left\lVert w-\widetilde{\beta}\right\rVert_2+\left\lVert S_M (w-\widetilde{\beta})\right\rVert\right)\\
&+\left(4ec_1\sqrt{\frac{2N(\lambda)}{n}}+ec_1\sqrt{\frac{2\log\frac{2 e \log(R_{\mathrm{lsc}} n)}{t}}{n}}+\frac{8ec_1H \log\frac{2 e \log(R_{\mathrm{lsc}} n)}{t}}{3n\sqrt{\lambda}}\right)\left\lVert L_{M,\lambda}^{\frac{1}{2}}S_M(\widetilde{\beta}-w)\right\rVert_{\mathcal{H}_M}.
\end{align*}
Also, note that
\begin{align*}
    &\frac{\mu}{4}\left\lVert S_M (\widetilde{\beta}-w)\right\rVert^2+\frac{\lambda}{2}\left\lVert\widetilde{\beta}-w\right\rVert_2^2\\
     \geq &\frac{\mu\wedge 2}{4}\left\lVert S_M (\widetilde{\beta}-w)\right\rVert^2+\frac{\mu\wedge2}{4}\lambda\left\lVert S_M (w-\widetilde{\beta})\right\rVert_{\mathcal{H}_M}^2\notag\\
    =&\frac{\mu\wedge2}{4}\left(\left\lVert S_M (\widetilde{\beta}-w)\right\rVert^2+\lambda\left\lVert S_M (w-\widetilde{\beta})\right\rVert_{\mathcal{H}_M}^2\right)\\
    =&\frac{\mu\wedge2}{4}\left(\left\lVert L_M^{\frac{1}{2}}S_M (\widetilde{\beta}-w)\right\rVert_{\mathcal{H}_M}^2+\lambda\left\lVert S_M (w-\widetilde{\beta})\right\rVert_{\mathcal{H}_M}^2\right)\\
    =&\frac{\mu\wedge2}{4}\left\lVert L_{M,\lambda}^{\frac{1}{2}}S_M (\widetilde{\beta}-w)\right\rVert_{\mathcal{H}_M}^2,
\end{align*}
and
\begin{align*}
    &\sqrt{\lambda}\left\lVert w-\widetilde{\beta}\right\rVert_2+\left\lVert S_M (w-\widetilde{\beta})\right\rVert
    \\ \leq &\sqrt{2\left(\lambda\left\lVert S_M (w-\widetilde{\beta})\right\rVert_{\mathcal{H}_M}^2+\left\lVert S_M (\widetilde{\beta}-w)\right\rVert^2\right)}\\
    =&\sqrt{2}\left\lVert L_{M,\lambda}^{\frac{1}{2}}S_M (\widetilde{\beta}-w)\right\rVert_{\mathcal{H}_M}.
\end{align*}
Thus, we have
\begin{align*}
    &\frac{\mu\wedge2}{4}\left\lVert L_{M,\lambda}^{\frac{1}{2}}S_M (\widetilde{\beta}-w)\right\rVert_{\mathcal{H}_M}^2\\
     \leq & 2^{2r}(c_2+\mu)\lambda^{2r}R^2\\
    &+\left(\frac{\Delta_3\left(\sqrt{M}+\sqrt{8\log\frac{2}{t}}\right)}{n\epsilon\sqrt{\lambda}}\right)^2\\
    &+2^{\frac{1-r}{2}}5^r\lVert L\rVert^{-\frac{1-r}{2}}\lambda^rR\left\lVert L_{M,\lambda}^{\frac{1}{2}}S_M (\widetilde{\beta}-w)\right\rVert_{\mathcal{H}_M}\\
    &+\left(4ec_1\sqrt{\frac{2N(\lambda)}{n}}+ec_1\sqrt{\frac{2\log\frac{2 e \log(R_{\mathrm{lsc}} n)}{t}}{n}}+\frac{8ec_1H \log\frac{2 e \log(R_{\mathrm{lsc}} n)}{t}}{3n\sqrt{\lambda}}\right)\left\lVert L_{M,\lambda}^{\frac{1}{2}}S_M(\widetilde{\beta}-w)\right\rVert_{\mathcal{H}_M}.
\end{align*}
Using $x^2\leq ax+b\Rightarrow x\leq a+\sqrt{b}$, we obtain the desired bound:
\begin{align*}
    &\left\lVert L_{M,\lambda}^{\frac{1}{2}}S_M (\widetilde{\beta}-w)\right\rVert_{\mathcal{H}_M}\\
     \leq & \frac{4}{\mu\wedge 2}\left(4ec_1\sqrt{\frac{2N(\lambda)}{n}}+ec_1\sqrt{\frac{2\log\frac{2 e \log(R_{\mathrm{lsc}} n)}{t}}{n}}+\frac{8ec_1H \log\frac{2 e \log(R_{\mathrm{lsc}} n)}{t}}{3n\sqrt{\lambda}}\right)\\
    &+\frac{\Delta_3\left(\sqrt{M}+\sqrt{8\log\frac{2}{t}}\right)}{n\epsilon\sqrt{\lambda}}\\
    &+\left(\frac{2^{\frac{5-r}{2}}5^r\lVert L\rVert^{-\frac{1-r}{2}}}{\mu\wedge2}+\frac{2^{r+1}\sqrt{c_2+\mu}}{\sqrt{\mu\wedge2}}\right)\lambda^rR.
\end{align*}
\subsubsection{Proof of Lemma~\ref{lem:nonpriv_lsc}}
Let
\begin{align*}
    c=&\Bigg[\frac{5}{\mu\wedge 2}\left(4ec_1\sqrt{\frac{2N(\lambda)}{n}}+ec_1\sqrt{\frac{2\log\frac{2 e \log(R_{\mathrm{lsc}} n)}{t}}{n}}+\frac{8ec_1H \log\frac{2 e \log(R_{\mathrm{lsc}} n)}{t}}{3n\sqrt{\lambda}}\right)\\
    &+\left(\frac{2^{15/2-3r}5^r\lVert L\rVert^{-\frac{1-r}{2}}}{\mu\wedge2}+\frac{2^{7/2-r}\sqrt{c_2+\mu}}{\sqrt{\mu\wedge2}}\right)\lambda^rR\Bigg]\vee\frac{1}{n}.
\end{align*}
Define
\begin{align*}
\hat{\beta}^c_\lambda&:=\left(\frac{c}{\left\lVert L_{M,\lambda}^{\frac{1}{2}} S_M (\hat{\beta}_{M,\lambda}-w)\right\rVert_{\mathcal{H}_M}}\wedge1\right)\hat{\beta}_{M,\lambda}\\
&+\left(1-\frac{c}{\left\lVert L_{M,\lambda}^{\frac{1}{2}}S_M (\hat{\beta}_{M,\lambda}-w)\right\rVert_{\mathcal{H}_M}}\wedge1\right)w.
\end{align*}
We claim $\left\lVert L_{M,\lambda}^{\frac{1}{2}} S_M (\hat{\beta}_{M,\lambda}^c-w)\right\rVert_{\mathcal{H}_M}< c$, which implies $\left\lVert L_{M,\lambda}^{\frac{1}{2}} S_M (\hat{\beta}_{M,\lambda}-w)\right\rVert_{\mathcal{H}_M}<c$.

We assume $\left\lVert L_{M,\lambda}^{\frac{1}{2}} S_M (\hat{\beta}_{M,\lambda}^c-w)\right\rVert_{\mathcal{H}_M}\geq \frac{1}{n}$, as the claim holds trivially otherwise. Also, $c\leq R_{\mathrm{lsc}}-2^{r+\frac{1}{2}}\lambda^rR$ implies $\left\lVert L_{M,\lambda}^{\frac{1}{2}} S_M (\hat{\beta}_{M,\lambda}^c-w)\right\rVert_{\mathcal{H}_M}\in [n^{-1},R_{\mathrm{lsc}}]$.

Definition of $\hat{\beta}_{M,\lambda}$ implies
\begin{eqnarray*}
    \mathcal{E}_n(S_M\hat{\beta}_{M,\lambda})+\frac{\lambda}{2}\lVert \hat{\beta}_{M,\lambda}\rVert_2^2\leq\mathcal{E}_n(S_M w)+\frac{\lambda}{2}\left\lVert w\right\rVert_2^2.
\end{eqnarray*}
This implies
\begin{equation*}
    \mathcal{E}_n(S_M\hat{\beta}_{M,\lambda}^c)+\frac{\lambda}{2}\lVert \hat{\beta}_{M,\lambda}^c\rVert_2^2\leq\mathcal{E}_n(S_M w)+\frac{\lambda}{2}\left\lVert w\right\rVert_2^2
\end{equation*}
by the convexity of the map $\beta\mapsto\mathcal{E}_n(S_M \beta)+\frac{\lambda}{2}\left\lVert\beta\right\rVert_2^2$. As in the proof of Theorem~\ref{thm:dprp_cls_smooth}, we obtain the upper bound of the excess risk as follows:
\begin{align}
    \mathcal{E}(S_M\hat{\beta}_{M,\lambda}^c)-\mathcal{E}(S_M w) \leq &\left[\frac{\lambda}{2}\left\lVert w\right\rVert_2^2-\frac{\lambda}{2}\lVert \hat{\beta}_{M,\lambda}^c\rVert_2^2\right]\notag\\
    &+\left[(\mathcal{E}-\mathcal{E}_n)(S_M\hat{\beta}_{M,\lambda}^c)-(\mathcal{E}-\mathcal{E}_n)(S_M w)\right]\label{eq:nonpriv_err_lsc_first}.
\end{align}
The first term can be bounded by Lemma~\ref{lem:liu_lem1}:
\begin{align*}
&\frac{\lambda}{2}\left\lVert w\right\rVert_2^2-\frac{\lambda}{2}\lVert \hat{\beta}_{M,\lambda}^c\rVert_2^2\\
    \leq&-\frac{\lambda}{2}\left\lVert \hat{\beta}_{M,\lambda}^c- w\right\rVert_2^2\\
    &+ \lambda^{\frac{1+r}{2}}\left\lVert L_M^{-\frac{r}{2}}S_M w\right\rVert_{\mathcal{H}_M}\left(\sqrt{\lambda}\left\lVert w-\hat{\beta}_{M,\lambda}^c\right\rVert_2+\left\lVert S_M (w-\hat{\beta}_{M,\lambda}^c)\right\rVert\right).
\end{align*}
Also,
\begin{align*}
\left\lVert L_M^{-\frac{r}{2}}S_M w\right\rVert_{\mathcal{H}_M}\leq&\lambda^{-\frac{1-r}{2}}(\lVert L_M\rVert+\lambda)^r\lVert L_M\rVert^{-\frac{1+r}{2}}\lVert L_{M,\lambda}^{-r}L^r\rVert R\\
    \leq&\lambda^{-\frac{1-r}{2}}\left(\frac{5}{2}\lVert L\rVert\right)^r\left(\frac{\lVert L\rVert}{2}\right)^{-\frac{1+r}{2}}\lVert L_{M,\lambda}^{-r}L^r\rVert R\\
    =&2^{\frac{1-r}{2}}5^r\lVert L\rVert^{-\frac{1-r}{2}}\lambda^{-\frac{1-r}{2}}\lVert L_{M,\lambda}^{-r}L^r\rVert R\\
    \leq&2^{\frac{r+2}{2}}5^r\lVert L\rVert^{-\frac{1-r}{2}}\lambda^{-\frac{1-r}{2}}R.
\end{align*}
The first inequality holds by Lemma~\ref{lem:liu_prop1}. The second inequality holds since $\omega_1,\ldots,\omega_M\in E_{cov}$ and $\lambda\in [0,\lVert L\rVert]$. The final inequality holds by Lemma~\ref{lem:liu_lem6}, which gives $\lVert L_{M,\lambda}^{-r}L_\lambda^r\rVert\leq 2^{r+\frac{1}{2}}$.

Thus, we have
\begin{align}
&\frac{\lambda}{2}\lVert w\rVert_2^2-\frac{\lambda}{2}\lVert\hat{\beta}_{M,\lambda}^c\rVert_2^2\notag\\
     \leq &-\frac{\lambda}{2}\lVert\hat{\beta}_{M,\lambda}^c-w\rVert_2^2\notag\\
    &+2^{\frac{r+2}{2}}5^r\lVert L\rVert^{-\frac{1-r}{2}}\lambda^rR\left(\sqrt{\lambda}\left\lVert w-\hat{\beta}_{M,\lambda}^c\right\rVert_2+\left\lVert S_M (w-\hat{\beta}_{M,\lambda}^c)\right\rVert\right)\label{eq:nonpriv_err_lsc_first_first}.
\end{align}
Next, we bound the second term of \eqref{eq:nonpriv_err_lsc_first}.
\begin{align}
    &|(\mathcal{E}-\mathcal{E}_n)(S_M\hat{\beta}_{M,\lambda}^c)-(\mathcal{E}-\mathcal{E}_n)(S_M w)|\notag\\
     \leq &\left(4ec_1\sqrt{\frac{2N(\lambda)}{n}}+ec_1\sqrt{\frac{2\log\frac{2 e \log(R_{\mathrm{lsc}} n)}{t}}{n}}+\frac{8ec_1H \log\frac{2 e \log(R_{\mathrm{lsc}} n)}{t}}{3n\sqrt{\lambda}}\right)\lVert L_{M,\lambda}^{\frac{1}{2}}(f-g)\rVert_{\mathcal{H}_M}\label{eq:nonpriv_err_lsc_first_second}
\end{align}
since $\left\lVert L_{M,\lambda}^{\frac{1}{2}}S_M (\hat{\beta}_{M,\lambda}^c-w)\right\rVert_{\mathcal{H}_M}\in[n^{-1},R_{\mathrm{lsc}}]$ and $(X_{1:n},Y_{1:n})\in E_{rad}$.

Next, we apply Assumption~\ref{assump:lsc} to provide a lower bound of $\mathcal{E}(S_M w)-\mathcal{E}(S_M \hat{\beta}_{M,\lambda}^c)$. To do so, we show that $\left\lVert S_M \hat{\beta}_{M,\lambda}^c-f_{\mathcal{H}_k}\right\rVert\leq R_{\textrm{lsc}}$. From the proof of Lemma~\ref{lem:last_term}, we have
\begin{equation*}
    \lVert S_M w-f_{\mathcal{H}_k}\rVert\leq 2^{r+\frac{1}{2}}\lambda^r R.
\end{equation*}
Thus,
\begin{align*}
\lVert S_M\hat{\beta}_{M,\lambda}^c-f_{\mathcal{H}_k}\rVert\leq&
    \lVert S_M (\hat{\beta}_{M,\lambda}^c-w)\rVert+\lVert S_M w-f_{\mathcal{H}_k}\rVert\\
    =&
    \lVert L_M^{\frac{1}{2}}S_M (\hat{\beta}_{M,\lambda}^c-w)\rVert_{\mathcal{H}_M}+\lVert S_M w-f_{\mathcal{H}_k}\rVert\\
    \leq&c+2^{r+\frac{1}{2}}\lambda^r R\\
    \leq&R_{\mathrm{lsc}}.
\end{align*}
Now we can lower bound $\mathcal{E}(S_M w)-\mathcal{E}(S_M \hat{\beta}_{M,\lambda}^c)$ using Assumption~\ref{assump:lsc}.
\begin{align}
\mathcal{E}(S_M \hat{\beta}_{M,\lambda}^c)-\mathcal{E}(S_M w)=&\mathcal{E}(S_M \hat{\beta}_{M,\lambda}^c)-\mathcal{E}(f_{\mathcal{H}_k})+\mathcal{E}(f_{\mathcal{H}_k})-\mathcal{E}(S_M w)\notag\\
 \geq &\frac{\mu}{2}\left\lVert S_M \hat{\beta}_{M,\lambda}^c-f_{\mathcal{H}_k}\right\rVert^2-\frac{c_2}{2}\lVert S_M w-f_{\mathcal{H}_k}\rVert^2\notag\\
 \geq &\frac{\mu}{4}\left\lVert S_M (\hat{\beta}_{M,\lambda}^c-w)\right\rVert^2-\frac{c_2+\mu}{2}\lVert S_M w-f_{\mathcal{H}_k}\rVert^2\notag\\
 \geq &\frac{\mu}{4}\left\lVert S_M (\hat{\beta}_{M,\lambda}^c-w)\right\rVert^2-2^{2r}(c_2+\mu)\lambda^{2r}R^2,\label{eq:nonpriv_lsc_lhs}
\end{align}
where the last inequality is from $\left\lVert a+b\right\rVert^2\geq\frac{1}{2}\left\lVert a\right\rVert^2-\left\lVert b\right\rVert^2$.
Combining \eqref{eq:nonpriv_err_lsc_first}, \eqref{eq:nonpriv_lsc_lhs}, \eqref{eq:nonpriv_err_lsc_first_first}, and \eqref{eq:nonpriv_err_lsc_first_second} we get
\begin{align}
&\frac{\mu}{4}\left\lVert S_M (\hat{\beta}_{M,\lambda}^c-w)\right\rVert^2+\frac{\lambda}{2}\left\lVert\hat{\beta}_{M,\lambda}^c-w\right\rVert_2^2\notag\\
 \leq &2^{2r}(c_2+\mu)\lambda^{2r}R^2\notag\\
&+2^{\frac{r+2}{2}}5^r\lVert L\rVert^{-\frac{1-r}{2}}\lambda^rR\left(\sqrt{\lambda}\left\lVert w-\hat{\beta}_{M,\lambda}^c\right\rVert_2+\left\lVert S_M (w-\hat{\beta}_{M,\lambda}^c)\right\rVert\right)\notag\\
&+\left(4ec_1\sqrt{\frac{2N(\lambda)}{n}}+ec_1\sqrt{\frac{2\log\frac{2 e \log(R_{\mathrm{lsc}} n)}{t}}{n}}+\frac{8ec_1H \log\frac{2 e \log(R_{\mathrm{lsc}} n)}{t}}{3n\sqrt{\lambda}}\right)\\
&\cdot\left\lVert L_{M,\lambda}^{\frac{1}{2}}S_M (\hat{\beta}_{M,\lambda}^c-w)\right\rVert_{\mathcal{H}_M}\notag.
\end{align}
\newpage
Also, note that
\begin{align*}
    &\frac{\mu}{4}\left\lVert S_M (\hat{\beta}_{M,\lambda}^c-w)\right\rVert^2+\frac{\lambda}{2}\left\lVert\hat{\beta}_{M,\lambda}^c-w\right\rVert_2^2\\
     \geq &\frac{\mu\wedge 2}{4}\left\lVert S_M (\hat{\beta}_{M,\lambda}^c-w)\right\rVert^2+\frac{\mu\wedge2}{4}\lambda\left\lVert S_M (w-\hat{\beta}_{M,\lambda}^c)\right\rVert_{\mathcal{H}_M}^2\notag\\
    =&\frac{\mu\wedge2}{4}\left(\left\lVert S_M (\hat{\beta}_{M,\lambda}^c-w)\right\rVert^2+\lambda\left\lVert S_M (w-\hat{\beta}_{M,\lambda}^c)\right\rVert_{\mathcal{H}_M}^2\right)\\
    =&\frac{\mu\wedge2}{4}\left(\left\lVert L_M^{\frac{1}{2}}S_M (\hat{\beta}_{M,\lambda}^c-w)\right\rVert_{\mathcal{H}_M}^2+\lambda\left\lVert S_M (w-\hat{\beta}_{M,\lambda}^c)\right\rVert_{\mathcal{H}_M}^2\right)\\
    =&\frac{\mu\wedge2}{4}\left\lVert L_{M,\lambda}^{\frac{1}{2}}S_M (\hat{\beta}_{M,\lambda}^c-w)\right\rVert_{\mathcal{H}_M}^2,
\end{align*}
and
\begin{align*}
    &\sqrt{\lambda}\left\lVert w-\hat{\beta}_{M,\lambda}^c\right\rVert_2+\left\lVert S_M (w-\hat{\beta}_{M,\lambda}^c)\right\rVert
    \\ \leq &\sqrt{2\left(\left\lVert S_M (\hat{\beta}_{M,\lambda}^c-w)\right\rVert^2+\lambda\left\lVert S_M (w-\hat{\beta}_{M,\lambda}^c)\right\rVert_{\mathcal{H}_M}^2\right)}\\
    =&\sqrt{2}\left\lVert L_{M,\lambda}^{\frac{1}{2}}S_M (\hat{\beta}_{M,\lambda}^c-w)\right\rVert_{\mathcal{H}_M}.
\end{align*}
Combining the above, we obtain
\begin{align*}
    &\frac{\mu\wedge2}{4}\left\lVert L_{M,\lambda}^{\frac{1}{2}}S_M (\hat{\beta}_{M,\lambda}^c-w)\right\rVert_{\mathcal{H}_M}^2\\
     \leq & 2^{2r}(c_2+\mu)\lambda^{2r}R^2\\
    &+2^{\frac{r+3}{2}}5^r\lVert L\rVert^{-\frac{1-r}{2}}\lambda^rR\left\lVert L_{M,\lambda}^{\frac{1}{2}}S_M (\hat{\beta}_{M,\lambda}^c-w)\right\rVert_{\mathcal{H}_M}\\
    &+\left(4ec_1\sqrt{\frac{2N(\lambda)}{n}}+ec_1\sqrt{\frac{2\log\frac{2 e \log(R_{\mathrm{lsc}} n)}{t}}{n}}+\frac{8ec_1H \log\frac{2 e \log(R_{\mathrm{lsc}} n)}{t}}{3n\sqrt{\lambda}}\right)\\
    &\cdot\left\lVert L_{M,\lambda}^{\frac{1}{2}}S_M (\hat{\beta}_{M,\lambda}^c-w)\right\rVert_{\mathcal{H}_M}.
\end{align*}
Using $x^2\leq ax+b\Rightarrow x\leq a+\sqrt{b}$ we have
\begin{align*}
    &\left\lVert L_{M,\lambda}^{\frac{1}{2}}S_M (\hat{\beta}_{M,\lambda}^c-w)\right\rVert_{\mathcal{H}_M}\\
     \leq & \frac{4}{\mu\wedge 2}\left(4ec_1\sqrt{\frac{2N(\lambda)}{n}}+ec_1\sqrt{\frac{2\log\frac{2 e \log(R_{\mathrm{lsc}} n)}{t}}{n}}+\frac{8ec_1H \log\frac{2 e \log(R_{\mathrm{lsc}} n)}{t}}{3n\sqrt{\lambda}}\right)\\
    &+\left(\frac{2^{\frac{r+7}{2}}5^r\lVert L\rVert^{-\frac{1-r}{2}}}{\mu\wedge2}+\frac{2^{r+1}\sqrt{c_2+\mu}}{\sqrt{\mu\wedge2}}\right)\lambda^rR\\
    &<c.
\end{align*}
Thus, $\left\lVert L_{M,\lambda}^{\frac{1}{2}}S_M (\hat{\beta}_{M,\lambda}^c-w)\right\rVert<c$ so $\left\lVert L_{M,\lambda}^{\frac{1}{2}}S_M (\hat{\beta}_{M,\lambda}-w)\right\rVert<c$ holds.
\subsubsection{Proof of Lemma~\ref{lem:rad}}
Lemma~\ref{lem:liu_lem5} shows that the following inequality holds:
\begin{align*}
    &\sup_{\lVert L_\lambda^{\frac{1}{2}}(f-g)\rVert_{\mathcal{H}_M}\leq r}|(\mathcal{E}-\mathcal{E}_n)(f)-(\mathcal{E}-\mathcal{E}_n)(g)|\\
     \leq &\left(8ec_1\sqrt{\frac{N_M(\lambda)}{n}}+ec_1\sqrt{\frac{2\log\frac{2 e \log(R_{\mathrm{lsc}} n)}{t}}{n}}+\frac{8ec_1\sup_{x\in\mathcal{X}}k_M(x,x) \log\frac{2 e \log(R_{\mathrm{lsc}} n)}{t}}{3n\sqrt{\lambda}}\right)\\
     &\cdot\lVert L_{\lambda}^{\frac{1}{2}}(f-g)\rVert_{\mathcal{H}_M} \\
     &\leq\left(4ec_1\sqrt{\frac{2N(\lambda)}{n}}+ec_1\sqrt{\frac{2\log\frac{2 e \log(R_{\mathrm{lsc}} n)}{t}}{n}}+\frac{8ec_1H \log\frac{2 e \log(R_{\mathrm{lsc}} n)}{t}}{3n\sqrt{\lambda}}\right)\lVert L_{\lambda}^{\frac{1}{2}}(f-g)\rVert_{\mathcal{H}_M}
\end{align*}
holds with probability at least $1-t$, where the last inequality holds by the definitions of $E_g$ and $E_{rf-dim}$.
\subsection{Excess risk analysis under $\ell_\infty$-based local strong convexity assumption}
Theorem~\ref{thm:dprp_cls_rate} holds when Assumption~\ref{assump:lsc} is replaced by an alternative of Assumption~\ref{assump:lsc} below.
\begin{assumption}[$\ell_\infty$ version of Assumption~\ref{assump:lsc}]\label{assump:new_lsc}
There exists $    e_{\infty}>4$ and $\mu>0$ such that
\begin{align*}
\mathcal{E}(f)-\mathcal{E}(f_{\mathcal{H}_k})\geq \frac{\mu}{2}\lVert f-f_{\mathcal{H}_k}\rVert^2,
\end{align*}
for all $f\in L^2(P_X)$ such that $\lVert f\rVert_{\infty}\leq     e_{\infty}\lVert f_{\mathcal{H}_k}\rVert_{\mathcal{H}_k}$.
\end{assumption}
\begin{theorem}\label{thm:dprp_cls_alt_rate} Under Assumptions~\ref{assump:lip}, \ref{assump:cap_real}-\ref{assump:kernel}, and \ref{assump:new_lsc}, there exist choices of $\lambda$ and $M$ such that Algorithm~\ref{alg:dprp_cls} achieves an excess risk bound $O((n^{-\frac{2r}{2r+\gamma}}+(n\epsilon)^{-\frac{4r}{4r+\gamma}})\log^2n)$ if $\epsilon\geq n^{-1}$.
\end{theorem}
\begin{proof}
The proof is almost the same as that of Theorem~\ref{thm:dprp_cls_rate}. We decompose the excess risk into
\begin{align}
    \mathcal{E}(S_M\widetilde{\beta})-\mathcal{E}(f_{\mathcal{H}_k}) \leq &\frac{c_2}{2}\lVert S_M\widetilde{\beta}-f_{\mathcal{H}_k}\rVert^2\\
     \leq & c_2\left(\lVert S_M \widetilde{\beta}-S_M w\rVert^2+\lVert S_M w-f_{\mathcal{H}_k}\rVert^2\right).
\end{align}
The second term was shown to be bounded by $O(\lambda^{2r})$. The first term can be bounded using the lemma below.
\begin{lemma}\label{lem:priv_err_lsc_alt_first}
Under the same assumptions as in Lemma~\ref{lem:priv_err_lsc_first}, if the following conditions hold:
\begin{align*}
    \lambda \geq &\frac{2\Delta_3}{(    e_{\infty}-4)\lVert f_{\mathcal{H}_k}\rVert_{\mathcal{H}_k}}\frac{\sqrt{M}+\sqrt{8\log\frac{2}{t}}}{n\epsilon}\vee\frac{102400e^2c_1^2}{(    e_{\infty}-4)^2(\mu\wedge2)^2\lVert f_{\mathcal{H}_k}\rVert_{\mathcal{H}_k}^2}\frac{2N(\lambda)}{n}\\
    &\vee\frac{1600e^1c_1^2}{(    e_{\infty}-4)^2(\mu\wedge2)^2\lVert f_{\mathcal{H}_k}\rVert_{\mathcal{H}_k}^2}\frac{2\log\frac{23\log(\alpha n)}{t}}{n}
    \vee\frac{320ec_1H}{(    e_{\infty}-4)(\mu\wedge2)\lVert f_{\mathcal{H}_k}\rVert_{\mathcal{H}_k}}\frac{\log\frac{2 e \log(R_{\mathrm{lsc}} n)}{t}}{n}\\
    \lambda \leq &\left(2^{r+\frac{5}{2}}+\frac{2^{\frac{r+11}{2}}5^r\lVert L\rVert^{-\frac{1-r}{2}}}{\mu\wedge2}+\frac{2^{r+3}\sqrt{c_2+\mu}}{\sqrt{\mu\wedge2}}\right)^{\frac{1}{r-1/2}}\left(\frac{2R}{(    e_{\infty}-4)\lVert f_{\mathcal{H}_k}\rVert_{\mathcal{H}_k}}\right)^{\frac{1}{r-1/2}}\\
    M \geq &(1024\kappa^2+4C^2+2C) N(\lambda)\lambda^{-(2r-1)}\log \frac{2}{t}\vee1,
\end{align*}
then the following inequality holds:
\begin{align*}
    &\left\lVert L_{M,\lambda}^{\frac{1}{2}}S_M (\widetilde{\beta}-w)\right\rVert_{\mathcal{H}_M}\\
     \leq & \frac{4}{\mu\wedge 2}\left(\sqrt{\frac{2N(\lambda)}{n}}+\sqrt{\frac{\log\frac{\log (R_{\mathrm{lsc}} n)}{t}}{n}}+\frac{\log\frac{\log (R_{\mathrm{lsc}} n)}{t}}{n\sqrt{\lambda}}\right)\\
    &+\frac{\Delta_3\left(\sqrt{M}+\sqrt{8\log\frac{2}{t}}\right)}{n\epsilon\sqrt{\lambda}}\\
    &+\left(\frac{2^{\frac{r+7}{2}}5^r\lVert L\rVert^{-\frac{1-r}{2}}}{\mu\wedge2}+\frac{2^{r+1}\sqrt{c_2+\mu}}{\sqrt{\mu\wedge2}}\right)\lambda^r.
\end{align*}
\end{lemma}
As in Section~\ref{pf:lsc}, we obtain
\begin{align*}
    \mathcal{E}(S_M\widetilde{\beta})-\mathcal{E}(f_{\mathcal{H}_k})=O\left(\left(n^{-\frac{2r}{2r+\gamma}}+(n\epsilon)^{-\frac{4r}{4r+\gamma}}\right)\log^2n\right),
\end{align*}
with probability at least $1-n^{-c}$ for some $c>0$.
\end{proof}
\subsubsection{Proof of Lemma~\ref{lem:priv_err_lsc_alt_first}}
It suffices to show that $\lVert S_M \widetilde{\beta}\rVert_{\infty}\leq     e_{\infty}\lVert f_{\mathcal{H}_k}\rVert_{\mathcal{H}_k}$ and then repeat the argument from the proof of Lemma~\ref{lem:priv_err_lsc_first}.
\begin{align*}
    \lVert S_M \widetilde{\beta}\rVert_{\infty} \leq &\lVert S_M (\widetilde{\beta}-\hat{\beta}_{M,\lambda})\rVert_{\infty}+\lVert S_M\hat{\beta}_{M,\lambda}\rVert_\infty.
\end{align*}
We bound each term. The first term is bounded as follows:
\begin{align*}
    \lVert S_M (\widetilde{\beta}-\hat{\beta}_{M,\lambda})\rVert_{\infty} \leq &\lVert S_M (\widetilde{\beta}-\hat{\beta}_{M,\lambda})\rVert_{\mathcal{H}_M}\\
    =&\lVert\widetilde{\beta}-\hat{\beta}_{M,\lambda}\rVert_{2}\\
     \leq &\frac{\lVert b\rVert}{n\lambda}\\
     \leq & \frac{\Delta_3\left(\sqrt{M}+\sqrt{8\log\frac{2}{t}}\right)}{(n\epsilon)\lambda}\\
     \leq &\frac{    e_{\infty}-4}{2}\lVert f_{\mathcal{H}_k}\rVert_{\mathcal{H}_k},
\end{align*}
as $\lambda\geq\frac{2\Delta_3}{(    e_{\infty}-4)\lVert f_{\mathcal{H}_k}\rVert_{\mathcal{H}_k}}\frac{\sqrt{M}+\sqrt{8\log\frac{2}{t}}}{n\epsilon}$.

The second term is bounded as follows:
\begin{align*}
    \lVert S_M\hat{\beta}_{M,\lambda}\rVert_\infty \leq & \lVert S_M\hat{\beta}_{M,\lambda}\rVert_{\mathcal{H}_M}\\
     \leq &\lVert S_M (\hat{\beta}_{M,\lambda}-w)\rVert_{\mathcal{H}_M}+\lVert S_M w\rVert_{\mathcal{H}_M}\\
     \leq &\lVert L_{M,\lambda}^{-\frac{1}{2}}\rVert_{\mathcal{H}_M}
    \left\lVert L_{M,\lambda}^{\frac{1}{2}} S_M (\hat{\beta}_{M,\lambda}-w)\right\rVert_{\mathcal{H}_M}+\lVert w\rVert_2\\
     \leq &\lambda^{-\frac{1}{2}}\mathrm{err}(n,\lambda,t)+4\lVert f_{\mathcal{H}_k}\rVert_{\mathcal{H}_k},
\end{align*}
as
\begin{align*}
    \lVert w\rVert_2=&\lVert \mathcal{G}_\lambda(C_M)S_M^\top f_{\mathcal{H}_k}\rVert_2\\
    =&\lVert S_M\mathcal{G}_\lambda(L_M)f_{\mathcal{H}_k}\rVert_2\\
     \leq &\lVert S_M\mathcal{G}_\lambda(L_M)L_{M,\lambda}^{\frac{1}{2}}\rVert\lVert L_{M,\lambda}^{-\frac{1}{2}}L^{\frac{1}{2}}\rVert\lVert L^{-\frac{1}{2}}f_{\mathcal{H}_k}\rVert\\
    =&\lVert S_M\mathcal{G}_\lambda(L_M)L_{M,\lambda}^{\frac{1}{2}}\rVert\lVert L_{M,\lambda}^{-\frac{1}{2}}L^{\frac{1}{2}}\rVert\lVert f_{\mathcal{H}_k}\rVert_{\mathcal{H}_k}\\
    =&\lVert L_{M,\lambda}^{\frac{1}{2}}\mathcal{G}_\lambda(L_M)L_M\mathcal{G}_\lambda(L_M)L_{M,\lambda}^{\frac{1}{2}}\rVert^{\frac{1}{2}}\lVert L_{M,\lambda}^{-\frac{1}{2}}L^{\frac{1}{2}}\rVert\lVert f_{\mathcal{H}_k}\rVert_{\mathcal{H}_k}\\
    =&\lVert L_{M,\lambda}L_M\mathcal{G}_\lambda (L_M)^2\rVert^{\frac{1}{2}}\lVert L_{M,\lambda}^{-\frac{1}{2}}L^{\frac{1}{2}}\rVert\lVert f_{\mathcal{H}_k}\rVert_{\mathcal{H}_k}\\
     \leq &\sqrt{2}\lVert L_{M,\lambda}^{-\frac{1}{2}}L^{\frac{1}{2}}\rVert\lVert f_{\mathcal{H}_k}\rVert_{\mathcal{H}_k}\\
     \leq & 4\lVert f_{\mathcal{H}_k}\rVert_{\mathcal{H}_k},
\end{align*}
where the first inequality follows from the fact that the eigenvalues of $L_{M,\lambda}L_M\mathcal{G}_\lambda (L_M)^2$ are of the form $u(u+\lambda)\frac{\mathbf{1}(u\geq\lambda)}{u^2}=\frac{u+\lambda}{u}\mathbf{1}(u\geq\lambda)\leq 2$, and the second inequality follows from Lemma~\ref{lem:liu_lem6}.

Thus,
\begin{align*}
    \lVert S_M\widetilde{\beta}\rVert_\infty \leq & \frac{\Delta_3\left(\sqrt{M}+\sqrt{8\log\frac{2}{t}}\right)}{(n\epsilon)\lambda}+\lambda^{-\frac{1}{2}}\mathrm{err}(n,\lambda,t)+4\lVert f_{\mathcal{H}_k}\rVert_{\mathcal{H}_k}.
\end{align*}
\begin{lemma}\label{lem:nonpriv_lsc_alt}
Under Assumption~\ref{assump:new_lsc}, if
\begin{align*}
    \lambda&\in [0,\lVert L\rVert\wedge1]\\
    M \geq & (1024\kappa^2+4C^2+2C) N(\lambda)\lambda^{-(2r-1)}\log \frac{2}{t}\vee1\\
    \mathrm{err}(n,\lambda,t) \leq & \frac{    e_{\infty}-4}{2}\sqrt{\lambda}\lVert f_{\mathcal{H}_k}\rVert_{\mathcal{H}_k}\leq \alpha,
\end{align*}
then
\begin{eqnarray*}
    \left\lVert L_{M,\lambda}^{\frac{1}{2}}S_M (\hat{\beta}_{M,\lambda}-w)\right\rVert_{\mathcal{H}_M}<\mathrm{err}(n,\lambda,t).
\end{eqnarray*}
\end{lemma}
Note that $\frac{5}{\mu\wedge2}\frac{8ec_1H \log\frac{2 e \log(R_{\mathrm{lsc}} n)}{t}}{3n\sqrt{\lambda}}\geq\frac{1}{n}$ since $\lambda\leq \frac{1600e^2c_1^2 H^2 }{9(\mu\wedge 2)^2}\log^2\frac{2 e \log(R_{\mathrm{lsc}} n)}{t}$. Thus,
\begin{align*}
    \mathrm{err}(n,\lambda,t)=&\frac{5}{\mu\wedge 2}\left(4ec_1\sqrt{\frac{2N(\lambda)}{n}}+ec_1\sqrt{\frac{2\log\frac{2 e \log(R_{\mathrm{lsc}} n)}{t}}{n}}+\frac{8ec_1H \log\frac{2 e \log(R_{\mathrm{lsc}} n)}{t}}{3n\sqrt{\lambda}}\right)\\
    &+\left(\frac{2^{\frac{r+7}{2}}5^r\lVert L\rVert^{-\frac{1-r}{2}}}{\mu\wedge2}+\frac{2^{r+1}\sqrt{c_2+\mu}}{\sqrt{\mu\wedge2}}\right)\lambda^rR.
\end{align*}
Also,
\begin{align*}
    \lambda^{-\frac{1}{2}}\mathrm{err}(n,\lambda,t) \leq &\frac{5}{\mu\wedge2}\left(8ec_1\sqrt{\frac{2N(\lambda)}{n\lambda}}+ec_1\sqrt{\frac{2\log\frac{2 e \log(R_{\mathrm{lsc}} n)}{t}}{n\lambda}}+\frac{8ec_1H \log\frac{2 e \log(R_{\mathrm{lsc}} n)}{t}}{3n\lambda}\right)\\
    &+\left(2^{r+\frac{1}{2}}+\frac{2^{\frac{r+7}{2}}5^r\lVert L\rVert^{-\frac{1-r}{2}}}{\mu\wedge2}+\frac{2^{r+1}\sqrt{c_2+\mu}}{\sqrt{\mu\wedge2}}\right)\lambda^{r-\frac{1}{2}}R\\
     \leq &(    e_{\infty}-4)\lVert f_{\mathcal{H}_k}\rVert_{\mathcal{H}_k},
\end{align*}
as
\begin{align*}
    \lambda \geq & \frac{102400e^2c_1^2}{(    e_{\infty}-4)^2(\mu\wedge2)^2\lVert f_{\mathcal{H}_k}\rVert_{\mathcal{H}_k}^2}\frac{2N(\lambda)}{n}\vee\frac{1600ec_1^2}{(    e_{\infty}-4)^2(\mu\wedge2)^2\lVert f_{\mathcal{H}_k}\rVert_{\mathcal{H}_k}^2}\frac{2\log\frac{23\log(\alpha n)}{t}}{n} \\
    &\vee\frac{320ec_1H}{(    e_{\infty}-4)(\mu\wedge2)\lVert f_{\mathcal{H}_k}\rVert_{\mathcal{H}_k}}\frac{\log\frac{2 e \log(R_{\mathrm{lsc}} n)}{t}}{n}
\end{align*}
implies
\begin{align*}
    \frac{5}{\mu\wedge2}\left(8ec_1\sqrt{\frac{2N(\lambda)}{n\lambda}}+ec_1\sqrt{\frac{2\log\frac{2 e \log(R_{\mathrm{lsc}} n)}{t}}{n\lambda}}+\frac{8ec_1H \log\frac{2 e \log(R_{\mathrm{lsc}} n)}{t}}{3n\lambda}\right)\leq\frac{3(    e_{\infty}-4)}{8}\lVert f_{\mathcal{H}_k}\rVert_{\mathcal{H}_k},
\end{align*}
and
\begin{align*}
    \lambda\leq \left(2^{r+\frac{5}{2}}+\frac{2^{\frac{r+11}{2}}5^r\lVert L\rVert^{-\frac{1-r}{2}}}{\mu\wedge2}+\frac{2^{r+3}\sqrt{c_2+\mu}}{\sqrt{\mu\wedge2}}\right)^{\frac{1}{r-1/2}}\left(\frac{2R}{(    e_{\infty}-4)\lVert f_{\mathcal{H}_k}\rVert_{\mathcal{H}_k}}\right)^{\frac{1}{r-1/2}}
\end{align*}
implies
\begin{align*}
    \left(2^{r+\frac{1}{2}}+\frac{2^{\frac{r+7}{2}}5^r\lVert L\rVert^{-\frac{1-r}{2}}}{\mu\wedge2}+\frac{2^{r+1}\sqrt{c_2+\mu}}{\sqrt{\mu\wedge2}}\right)\lambda^{r-\frac{1}{2}}R\leq\frac{(    e_{\infty}-4)\lVert f_{\mathcal{H}_k}\rVert_{\mathcal{H}_k}}{8}.
\end{align*}
Finally, $\lambda\leq \frac{4\alpha^2}{(    e_{\infty}-4)^2}\lVert f_{\mathcal{H}_k}\rVert^{-2}$ implies $\mathrm{err}(n,\lambda,t)\leq\sqrt{\lambda}\lVert f_{\mathcal{H}_k}\rVert_{\mathcal{H}_k}\leq\alpha$.

Therefore,
\begin{align*}
    \lVert S_M\widetilde{\beta}\rVert_\infty\leq    e_{\infty}\lVert f_{\mathcal{H}_k}\rVert_{\mathcal{H}_k}.
\end{align*}

\subsubsection{Proof of Lemma~\ref{lem:nonpriv_lsc_alt}}
Lemma~\ref{lem:nonpriv_lsc_alt} can be established by showing that $\left\lVert S_M \hat{\beta}_{M,\lambda}^c\right\rVert_\infty\leq    e_{\infty}\left\lVert f_{\mathcal{H}_k}\right\rVert_{\mathcal{H}_k}$ and by repeating the argument in the proof of Lemma~\ref{lem:nonpriv_lsc}. We verify that this condition holds as follows:
\begin{align*}
    \lVert S_M\hat{\beta}_{M,\lambda}^c\rVert_\infty \leq & \lVert S_M\hat{\beta}_{M,\lambda}^c\rVert_{\mathcal{H}_M}\\
     \leq &\lVert S_M (\hat{\beta}_{M,\lambda}^c-w)\rVert_{\mathcal{H}_M}+\lVert S_M w\rVert_{\mathcal{H}_M}\\
     \leq &\lVert L_{M,\lambda}^{-\frac{1}{2}}\rVert_{\mathcal{H}_M}
    \left\lVert L_{M,\lambda}^{\frac{1}{2}} S_M (\hat{\beta}_{M,\lambda}^c-w)\right\rVert_{\mathcal{H}_M}+\lVert w\rVert_2\\
     \leq &\lambda^{-\frac{1}{2}}\mathrm{err}(n,\lambda,t)+4\lVert f_{\mathcal{H}_k}\rVert_{\mathcal{H}_k}\\
     \leq &     e_{\infty}\lVert f_{\mathcal{H}_k}\rVert_{\mathcal{H}_k}.
\end{align*}

\subsection{Proof of Theorem~\ref{thm:dprff_cls_rate}}
As in the proof of Theorem~\ref{thm:dprp_reg_rate}, we prove the theorem under Assumption~\ref{assump:cap_real}.

The argument is almost the same as that in the proof of Theorem~\ref{thm:dprp_cls_rate}. The main difference lies in the order of $M$: we replace Lemmas~\ref{lem:half_operator} and~\ref{lem:full_operator} with Lemmas~\ref{lem:half_operator_rff} and~\ref{lem:full_operator_rff}, which yield a similar high-probability bound when $M\gtrsim\mathcal{F}_\infty(\lambda)^{2-2r}(\lambda^{-1} N(\lambda))^{2r-1}\log n$. Following the argument in Section~\ref{pf:lsc}, we obtain the rate $O\left(\left(n^{-\frac{2r}{2r+\gamma}}+(n\epsilon)^{-\frac{4r}{4r+\gamma_\alpha}}\right)\log^2n\right)$ with probability at least $1 - n^{-c}$.
\subsection{Proof of Theorem~\ref{thm:dphall_cls_smooth_rate}}
As in the proof of Theorem~\ref{thm:dprp_reg_rate}, we prove the theorem under Assumption~\ref{assump:cap_real}.

The algorithm of \citet{hall2013differential} is presented in Algorithm~\ref{alg:hall_cls}.
\begin{algorithm}
    \caption{Differentially private kernel learning algorithm of \citet{hall2013differential}}\label{alg:hall_cls}
\begin{algorithmic}
  \Input:\textrm{  }given data $\{(x_i,y_i)\}_{i=1}^n$, kernel $k$.
  \EndInput

  \Output:\textrm{  }a differentially private solution $\widetilde{f}$.
  \EndOutput
\begin{enumerate}
    \item $\hat{f}:=\argmin_{f\in\mathcal{H}_k}\frac{1}{n}\sum_{i=1}^n l(y_i,f(x_i))+\frac{\lambda}{2}\left\lVert f\right\rVert_{\mathcal{H}_k}^2$.
    \item From a Gaussian process $h:\mathcal{X}\times\Omega\rightarrow\mathbb{R}$ with mean 0 and covariance function $k$, draw a random function $h(\cdot,\omega_0)$.
\end{enumerate}
\Return $\widetilde{f}:=\hat{f}+\frac{c_1\kappa\left(1+\sqrt{2\log\frac{1}{\delta}}\right)}{n\lambda\epsilon}h(\cdot,\omega_0)$.
\end{algorithmic}
\end{algorithm}
Then
\begin{align*}
\mathcal{E}(\widetilde{f})-\mathcal{E}(\hat{f})\leq&\frac{c_2}{2}\left\lVert\frac{c_1\kappa\left(1+\sqrt{2\log\frac{1}{\delta}}\right)}{n\lambda\epsilon} h(\cdot,\omega_0)\right\rVert^2\\
    =&\frac{c_2\kappa^2\left(1+\sqrt{2\log\frac{1}{\delta}}\right)^2}{2(n\epsilon)^2\lambda^2}\left\lVert h(\cdot,\omega_0)\right\rVert^2.
\end{align*}
We apply Proposition~\ref{prop:gaussian_vec} to the finite-dimensional approximation $h_N=\sum_{j=1}^N\sqrt{\mu_j}Z_j\phi_j$ of $\left\lVert h(\cdot,\omega)\right\rVert^2$ to obtain a high-probability bound, where $\{Z_j\}_{j=1}^N$ are i.i.d. standard normal variables. As the random variables $\langle h,\phi_j\rangle$ are independent Gaussian random variables with distribution $\mathcal{N}(0,\mu_j)$, we may apply Proposition~\ref{prop:gaussian_vec} to conclude that
\begin{eqnarray*}
    \mathbb{P}\left(\left\lVert h_N\right\rVert^2\geq \kappa^2+2\kappa^2\sqrt{\log\frac{1}{t}}+2\kappa^2 \log\frac{1}{t}\right)\leq t,
\end{eqnarray*}
which follows from the fact that $\mu_l\leq \sum_l\mu_l\leq\kappa^2$ and $\sum_l\mu_l^2\leq (\sum_l\mu_l)^2$. Taking $N\rightarrow\infty$, we obtain
\begin{eqnarray*}
    \mathbb{P}\left(\left\lVert h(\cdot,\omega)\right\rVert^2\geq \kappa^2+2\kappa^2\sqrt{\log\frac{1}{t}}+2\kappa^2 \log\frac{1}{t}\right)\leq t.
\end{eqnarray*}
Denote $c=\kappa^2+2\kappa^2\sqrt{\log\frac{1}{t}}+2\kappa^2 \log\frac{1}{t}$. Then
\begin{eqnarray*}
    \mathcal{E}(\widetilde{f})-\mathcal{E}(\hat{f})\leq\frac{c_2\kappa^2\left(1+\sqrt{2\log\frac{1}{\delta}}\right)^2}{2(n\epsilon)^2\lambda^2}c
\end{eqnarray*}
with probability at least $1-t$. Following the proof of Theorem~\ref{thm:dprff_cls_smooth}, we can show
\begin{eqnarray*}
    \left\lVert \hat{f}-f_0\right\rVert^2=O\left(\left(\lambda^{2r}+\frac{N(\lambda)}{n}\right)\log^2\frac{1}{t}\right),
\end{eqnarray*}
with probability at least $1-t$ for $t\in(0,e^{-1}]$. Let the random feature be defined as $\varphi(x,\omega):=(\sqrt{\mu_1}\phi_1(x),\sqrt{\mu_2}\phi_2(x),\cdots)$. Then, by Mercer's theorem, we have $\mathbb{E}_\omega[\varphi(x,\omega)\varphi(y,\omega)^\top]=k(x,y)$, and
\begin{align*}
    \lVert L_\lambda^{-\frac{1}{2}}\varphi(\cdot,\omega)\rVert^2=N(\lambda).
\end{align*}
Therefore, $\varphi$ satisfies Assumption~\ref{assump:compat} with $\alpha = \gamma$, and thus falls under the setting of Theorem~\ref{thm:dprff_cls_rate}.

Thus, the excess risk of Algorithm~\ref{alg:hall_cls} is given as follows:
\begin{eqnarray*}
    \mathcal{E}(\widetilde{f})-\mathcal{E}(f_0)=\widetilde{O}\left(\left(\lambda^{2r}+\frac{N(\lambda)}{n}+\frac{1}{(n\epsilon)^2\lambda^2}\right)\log^2\frac{1}{t}\right),
\end{eqnarray*}
with probability $1-t$. Setting $\lambda=O(n^{-\frac{1}{2r+\gamma}}\vee (n\epsilon)^{-\frac{2}{2r+2}})$ gives the excess risk bound $\widetilde{O}\left(\left(n^{-\frac{2r}{2r+\gamma}}+ (n\epsilon)^{-\frac{4r}{2r+2}}\right)\log^2\frac{1}{t}\right)$.

\subsection{Proof of Theorem~\ref{thm:dphall_cls_rate}}
As in the proof of Theorem~\ref{thm:dphall_cls_smooth_rate}, we can leverage the proof of Theorem~\ref{thm:dprff_cls_rate} to derive the excess risk of Algorithm~\ref{alg:hall_cls} under Assumption~\ref{assump:lsc}:
\begin{eqnarray*}
    \mathcal{E}(\widetilde{f})-\mathcal{E}(f_0)=O\left(\left(\lambda^{2r}+\frac{N(\lambda)}{n}+\frac{1}{(n\epsilon)^2\lambda^2}\right)\log^2\frac{1}{t}\right),
\end{eqnarray*}
with probability $1-t$, where $t=n^{-c}$. Setting $\lambda=O(n^{-\frac{1}{2r+\gamma}}\vee (n\epsilon)^{-\frac{2}{2r+2}})$ yields the excess risk bound $O\left(\left(n^{-\frac{2r}{2r+\gamma}}+ (n\epsilon)^{-\frac{4r}{2r+2}}\right)\log^2n\right)$.

\subsection{Proof of Theorem~\ref{thm:lower_bd_general}}
If a loss function is $\mu$-strongly convex, then the minimax risk can be lower bounded as
\begin{eqnarray*}
    \mathcal{R}(\mathcal{P},\mathcal{E},\epsilon,\delta)\geq\frac{\mu}{2}\inf_{M\in\mathcal{M}_{\epsilon,\delta}}\sup_{P\in\mathcal{P}}\mathbb{E}\left[(\widetilde{f}(X)-f_{\mathcal{H}_k}(X))^2\right].
\end{eqnarray*}
The problem becomes an estimation problem, so we can apply the argument in the proof of Theorem~\ref{thm:lower_bd} to obtain the minimax lower bound. For example, define
\begin{align*}
    l_\mu(y,\hat{y})=\log(1+e^{-y\hat{y}})+\frac{\mu}{2}\hat{y}^2
\end{align*}
then $l_\mu$ is $\mu$-strongly convex and satisfies Assumption~\ref{assump:lsc}. Then for $R>0$, for any $f^\star\in\mathcal{H}_k$ satisfying $\lVert f^\star\rVert_{\mathcal{H}_k}\leq \kappa^{-1}\sup\{R : (1+e^R)R\leq\mu^{-1}\}$, the distribution $P_{X,Y}$ such that
\begin{align*}
    \mathbb{P}(Y=1|X=x)=&\frac{e^{f^\star(x)}}{1+e^{f^\star(x)}}+\mu f^\star(x)\\
    \mathbb{P}(Y=-1|X=x)=&\frac{1}{1+e^{f^\star(x)}}-\mu f^\star(x)
\end{align*}
satisfies $f^\star=\argmin_{f\in\mathcal{H}_k}\mathbb{E}[l_\mu(Y,f(X))]$. This holds since $f^\star(x)=\argmin_{\hat{y}\in\mathbb{R}} \mathbb{E}[l_\mu(Y,\hat{y})|X=x]$. Indeed, for fixed $x$, the derivative of the conditional risk at $\hat y$ is
\[
    \sigma(\hat y)-\mathbb P(Y=1\mid X=x)+\mu \hat y,
\]
where $\sigma(t)=e^t/(1+e^t)$. Hence, with
\[
    \mathbb P(Y=1\mid X=x)=\sigma(f^\star(x))+\mu f^\star(x),
\]
the derivative vanishes at $\hat y=f^\star(x)$. Since $l_\mu$ is $\mu$-strongly convex in $\hat y$, this minimizer is unique. Thus, we can choose $f_v$ and $\nu$ as in the proof of Theorem~\ref{thm:lower_bd}.
\section{Alternative algorithm for DP linear regression}\label{sec:linear}
As a special case, our results recover differentially private linear ERM. Setting $k$ as the inner product kernel and $\mathcal{X}$ as the Euclidean space, Algorithms~\ref{alg:dprp_reg} and~\ref{alg:dprp_cls} become random projection-based differentially private linear ERM algorithms. Our theoretical analysis can be directly applied to derive dimension-free excess risk bounds for these algorithms. For instance, Theorem \ref{thm:dprp_linear_reg} states the excess risk bound of random projection-based sufficient statistics perturbation for private linear regression. The bound can be obtained by combining the following fact and Lemmas~\ref{lem:nonpriv_rp} and~\ref{lem:priv_ridge}:
\begin{equation*}
    N(\lambda)\leq\mathrm{tr}(\Sigma)\lambda^{-1} \wedge \mathrm{rank}(\Sigma),
\end{equation*}
where $\Sigma\coloneqq \mathbb{E}\left[XX^\top\right]$. 
\begin{theorem}\label{thm:dprp_linear_reg}
Under Assumption \ref{assump:kernel}, for $t\in(0,e^{-1}]$, there exist $n_0>0$ and choices of $T,M$, and $\lambda$ such that Algorithm \ref{alg:dprp_reg} achieves an excess risk bound of 
\begin{equation*}
    O\left(\left(\frac{\mathrm{rank}(\Sigma)}{n}\wedge\frac{1}{\sqrt{n}}+\frac{\sqrt{\mathrm{rank}(\Sigma)}}{n\epsilon}\wedge\frac{1}{(n\epsilon)^{2/3}}\right)\log^2\frac{1}{t}\right),
\end{equation*}
with probability at least $1-t$ if $\epsilon\geq n^{-1}$ and $n\geq n_0$.
\end{theorem}
Similarly, the following result holds for random projection-based objective perturbation methods applied to Lipschitz-smooth loss functions.
\begin{theorem}\label{thm:dprp_linear_cls_smooth}
Under Assumptions~\ref{assump:lip} and \ref{assump:kernel}, there exist $n_0>0$ and choices of $M$ and $\lambda$ such that Algorithm \ref{alg:dprp_reg} achieves an excess risk bound of 
\begin{equation*}
    O\left(\left(\frac{1}{\sqrt{n}}+\frac{\sqrt{\mathrm{rank}(\Sigma)}}{n\epsilon}\wedge\frac{1}{(n\epsilon)^{2/3}}\right)\log^2\frac{1}{t}\right),
\end{equation*}
with probability at least $1-t$ if $\epsilon\geq n^{-1}$ and $n\geq n_0$.
\end{theorem}
Theorems~\ref{thm:dprp_linear_reg} and~\ref{thm:dprp_linear_cls_smooth} bound the privacy cost of the algorithms by $O\left(\frac{\sqrt{\mathrm{rank}(\Sigma)}}{n\epsilon} \wedge \frac{1}{(n\epsilon)^{2/3}}\right)$ for both squared loss and Lipschitz-smooth loss functions. These match the optimal bounds established in \cite{song2020characterizing} and \cite{arora2022differentially}, which were achieved by noisy gradient descent-based private algorithms. To the best of our knowledge, this is the first result achieving a dimension-free privacy cost of $O\left(\frac{\sqrt{\mathrm{rank}(\Sigma)}}{n\epsilon} \wedge \frac{1}{(n\epsilon)^{2/3}}\right)$ using a technique other than noisy gradient descent. Furthermore, we note that Theorem~\ref{thm:dprp_linear_reg} is not covered by \cite{arora2022differentially, song2020characterizing}, as the loss function is not Lipschitz due to the unbounded response. Nevertheless, it still achieves a privacy cost of $O\left(\frac{\sqrt{\mathrm{rank}(\Sigma)}}{n\epsilon}\right)$. This demonstrates that the Bernstein condition alone suffices to guarantee the optimal rate.

Finally, we provide the excess risk bound for random projection-based objective perturbation under the local strong convexity condition.
\begin{theorem}\label{thm:dprp_linear_cls_lsc}
Under Assumptions~\ref{assump:lip}, \ref{assump:kernel}, and \ref{assump:lsc}, for $c>0$, there exist $n_0>0$ and choices of $M$ and $\lambda$ such that Algorithm \ref{alg:dprp_reg} for Euclidean space achieves an excess risk bound of 
\begin{equation*}
    O\left(\left(\frac{\mathrm{rank}(\Sigma)}{n}\wedge\frac{1}{\sqrt{n}}+\frac{\sqrt{\mathrm{rank}(\Sigma)}}{n\epsilon}\wedge\frac{1}{(n\epsilon)^{2/3}}\right)\log^2n\right),
\end{equation*}
with probability at least $1-n^{-c}$ if $\epsilon\geq n^{-1}$ and $n\geq n_0$.
\end{theorem}
Theorem~\ref{thm:dprp_linear_cls_lsc} achieves the same order of privacy cost as in the least squares regression case.
\subsection{Proof of Theorem~\ref{thm:dprp_linear_reg}}
\begin{lemma}\label{lem:eff_dim_linear}
    Let $\mathcal{X}\subset\mathbb{R}^d$ such that $\sup_{x\in\mathcal{X}}\left\lVert x\right\rVert_2\leq \kappa$. Denote $\Sigma:=\mathbb{E}\left[XX^\top\right]$. Then $N(\lambda)\leq \kappa^2\lambda^{-1}\wedge\mathrm{rank}(\Sigma)$.
\end{lemma}
\begin{proof}
\begin{align*}
    N(\lambda)=&\mathrm{tr}(L_\lambda^{-1}L)\\
    =&\mathrm{tr}(\Sigma_{\lambda}^{-1}\Sigma)\\
    =&\sum_{l=1}^{\mathrm{rank}(\Sigma)}\frac{\mu_l}{\mu_l+\lambda}\\
    \leq&\mathrm{rank}(\Sigma).
\end{align*}
Also,
\begin{align*}
N(\lambda)=&\mathrm{tr}(\Sigma_{\lambda}^{-1}\Sigma)\\
    =&\mathrm{tr}(\Sigma^{\frac{1}{2}}\Sigma_{\lambda}^{-1}\Sigma^{\frac{1}{2}})\\
    \leq&\lambda^{-1}\mathrm{tr}(\Sigma)\\
    \leq&\kappa^2\lambda^{-1}.
\end{align*}
\end{proof}
Applying Lemmas~\ref{lem:priv_ridge} and \ref{lem:nonpriv_rp} in the linear setting (that is, $r = \frac{1}{2}$ and $\gamma = 1$), we obtain the bound
\begin{eqnarray*}
    O\left(\frac{M}{(n\epsilon)^2\lambda}+\frac{1}{n^2\lambda}\log^2\frac{1}{t}+\frac{N(\lambda)}{n}\log\frac{1}{t}+\frac{\lambda N(\lambda)}{M}\log^2\frac{1}{t}+\lambda\right),
\end{eqnarray*}
with probability at least $1-t$. Setting $M=(n\epsilon)\lambda \sqrt{N(\lambda)}\log\frac{1}{t}$, we have
\begin{eqnarray*}
    O\left(\left(\frac{\sqrt{N(\lambda)}}{n\epsilon}+\frac{1}{n^2\lambda}+\frac{N(\lambda)}{n}+\lambda\right)\log^2\frac{1}{t}\right).
\end{eqnarray*}
By Lemma~\ref{lem:eff_dim_linear}, setting $\lambda=\left(\frac{\sqrt{\mathrm{rank}(\Sigma)}}{n\epsilon}\wedge\frac{1}{(n\epsilon)^{2/3}}\right)\vee\left(\frac{1}{\sqrt{n}}\wedge\frac{\mathrm{rank}(\Sigma)}{n}\right)$ yields the following excess risk bound:
\begin{equation*}
    O\left(\left(\left(\frac{\sqrt{\mathrm{rank}(\Sigma)}}{n\epsilon}\wedge\frac{1}{(n\epsilon)^{2/3}}\right)\vee\left(\frac{1}{\sqrt{n}}\wedge\frac{\mathrm{rank}(\Sigma)}{n}\right)\right)\log^2\frac{1}{t}\right).
\end{equation*}
\subsection{Proof of Theorem~\ref{thm:dprp_linear_cls_smooth}}
It was shown in the proof of Theorem~\ref{thm:dprp_cls_smooth} that for $r = 1/2$ and $\gamma = 1$, the choices of $\lambda=O(n^{-\frac{1}{2}}\vee\frac{\sqrt{M}}{n\epsilon})$ and $M=O(N(\lambda)\log\frac{1}{t})$ yield the excess risk bound
\begin{align*}
    O\left(\lambda\log^2\frac{1}{t}+\frac{M}{(n\epsilon)^2\lambda}+\frac{\sqrt{\log \frac{1}{t}}}{\sqrt{n}}\right),
\end{align*}
with probability at least $1-t$. Applying Lemma~\ref{lem:eff_dim_linear} and setting $\lambda=O\Big(\Big(\frac{\sqrt{\mathrm{rank}(\Sigma)}}{n\epsilon}\wedge\frac{1}{(n\epsilon)^{2/3}}\Big)\vee\frac{1}{\sqrt{n}}\Big)$, we obtain the excess risk bound as follows:
\begin{equation*}
    \widetilde{O}\left(\left(\frac{\sqrt{\mathrm{rank}(\Sigma)}}{n\epsilon}\wedge\frac{1}{(n\epsilon)^{2/3}}\right)\vee\frac{1}{\sqrt{n}}\right).
\end{equation*}
\subsection{Proof of Theorem~\ref{thm:dprp_linear_cls_lsc}}
It was shown in the proof of Theorem~\ref{thm:dprp_cls_rate} that for $r=1/2$ and $\gamma=1$, the choices of $\lambda=O(n^{-\frac{1}{1+\gamma}}\vee(n\epsilon)^{-\frac{2}{2+\gamma}})$ and $M=O(N(\lambda)\log n)$ yield the excess risk bound:
\begin{eqnarray*}
    O\left(\left(\frac{N(\lambda)}{n}+\lambda+\frac{\mathrm{rank}(\Sigma)}{(n\epsilon)^2\lambda}\wedge\frac{1}{(n\epsilon)^2\lambda^2}\right)\log^2n\right),
\end{eqnarray*}
with probability at least $1-n^{-c}$. Choosing $\lambda=O\left(\left(\frac{\sqrt{\mathrm{rank}(\Sigma)}}{n\epsilon}\wedge\frac{1}{(n\epsilon)^{2/3}}\right)\vee\frac{1}{n}\right)$ leads to the following excess risk bound:
\begin{equation*}
    O\left(\left(\left(\frac{\sqrt{\mathrm{rank}(\Sigma)}}{n\epsilon}\wedge\frac{1}{(n\epsilon)^{2/3}}\right)\vee\frac{1}{n}\right)\log^2n\right).
\end{equation*}
\newpage
\section{Auxiliary lemmas}
In this section, we provide proofs of the auxiliary lemmas. 
Some of these results have appeared in the literature, but minor differences in the theoretical setting require slight modifications. 
For completeness, we include full proofs.
\begin{proposition}[Corollary 6.34 in \cite{zhang2023mathematical}]\label{prop:bern_rademacher}
    Suppose the given family of functions $\mathcal{F}$ satisfies
    \begin{align*}
        \mathrm{Var}\left(f(X)\right) \leq &\sigma^2\\
        \sup_{x\in\mathcal{X}}\left(\mathbb{E}\left[f(X)\right]-f(x)\right) \leq & M
    \end{align*}
    for some constants $\sigma,M>0$. Then, for i.i.d. random variables $X_1, \ldots, X_n$, the following inequality holds:
            \begin{align*}
                \mathbb{P}\left(\left|\frac{1}{n}\sum_{i=1}^n f(X_i)-\mathbb{E}\left[f(X)\right]\right|\leq 4R_n(\mathcal{F})+\frac{4M\log\frac{2}{t}}{3n}+\sqrt{\frac{2\sigma^2\log\frac{2}{t}}{n}},\forall f\in\mathcal{F}\right)\geq 1-t.
            \end{align*}
    \end{proposition}
\begin{lemma}[Lemma 5 in \citet{liu2025improved}]\label{lem:liu_lem5}
For a given positive definite kernel $k$, there exists a constant $c_1>0$ such that the following inequality holds
\begin{align*}
    &|(\mathcal{E}-\mathcal{E}_n)(f)-(\mathcal{E}-\mathcal{E}_n)(g)|\\
    \leq &\left(4ec_1\sqrt{\frac{N(\lambda)}{n}}+ec_1\sqrt{\frac{2\log\frac{2 e \log(R_{\mathrm{lsc}} n)}{t}}{n}}+\frac{8ec_1H \log\frac{2 e \log(R_{\mathrm{lsc}} n)}{t}}{3n\sqrt{\lambda}}\right)\lVert L_{\lambda}^{\frac{1}{2}}(f-g)\rVert_{\mathcal{H}_k},
\end{align*}
for all $f,g\in\mathcal{H}_M$ such that $\left\lVert L_\lambda^{\frac{1}{2}}(f-g)\right\rVert_{\mathcal{H}}\in[n^{-1},\alpha]$ with probability at least $1-t$ for any $t\in (0,1)$.
\end{lemma}
\begin{proof}
     Let $\mathcal{F}_0$ be the family of functions such that $l(y,f(x))-l(y,g(x))$ for some $f,g\in\mathcal{H}_M$ satisfying $\lVert L_{M,\lambda}^{\frac{1}{2}}(f-g)\rVert_{\mathcal{H}_M}\leq r$. We can apply Proposition~\ref{prop:bern_rademacher} and find that
     \begin{align*}
         &\left|\frac{1}{n}\sum_{i=1}^n(l(y_i,f(x_i))-l(y_i,g(x_i)))-\mathbb{E}\left[(l(Y,f(X))-l(Y,g(X)))\right]\right|\\
          \leq & 4R_n(\mathcal{F}_0)+c_1r\sqrt{\frac{2\log\frac{2}{t}}{n}}+\frac{8c_1\kappa r\log\frac{2}{t}}{3n\sqrt{\lambda}}
     \end{align*}
     holds for every $f,g$ satisfying $\lVert L_\lambda^{\frac{1}{2}}(f-g)\rVert_{\mathcal{H}_M}\leq r$ with probability at least $1-t$, as
    \begin{align*}
        \sup_{x,y}l(y,f(x))-l(y,g(x)) \leq &\sup_{x}c_1|f(x)-g(x)|\\
        =&\sup_{x}c_1|\langle f-g,k(x,\cdot)\rangle|\\
         \leq &\sup_x c_1\lVert f-g\rVert_{\mathcal{H}_M}\lVert k(x,\cdot)\rVert_{\mathcal{H}_M}\\
         \leq & c_1\kappa\lVert f-g\rVert_{\mathcal{H}_M}\\
         \leq & c_1\kappa\lVert L_{M,\lambda}^{-\frac{1}{2}}\rVert\lVert L_{M,\lambda}^{\frac{1}{2}}(f-g)\rVert_{\mathcal{H}_M}\\
         \leq &\frac{c_1\kappa r}{\sqrt{\lambda}},
    \end{align*}
    and
    \begin{align*}
        \mathbb{E}\left[\left(l(Y,f(X))-l(Y,g(X))\right)^2\right] \leq & c_1^2\mathbb{E}\left[(f(X)-g(X))^2\right]\\
        =&c_1^2\lVert f-g\rVert^2\\
        =&c_1^2\lVert L_M^{\frac{1}{2}}(f-g)\rVert_{\mathcal{H}_M}^2\\
         \leq & c_1^2\lVert L_{M,\lambda}^{\frac{1}{2}}(f-g)\rVert_{\mathcal{H}_M}^2\\
         \leq & c_1^2r^2.
    \end{align*}
    Then the following holds:
        \begin{align*}
            &\mathbb{E}\left[\sup_{\lVert L_\lambda^{\frac{1}{2}} f\rVert_{\mathcal{H}_k}\leq r}\frac{1}{n}\sum_{i=1}^n\sigma_if(X_i)\Bigg|X_1,\ldots,X_n\right]\\
            =&\mathbb{E}\left[\sup_{\lVert L_\lambda^{\frac{1}{2}} f\rVert_{\mathcal{H}_k}\leq r}\frac{1}{n}\left\langle\sum_{i=1}^n\sigma_ik(x_i,\cdot),f\right\rangle_{\mathcal{H}_k}\Bigg|X_1,\ldots,X_n\right]\\
            =&\mathbb{E}\left[\sup_{\lVert L_\lambda^{\frac{1}{2}} f\rVert_{\mathcal{H}_k}\leq r}\frac{1}{n}\left\langle\sum_{i=1}^n\sigma_iL_\lambda^{-\frac{1}{2}}k(x_i,\cdot),L_\lambda^{\frac{1}{2}}f\right\rangle_{\mathcal{H}_k}\Bigg|X_1,\ldots,X_n\right]\\
             \leq &\frac{r}{n}\mathbb{E}\left[\left\lVert\sum_{i=1}^n\sigma_iL_\lambda^{-\frac{1}{2}}k(x_i,\cdot)\right\rVert_{\mathcal{H}_k}\Bigg|X_1,\ldots,X_n\right]\\
             \leq &\frac{r}{n}\sqrt{\mathbb{E}\left[\left\lVert\sum_{i=1}^n\sigma_iL_\lambda^{-\frac{1}{2}}k(x_i,\cdot)\right\rVert_{\mathcal{H}_k}^2\Bigg|X_1,\ldots,X_n\right]}\\=&r\sqrt{\frac{N(\lambda)}{n}}.
        \end{align*}
        Then
        \begin{align*}
            R_n(\mathcal{F}_0)=&\mathbb{E}\left[\sup_{\lVert L_\lambda^{\frac{1}{2}}(f-g)\rVert_{\mathcal{H}_M}\leq r}\frac{1}{n}\sum_{i=1}^n\sigma_i(l(Y_i,f(X_i))-l(Y_i,g(X_i)))\right]\\
             \leq & \mathbb{E}\left[\sup_{\lVert L_\lambda^{\frac{1}{2}}(f-g)\rVert_{\mathcal{H}_M}\leq r}\frac{1}{n}\sum_{i=1}^n\sigma_i c_1(f(X_i)-g(X_i))\right]\\
             \leq & \mathbb{E}\left[\sup_{\lVert L_\lambda^{\frac{1}{2}}f\rVert_{\mathcal{H}_M}\leq r}\frac{1}{n}\sum_{i=1}^n\sigma_i c_1f(X_i)\right]\\
             \leq & c_1r\sqrt{\frac{N(\lambda)}{n}},
        \end{align*}
        where the second inequality is from Lemma 5 in \cite{meir2003generalization}. Thus, the following inequality holds:
        \begin{align*}
            \sup_{\lVert L_\lambda^{\frac{1}{2}}(f-g)\rVert_{\mathcal{H}_M}\leq e^{i+1}n^{-1}}|(\mathcal{E}-\mathcal{E}_n)(f)-(\mathcal{E}-\mathcal{E}_n)(g)|\leq 4c_1r\sqrt{\frac{N(\lambda)}{n}}+c_1r\sqrt{\frac{2\log\frac{2}{t}}{n}}+\frac{8c_1\kappa r\log\frac{2}{t}}{3n\sqrt{\lambda}},
        \end{align*}
        with probability at least $1-t$. Next, we partition the interval $[n^{-1},\alpha]$ into $I=\{[e^in^{-1},e^{i+1}n^{-1}]|i=0,\ldots, \lfloor \log (R_{\mathrm{lsc}} n)\rfloor\}$. Then
        \begin{align*}
            &\sup_{\lVert L_{M,\lambda}^{\frac{1}{2}}(f-g)\rVert_{\mathcal{H}_M}\leq e^{i+1}n^{-1}}|(\mathcal{E}-\mathcal{E}_n)(f)-(\mathcal{E}-\mathcal{E}_n)(g)|
            \\ \leq & \left(4c_1\sqrt{\frac{N(\lambda)}{n}}+c_1\sqrt{\frac{2\log\frac{2|I|}{t}}{n}}+\frac{8c_1\kappa \log\frac{2|I|}{t}}{3n\sqrt{\lambda}}\right)e^{i+1}n^{-1}\\
             \leq & \left(4ec_1\sqrt{\frac{N(\lambda)}{n}}+ec_1\sqrt{\frac{2\log\frac{2|I|}{t}}{n}}+\frac{8ec_1\kappa \log\frac{2|I|}{t}}{3n\sqrt{\lambda}}\right)e^{i}n^{-1}\\
             \leq & \left(4ec_1\sqrt{\frac{N(\lambda)}{n}}+ec_1\sqrt{\frac{2\log\frac{2|I|}{t}}{n}}+\frac{8ec_1\kappa \log\frac{2|I|}{t}}{3n\sqrt{\lambda}}\right)\lVert L_{M,\lambda}^{\frac{1}{2}}(f-g)\rVert_{\mathcal{H}_M}\\
             \leq & \left(4ec_1\sqrt{\frac{2N(\lambda)}{n}}+ec_1\sqrt{\frac{2\log\frac{2 e \log(R_{\mathrm{lsc}} n)}{t}}{n}}+\frac{8ec_1H \log\frac{2 e \log(R_{\mathrm{lsc}} n)}{t}}{3n\sqrt{\lambda}}\right)\lVert L_{M,\lambda}^{\frac{1}{2}}(f-g)\rVert_{\mathcal{H}_M},
        \end{align*}
        holds for every $0\leq i\leq |I|$ with probability at least $1-t$.
\end{proof}
\begin{lemma}[Lemma 6 in \citet{liu2025improved}]\label{lem:liu_lem6}
\begin{align*}
&\left\lVert L_{M,\lambda}^{-r}L_\lambda^r\right\rVert
\\ \leq &\left(1-\lambda^{-r+\frac{1}{2}}\left\lVert L_\lambda^{-\frac{1}{2}}(L_M-L)\right\rVert^{2r-1}\left\lVert L_\lambda^{-\frac{1}{2}}(L_M-L)L_\lambda^{-\frac{1}{2}}\right\rVert^{2-2r}\right)^{-1}\\
&\cdot\left(1-\left\lVert L_\lambda^{-\frac{1}{2}}(L_M-L)L_\lambda^{-\frac{1}{2}}\right\rVert\right)^{1-r}
\end{align*}
\end{lemma}
\begin{proof}
    \begin{align*}
        \left\lVert L_{M,\lambda}^{-r}L_\lambda^r\right\rVert
        =&
        \left\lVert \left(\left(I+L_\lambda^{-r}(L_M-L)L_\lambda^{-(1-r)}\right)L_\lambda^{1-r}L_{M,\lambda}^{-(1-r)}\right)^{-1}\right\rVert\\
        =&
        \left\lVert L_{M,\lambda}^{1-r}L_\lambda^{-(1-r)}\left(I+L_\lambda^{-r}(L_M-L)L_\lambda^{-(1-r)}\right)^{-1}\right\rVert\\
         \leq &\left\lVert L_{M,\lambda}^{1-r}L_\lambda^{-(1-r)}\right\rVert\left(1-\left\lVert L_\lambda^{-r}(L_M-L)L_\lambda^{-(1-r)}\right\rVert\right)^{-1}.
        \end{align*}
    The first term is bounded as follows:
    \begin{align*}
        \left\lVert L_{M,\lambda}^{1-r}L_\lambda^{-(1-r)} \right\rVert\leq&
        \left\lVert L_{M,\lambda}^{\frac{1}{2}}L_\lambda^{-\frac{1}{2}}\right\rVert^{2-2r}\\
        =&\left\lVert L_{\lambda}^{-\frac{1}{2}} L_{M,\lambda}L_\lambda^{-\frac{1}{2}}\right\rVert^{1-r}\\
        =&\left\lVert L_{\lambda}^{-\frac{1}{2}} (L_{M}-L)L_\lambda^{-\frac{1}{2}}+I\right\rVert^{1-r}\\
         \leq & \left(1-\left\lVert L_{\lambda}^{-\frac{1}{2}} (L_{M}-L)L_\lambda^{-\frac{1}{2}}\right\rVert\right)^{1-r}.
    \end{align*}
    Similarly, the quantity $\left\lVert L_\lambda^{-r}(L_M-L)L_\lambda^{-(1-r)}\right\rVert$ in the second term is bounded as follows:
    \begin{align*}
        &\left\lVert L_\lambda^{-r}(L_M-L)L_\lambda^{-(1-r)}\right\rVert\\
        \leq&\left\lVert L_\lambda^{-r+\frac{1}{2}}\right\rVert\left\lVert L_\lambda^{-\frac{1}{2}}(L_M-L)L_\lambda^{-(1-r)}\right\rVert\\
        \leq&\lambda^{-r+\frac{1}{2}}\left\lVert L_\lambda^{-\frac{1}{2}}(L_M-L)L_\lambda^{-(1-r)}\right\rVert\\
        \leq&\lambda^{-r+\frac{1}{2}}\left\lVert L_\lambda^{-\frac{1}{2}}(L_M-L)\right\rVert^{2r-1}\left\lVert L_\lambda^{-\frac{1}{2}}(L_M-L)L_\lambda^{-\frac{1}{2}}\right\rVert^{2-2r}
    \end{align*}
    where the last inequality follows from applying Lemma~\ref{lem:prop9} with $X=L_\lambda^{-\frac{1}{2}}(L_M-L)$, $A=L_\lambda^{-\frac{1}{2}}$, and $\sigma=2-2r$.
    \begin{lemma}[Proposition 9 in \cite{randomfeature}]\label{lem:prop9}
        Let $\mathcal{H}$ and $\mathcal{K}$ be two separable Hilbert spaces and $X:\mathcal{H}\rightarrow\mathcal{K}, A:\mathcal{H}\rightarrow\mathcal{H}$ be bounded linear operators, where $A$ is positive semidefinite. Then
        \begin{align*}
            \lVert XA^\sigma\rVert\leq\lVert X\rVert^{1-\sigma}\lVert XA\rVert^\sigma
        \end{align*}
        holds for $\sigma\in [0,1]$.
    \end{lemma}
\end{proof}
\begin{lemma}[Lemma 1 in \citet{liu2025improved}]\label{lem:liu_lem1}
\begin{align*}
&\left\lVert S_M w\right\rVert_{\mathcal{H}_M}^2-\left\lVert S_M\widetilde{\beta}\right\rVert_{\mathcal{H}_M}^2\\
 \leq &-\left\lVert S_M\widetilde{\beta}-S_M w\right\rVert_{\mathcal{H}_M}^2\\
&+ 2\lambda^{-\frac{1-r}{2}}\left\lVert L_M^{-\frac{r}{2}}S_M w\right\rVert_{\mathcal{H}_M}\left(\sqrt{\lambda}\left\lVert S_M (w-\widetilde{\beta})\right\rVert_{\mathcal{H}_M}+\left\lVert S_M (w-\widetilde{\beta})\right\rVert\right).
\end{align*}
\end{lemma}

\begin{proof}
\begin{align*}
    &\left\lVert S_M w\right\rVert_{\mathcal{H}_M}^2-\left\lVert S_M\widetilde{\beta}\right\rVert_{\mathcal{H}_M}^2\\=&-\left\lVert S_M\widetilde{\beta}-S_M w\right\rVert_{\mathcal{H}_M}^2+2\langle S_M (\widetilde{\beta}-w),S_M w\rangle_{\mathcal{H}_M}\\
    =&-\left\lVert S_M\widetilde{\beta}-S_M w\right\rVert_{\mathcal{H}_M}^2+2|\langle S_M (w-\widetilde{\beta}),S_M w\rangle_{\mathcal{H}_M}|\\
    =&-\left\lVert S_M\widetilde{\beta}-S_M w\right\rVert_{\mathcal{H}_M}^2+2|\langle L_M^{-\frac{r}{2}} S_M (w-\widetilde{\beta}),L_M^{\frac{r}{2}}S_M w\rangle_{\mathcal{H}_M}|\\
     \leq &-\left\lVert S_M\widetilde{\beta}-S_M w\right\rVert_{\mathcal{H}_M}^2+2\left\lVert L_M^{-\frac{r}{2}}S_M w\right\rVert_{\mathcal{H}_M}\left\lVert L_M^{\frac{r}{2}} S_M (w-\widetilde{\beta})\right\rVert_{\mathcal{H}_M}\\
    =&-\left\lVert S_M\widetilde{\beta}-S_M w\right\rVert_{\mathcal{H}_M}^2\\
    &+2\lambda^{-\frac{1-r}{2}}\left\lVert L_M^{-\frac{r}{2}}S_M w\right\rVert_{\mathcal{H}_M}\left\lVert \lambda^{\frac{1-r}{2}}L_M^{\frac{r}{2}} S_M (w-\widetilde{\beta})\right\rVert_{\mathcal{H}_M}\\
     \leq &-\left\lVert S_M\widetilde{\beta}-S_M w\right\rVert_{\mathcal{H}_M}^2\\
    &+2\lambda^{-\frac{1-r}{2}}\left\lVert L_M^{-\frac{r}{2}}S_M w\right\rVert_{\mathcal{H}_M}\left\lVert ((1-r)\lambda^{\frac{1}{2}}+rL_M^{\frac{1}{2}}) S_M (w-\widetilde{\beta})\right\rVert_{\mathcal{H}_M}\\
     \leq &-\left\lVert S_M\widetilde{\beta}-S_M w\right\rVert_{\mathcal{H}_M}^2\\
    &+2\lambda^{-\frac{1-r}{2}}\left\lVert L_M^{-\frac{r}{2}}S_M w\right\rVert_{\mathcal{H}_M}\left(\sqrt{\lambda}\left\lVert S_M (w-\widetilde{\beta})\right\rVert_{\mathcal{H}_M}+\left\lVert L_{M}^{\frac{1}{2}}S_M (w-\widetilde{\beta})\right\rVert_{\mathcal{H}_M}\right)\\
     \leq &-\left\lVert S_M\widetilde{\beta}-S_M w\right\rVert_{\mathcal{H}_M}^2\\
    &+2\lambda^{-\frac{1-r}{2}}\left\lVert L_M^{-\frac{r}{2}}S_M w\right\rVert_{\mathcal{H}_M}\left(\sqrt{\lambda}\left\lVert S_M (w-\widetilde{\beta})\right\rVert_{\mathcal{H}_M}+\left\lVert S_M (w-\widetilde{\beta})\right\rVert\right).
\end{align*}
\end{proof}
\begin{lemma}[Proposition 1 in \citet{liu2025improved}]\label{lem:liu_prop1}
For $r\leq 2a+1$,
\begin{align*}
\left\lVert L_M^{-a}S_M w\right\rVert_{\mathcal{H}_M}\leq&\lambda^{\left(\left(r-\frac{1}{2}\right)-a\right)\wedge0}(\left\lVert L_M\right\rVert+\lambda)^r\left\lVert L_M\right\rVert^{-a-\frac{1}{2}}\left\lVert L_{M,\lambda}^{-r}L^{r}\right\rVert R\\
    \lVert S_M w-f_{\mathcal{H}_k}\rVert\leq&\lambda^r\left\lVert L_{M,\lambda}^{-r}L^r\right\rVert R.
\end{align*}
\end{lemma}
\begin{proof}
\begin{align*}
\left\lVert L_M^{-a}S_M w\right\rVert_{\mathcal{H}_M}=&\left\lVert L_M^{-a-\frac{1}{2}}S_M \mathcal{G}_\lambda(C_M)S_M^\top f_{\mathcal{H}_k}\right\rVert\\
    =&\left\lVert L_M^{-a-\frac{1}{2}}S_M \mathcal{G}_\lambda(C_M)S_M^\top L^r g\right\rVert\\
    =&\left\lVert L_M^{-a+\frac{1}{2}}\mathcal{G}_\lambda(L_M) L^r g\right\rVert.
\end{align*}
Then
\begin{eqnarray*}
    \left\lVert L_M^{-a}S_M w\right\rVert_{\mathcal{H}_M}\leq\left\lVert L_M^{-a+\frac{1}{2}}\mathcal{G}_\lambda(L_M)L_{M,\lambda}^{r}\right\rVert\left\lVert L_{M,\lambda}^{-r}L^r\right\rVert\left\lVert g\right\rVert.
\end{eqnarray*}
The eigenvalues of $L_M^{-a+\frac{1}{2}}\mathcal{G}_\lambda(L_M)L_{M,\lambda}^{r}$ are of the form $(t+\lambda)^rt^{-a-\frac{1}{2}}\mathbf{1}(t\geq \lambda)$. If $r\geq a+1/2$ then
\begin{equation*}
    (t+\lambda)^rt^{-a-\frac{1}{2}}\mathbf{1}(t\geq \lambda)\leq (\left\lVert L_M\right\rVert+\lambda)^r\left\lVert L_M\right\rVert^{-a-\frac{1}{2}}.
\end{equation*}
If $r\leq a+1/2$ then
\begin{equation*}
    (t+\lambda)^rt^{-a-\frac{1}{2}}\mathbf{1}(t\geq \lambda)\leq (2\lambda)^r\lambda^{-a-\frac{1}{2}}\leq 2\lambda^{r-a-1/2}.
\end{equation*}
Also,
\begin{align*}
    \left\lVert L_M^{-a}S_M w\right\rVert_{\mathcal{H}_M} \leq & \left\lVert L_M^{-a+\frac{1}{2}}\mathcal{G}_\lambda(L_M)L_{M,\lambda}^{r-1/2}\right\rVert\left\lVert L_{M,\lambda}^{-r+1/2}L^{r-1/2}\right\rVert\left\lVert f\right\rVert_{\mathcal{H}_k}\\
     \leq & \lambda^{\left(a-\left(r-1\right)\right)\wedge0}(\left\lVert L_M\right\rVert+\lambda)^{r-1/2}\left\lVert L_M\right\rVert^{-a-\frac{1}{2}}\left\lVert L_{M,\lambda}^{-r+1/2}L^{r-1/2}\right\rVert\left\lVert f\right\rVert_{\mathcal{H}_k}
\end{align*}
if $r\leq 2a+3/2$.
\end{proof}
\begin{lemma}[Proof in Lemma 8 in \citet{randomfeature}]\label{lem:half_operator_rff}
Define $L_M:=\frac{1}{M}\sum_{i=1}^M\varphi_M\varphi_M^\top$. There exists a universal constant $C,c>0$ such that
    \begin{eqnarray*}
        \mathbb{P}_\omega \left(E_{half-rff}\right)\geq 1-t
    \end{eqnarray*}
    where
    \begin{equation*}
        E_{half-rff}:=\Bigg\{w:\left\lVert L_\lambda^{-\frac{1}{2}}(L-L_M)\right\rVert\leq \sqrt{\frac{4\kappa^2N(\lambda)\log\frac{2}{t}}{M}}+\frac{4\kappa\sqrt{\mathcal{F}_\infty(\lambda)}\log\frac{2}{t}}{M}\Bigg\}.
    \end{equation*}
\end{lemma}
\begin{lemma}[Proof in Lemma 8 in \citet{randomfeature}]\label{lem:full_operator_rff}Define $L_M:=\frac{1}{M}\sum_{i=1}^M\varphi_M\varphi_M^\top$. There exists a universal constant $C>0$ such that
\begin{equation*}\mathbb{P}_\omega \left(
    E_{full-rff}\right)\geq1-t
\end{equation*}
where
\begin{equation*}
    E_{full-rff}:=\left\{w:\left\lVert L_\lambda^{-\frac{1}{2}}(L-L_M)L_\lambda^{-\frac{1}{2}}\right\rVert\leq \frac{2(1+\mathcal{F}_\infty(\lambda))\log\frac{8\kappa^2}{\lambda t}}{3M}+\sqrt{\frac{2\mathcal{F}_\infty(\lambda)\log\frac{8\kappa^2}{\lambda t}}{M}}\right\}.
\end{equation*}
\end{lemma}

\begin{lemma}\label{lem:zero_grad}
For any $\beta\in\mathbb{R}^M$,
    \begin{eqnarray*}
    \mathbb{E}\left[l_{\hat{y}}(Y,f_{\mathcal{H}_k}(X))(S_M \beta(X))\right]=0
\end{eqnarray*}
holds almost surely. Thus,
\begin{eqnarray*}
    \mathcal{E}(S_M \beta)-\mathcal{E}(f_{\mathcal{H}_k})\leq\frac{c_2}{2}\left\lVert S_M \beta-f_{\mathcal{H}_k}\right\rVert^2
\end{eqnarray*}
holds under Assumption \ref{assump:lip}.
\end{lemma}
\begin{proof}
Let $g$ be an arbitrary element of $\mathcal{H}_k$. From the definition of $f_{\mathcal{H}_k}$,
\begin{eqnarray*}
\mathbb{E}\left[l_{\hat{y}}(Y,f_{\mathcal{H}_k}(X))g(X)\right]=0
\end{eqnarray*}
holds for any $g\in\mathcal{H}_k$. Let $\mathcal{E}_{\textrm{excess}}:\mathbb{R}\rightarrow\mathbb{R}$ be a function defined by
\begin{eqnarray*}
    \mathcal{E}_{\textrm{excess}}(t):=\mathcal{E}(f_{\mathcal{H}_k}+tg)-\mathcal{E}(f_{\mathcal{H}_k}).
\end{eqnarray*}
Then, we have
\begin{eqnarray*}
    0=\mathcal{E}_{\textrm{excess}}^\prime(0)=\mathbb{E}\left[l_{\hat{y}}(Y,f_{\mathcal{H}_k}(X)+tg(X))g(X)\right]\Big|_{t=0}=\mathbb{E}\left[l_{\hat{y}}(Y,f_{\mathcal{H}_k}(X))g(X)\right].
\end{eqnarray*}
This implies that for any element $g$ of the closure of $\mathcal{H}_k$ satisfies
\begin{align*}
    \mathbb{E}\left[l_{\hat{y}}(Y,f_{\mathcal{H}_k}(X))g(X)\right]=0.
\end{align*}
Also, for any $\beta \in \mathbb{R}^M$, we have
\begin{align*}
&\mathbb{E}\left[l_{\hat{y}}(Y,f_{\mathcal{H}_k}(X))(S_M \beta(X))\right]\\
=&\mathbb{E}\left[l_{\hat{y}}(Y,f_{\mathcal{H}_k}(X))(PS_M \beta(X))\right]+\mathbb{E}\left[l_{\hat{y}}(Y,f_{\mathcal{H}_k}(X))((I-P)S_M \beta(X))\right]\\
=&\mathbb{E}\left[l_{\hat{y}}(Y,f_{\mathcal{H}_k}(X))((I-P)S_M \beta(X))\right]\\
=&\mathbb{E}_X\left[\mathbb{E}_{Y|X}\left[l_{\hat{y}}(Y,f_{\mathcal{H}_k}(X))|X\right]((I-P)S_M \beta(X))\right]\\
=&\langle\mathbb{E}_{Y|X}\left[l_{\hat{y}}(Y,f_{\mathcal{H}_k}(X))|\cdot\right],(I-P)S_M \beta\rangle\\
=&0,
\end{align*}
where the second equality holds since $P$ is the projection to the closure of $\mathcal{H}_k$, and the third equality holds since $(I-P)S_M\equiv0$ almost surely by Lemma~\ref{lem:secondterm}.
Thus,
\begin{align*}
    \mathbb{E}\left[l_{\hat{y}}(Y,f_{\mathcal{H}_k}(X))(S_M \beta(X))\right]=0
\end{align*}
almost surely for every $\beta\in\mathbb{R}^M$. Thus, under Assumption~\ref{assump:lip},
\begin{align*}
    \mathcal{E}(S_M\beta)-\mathcal{E}(f_{\mathcal{H}_k})\leq&\mathbb{E}[l_{\hat{y}}(Y,f_{\mathcal{H}_k}(X))(S_M\beta-f_{\mathcal{H}_k}(X))+\frac{c_2}{2}(S_M\beta-f_{\mathcal{H}_k}(X))^2]\\
    =&\frac{c_2}{2}\lVert S_M\beta-f_{\mathcal{H}_k}\rVert^2.
\end{align*}
\end{proof}
\newpage
\section{Relation between $N(\lambda)$ and the decay rate of $\mu_l$}\label{sec:cap_cap_real}
In this section, we show Assumption~\ref{assump:cap_real} is equivalent to the upper bound condition in Assumption~\ref{assump:cap}.

Suppose $N(\lambda)\leq Q\lambda^{-\gamma}$ if $\lambda>0$. For $l_0\geq 1$, we have
\begin{align*}
    \mathcal{N}(\mu_{l_0})=&\sum_{l=1}^\infty\frac{\mu_l}{\mu_{l_0}+\mu_l}\\
     \geq &\sum_{l\leq l_0}\frac{\mu_l}{\mu_{l_0}+\mu_l}\\
     \geq &\sum_{l\leq l_0}\frac{1}{2}\\
    =&\frac{l_0}{2},
\end{align*}
so $N(\lambda)\leq Q\lambda^{-\gamma}$ implies that
\begin{eqnarray*}
    \mu_{l_0}\leq(2Q)^{\frac{1}{\gamma}}l_0^{-1/\gamma}.
\end{eqnarray*}
Now suppose that $\mu_l\leq ql^{-\frac{1}{\gamma}}$ for $l\geq 1$. For $\lambda\in(0,1)$, define $l_\lambda:=\min\left\{l\in\mathbb{N}:ql^{-\frac{1}{\gamma}}\leq\lambda\right\}$. Then, we have:
\begin{align*}
N(\lambda) \leq &\mathcal{N}(ql_\lambda^{-\frac{1}{\gamma}})\\
    =&\sum_{l=1}^\infty\frac{\mu_{l}}{\mu_l+ql_\lambda^{-\frac{1}{\gamma}}}\\
     \leq &\sum_{l=1}^{l_\lambda}1+\sum_{l>l_\lambda}\frac{\mu_l}{ql_\lambda^{-\frac{1}{\gamma}}}\\
     \leq & l_\lambda+l_\lambda^{\frac{1}{\gamma}}\sum_{l>l_\lambda}l^{-\frac{1}{\gamma}}\\
     \leq &\frac{1}{1-\gamma}l_\lambda\\
     \leq &\frac{1}{1-\gamma}\left(1+q^{\gamma}\lambda^{-\gamma}\right)\\
     \leq &\frac{1+q^{\gamma}}{1-\gamma}\lambda^{-\gamma}.
\end{align*}
\newpage
Also, if $\lambda\geq 1$, then
\begin{align*}
    N(\lambda)=&\sum_{l\geq 1}\frac{\mu_l}{\lambda+\mu_l}\\
     \leq &\sum_{l\geq 1}\frac{\mu_l}{\lambda}\\
     \leq &\lambda^{-1}\sum_{l\geq 1}ql^{-\frac{1}{\gamma}}\\
     \leq &\frac{q}{1-\gamma}\lambda^{-1}\\
     \leq &\frac{q}{1-\gamma}\lambda^{-\gamma}.
\end{align*}
Thus,
\begin{equation*}
    N(\lambda)\leq\max\left\{\frac{1+q^{\gamma}}{1-\gamma},\frac{q}{1-\gamma}\right\}\lambda^{-\gamma}.
\end{equation*}
\newpage

\section{Minimax optimal estimator for differentially private kernel regression when $P_X$ is known}\label{sec:emb}
In this section, we develop a minimax-optimal estimator when the covariate distribution $P_X$ is known. To this end, we assume the embedding property together with the fluctuation condition (Assumption~\ref{assump:fluct}), a standard requirement for refined convergence analysis in kernel learning \citep{embedding_property0,embedding_property,embedding_property2,zhang2023misspecified_krr}.
\begin{assumption}[Embedding property]\label{assump:emb}
There exist $\beta\in [\gamma,1]$ and $c_4>0$ such that $\sup_x k_\beta(x,x)\leq c_4$, where $k_\beta$ is defined by
\begin{equation*}
    k_\beta(x,y)\coloneqq \sum_l\mu_l^\beta\phi_l(x)\phi_l(y).
\end{equation*}
\end{assumption}
The embedding property always holds for $\beta=1$ because $k_1(x, x) = k(x, x) \leq \kappa^2$. Thus, the condition is always satisfied at $\beta=1$. Furthermore, since the embedding property implies a polynomial eigenvalue decay of order $1/\beta$, we assume $\gamma \leq \beta$ \citep{embedding_property}. Similar to the eigenvalue assumption, the embedding property characterizes the complexity of the hypothesis space. Specifically, smaller values of $\beta$ correspond to smaller hypothesis spaces. In standard kernel regression, the embedding property typically does not affect the optimal rate. However, it plays a critical role in the hard learning scenario, where the target function lies outside the RKHS (e.g., the conditional mean $\mathbb{E}[Y|X]$). In such cases, \citet{embedding_property} demonstrated that increasing $\beta$ significantly degrades the excess risk rate. However, as noted by \citet{embedding_property}, finding a minimax-optimal estimator when the minimax lower bound depends on $\beta$ remains a challenging open problem in kernel learning theory.

In our analysis, we assume the embedding property holds alongside the fluctuation condition. These two conditions are closely related; specifically, the embedding property implies $\lVert\phi_l\rVert_{L^1(P_X)}^2 \geq c_4^{-1}\mu_l^{\beta}$ (see Appendix~\ref{sec:emb_fluct}). Together, they determine the order of the $L^1$-norm decay $\lVert\phi_l\rVert_{L^1(P_X)} \asymp \mu_l^{\beta/2}$. Under these conditions, we construct an estimator that achieves the minimax-optimal rate $O(n^{-\frac{2r}{2r+\gamma}} + (n\epsilon)^{-\frac{4r}{2r+\gamma+\beta}})$.
\begin{theorem}\label{thm:lower_bd_fluct_emb}Under Assumptions~\ref{assump:fluct} and~\ref{assump:emb}, and $\delta=o(n^{-1})$, when the covariate distribution $P_X$ is known, the minimax risk is
\begin{equation*}
    \mathfrak{R}(\mathcal{P}_2,\mathcal{E},\epsilon,\delta)\asymp
    n^{-\frac{2r}{2r+\gamma}}+(n\epsilon)^{-\frac{4r}{2r+\gamma+\beta}}\wedge1,
\end{equation*}
for any $\eta>0$.
\end{theorem}
\begin{proof}
    Since $P_X$ is known, the eigenfunctions $\{\phi_l\}_{l=1}^\infty$ of the integral operator $L$ are also known. For simplicity, we assume $\mathcal{Y}\subset[-1,1]$. Suppose that a kernel $k$ and data distribution $P_{X,Y}$ satisfy Assumptions \ref{assump:bern}-\ref{assump:source}. We estimate the coefficients of $f_{\mathcal{H}_k}$ in the eigenbasis $\{\phi_l\}_{l=1}^\infty$ using the fact that
    \begin{eqnarray*}
        f_{\mathcal{H}_k}=\sum_{l=1}^\infty \mathbb{E}\left[Y\phi_l(X)\right]\phi_l.
    \end{eqnarray*}
    Denote $\theta_l\coloneqq \mathbb{E}\left[Y\phi_l(X)\right]$ for $l\geq 1$.
    
    Let $l_0$ be a given positive integer. Given a dataset $\{(x_i,y_i)\}_{i=1}^n$, we construct the private estimator as follows:
   \begin{enumerate}
       \item Evaluate $\hat{\theta}\in\mathbb{R}^{l_0}$ where $        \hat{\theta}_l\coloneqq \frac{1}{n}\sum_{i=1}^ny_i\phi_l(x_i)$ for $l=1,\ldots,l_0$.
       \item $\widetilde{\theta}\coloneqq \hat{\theta}+\frac{2\sqrt{c_4} l_0^{\beta/2}\left(1+\sqrt{2\log\frac{1}{\delta}}\right)}{n\epsilon}\varepsilon_{l_0}$ where $\varepsilon_{l_0}\sim \mathcal{N}(0,I_{l_0})$.
       \item Define the private estimator $\widetilde{f}$ as $\widetilde{f}\coloneqq \sum_{l=1}^{l_0}\widetilde{\theta}_l\phi_l$.
   \end{enumerate}
   From the definition of $k_\beta$, it follows that $\sum_{l\leq l_0}\phi_l^2(x)\leq k_\beta(x,x)\lambda_{l_0}^{-\beta}\leq c_4 \lambda_{l_0}^{-\beta}$ holds for all $l_0\geq 1$ and $x\in\mathcal{X}$. Thus, the $\ell_2$ sensitivity of $\hat{\theta}$ is bounded by $2\sqrt{c_4} \lambda_{l_0}^{\beta/2}$. By Proposition~\ref{prop:gaussian}, releasing $\widetilde{\theta}$ is $(\epsilon,\delta)$-DP, and hence $\widetilde{f}$ is also $(\epsilon,\delta)$-DP. 
    
    We proceed to bound the excess risk of $\widetilde{f}$. 
    Note that
    \begin{eqnarray*}
        \mathbb{E}[\mathcal{E}(\widetilde{f})]-\min_{f\in\mathcal{H}_k}\mathcal{E}(f)&=&
        \mathbb{E}\left[\lVert \widetilde{f}-f_{\mathcal{H}_k}\rVert^2\right]\\
        &=&\mathbb{E}\left[\sum_{l=1}^{l_0} (\widetilde{\theta}_l-\theta_l)^2\right]+\sum_{l>l_0}\theta_l^2\\
        &=&\mathbb{E}\left[\sum_{l=1}^{l_0} (\widetilde{\theta}_l-\hat{\theta}_l)^2\right]+\mathbb{E}\left[\sum_{l=1}^{l_0}(\hat{\theta}_l-\theta_l)^2\right]+\sum_{l>l_0}\theta_l^2.
    \end{eqnarray*}
    The first term can be explicitly calculated as $\frac{4c_4\lambda_{l_0}^{\beta}l_0\left(1+\sqrt{2\log\frac{1}{\delta}}\right)^2}{(n\epsilon)^2}$. The second term is bounded by $O(l_0/n)$ since
    \begin{eqnarray*}
        \mathbb{E}\left[(\hat{\theta}_l-\theta_l)^2\right]&=&\mathbb{E}\left[\frac{1}{n^2}\sum_{i=1}^n \left(y_i\phi_l(x_i)-\mathbb{E}\left[Y\phi_l(X)\right]\right)^2\right]\\
        &\leq&\frac{1}{n}\mathbb{E}\left[Y^2\phi_l^2(X)\right]\\
        &=&\frac{1}{n}\mathbb{E}\left[\mathbb{E}\left[Y^2|X\right]\phi_l^2(X)\right]\\
        &\leq&\frac{2\sigma^2+2\kappa^{2r+1}R^2}{n}\mathbb{E}\left[\phi_l^2(X)\right]\\
        &=&\frac{2\sigma^2+2\kappa^{2r+1}R^2}{n},
    \end{eqnarray*}
    where we use
    \begin{eqnarray*}
        \mathbb{E}\left[Y^2|X\right]&\leq&2\mathbb{E}\left[(Y-f_{\mathcal{H}_k}(X))^2|X\right]+2f_{\mathcal{H}_k}^2(X)\\
        &\leq&2\sigma^2+2\lVert f_{\mathcal{H}_k}\rVert_{\mathcal{H}_k}^2\lVert k(X,\cdot)\rVert_{\mathcal{H}_k}^2\\
        &=&2\sigma^2+2\lVert L^{-\frac{1}{2}}f_{\mathcal{H}_k}\rVert_{}^2\lVert k(X,\cdot)\rVert_{\mathcal{H}_k}^2\\
        &=&2\sigma^2+2\lVert L^{r-\frac{1}{2}}g\rVert_{}^2\lVert k(X,\cdot)\rVert_{\mathcal{H}_k}^2\\
        &\leq&2\sigma^2+2\lVert L^{r-\frac{1}{2}}\rVert^2\lVert g\rVert^2\lVert k(X,\cdot)\rVert_{\mathcal{H}_k}^2\\
        &\leq&2\sigma^2+2\kappa^{2r+1}R^2.
    \end{eqnarray*}
    The third term can be bounded from Assumption \ref{assump:source}. Specifically, this assumption guarantees the existence of a function $g\in L^2(P_{X})$ such that $f_{\mathcal{H}_k}=L^rg$. Letting $g=\sum_{l} \alpha_l\phi_l$ for some coefficients $\alpha_l$, it follows that $\sum_l\alpha_l^2\leq R^2$. Then, $\theta_l=\mu_l^{r}\alpha_l$, and
    \begin{eqnarray*}
        \sum_{l>l_0}\theta_l^2&=&\sum_{l>l_0}\mu_l^{2r}\alpha_l^2\\
        &\leq&\lambda^{2r}_{l_0}\sum_{l>l_0}\alpha_l^2\\
        &\leq&\lambda^{2r}_{l_0}R^2.
    \end{eqnarray*}
    Combining the bound for each term, we obtain the excess risk bound of $\widetilde{f}$ as $O\Bigg(\frac{l_0^{\frac{\beta}{\gamma}+1}}{(n\epsilon)^2}+\frac{l_0}{n}+l_0^{-\frac{2r}{\gamma}}\Bigg)$. Choosing $l_0$ to be the closest positive integer to $(n\epsilon)^{\frac{2\gamma}{2r+\beta+\gamma}}\wedge n^{\frac{\gamma}{2r+\gamma}}$, we obtain the desired upper bound.
\end{proof}
We highlight several key implications of Theorem~\ref{thm:lower_bd_fluct_emb}. First, as $\epsilon\rightarrow\infty$, our lower bound recovers the non-private minimax rate $O(n^{-\frac{2r}{2r+\gamma}})$ for kernel regression under the same assumptions \citep{caponnetto2007optimal}. However, the private minimax rate depends on the embedding parameter $\beta$---a phenomenon that has not been observed in non-private settings. This indicates that factors beyond the eigenvalue decay and the source condition fundamentally influence the complexity of differentially private kernel learning. Second, we show that when the marginal distribution $P_X$ is known, one can construct an estimator that achieves the minimax lower bound. Consequently, Theorem~\ref{thm:lower_bd_fluct_emb} establishes the minimax-optimal rate for differentially private nonparametric regression when $P_X$ is known. This extends the results of \citet{cai2024optimal}. Specifically, they considered nonparametric regression where $P_X$ is a uniform distribution on $[0,1]$ and the target function lies in a Besov space $B_{p,q}$ over $[0,1]$ for $p\geq 2$, $q\geq1$, $\alpha>1/2+1/p$. They established a minimax-optimal rate of $O(n^{-\frac{2\alpha}{2\alpha+1}}+(n\epsilon)^{-\frac{4\alpha}{2\alpha+2}})$. In contrast, Theorem~\ref{thm:lower_bd_fluct_emb} applies to a much broader class of high-dimensional marginal distributions when the target function belongs to $B_{2,2}^\alpha$ over $\mathcal{X}$. Let $\mathcal{X}$ be a bounded and connected open subset of $\mathbb{R}^d$. Then, $B_{2,2}^\alpha(\mathcal{X})$ can be identified with an RKHS associated with a bounded kernel $k_\alpha$.  If $P_X$ is quasi-uniform (i.e., its density with respect to the uniform distribution over $\mathcal{X}$ is bounded away from zero and infinity), then Assumptions~\ref{assump:bern}-\ref{assump:emb} are satisfied with $r=1/2$, $\gamma=d/2\alpha$, and $\beta=d/2\alpha$ \citep{embedding_property}. This yields a minimax-optimal rate $O(n^{-\frac{2\alpha}{2\alpha+d}}+(n\epsilon)^{-\frac{4\alpha}{2\alpha+2d}})$. In particular, when $d=1$, this recovers the minimax-optimal rate established by \citet{cai2024optimal}.
\newpage
\section{On the relation between Assumption~\ref{assump:emb} and Assumption~\ref{assump:fluct}}\label{sec:emb_fluct}
In this section, we show that $\lVert \phi_l\rVert_{L^1(P_X)}^2\geq c_4^{-1}\mu_l^\beta$ holds under Assumption~\ref{assump:emb}. Specifically, we show that if $\sup_{x\in\mathcal{X}}k_\beta(x,x)\leq c_4$ then
\begin{align*}
    \lVert\phi_l\rVert_{L^1(P_X)}\geq\frac{\mu_l^{\beta/2}}{\sqrt{c_4}}.
\end{align*}
We first show that $\lVert \phi_l\rVert_{L^1(P_X)}\geq\frac{\sqrt{\mu_l}}{\kappa}$:
\begin{align*}
    \lVert\phi_l\rVert_{L^1(P_X)}\geq& \frac{\lVert \phi_l\rVert_{L^2(P_X)}^2}{\lVert \phi_l\rVert_{L^\infty(P_X)}}\\
    =&\frac{1}{\lVert \phi_l\rVert_{L^\infty(P_X)}}\\
    =&\frac{1}{\sup_{x}|\langle \phi_l,k(x,\cdot)\rangle_{\mathcal{H}_k}|}\\
    \geq&\frac{1}{\sup_{x}\lVert k(x,\cdot)\rVert_{\mathcal{H}_k}\lVert\phi_l\rVert_{\mathcal{H}_k}}\\
    \geq&\frac{1}{\kappa\mu_{l}^{-\frac{1}{2}}}\\
    \geq &\frac{\sqrt{\mu_l}}{\kappa}.
\end{align*}
Next, we show that the embedding property implies $\lVert \phi_l\rVert_{L^1(P_X)}\geq\frac{\mu_l^{\beta/2}}{\sqrt{c_4}}$. This can be shown using the previous argument for $k_\beta$ instead of $k$, as $\phi_l$ is an eigenfunction of $k_\beta$ with eigenvalue $\mu_l^{\beta}$.
\newpage
\section{Moments of Gaussian process}
\begin{lemma}\label{lem:gp_moments}
A Gaussian process $h$ with mean 0 and covariance operator $\Sigma$ satisfies the Bernstein condition:
\begin{align*}
\mathbb{E}\left[\left|\left\lVert h\right\rVert^2-\mathbb{E}\left[\left\lVert h\right\rVert^2\right]\right|^p\right]\leq& Cp!\left(4\mathbb{E}\left[\left\lVert h\right\rVert^2\right]\right)^p=Cp!\left(4\mathrm{tr}(\Sigma)\right)^p,\\
    \mathbb{E}\left[\left\lVert h\right\rVert^{2p}\right]\leq& Cp!\left(8\mathbb{E}\left[\left\lVert h\right\rVert^2\right]\right)^p=Cp!\left(8\mathrm{tr}(\Sigma)\right)^p,
\end{align*}
for some $C>0$. For example, $C=2$ satisfies the condition.
\end{lemma}
\begin{proof}
Denote the eigenvalues of $\Sigma$ as $\{\mu_l\}_{l=1}^\infty$. Note that
\begin{equation*}
    \mathbb{E}\left[\left|\left\lVert h\right\rVert^2-\mathbb{E}\left[\left\lVert h\right\rVert^2\right]\right|^2\right]=2\sum_l\mu_l^2.
\end{equation*}

    By Proposition~\ref{prop:gaussian_vec}, the following holds:
    \begin{equation*}
        \mathbb{P}\left(\left|\left\lVert h\right\rVert^2-\mathbb{E}\left[\left\lVert h\right\rVert^2\right]\right|\geq 2\sqrt{\sum_l\mu_l^2}\sqrt{x}+2\max_l\mu_l x\right)\leq e^{-x}.
    \end{equation*}
    Using $\sqrt{\sum_l\mu_l^2},\max_l\mu_l\leq \sum_l\mu_l=\mathrm{tr}(\Sigma)$, we have
\begin{align*}
\mathbb{P}\left(\left|\left\lVert h\right\rVert^2-\mathbb{E}\left[\left\lVert h\right\rVert^2\right]\right|\geq 2\mathrm{tr}(\Sigma)\sqrt{x}+2\mathrm{tr}(\Sigma) x\right)\leq e^{-x}.
\end{align*}
    For $t\geq 0$ we can represent the inequality as follows:
    \begin{equation*}
        \mathbb{P}\left(\left|\left\lVert h\right\rVert^2-\mathbb{E}\left[\left\lVert h\right\rVert^2\right]\right|\geq t\right)\leq \mathrm{exp}\left(-\left(\sqrt{\frac{t}{2\mathrm{tr}(\Sigma)}+\frac{1}{4}}-\frac{1}{2}\right)^2\right).
    \end{equation*}
    Then
\begin{align*}
\mathbb{E}\left[\left|\left\lVert h\right\rVert^2-\mathbb{E}\left[\left\lVert h\right\rVert^2\right]\right|^p\right]=&\int_0^\infty pt^{p-1}\mathbb{P}\left(\left|\left\lVert h\right\rVert^2-\mathbb{E}\left[\left\lVert h\right\rVert^2\right]\right|\geq t\right)dt\\
        \leq&\int_0^\infty pt^{p-1}\mathrm{exp}\left(-\left(\sqrt{\frac{t}{2\mathrm{tr}(\Sigma)}+\frac{1}{4}}-\frac{1}{2}\right)^2\right)dt\\
        \leq&\int_0^{4\mathrm{tr}(\Sigma)} pt^{p-1}\mathrm{exp}\left(-\left(\sqrt{\frac{t}{2\mathrm{tr}(\Sigma)}+\frac{1}{4}}-\frac{1}{2}\right)^2\right)dt\\
        &+\int_{4\mathrm{tr}(\Sigma)}^\infty pt^{p-1}\mathrm{exp}\left(-\frac{t}{4\mathrm{tr}(\Sigma)}\right)dt\\
        \leq&4^p\mathrm{tr}(\Sigma)^p+p\Gamma(p)4^p\mathrm{tr}(\Sigma)^p\\
        =&(p!+1)4^p\mathrm{tr}(\Sigma)^p\\
        \leq&2 p!(4\mathrm{tr}(\Sigma))^p.
\end{align*}
The second inequality holds since $(t-1)^2\geq\frac{1}{2}(t^2-1)$ if $t>4$. Also, note that for any random variable $X$
\begin{align*}
\mathbb{E}\left[\left|X\right|^p\right]\leq&\mathbb{E}\left[\left(\left|\mathbb{E}\left[X\right]\right|+\left|X-\mathbb{E}\left[X\right]\right|\right)^p\right]\\
        \leq&2^{p-1}\mathbb{E}\left[\left|\mathbb{E}\left[X\right]\right|^p+\left|X-\mathbb{E}\left[X\right]\right|^p\right]\\
        =&2^{p-1}\left(\left|\mathbb{E}\left[X\right]\right|^p+\mathbb{E}\left[\left|X-\mathbb{E}\left[X\right]\right|^p\right]\right)
\end{align*}
holds. Since $\mathbb{E}\left[\left\lVert h\right\rVert^2\right]=\sum_l\mu_l=\mathrm{tr}(\Sigma)$, we obtain
\begin{align*}
\mathbb{E}\left[\left\lVert h\right\rVert^{2p}\right]\leq&2^{p-1}\left(\mathbb{E}\left[\left\lVert h\right\rVert^{2}\right]^p+\mathbb{E}\left[\left|\left\lVert h\right\rVert^{2}-\mathbb{E}\left[\left\lVert h\right\rVert^{2}\right]\right|^p\right]\right)\\
        \leq&2^{p-1}\left(\mathrm{tr}^p(\Sigma)+(p!+1)4^p\mathrm{tr}(\Sigma)^p\right)\\
        \leq&2^{p-1}(p!+2)4^p\mathrm{tr}(\Sigma)^p\\
        \leq&2p!(8\mathrm{tr}(\Sigma))^p.
\end{align*}
\end{proof}
\newpage

\bibliography{reference}

\end{document}